\documentclass[mnsc,nonblindrev]{informs4}

\OneAndAHalfSpacedXI

\usepackage{amsmath,amssymb,amsfonts}

\usepackage{natbib}
 \bibpunct[, ]{(}{)}{,}{a}{}{,}%
 \def\bibfont{\small}%

\usepackage{hyperref}
\usepackage{rotating}
\usepackage{fancyvrb}

\usepackage{url}            
\usepackage{booktabs}       
\usepackage{nicefrac}       
\usepackage{comment}

\usepackage{framed}
\usepackage{caption}
\usepackage{subcaption}
\usepackage{rotating}
\usepackage{bbm}
\usepackage{color}
\usepackage{xcolor}
\usepackage{array}
\usepackage [english]{babel}
\usepackage [autostyle, english = american]{csquotes}
\MakeOuterQuote{"}
\usepackage{enumitem}
\usepackage{algorithm}
\usepackage{algpseudocode}
\usepackage{multirow}
\usepackage{pdflscape}
\usepackage{tabularx}  
\usepackage{longtable}
\usepackage{changepage}
\usepackage[normalem]{ulem}

\newcommand{\E}{\mathbb{E}}
\newcommand{\R}{\mathbb{R}}

\newcommand{\edit}[1]{{\color{black}#1}}

\newcommand{\ie} {{\em i.e.\/}, }
\newcommand{\eg} {{\em e.g.\/}, }

\newcommand{\gammah}{\gamma^{\textup{H}}}
\newcommand{\gammal}{\gamma^{\textup{L}}}
\newcommand{\Dh}{D^{\textup{H}}}
\newcommand{\Dl}{D^{\textup{L}}}
\newcommand{\ul}{\underline{u}}
\newcommand{\uh}{\overline{u}}

\usepackage{xcolor}
\usepackage{framed}

\colorlet{responsecolor}{black}
\definecolor{shadecolor}{rgb}{0.54, 0.81, 0.94}

\TheoremsNumberedThrough     
\ECRepeatTheorems
\JOURNAL{Any INFORMS Journal}

\EquationsNumberedThrough    

\MANUSCRIPTNO{IJDS-0001-1922.65}

\begin{document}
\RUNTITLE{Deep Reinforcement Learning for Inventory Networks: Toward Reliable Policy Optimization}

\TITLE{Deep Reinforcement Learning for Inventory Networks: Toward Reliable Policy Optimization}
\ARTICLEAUTHORS{%
\AUTHOR{Matias Alvo}

\AFF{Graduate School of Business, Columbia University, 
\EMAIL{malvo26@gsb.columbia.edu}}

\AUTHOR{Daniel Russo}

\AFF{Graduate School of Business, Columbia University, 
\EMAIL{djr2174@gsb.columbia.edu}}

\AUTHOR{Yash Kanoria}

\AFF{Graduate School of Business, Columbia University, 
\EMAIL{yk2577@gsb.columbia.edu}}

\AUTHOR{Minuk Lee}
\AFF{Columbia University, 
\EMAIL{ml4723@columbia.edu}}
}

\ABSTRACT{
We argue that inventory management presents unique opportunities for the reliable application of deep reinforcement learning (DRL). To enable this, we emphasize and test two complementary techniques. The first is \textit{Hindsight Differentiable Policy Optimization (HDPO)}, \edit{which uses pathwise gradients from offline counterfactual simulations to directly and efficiently optimize policy performance. Unlike standard policy gradient methods that rely on high-variance score-function estimators, HDPO computes gradients by differentiating through the known system dynamics.}
Via extensive benchmarking, we show that HDPO recovers near-optimal policies in settings with known or bounded optima, is more robust than variants of the REINFORCE algorithm, and significantly outperforms generalized newsvendor heuristics on problems using real time series data. Our second technique aligns neural policy architectures with the topology of the inventory network. We exploit Graph Neural Networks (GNNs) as a natural inductive bias for encoding supply chain structure, demonstrate that they can represent optimal and near-optimal policies in two theoretical settings, and empirically show that they reduce data requirements across six diverse inventory problems.
A key obstacle to progress in this area is the lack of standardized benchmark problems. To address this gap, we open-source a suite of benchmark environments, along with our full codebase, to promote transparency and reproducibility. All resources are available at \url{https://github.com/MatiasAlvo/Neural_inventory_control}.
}


\maketitle

\section{Introduction}
\label{sec: introduction}

Inventory management deals with designing replenishment policies that minimize costs related to holding, selling, purchasing, and transporting goods. Such problems are typically modeled as Markov Decision Processes (MDPs), but the triple curse of dimensionality \citep{powell2007approximate} (\ie exponentially large state, action, and outcome spaces) usually renders them computationally intractable. Since the 1950s, researchers have attempted to circumvent this curse by formulating specialized models in which the optimal policy has a simple structure \citep{arrow1958studies, scarf1960optimality, clark1960optimal, federgruen1984computational}. These policies then form the basis of simple heuristics that work well in slightly broader – but still very narrow – problem settings \citep[see \eg][]{federgruen1984approximations,xin2021understanding}. Practical inventory management problems are often far more complex due to the presence of multiple inventory storage locations, dynamic demand patterns, and complexities in customer behavior during stock-outs.

{\bf The Promise and Pitfalls of Deep Reinforcement Learning.}
Deep reinforcement learning (DRL) offers a fundamentally different approach. Rather than deriving policies through mathematical analysis or human intuition, DRL ``learns'' effective policies through interaction with a simulator grounded in real data. Practitioners focus on ensuring simulator fidelity, and solution quality then improves with increased computation—larger neural networks (NNs), more training time, and more data. This approach has produced remarkable successes in domains once thought intractable: arcade games \citep{mnih2015human}, board games \citep{silver2017mastering}, and robotics \citep{levine2016end}.

For inventory management, DRL is particularly appealing. A high-fidelity simulator can capture real demand patterns, complex network structures, and operational constraints that would overwhelm analytical methods. Recent works have begun exploring DRL for operations problems including queueing control \citep{dai2022queueing}, ride-hailing  \citep{feng2021scalable, oda2018movi, tang2019deep}, and inventory management itself  \citep{gijsbrechts2021can}.

However, the DRL literature itself reveals concerning fragility. One standard approach employs policy gradient methods like REINFORCE \citep{williams1992simple}, where an NN policy generates actions, and parameters are updated to increase the probability of actions that lead to good outcomes. Successful applications require extensive modifications: natural gradient methods to handle ill-conditioning \citep{kakade2001natural}, entropy regularization to prevent premature convergence \citep{haarnoja2018soft}, variance-reducing baselines \citep{greensmith2004variance}, and actor-critic architectures \citep{konda1999actor}. Even then, many ``tricks'' \citep{huang202237} and careful hyperparameter tuning \citep{henderson2018deep} are essential. Skeptics have argued these methods are little better than random search \citep{mania2018simple}.

{\bf The Scale Challenge in Inventory Networks.} 
These generic difficulties are dramatically amplified in large-scale inventory networks.  Consider a network with 50 stores and 5 warehouses—the state vector must track inventory levels and outstanding orders across all locations, as well as demand histories for all 50 stores. The action space specifies replenishment quantities for potentially hundreds of edges. State and action spaces become very high dimensional and these dimensions are not artificial ones that can be ignored through representation learning: each store's inventory level genuinely matters for its own replenishment decisions.

In this high-dimensional setting, we identify two fundamental problems: 

\begin{itemize}
    \item {\bf Problem 1: Weak and diffused learning signals.} REINFORCE-style methods learn through trial and error—when a store stocks out, it does not immediately learn to order more. That insight only emerges if random variations in its ordering happen to correlate with better system-wide outcomes over many periods. Credit assignment becomes much harder as networks grow, with the reward signal dispersed across thousands of decisions over time.
    \item {\bf Problem 2: Inability to exploit network structure in standard neural architectures.} Generic fully-connected policy architectures attempt to learn the entire state-to-action mapping as one monolithic function. They easily overfit by learning spurious correlations (e.g., where store 7's orders depend on details of store 42's inventory) and cannot transfer insights across similar parts of the network—learning to manage one store does not help with managing an identical store elsewhere.
\end{itemize}

\subsection{Our Approach}

We address these problems through an integrated framework:

\begin{itemize}
    \item {\bf For Problem 1, we employ Hindsight Differentiable Policy Optimization (HDPO).} Unlike REINFORCE, which requires random exploration to discover good actions, HDPO exploits a key property of inventory networks: the physics is known. Shipping one unit increases the receiver's inventory by exactly one unit (after a fixed lead time). Combined with historical demand data, we can simulate—and differentiate through—entire trajectories. This is enabled by purposeful modeling choices that ensure action spaces are continuous and total costs are differentiable. While a scalar reward signal only indicates how a policy performed on a given scenario, gradient computation evaluates all local counterfactuals—how total cost would have changed under any infinitesimal change to policy parameters—and reveals the direction of steepest improvement. This appears to have especially important benefits in large-scale inventory networks, where the learning signal otherwise becomes hopelessly diffuse. The result is stable, efficient learning without the high variance and extensive training tricks common in pure trial-and-error methods (like REINFORCE).
     \item {\bf For Problem 2, we propose to train Graph Neural Network (GNN) \citep{prince2023understanding} policy architectures with HDPO.} Under GNNs, the policy's computation graph—the directed graph of relationships captured in the neural architecture—mirrors the physical inventory network along which goods flow. By default each location bases decisions primarily on local state information, while learnable message passing enables coordination when locations share resources or constraints. This design naturally decomposes the problem while allowing information flow where it is needed. Shared parameters across similar nodes enable transfer learning: what the network learns about managing one store automatically applies to similar stores elsewhere.
\end{itemize}

Together, these techniques comprise a careful effort to bring massive advances in compute infrastructure (GPUs), automatic differentiation tools (TensorFlow \citep{tensorflow2015-whitepaper}, PyTorch \citep{paszke2019pytorch}, and JAX \citep{jax2018github}), and neural architecture design (GNNs) to bear on inventory network control problems. While differentiable simulation has existed in operations research for decades \citep[see e.g.][]{glasserman1995sensitivity}, it is the confluence of these modern tools that makes it practical to train sophisticated neural policies that achieve near-optimal performance without extensive tuning. 

 \subsection{Benchmarking and Empirical Findings}

We develop a comprehensive suite of benchmark problems spanning various network topologies, demand distributions, and structural assumptions. These include instances where optimal costs are known or tightly bounded, enabling comparison to certifiable optima—a rare opportunity in deep RL evaluation. Other problems use publicly available sales data from Corporación Favorita, one of Ecuador's largest grocery retailers, introducing realistic nonstationarity. We open-source these benchmarks to promote reproducible evaluation of data-driven techniques in inventory control. In contrast to machine learning, where standardized benchmarks such as ImageNet \citep{krizhevsky2012imagenet} serve as widely accepted datasets and tasks for evaluating various approaches, inventory management lacks such universally recognized benchmarks against which numerous policies can be rigorously tested. Creating new ones is of particular importance.

The paper presents three main empirical findings:

\begin{enumerate}
    \item {\bf Vanilla HDPO achieves near-optimal performance with impressive consistency.} Section \ref{sec:vanilla-hdpo} evaluates ``vanilla'' HDPO—using standard NNs without structural modifications—across diverse inventory problems. In settings with known optima, HDPO achieves average optimality gaps of just 0.03\%, 0.25\%, 0.35\%, and 0.15\% across four benchmark classes (single-store backlogged, single-store lost demand, serial networks, and one-warehouse multi-store systems). The performance can also be unusually robust. To show this, we revisit a single-store setting studied by \citet{gijsbrechts2021can}: HDPO achieves gaps as low as detectable across all six instances we tested for each of twelve different hyperparameter choices. Using A3C (a REINFORCE-based DRL algorithm), \citet{gijsbrechts2021can} instead reported the need for extensive hyperparameter tuning (testing $\sim$250 hyperparameter configurations per instance) and still attained optimality gaps that were an order of magnitude larger (3.0-6.7\%).

    Once we turn to settings with realistic demand data, we can no longer benchmark against the optimum itself. We instead compare against sophisticated heuristics that also leverage deep learning: these methods first train an NN to forecast demand, then apply newsvendor-type ordering policies that restock up to an inventory level corresponding to a (trainable) quantile of the forecasted demand distribution. HDPO's end-to-end optimization—directly mapping historical demand windows to decisions—outperforms these two-stage forecast-then-optimize methods by up to 22\%.

    While these results are encouraging for problems of moderate scale, Section~\ref{sec:vanilla-unrelated-stores} reveals fundamental limitations as networks grow larger.
    
    \item {\bf Vanilla architectures are blind to inventory network structure, causing poor scaling.} Section~\ref{sec:vanilla-unrelated-stores} reveals this through a diagnostic experiment with many unrelated stores that have no need to coordinate. Vanilla fully-connected NNs cannot see this independence—they are blind to inventory network structure and cannot even distinguish which state variables correspond to which locations. This makes them prone to learning spurious correlations, like policies where store 7's orders depend on complex functions of stores 37 and 42's inventory levels.

    Our experiments show performance degrades markedly with the number of locations unless the volume of data-per-store scales with the number of stores. Yet the underlying problem has not become harder—it is the same single-store problem replicated many times. An architecture that recognized this structure would actually improve with more stores, using each as additional training data. This dichotomy motivates our introduction of GNNs as a general architecture that exploits the structure of the physical inventory network.

    \item {\bf GNN architectures provide dramatic sample efficiency gains at scale.} Section \ref{sec:gnn-section} demonstrates how GNNs address these scaling challenges. Across six diverse inventory settings—including serial systems, one-warehouse and many-warehouse networks with up to 56 locations—GNNs consistently outperform vanilla architectures when data is limited. The improvements can be substantial: 30\% performance gains with just 128 training samples in 50-store networks, and up to 64× reduction in data requirements to achieve target performance levels.
    
    Section~\ref{sec: GNN_performance_drivers} conducts focused experiments that isolate three key mechanisms underlying GNNs' success: automatic decomposition of independent decisions, coordination through message passing when locations interact, and dynamic exploitation of network flexibility (e.g., choosing which warehouse supplies each store based on current conditions).
\end{enumerate}
  
The remainder of this paper presents these techniques and empirical findings in detail.

 \subsection{Summary of contributions.}
 Our work makes several contributions. While the building blocks of HDPO---differentiable simulation \citep{glasserman1995sensitivity} and recognition of exogenous noise in MDPs \citep{powell2007approximate}---have long existed in operations research, and while \citet{madeka2022deep} recently applied HDPO to single-location inventory problems at Amazon, we are the first to: (i) systematically evaluate HDPO across complex multi-location networks  (up to 56 locations)  where coordination is essential, representing a massive increase in problem scale compared to prior work; (ii) introduce and train GNN policy architectures with pathwise gradients in reinforcement learning (RL) ---a novel combination that prevents data requirements from exploding with inventory network size, achieving up to 64× reduction compared to standard policy architectures; and (iii) provide comprehensive, reproducible benchmarks spanning both stylized problems with known optima and realistic scenarios constructed from open datasets. Through extensive empirical validation, we demonstrate that this integrated framework of HDPO and GNNs provides a promising and scalable approach to DRL for large-scale inventory network control.

\subsection{Further related literature 
\label{subsec: further-related-literature} }

\paragraph{\bf Inventory management.}

The field of inventory management boasts a rich and extensive history. In seminal research, \cite{arrow1958studies} demonstrated that in a setting involving a single location with positive vendor lead times and assuming backlogged demand (\ie customers are willing to wait for unavailable items), an optimal one-dimensional base-stock policy exists. Specifically, they illustrated that the state, which initially has the same dimension as the lead time, can be condensed into a single parameter. Conversely, in situations where demand is considered lost (\ie customers do not wait for unavailable products), the optimal policy may rely on the entire inventory pipeline, encompassing the quantity of units arriving each day until the lead time. This assumption may better reflect reality, particularly in brick-and-mortar retail, where only 9-22\% of customers are inclined to postpone their purchases when a specific item is unavailable in-store \citep{corsten2004stock}. \cite{zipkin2008old} highlighted that a base-stock policy may perform poorly in such contexts. Given the computational complexity associated with solving this problem, numerous heuristics have been proposed, including myopic policies \citep{morton1971near, zipkin2008old}, constant-order policies \citep{reiman2004new}, capped base-stock policies \citep{xin2021understanding}, and linear programming-based techniques \citep{sun2014quadratic}.

In multi-echelon inventory networks, inventory can be held across numerous locations. \cite{clark1960optimal} demonstrated that within a serial system, where inventory progresses linearly through sequential locations, a simple extension of base-stock policies proves to be optimal under a backlogged demand assumption. Similarly, \cite{federgruen1984approximations} introduced a well-performing policy for a setting in which a central ``transshipment'' warehouse, unable to hold inventory, supplies multiple stores. Yet, the task of determining the optimal policy, even in some straightforward inventory network setups, remains daunting. For a comprehensive review of works addressing multi-echelon inventory problems, we direct readers to \cite{de2018typology}.

\paragraph{\bf DRL algorithms.}
To describe the literature, we will differentiate between two types of algorithms. Similar to value-iteration methods in dynamic programming, Q-learning \citep{watkins1992q} makes iterative updates to an estimate of the optimal state-action value function (usually denoted as $Q^*$). DRL implementations employ NNs to approximate the value function. Influential work by \citep{mnih2015human} introduces a variant of this idea called deep Q-networks (DQN).

Closer in spirit to policy-iteration methods in dynamic programming, policy gradient methods make iterative improvements to an initial policy. The REINFORCE algorithm by \citep{williams1992simple} relies on the use of randomized policies and inverse propensity weighting to produce unbiased policy gradients. See Appendix \ref{appendix:score-function-gradient} for a discussion. Most successful implementations use additional techniques to reduce the variance of gradient estimates, improve optimization geometry, and prevent convergence to deterministic policies. Actor-critic methods maintain an estimate of the value function under the current policy and leverage this to reduce the variance of gradient estimates. The A3C algorithm by \cite{mnih2016asynchronous} is an influential variant that exploits distributed computation. Building off of works by \cite{kakade2001natural} and \cite{schulman2015trust}, the Proximal Policy Optimization (PPO) algorithm \citep{schulman2017proximal} also uses regularization intended to stabilize policy changes and improve optimization geometry. Researchers sometimes achieve fantastic results with these algorithms, but as mentioned in the introduction, critical inspections suggest that subtle implementation details \citep{huang202237} and careful hyperparameter choices \citep{henderson2018deep} have an enormous impact on reported results.

As discussed below, we use a policy gradient type algorithm, but use deterministic, rather than randomized policies, and rely on the use of a differentiable simulator to calculate gradients. 

\paragraph{\bf HDPO and differentiable simulators.}

HDPO builds on two well-established ideas in the stochastic simulation literature: using offline historical data to evaluate policy performance, and exploiting pathwise gradients when cost functions are smooth with respect to actions. Rather than relying on REINFORCE-style randomized estimators, these properties allow for deterministic policies and model-based computation of gradients with respect to policy parameters. See \cite[Chapter 7]{glasserman2004monte} for a comparison of likelihood ratio and pathwise derivative estimators, and \cite{glasserman1995sensitivity} for an early application to inventory policies.

Recent software developments make it practical to scale these ideas to high-dimensional policy classes. Modern frameworks such as TensorFlow \citep{tensorflow2015-whitepaper}, PyTorch \citep{paszke2019pytorch}, and JAX \citep{jax2018github} allow researchers to construct complex NN policies, compute gradients via automatic differentiation, and train models efficiently on GPUs. Our use of HDPO is also inspired by recent progress in robotics, where differentiable physics simulators have enabled high-performance control policies \citep{freeman2021brax, hu2019difftaichi}. We believe this paradigm holds great promise for solving large-scale inventory management problems.

The recent work of \cite{madeka2022deep} also applies HDPO to inventory management. We became aware of their study while our numerical experiments were already underway and view it as complementary to our own. Their work demonstrates that HDPO can outperform internal baselines at Amazon, providing a compelling case study of its real-world potential. 
Our paper makes three contributions not present in theirs. First, while their focus is on single-location problems, we consider networks with many locations—introducing added complexity in both state and action spaces and motivating the GNN architecture described in Section~\ref{sec:gnn-section}. Second, to our knowledge, we provide the first systematic evaluation of HDPO across diverse inventory control settings, including comparisons against certifiable lower bounds. Third, we develop and open-source a suite of benchmark problems and algorithms, supporting reproducibility and enabling standardized comparisons across methods.

\paragraph{\bf Policy-gradient methods for inventory control.}

Several works apply policy gradient type algorithms to inventory control. With the exception of \cite{madeka2022deep}, which we discussed above, all works apply advanced variants of REINFORCE algorithms; these are fundamentally different from HDPO.
A notable example is the work of \cite{gijsbrechts2021can}, which applies the A3C algorithm to address three classical inventory control problems. They put great effort into getting A3C to work well —including tuning hyperparameters separately for each set of problem primitives— and report the resulting performance on well-documented benchmarks. In some problems, they include comparisons against the true optimal cost. For these reasons, their work serves as a crucial benchmark in Section \ref{subsec:hdrl-and-hdpo-vs-reinforce}. Unlike A3C, we find HDPO performs  robustly well without requiring extensive parameter tuning or other tricks.  

Some follow up works have applied the PPO algorithm, another REINFORCE style policy gradient algorithm that is a successor to A3C \citep{vanvuchelen2020use, van2023using, vanvuchelen2023use, kaynov2024deep}. 
Unfortunately, it is difficult to rigorously benchmark HDPO's performance by comparing it against our own implementation of PPO; research suggests that hundreds of implementation tricks impact the performance of PPO \citep{huang202237}.  (Note that the importance of these tricks already points to the benefits of HDPO.)
We chose not to invest substantial effort in implementing and tinkering with PPO. Similar to A3C, PPO involves consideration of many tunable hyperparameters, contrasting with the simplicity of HDPO. Additionally, the results of \citep{vanvuchelen2020use, van2023using, vanvuchelen2023use, kaynov2024deep} all demonstrate cases where it is sometimes surpassed by hand-coded heuristics. For instance, \cite{kaynov2024deep} applied PPO to problems involving one warehouse and multiple locations, finding that PPO was often outperformed by an echelon-stock heuristic and emphasized the need for "better DRL algorithms tailored to inventory decision making". In a similar problem setup,  HDPO achieved an average optimality gap smaller than $0.15\%$ (refer to the last row of Table \ref{table:optimal-exps-summary} in Section \ref{subsec:optimal_exps}). 
Based on this evidence, we do not expect PPO to offer fundamentally different performance from the results in \cite{gijsbrechts2021can}, against which we carefully compare.

{\bf Further related literature.} Appendix~\ref{appendix:extended-lit-review} provides an extended review of related work on GNNs in RL, as well as other ML approaches for inventory management. 

\section{A Differentiable Framework for Inventory Network Control}
\label{sec: problem formulation}

This section formulates a rich class of inventory control problems in a manner that facilitates application of HDPO, our primary solution approach.  One key feature of inventory control is that there is, to a certain extent, a ``known physics'' to the problem: shipping an additional unit of inventory along an edge increases the receiving node's inventory by exactly one unit. This contrasts with many RL domains where system dynamics must be learned from scratch. Additionally, with purposeful but plausible modeling choices, we assume that (1) historical demand scenarios are observable and exogenous, (2) it is possible to adjust inventory flows in tiny (``continuous'') increments and (3) costs are smooth functions of the inventory flows across time. Together, these features enable us to backtest a policy's performance across a set of demand scenarios (whether historical or generated) and to directly improve policy performance by differentiating through \edit{a} simulator.

We begin with a general description of the inventory network control problems we study (Sections \ref{sec: inventory_network_control} and \ref{sec:structural-features}). We then turn to the abstract class of policies over which we optimize—emphasizing the state and action representations for, respectively, inputs and outputs of our policies (Section~\ref{sec: policies}). Section \ref{sec:hdpo} details HDPO and the key modeling assumptions that enable its application,  with Section \ref{sec:scope-and-modeling-assumptions} providing further discussion of these key assumptions.

\subsection{Inventory network representation and control}\label{sec: inventory_network_control}
The inventory network is modeled as a directed graph $\mathcal{G} = (\mathcal{N}, \mathcal{E})$, where $\mathcal{N}$ is the set of nodes (locations), and $\mathcal{E}$ contains a directed edge for each supplier-receiver relationship. For each node $k$, we define its set of \textit{suppliers} as $\mathcal{N}_{\text{sup}}^k = \{ j \in \mathcal{N} \mid (j,k) \in \mathcal{E} \}$, and its set of \textit{receivers} as $\mathcal{N}_{\text{rec}}^k = \{ j \in \mathcal{N} \mid (k,j) \in \mathcal{E} \}$.  Figures  \ref{fig:serial-system} and \ref{fig:one-warehouse-many-stores} and \ref{fig:many-warehouses-many-stores} depict examples of inventory networks we study.

\begin{figure}[htbp!]
\captionsetup[subfigure]{justification=centering}
\centering
\begin{subfigure}{.33\textwidth}
  \centering
  \includegraphics[width=1.0\linewidth, trim={0 -2.3cm 0 0.7cm},clip]{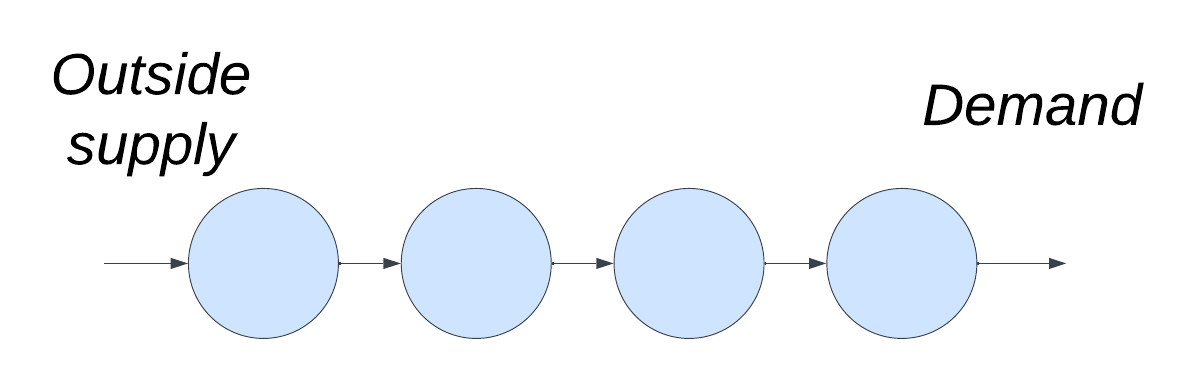}
  \caption{Serial \newline}
  \label{fig:serial-system}
\end{subfigure}%
\begin{subfigure}{.33\textwidth}
  \centering
  \includegraphics[width=0.8\linewidth, trim={0 0.0cm 0 0.0cm},clip]{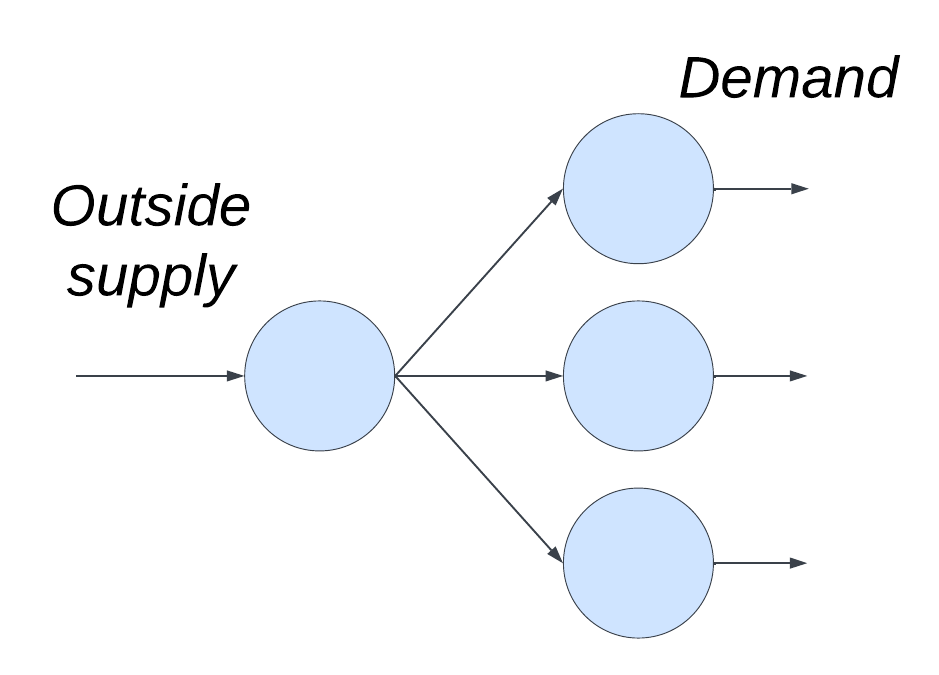}
  \caption{One-warehouse many-stores.\newline}
  \label{fig:one-warehouse-many-stores}
\end{subfigure}%
\begin{subfigure}{.33\textwidth}
  \centering
  \includegraphics[width=0.9\linewidth, trim={0 0.0cm 0 0.0cm},clip]{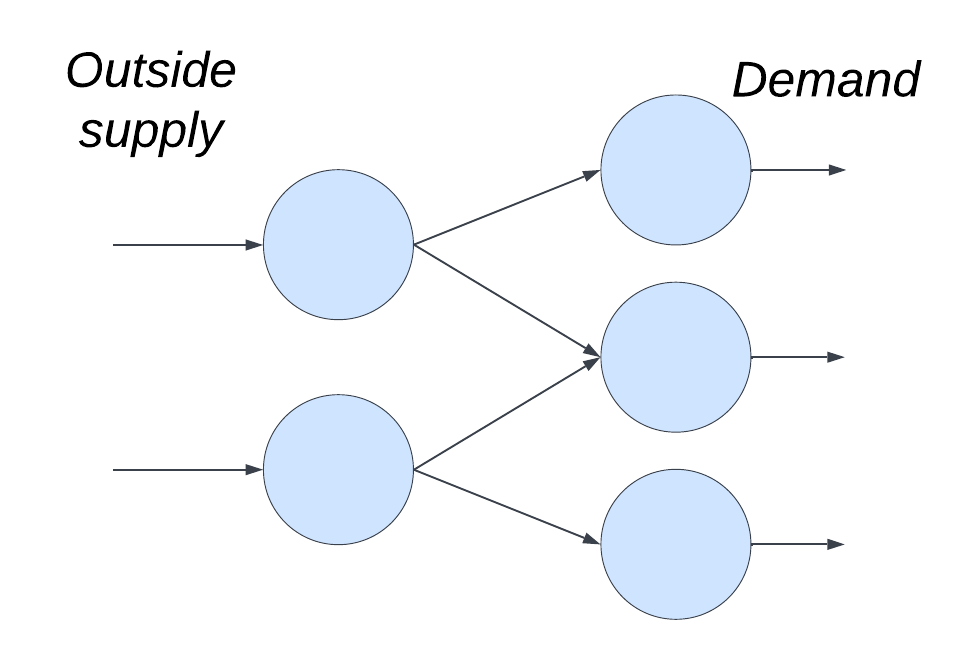}
  \caption{Many-warehouses many-stores.}
  \label{fig:many-warehouses-many-stores}
\end{subfigure}%

\caption{Diagrams of network topologies for inventory management problems. Arrows indicate the direction in which inventory flows.}
\label{fig:inventory-networks}
\end{figure}

We classify nodes into one of three types. First, a special node, labeled $0$, represents an \textit{external supply source} with unlimited inventory; this node does not hold inventory or make decisions beyond supplying connected nodes. Any edge of the form $(0, k) \in \mathcal{E}$ indicates that node $k$ can order from this external source.  Second, \emph{stores} are nodes where demand arises. These are denoted $\mathcal{N}_{\text{st}} \subseteq \mathcal{N}_+$, where $\mathcal{N}_+ := \mathcal{N} \setminus \{0\}$ represent the set of all physical locations in the network. Third, \emph{distribution centers} (DCs)  are physical locations that supply inventory to other nodes, defined as $\mathcal{N}_{\text{dc}} = \{ k \in \mathcal{N}_+ \mid \mathcal{N}_{\text{rec}}^k \neq \emptyset \}$. We will refer to a node that supplies a store as a \textit{warehouse}. 

\paragraph{\bf Inventory control dynamics.} The control problem unfolds over discrete time periods $t\in \{1,\ldots, T\}$. At each period, the decision-maker must determine how much inventory to allocate along each directed edge of the graph. Specifically, the action $a_t$ specifies an allocation $a^e_t \geq 0$ on each edge $e$, representing inventory flows between supplier-receiver node pairs. Allocations are only feasible if a supplier has enough inventory on hand to fulfill them.

Once allocated, inventory flows through the network. When $a^e_t$ units are allocated along edge $e$ at time $t$, exactly $a^e_t$ units arrive at the receiving node after a fixed edge-dependent lead time $L^e$. This arrival increases the inventory on hand at that node by exactly the allocated amount. We require $L^e \geq 2$ to avoid degenerate cases and simplify notation\footnote{Requiring $L^e \geq 1$ ensures that inventory orders placed at time $t$ cannot arrive and be used to fulfill demand at time $t$. The stricter requirement that $L^e \geq 2$ means that the set of outstanding orders (later denoted $Q_t$) is not empty, an edge case that requires slightly different notation.}.   

Simultaneously, demand from customers $\xi_t^k$ depletes inventory at store locations. Define $\xi_t = (\xi^k_t)_{k \in \mathcal{N}_{\text{st}}}$. We model two possibilities for unmet demand. Under a \emph{lost demand} assumption, customers who arrive when inventory is unavailable leave immediately and do not return—the sale is permanently lost. Under a \emph{backlogged demand} assumption, customers are willing to wait for unavailable items, creating a queue of unfulfilled orders that must be satisfied when inventory becomes available.

The decision-maker incurs costs based on inventory positions and flows. There are fixed per-period per-unit holding costs $h^k$ at each node $k$ for maintaining inventory, underage costs $p^k$ for unmet demand (stockouts in the lost demand case, or backorders in the backlogged case), and edge-level costs $\lambda^e$ representing per-unit shipment costs or procurement costs when the origin node is the outside supplier.

While this problem is complicated due to high-dimensionality and, potentially, complex patterns in demand, it clearly exhibits the natural "known physics" we highlighted at the outset.

\begin{figure}[ht]
    \centering
    \includegraphics[width=0.5\textwidth]{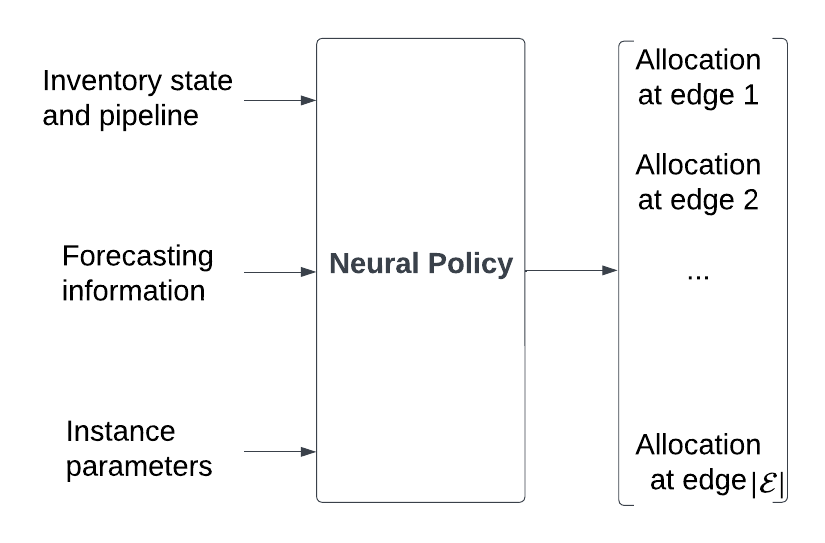}
    \caption{Inputs and outputs of a neural policy.}
    \label{fig:neural-policies}
\end{figure}

\subsection{Structural Features Defining a Setting.} \label{sec:structural-features}
Our study constructs a range of inventory control settings by varying four structural features:

\begin{enumerate}
\item {\bf Network topology.} Ranges from single-location systems to complex multi-echelon networks. Our experiments include the following network structures: single-store, serial (Figure~\ref{fig:serial-system}), one-warehouse many-stores (OWMS; Figure~\ref{fig:one-warehouse-many-stores}), and many-warehouses many-stores (MWMS; Figure~\ref{fig:many-warehouses-many-stores}) topologies.
\item {\bf Unmet demand assumption.} Determines whether unmet demand is (i) \textit{backlogged}, meaning customers wait until the product is restocked, or (ii) \textit{lost}, meaning unmet demand disappears.
\item {\bf Demand process.} Specifies how demand traces are generated. Importantly, the policy does not assume knowledge of the demand distribution and learns directly from observed scenarios. We consider both synthetically generated processes (e.g., from Poisson or i.i.d. multivariate normal distributions) and realistic demand datasets (see Section~\ref{sec:vanilla-hdpo-realistic}; details in Appendix~\ref{appendix:realistic-demand-dataset}).
\item {\bf Structural constraints.} Captures additional structural assumptions imposed on the network. For example, a location may be designated as a \textit{transshipment center}, meaning it cannot hold inventory—equivalently, it is assigned an infinite holding cost.
\end{enumerate}

\subsection{Neural policies and state/action representation}\label{sec: policies}
We consider neural policies that map states  to actions in inventory networks. The design of these policies --- specifically how we structure their inputs and outputs --- represents a key design choice that enables effective learning and optimization.

\subsubsection{State representation as policy input:}
The state $S_t$ serves as input to the policy, implicitly defining the information on which decisions are based. The state has both \emph{dynamic components}---which evolve throughout the $T$ periods and are influenced by the agent's actions---and \emph{static components} --- which are exogenously determined and fixed upfront at the start of any episode of decision-making. 

The dynamic components of the state consist of the concatenation of \textit{local states} across nodes in the network given by $(I^k_t, Q^k_t, \mathcal{F}_t^k)_{k \in \mathcal{N}_+}.$ This consists of 
    \begin{enumerate}
    \item A \emph{physical inventory state}, which includes inventory on hand $I^k_t \in \mathbb{R}$ and information $Q^k_t$ on outstanding orders expected to arrive at the node in the future. We assume that each edge $e \in \mathcal{E}$ has a fixed lead time $L^e \geq 2$. For each node $k \in \mathcal{N}_+$, define $\bar{L}^k = \max_{e = (j,k)} L^e$ as its maximum incoming lead time. Then each $Q^k_t \in \mathbb{R}_+^{\bar{L}^k - 1}$ can be represented as a fixed-length vector that tracks outstanding orders placed over the previous $\bar{L}^k - 1$ periods.\footnote{Random lead times can be handled by letting $Q^k_t$ track orders until arrival and updating $I^k_t$ based on realized delays.}
    \item \emph{Forecasting information} $\mathcal{F}_t^k$ contains all information relevant to predicting future demand at location $k$. In our experiments, this usually consists of recent and historical demand observations. It could, in principle, include product features, contextual information, etc.
\end{enumerate}
In addition, we pass static information to the policy by including it as part of the state. This enables meta-learning: a single neural policy can adapt its behavior across different problem instances by conditioning on information like their specific cost structures, lead times, or inventory network structure.
\begin{enumerate}
    \item   \emph{Instance parameters} $\mathcal{R}^k$ include node-specific features such as costs, lead times, and other parameters known to the decision-maker at the initial period. These do not change over time. In our training, we sometimes vary $\mathcal{R} = (\mathcal{R}^k)_{k \in \mathcal{N}_+}$ by drawing it randomly from a specified distribution, enabling \emph{meta-learning} across related problem instances.
    \item The \emph{inventory network} $G$ encodes fixed supplier relationships. In all of our policies, this is used to enforce feasibility of actions (see below). In our application of GNN policy architectures, the computation graph of the policy also mimics $G$. While we do not engage in meta-learning across graph structures, this is possible within our framework and is a natural extension to pursue. 
\end{enumerate}

\subsubsection{Action space and feasible allocations: \label{section:action-and-feasible-allocations} }
The output of a policy specifies an allocation $a^e_t$ on each edge $e \in \mathcal{E}$, representing inventory flows between supplier–receiver node pairs. Allocations must respect feasibility: no DC may ship more than its available inventory.

We adopt a mechanism that ensures the NN produces feasible actions while remaining differentiable in the a NN parameters. Since the policy must output one allocation per edge, the network generates an \textit{intermediate output} $b^{(j,k)}$ for each edge $(j,k) \in \mathcal{E}$. These outputs are passed through \textit{feasibility enforcement layers} (FELs), which transform them into valid allocations. As illustrated in Figure~\ref{fig:neural-policies}, we treat FELs as part of the neural policy.

To ensure that a distribution center (DC) does not allocate more inventory than it holds, we apply hard-coded, differentiable functions independently to each supplier node $j \in \mathcal{N}_{\text{dc}}$. Each FEL maps the intermediate outputs $(b^{(j,k)})_{k \in \mathcal{N}_{\text{rec}}^j}$ to feasible allocations $(a^{(j,k)})_{k \in \mathcal{N}_{\text{rec}}^j} \in \mathbb{R}_+^{|\mathcal{N}_{\text{rec}}^j|}$ that satisfy the constraint $\sum_{k \in \mathcal{N}_{\text{rec}}^j} a^{(j,k)} \leq I^j$. While this construction is specific to our setting, the same principle may apply more broadly by designing appropriate differentiable mappings. Full definitions of the FELs used in our experiments are provided in Appendix~\ref{appendix:feasibility-enforcement}.

\subsection{Hindsight differentiable policy optimization}\label{sec:hdpo}
 We now describe how to optimize the neural policies introduced in the last subsection, exploiting special structures to enable efficient learning through pathwise gradients.

\subsubsection{Hindsight Differentiable  structure of inventory control. \label{sec:HD-structure} }
The objective of the decision-maker is to solve 
\begin{align}
\label{eq:mdp-objective}
    \begin{split}
    \min_{\theta} & \ \mathbb{E}_{\xi} \left [ \,\tfrac{1}{T} {\textstyle \sum_{t \in [T]}} c(S_t, \pi_{\theta}(S_t), \xi_t) \big | S_1 \right ] \quad  \textup{subject to} \quad S_{t+1} = f(S_t, \pi_{\theta}(S_t), \xi_t) \ \ \forall t \in [T]
    ,
    \end{split}
\end{align}
with $[T] = \{1,\ldots,T \}$, and where the system function $f$ governs state transitions, the cost function $c$ governs per-period costs, and $\xi = (\xi_1, \ldots, \xi_T)$ is a sequence of demand realizations. The policy $\pi_{\theta}$ is a (sufficiently) smooth function of $\theta \in \mathbb{R}^{b}$ which maps states to actions. Nonstationary policies can be represented by incorporating time into the state $S_t$. Instead, our numerical results largely focus on problems with large $T$ where stationary policies are nearly optimal.

Our approach to solving this relies on three key properties that, together, enable what we call \textit{hindsight differentiable policy optimization} (HDPO):

\begin{enumerate}
    \item {\bf Samples of exogenous demand scenarios.} We assume access to $H$ scenarios $(\bar{S}^h_1, \bar{\xi}^h_{1:T})$ where the sequence $(\xi_1, \ldots, \xi_T)$, with $\xi_t = (\xi^k_t)_{k \in \mathcal{N}_{\text{st}}}$, represents uncensored demand realizations at each store over time. The initial state $\bar{S}^h_1$ specifies initial inventory state and initial forecasting information, alongside instance parameters, which we usually construct synthetically in our simulations. 
    \item {\bf Known deterministic dynamics under given demand scenarios.} The transition function $f$ and cost function $c$ are fully specified. The state transitions of inventory on hand capture the "known physics" of inventory flow highlighted in Section \ref{sec: inventory_network_control}. The outstanding orders vector $Q^k_t$ shifts forward one period, with the oldest entry becoming available inventory and new orders appended based on current allocations. When forecasting information $\mathcal{F}^k_t$ consists of a sliding window of recent demands, it similarly shifts forward, dropping the oldest demand observation and incorporating the newly realized demand $\xi^k_t$. Cost functions are also known and come from applying fixed per-period per-unit  holding costs and edge-level costs defined in Section \ref{sec: inventory_network_control}.
    \item {\bf Continuous actions and differentiability of costs.}  Both $f$ and $c$ are continuous and almost everywhere differentiable with respect to the allocation decisions.  In particular, our model treats ordering quantities as continuous variables and---since all costs are variable, applied per-unit---assigns costs as continuous functions of these quantities.  Together, these imply that for fixed $S_1$ and $\xi$, the total cost $\sum_{t \in [T]} c(S_t, a_t, \xi_t)$ is continuous and a.e.\ differentiable in $(a_1, \ldots, a_T)$.
\end{enumerate}

 \subsubsection{Framework for training and evaluating policies. \label{sec:meta-learning-framework} } 

Given these structural properties, we can now describe our meta-learning framework for training neural policies across diverse inventory control problems. This framework, shown in Figure \ref{fig:training_and_evaluation}, leverages properties 1 and 2—access to exogenous demand scenarios and known deterministic dynamics—to enable reliable policy evaluation and comparison across multiple problem instances without requiring costly online experimentation. A single policy is trained across all instances by conditioning on instance-specific features, allowing it to generalize without retraining.

\setlength{\fboxsep}{6pt}              
\begin{figure}[htbp]
  \centering
  \fbox{%
    \begin{minipage}{.9\linewidth}
      \begin{enumerate}[leftmargin=*]
         \item Fix a {\bf \emph{setting}} by specifying the key structural features in Section \ref{sec:structural-features}---for example, ``one warehouse supplying multiple stores, with Poisson demand, where unmet demand is lost forever.''
    \item Fix a {\bf \emph{meta-instance}}, which specifies how instance parameters $\mathcal{R}$ are sampled (e.g., distributions over cost coefficients, lead times, edge costs, etc.).
    \item Sample many instance parameters $\mathcal{R}$. 
    \item For each instance, sample one or more demand {\bf \emph{scenarios}}  $(\bar{S}^h_1, \bar{\xi}^h_{1:T})$. 
    \item Optimize the shared policy to minimize average cost across (instance, scenario) pairs.
    \item Evaluate the policy on held-out instances from the same meta-instance distribution and held-out scenarios.  
      \end{enumerate}
    \end{minipage}}
  \caption{Six‑step experimental pipeline and key terminology.}\label{fig:training_and_evaluation}
\end{figure}

 \subsubsection{HDPO for optimizing policy parameters.} 
We implement the optimization step of our meta-learning framework using hindsight differentiable policy optimization (HDPO), our core method for training neural policies across instances and demand scenarios (step 5 of Figure~\ref{fig:training_and_evaluation}). The training process is illustrated in Figure \ref{fig: hdpo}, with the complete algorithm detailed in Algorithm \ref{alg:hdpo_full} in Appendix \ref{appendix:hdpo}. HDPO enables stable and efficient training by leveraging three key structural features that we carefully enforce: demand realizations are exogenous and fixed during training; inventory dynamics are known and deterministic conditional on actions and demand; and both cost and transition functions are smooth. Together, these allow us to treat the simulator as a differentiable computation graph and use standard backpropagation to compute how changes in policy parameters influence cumulative cost by tracing their effect on actions over time. In practice, we implement this in PyTorch, which supports automatic differentiation and GPU-accelerated training of expressive neural architectures. Further implementation details are provided in Appendix~\ref{appendix:implementation-and-experiments}.

HDPO stands in contrast to REINFORCE-style policy gradient methods, which estimate gradients based only on observed rewards and actions — without access to how counterfactual action changes would have altered outcomes. These methods rely on the likelihood-ratio trick to construct unbiased but often high-variance gradient estimates. In contrast, HDPO exploits the simulator's structure to compute low-variance gradients of expected cost with respect to policy parameters. In effect, it uses each demand scenario to assess the impact of all infinitesimal parameter perturbations, rather than learning only from the single trajectory that a given policy happens to generate. Appendix~\ref{appendix:score-function-gradient} provides additional details on the computation of gradient estimators for REINFORCE-style policy gradient methods.

\setlength{\fboxsep}{6pt}     
\begin{figure}[htbp]
  \centering
  \fbox{%
    \begin{minipage}{.9\linewidth}
      \begin{enumerate}[leftmargin=*]
        \item Sample a mini-batch of (instance, scenario) pairs.
        \item Simulate forward trajectories by recursively applying $f$ and $\pi_{\theta}$, accumulating per-period costs. 
        \item Backpropagate gradients with respect to the total cost incurred through the entire horizon and across the entire mini-batch. 
        \item Update policy parameters $\theta$ using standard optimizers (we use Adam). Go to step 1. 
        \end{enumerate}
    \end{minipage}}
  \caption{Hindsight differentiable policy optimization training process.} 
  \label{fig: hdpo}
\end{figure}

\subsection{Discussion of modeling assumptions. \label{sec:scope-and-modeling-assumptions}}

Our framework relies on two key modeling choices.

First, our formulation departs from classical inventory management, which typically considers discrete ordering decisions and substantial fixed ordering costs. In traditional single-product models, these fixed costs lead to the well-known (s,S) policies that wait until inventory depletes significantly before placing large replenishment orders. However, our continuous formulation reflects the operational reality of large retailers who manage diverse product assortments on fixed delivery schedules. From this perspective, the fixed costs of shipments are amortized across many products, making it natural to exclude shipment-level fixed costs when optimizing orders for individual products. This continuity assumption, combined with differentiable dynamics, enables gradient computation through entire simulation trajectories.

Second, we assume demand is \emph{exogenous}---that customer arrivals are unaffected by inventory availability---and we have access to realistic samples of realized demand scenarios $\bar{\xi}^h_{1:T}$. While the exogeneity assumption is strong, it is ubiquitous in the inventory management literature \citep{gijsbrechts2021can, xin2021understanding}, which tends to focus on balancing costs while \emph{meeting} unpredictable demand rather than purposefully \emph{shaping} future demand through stocking decisions. We remain agnostic to the source of these uncensored demand samples. While addressing demand censoring is crucial for simulator fidelity, this challenge is orthogonal to our methodological contributions in policy optimization. Depending on the context, these samples could result from a preprocessing step that imputes true demand from sales data and auxiliary signals (e.g., traffic to the product page on a website, as done in \cite{madeka2022deep}) or could be the output of a generative model trained on partially censored observations. We investigate the impact of censoring in Appendix~\ref{appendix:censored-demand}.  We further discuss our broader modeling assumptions in Section~\ref{sec:discussion-on-main-assumptions}.

\section{Vanilla HDPO: Effectiveness in Moderate-Sized Networks and the Need for Structure at Scale}
\label{sec:vanilla-hdpo}

This section summarizes the results of a comprehensive evaluation of the "Vanilla" HDPO approach across inventory control problems with diverse structural features and moderate network sizes. We focus on the \textit{Vanilla NN} architecture—a monolithic MLP that does not exploit any structural properties of the inventory network. Our key finding is that HDPO performs very well in moderate-sized inventory networks when sufficient demand traces are available, even without sophisticated modifications and using a straightforward policy network architecture. We conclude the section, and motivate Section \ref{sec:gnn-section}, by observing fundamental limitations of Vanilla NNs as we scale the size of inventory networks. 

Results in this section focus on optimization and assume access to a \textit{large} dataset of uncensored demand. Appendix~\ref{appendix:vanilla-hdpo-sample-efficiency} evaluates sample efficiency in the one-store setting, showing that HDPO achieves optimality gaps below $1\%$ with as few as $128$ traces.

In Section \ref{subsec:optimal_exps} we evaluate Vanilla HDPO on inventory control problems whose hidden structure enables us to either compute or tightly bound the cost attained by an optimal policy. Here, HDPO works with "raw" input features, without tailoring the methodology to each problem's special structure. We find that HDPO  achieves (essentially) optimal average cost for each of these problems. As highlighted in the introduction, it is quite rare to benchmark deep RL against the global optimum; our ability to do so rests on decades of research in operations that has uncovered the hidden structure of the global optimum in some settings of interest.

Section \ref{subsec:hdrl-and-hdpo-vs-reinforce} compares our results with those reported by \cite{gijsbrechts2021can}, highlighting the benefits of HDPO over popular policy gradient methods in deep RL. HDPO exhibits much better performance on the problems tested in \cite{gijsbrechts2021can}, and does so consistently without the need for a plethora of tricks or hyperparameter tuning. 

Later, Section~\ref{sec:vanilla-hdpo-realistic} evaluates HDPO performance using real-world data from the \textit{Favorita} dataset, where the primary challenge lies in addressing nonstationary sales patterns. HDPO directly maps historical demand windows and inventory information to ordering decisions, while \textit{generalized newsvendor policies} use a two-step approach: first forecasting the demand distribution, then optimizing based on quantiles of that forecast. Our findings indicate that HDPO's direct mapping approach consistently outperforms the two-step generalized newsvendor methods in a lost demand model.

Finally, Section \ref{sec:vanilla-unrelated-stores} assesses the sample efficiency of Vanilla HDPO in a setting with multiple identical but unrelated stores. Since the policy architecture has no structural bias for recognizing that stores are identical and independent, the Vanilla NN must infer this structure from data alone. Empirically, we find that performance worsens as the number of stores increases, which motivates our GNN architecture presented in Section \ref{sec:gnn-section}.

{\bf Vanilla NN architecture.}  Numerical results in this section use a \textit{Vanilla NN} architecture, shown in Figure~\ref{fig:vanilla nn} in Appendix \ref{appendix:vanilla-architecture}. This consists of a fully connected multilayer perceptron (MLP) that takes the full state vector as input and gives one intermediate output per edge in the inventory network. Full architectural details are provided in Appendix \ref{appendix:gnn-architecture}. We treat this as a basic instantiation of HDPO—capturing an “unstructured” application of the method—and use it as a baseline for comparison with our proposed GNN-based architecture in Section~\ref{sec:gnn-section}.

\subsection{HDPO recovers near-optimal policies in problems with hidden structure
\label{subsec:optimal_exps}}

Here, we evaluate the performance of HDPO in four distinct settings where the optimal cost is either exactly known or tightly bounded. These settings span a wide range of structural features—including single-node and multi-node systems, serial and OWMS network topologies, and both backlogged and lost-demand assumptions. The results underscore the flexibility of the method: HDPO consistently achieves near-optimal performance across structurally diverse inventory control problems by operating directly on raw state inputs.

\begin{table}[h!]
\centering
\caption{Overview of settings used in Section~\ref{subsec:optimal_exps}.}
\label{tab:opt-exp-settings-summary}
\begin{tabular}{@{}lllll@{}}
\toprule
\textbf{ID} & 
\shortstack{\textbf{Network} \\ \textbf{Structure}} & 
\shortstack{\textbf{Unmet Demand} \\ \textbf{Assumption}} & 
\shortstack{\textbf{Demand} \\ \textbf{Distribution}} & 
\shortstack{\textbf{Structural} \\ \textbf{Constraints}} \\ \midrule
S1 & Single-store & Backlogged & Normal & -- \\
S2 & Single-store & Lost & Poisson & Discrete allocations \\
S3 & Serial & Backlogged & Normal & -- \\
S4 & OWMS & Backlogged & Normal & Transshipment warehouse \\
\bottomrule
\end{tabular}
\end{table}

{\bf Benchmarks.} We study the four settings summarized in Table~\ref{tab:opt-exp-settings-summary}. We begin with two settings involving a single store: one where unmet demand is assumed to be backlogged \citep{arrow1958studies} (S1), and the other where it is considered lost \citep{zipkin2008old} (S2). Next, we consider a setting with a serial network structure (Figure~\ref{fig:serial-system}), operating under a backlogged demand assumption \citep{clark1960optimal} (S3). Finally, we analyze the setting introduced by \citet{federgruen1984approximations}, which involves a OWMS network (Figure~\ref{fig:one-warehouse-many-stores}), also under the assumption of backlogged demand (S4). In this setting, as the warehouse cannot hold inventory, we refer to it as a \textit{transshipment center}.
For each setting, we generate multiple instances, training a separate neural policy per instance by sampling multiple scenarios. This setup facilitates heuristic baseline design, enables controlled comparisons, and aligns with prior work \citep{gijsbrechts2021can, xin2021understanding}.

{\bf Experiment specifications.} Demand scenarios are split into train, development (dev), and test sets. Policies are trained using the Adam optimizer with early stopping based on dev set performance. Reported metrics are computed on the test set, excluding a fixed number of warm-up periods.
To approximate average loss, we train stationary policies over long episodes. Train and dev episodes span 50 and 100 periods, with evaluation over the final 20 and 40 periods, respectively. Test episodes span 5,000 periods, with costs computed over the final 2,000 to reflect steady-state behavior.
We use 32{,}768 scenarios for each of the train, dev, and test sets. Hyperparameters are fixed across instances within each setting. For single-location settings, we report the cost of a single run per instance; for multi-location settings, we report the best of three runs to mitigate occasional suboptimal outcomes and ensure robust evaluation. Precise implementation details for the numerical results are presented in Appendix~\ref{appendix:optimal-exps}.

{\bf Results.}
Table~\ref{table:optimal-exps-summary} summarizes the performance of Vanilla HDPO across multiple instances for each setting. Our findings show that the Vanilla NN reliably achieves near-optimal policies in instances with a state space of up to 63 dimensions, even with raw state inputs. Our results also show that HDPO effectively handles inventory problems with network structures when constraints are differentiable. This includes a serial network structure with 4 locations and a warehouse with 3 to 10 stores.
Computational constraints were not a limitation in these experiments. HDPO consistently achieved performance within 1\% of optimality within 7 and 2 minutes for settings S1 and S2, respectively (see Tables \ref{table: backlogged detailed} and \ref{table: lost demand detailed} in Appendix \ref{appendix:optimal-exps}). For settings S3 and S4, it achieved this performance in under 37 and 24 minutes, respectively (see Tables \ref{table:serial} and \ref{table: trans_shipment detailed} in the Appendix).

\begin{table}[h]
\centering
\caption{Vanilla NN performance in settings where the optimal cost can be bounded or computed. We report performance under the best-performing hyperparameter configuration.}
\label{table:optimal-exps-summary}

\begin{tabular}{@{}llllll@{}}
\toprule
\textbf{Setting} & 
\shortstack{\textbf{Number of} \\ \textbf{locations}} & 
\shortstack{\textbf{Raw state} \\ \textbf{dimension}} & 
\shortstack{\textbf{Instances} \\ \textbf{tested}} & 
\shortstack{\textbf{Average} \\ \textbf{opt. gap}} & 
\shortstack{\textbf{Max.} \\ \textbf{opt. gap}} \\ 
\midrule
S1 & 1 & 1 to 20 & 24 & 0.03\% & 0.17\% \\
S2 & 1 & 1 to 4 & 16 & $<$0.25\% & $<$0.25\% \\
S3 & 4 & 10 to 13 & 16 & \edit{0.54}\% & \edit{0.78}\% \\
S4 & 4 to 11 & 9 to 63 & 24 & $\leq$\edit{0.13}\% & $\leq$\edit{0.24}\% \\
\bottomrule
\end{tabular}
\end{table}

\subsection{On the benefits of HDPO
\label{subsec:hdrl-and-hdpo-vs-reinforce}}

As explained in Section \ref{sec:hdpo}, HDPO leverages important problem properties (known system, observation of historical scenarios, and differentiability) to enable 
an efficient search over the parameters of a neural policy. Is this important, or would generic deep RL methods have matched the exceptional performance summarized in Table \ref{table:optimal-exps-summary}? 

Some insight can be gained by comparing our results to those reported in \citet{gijsbrechts2021can}, who evaluated an A3C-based approach on 6 of the 16 instances from the classic inventory control study by \citet{zipkin2008old}. This corresponds to Setting S2, which features a single store facing stationary Poisson demand, with lost sales and discrete allocations. We chose not to implement other policy-gradient algorithms, as their performance is often highly sensitive to implementation details and requires substantial engineering effort to ensure reliable training.

{\bf Benchmarks.} We consider Setting S2 and replicate 6 of the 16 instances from the testbed in \citet{zipkin2008old} by setting the holding cost $h = 1$, underage cost $p \in \{4, 9\}$, and lead time $L \in \{2, 3, 4\}$.

{\bf Baselines.} We consider the results reported by \cite{gijsbrechts2021can} regarding their A3C method and we include results from the Capped Base-Stock (CBS) heuristic (see Appendix \ref{appendix:capped-base-stock-policy} for a definition) proposed by \cite{xin2021understanding}, which achieved an average optimality gap below 0.7\% in the 16-instance test bed.

{\bf Experiment specifications}.  We trained our model as if actions were continuous and rounded prescribed allocations to the nearest integer at test time. To test robustness to hyperparameter choices, we considered architectures with two and three hidden layers, learning rates of $10^{-4}, 10^{-3}, 10^{-2}$, and batch sizes of $1024$ and $8192$, resulting in $2 \times 3 \times 2 = 12$ hyperparameter combinations. Each configuration was run once on all instances.
We evaluate performance using the optimality gap with respect to the exact dynamic programming costs reported by \cite{zipkin2008old}. For each instance, we compare the \textbf{worst} loss achieved by HDPO across hyperparameter settings to the \textbf{best} loss reported for A3C by \cite{gijsbrechts2021can}, which corresponds to the best outcome among approximately 250 hyperparameter combinations. For CBS, we used the optimal parameters reported by \cite{xin2021understanding} but re-evaluated the policy in our environment.
Note that \cite{zipkin2008old} report lower bounds on the optimal cost rounded to two decimal places, whereas we did not round the losses obtained by HDPO or CBS. Since the smallest optimal cost is approximately $4.04$, gaps smaller than $0.01 / 4.04 \times 100 \approx 0.25\%$ are considered below the resolution of the benchmark and thus not detectable.

{\bf Results.} Table \ref{table: hdpo-vs-a3c} summarizes the performance of each approach on the 6-instance test bed. While the approach presented by \cite{gijsbrechts2021can} falls short of surpassing the CBS heuristic in any of these instances, with gaps between $3.0\%$ and $6.7\%$ for the best run, HDPO consistently outperforms the CBS heuristic, demonstrating near-optimal performance across the board. HDPO achieved the minimum detectable optimality gap of $0.25\%$ for all hyperparameter configurations on these 6 instances, and for all but one run across the complete 16-instance test bed (192 total runs). In contrast, \cite{gijsbrechts2021can} reported that their A3C algorithm exhibited high sensitivity to hyperparameters, necessitating extensive tuning.

\begin{table}[h!]
\begin{center}
\caption{
Optimality gaps in 6 instances studied in \citet{zipkin2008old} for A3C \citep[costs reported by][reflecting best run across around 250 hyperparameters]{gijsbrechts2021can}, CBS heuristic (parameters from \citep{xin2021understanding}) and HDPO (worst run across 12 hyperparameter settings per instance).
\label{table: hdpo-vs-a3c}}

\begin{tabular}{>{\raggedleft}p{3.3cm}>{\raggedleft}p{1.6cm}>{\raggedleft}p{1.6cm}>{\raggedleft}p{1.6cm}>{\raggedleft}p{1.6cm}>{\raggedleft}p{1.6cm}>{\raggedleft\arraybackslash}p{1.5cm}}
\toprule
Approach & $L$ = 2  \\$p$ = 4 & $L$ = 2 \\$p$ = 9 & $L$ = 3 \\$p$ = 4 & $L$ = 3 \\$p$ = 9 & $L$ = 4 \\$p$ = 4 & $L$ = 4 \quad$p$ = 9 \\
\midrule
                A3C (best run) &    3.20\% &    4.80\% &    3.00\% &    3.10\% &    6.70\% &    3.40\% \\
  CBS &   $<$ 0.25\% &   0.43\% &   0.67\% &   1.34\% &   1.63\% &   1.04\% \\
  HDPO (worst run) &    $<$ 0.25\% &   $<$ 0.25\% &  $<$ 0.25\% &   $<$ 0.25\% &   $<$ 0.25\% &  $<$ 0.25\% \\
\bottomrule
\end{tabular}

\end{center}
\end{table}

Details on performance across different hyperparameter choices are provided in Table~\ref{table: lost demand by hyperparams} in Appendix \ref{appendix:single-store-lost-demand-exps}. The results demonstrate consistently reliable behavior. Notably, under a specific hyperparameter configuration, our algorithm converges to solutions with less than $1\%$ optimality gap in under $\edit{55}$ seconds on average.

Appendix~\ref{appendix:vanilla-hdpo-sample-efficiency} further analyzes HDPO’s sample efficiency in this setting. It achieves optimality gaps below 1\% with as few as 128 train scenarios and performance indistinguishable from optimal with 512. These findings suggest that even retailers with a relatively small product catalog may benefit from applying HDPO.

\subsection{HDPO outperforms generalized newsvendor heuristics in problems with time-series data 
\label{sec:vanilla-hdpo-realistic} }

The previous set of experiments evaluated HDPO on stylized problems and showed that it can achieve near-optimal performance across a broad range of settings without leveraging domain-specific knowledge. Here, we assess the applicability of HDPO to realistic settings by testing it on single-store problems constructed from real-world demand data. The testbed isolates a single challenging problem feature: nonstationarity in real-world demand patterns. By simplifying other problem features, we facilitate the adaptation of ideas from the operations literature to design heuristic policies. Notably, our application of HDPO did not appear to leverage these simplifying assumptions, and we expect that its performance would not undergo substantial changes if, for example, lead times were random.

{\bf Benchmarks.}
In setting S5, we consider a single-store network with lost demand and realistic demand traces from the publicly available \textit{Corporación Favorita Grocery Sales Forecasting} dataset \citep{kagglefavorita}, which contains over 200{,}000 daily time series for 4{,}000 products across 54 stores in Ecuador. To improve data quality, we apply several filtering steps (see Appendix~\ref{appendix:realistic-demand-dataset} for details). Each (store, product)-level trace is treated as an independent scenario.
We define multiple meta-instances by specifying a mean underage cost $\hat{p} \in \{2, 3, 4, 6, 9, 13, 19\}$, which determines the range from which instance-level underage costs are sampled. Holding costs are fixed, and lead times are drawn uniformly from a fixed range.

\textbf{Baselines.} We implement a class of \textit{generalized newsvendor policies} that operate in two steps: first, they forecast cumulative demand over the product’s lead time; then, they place orders by selecting a fixed quantile of the predicted demand distribution. We evaluate three variants—\textit{Newsvendor}, \textit{Fixed Quantile}, and \textit{Transformed Newsvendor}—which differ in how the quantile is chosen. The most flexible variant learns a mapping from underage and holding costs to a quantile. Full definitions are provided in Appendix~\ref{appendix:generalized-newsvendor}. All policies in this class rely on a shared, offline-trained demand forecaster.
These policies are motivated by the fact that, under backlogged demand, future costs depend on cumulative demand until replenishment. Adjusting the quantile provides a way to account for differences in unit economics or model misspecification.

{\bf Experiment specifications.}
We partition demand scenarios temporally into train, dev, and test splits. Each is preceded by a 32-week pre-period—comprising 16 weeks of input history and 16 weeks of warm-up—to provide historical context and reduce sensitivity to initial conditions (see Appendix \ref{appendix:global-settings} for the detailed split specifications). Inventory is reinitialized to zero at the start of the warm-up period.
Performance is evaluated in terms of \textit{profit} (equivalent to minimizing cost with $p$ as the underage cost), measured relative to a non-admissible \textit{Just-in-time policy} that perfectly anticipates demand and orders exactly what is needed. This policy serves as an upper bound on achievable profit.
For the Vanilla NN, we fix hyperparameters across all meta-instances based on early experimentation and run each configuration once per meta-instance. For baselines, we train a quantile forecaster offline and tune policy parameters using the same procedure as for neural policies.

{\bf Results.}
Figure \ref{fig: real_data_lost_demand} presents the outcomes for the lost demand setting, across the 7 meta-instances tested. It depicts the total profit as a percentage of the value achieved by the Just-in-time benchmark. The figure illustrates that HDPO consistently outperforms all heuristics across instances, achieving over $80\%$ of the hindsight optimal profit when $\hat{p}$ takes a value of $9$ or larger. Our agent surpasses the best generalized newsvendor policy by up to $\edit{24}\%$, and the performance gap diminishes as the average unit underage cost increases. Additional details on the performance of HDPO in this setting can be seen in Table \ref{table:real-data-lost-demand} in Appendix \ref{appendix:one-store-real-setting}.

For further analysis, we compute the \textit{average weekly revenue} and \textit{average weekly holding cost} by summing the revenue $p\min\{\xi_t, I_t\}$ and holding cost $h(I_t - \xi_t)^+$ across all scenarios and weeks, then dividing by the total number of scenario-weeks. Figures~\ref{fig: real-data-average-revenue} and~\ref{fig: real-data-average-holding-cost} show the average weekly revenue and holding cost, respectively, normalized by the average unit underage cost in the setting with lost demand. The figures illustrate that HDPO consistently achieves higher revenues and lower holding costs compared to competitive benchmarks. The exception is the newsvendor policy, which severely over-orders and thus incurs very high holding costs—but as a result achieves higher revenues than HDPO.

\begin{figure}[htbp!]
\captionsetup[subfigure]{justification=centering}
\begin{subfigure}[t]{.32\textwidth}
  \centering
  \includegraphics[width=1.0\linewidth]{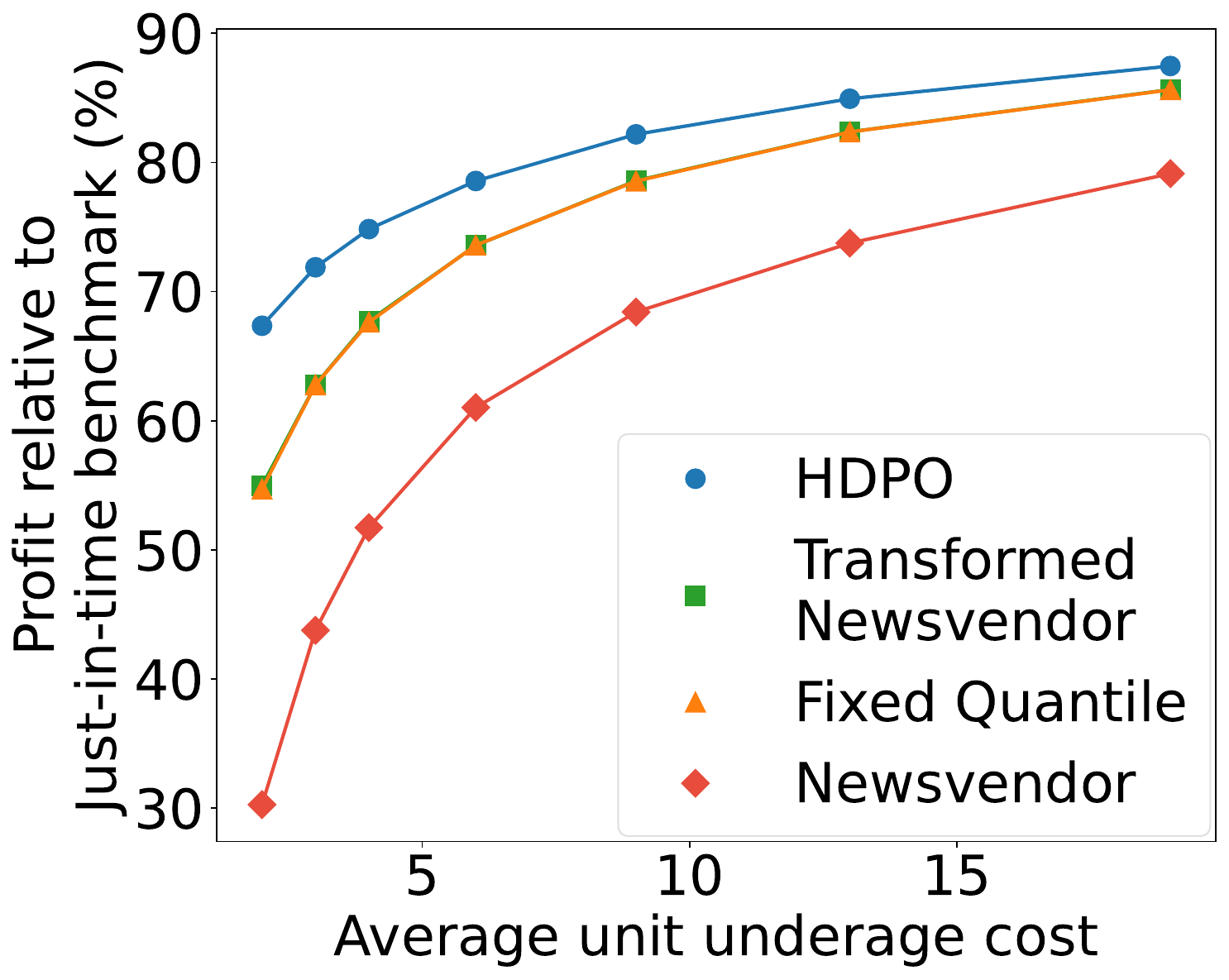}
  \caption{Profit relative to Just-in-time policy. Orange and green lines overlap. \\ \textbf{Higher is better.}}
  \label{fig: real_data_lost_demand}
\end{subfigure}%
\hspace{0.015\textwidth}
\begin{subfigure}[t]{.32\textwidth}
  \centering
  \includegraphics[width=1.0\linewidth]{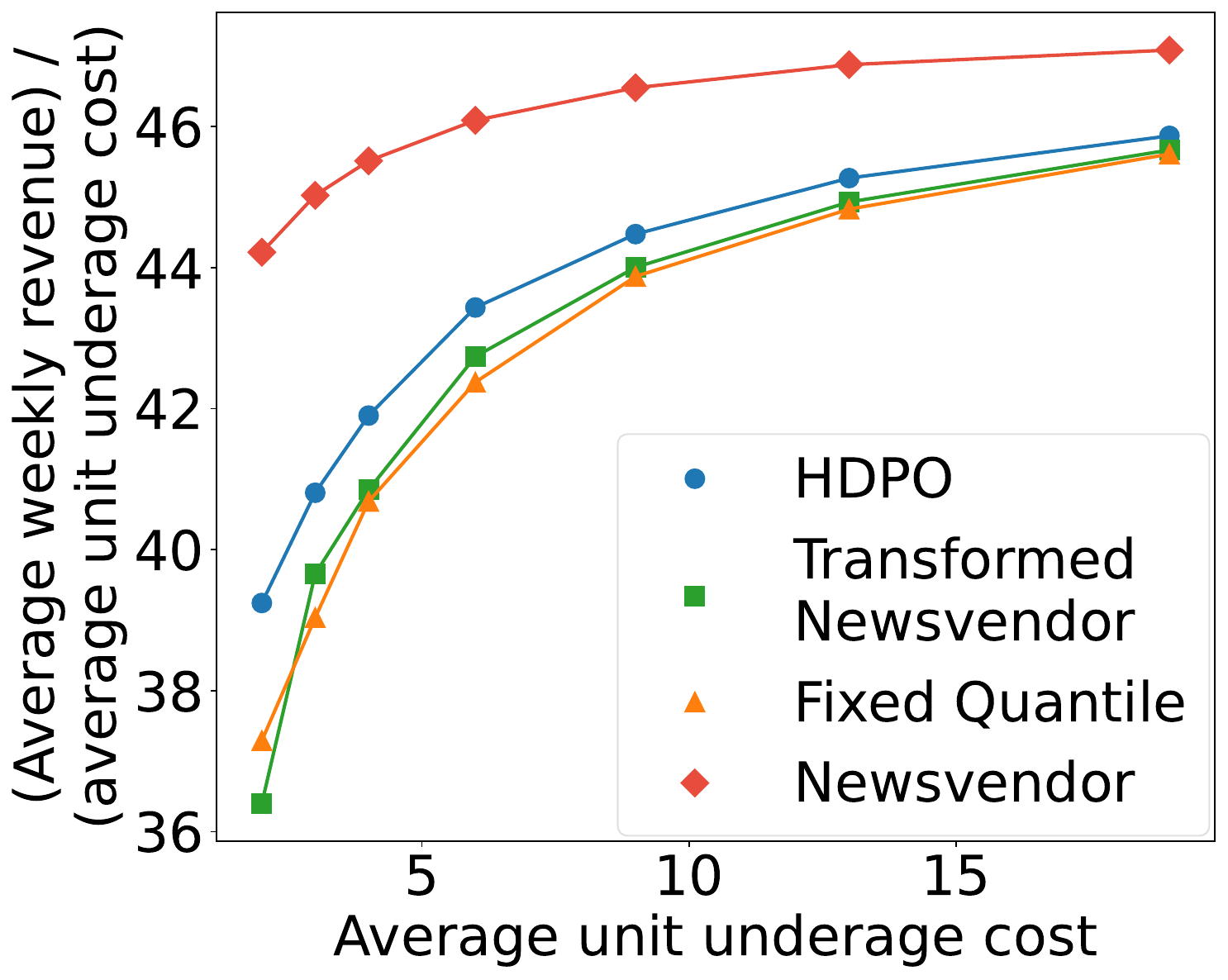}
  \caption{Average weekly revenue, normalized by the average unit underage cost. \\ \textbf{Higher is better.}}
  \label{fig: real-data-average-revenue}
\end{subfigure}
\hspace{0.015\textwidth}
\begin{subfigure}[t]{.32\textwidth}
  \centering
  \includegraphics[width=1.0\linewidth]{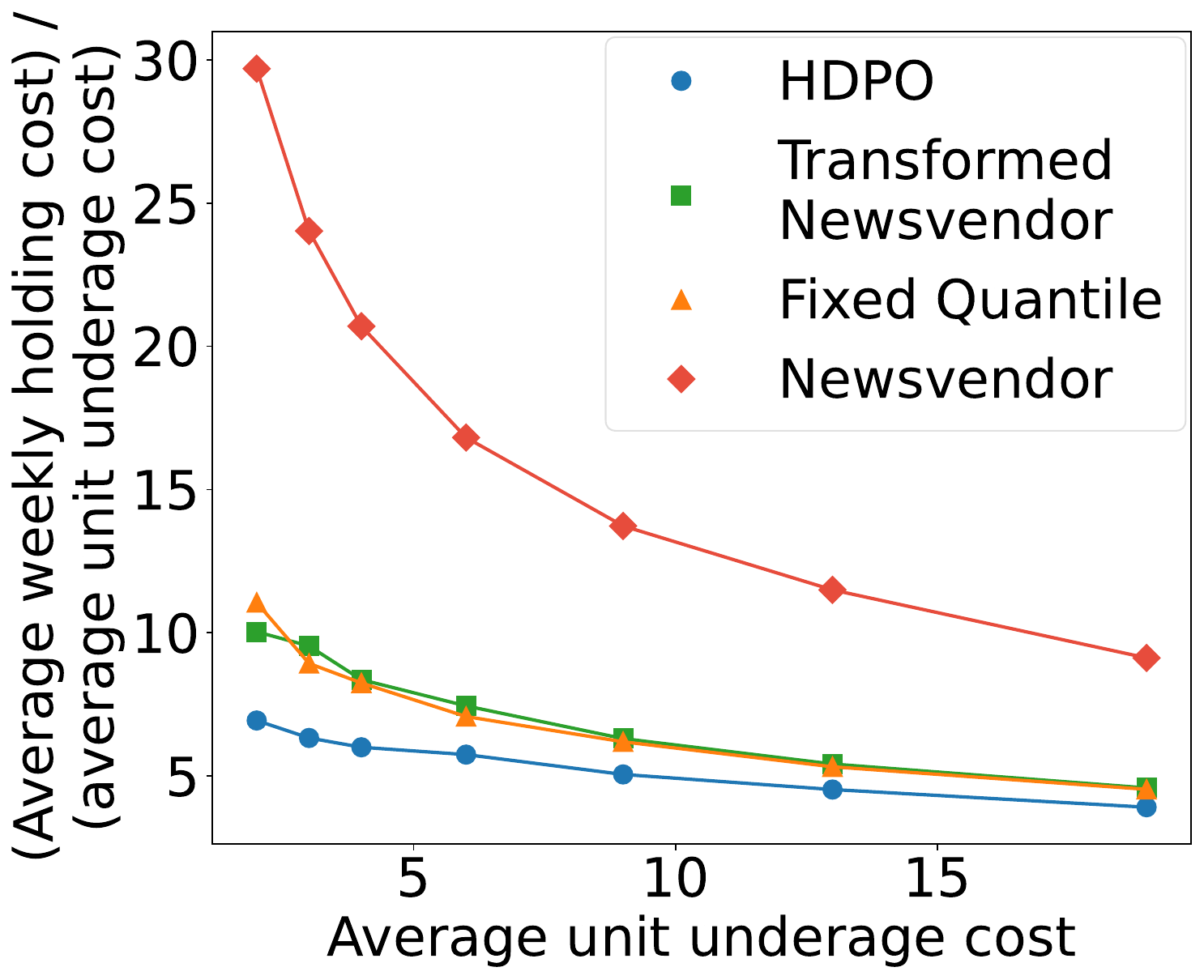}
  \caption{Average weekly holding cost, normalized by the average unit underage cost. \\ \textbf{Lower is better.}}
  \label{fig: real-data-average-holding-cost}
\end{subfigure}

\caption{Relative performance for setting with one store, considering $h=1$, realistic demand data and a lost demand assumption, for different average unit underage costs.
}
\label{fig: real data main results}
\end{figure}

Our findings in this non-stationary demand setting align with prior results reported in the literature for i.i.d.\ demand environments. Generalized newsvendor policies—like their base-stock counterparts—are designed to perform well under backlogged demand, where all incoming demand depletes available inventory. In contrast, in a lost demand setting, demand does not deplete inventory upon a stock-out, posing a challenge for policies optimized for backlogged settings. While adjusting the target quantile can partially address this issue, the underestimation of inventory may depend on the current inventory state in a non-trivial manner.
Our results emphasize that adjusting a quantile alone cannot systematically compensate for a policy that merely tracks the inventory position. We analyze this further in Appendix~\ref{sec: what-heuristics-miss}.

\subsection{Limitations: The need for structure-aware neural network architectures when addressing large inventory networks} \label{sec:vanilla-unrelated-stores}

The Vanilla NN architecture treats inventory network control problems as unstructured optimization tasks. It consists of a single fully-connected MLP that ingests the entire network state as a flattened vector and outputs decisions for all edges simultaneously. This architectural choice has a critical limitation: the network has no built-in awareness of the inventory network's topology or functional relationships between locations.

While theoretically capable of learning any policy, this structure-agnostic architecture cannot distinguish between sensible and spurious patterns without extensive data. It lacks mechanisms to favor learning that ``each store should order primarily based on its own inventory'' (likely sensible) over patterns like ``store 7 should order more when store 11 has high inventory'' (likely spurious). Without bias toward natural problem decompositions, meaningful and coincidental patterns appear equally expressible.

To crystallize this limitation, we design a sanity check that exposes the inefficiency of structure-agnostic learning.

{\bf Benchmark.} 
We introduce a setting we call S10 as a variation of S2, extending the single-store lost-demand setup to multiple identical, independent store copies, each facing Poisson demand, with no relationship between the stores. We vary the number of store copies while keeping all parameters fixed. 

This setting provides an ideal test case: as we add more identical stores, we are not actually increasing problem complexity. We are simply replicating the same single-store problem multiple times. An architecture that recognizes this structure should actually exhibit \emph{improving} performance with more stores, as each additional store provides more training data for what is fundamentally the same decision problem.

{\bf Experiment specifications.} See Appendix \ref{appendix:value-of-weight-sharing}

{\bf Results.}
Figure \ref{fig:weight_sharing-vanilla} reveals that the Vanilla NN fundamentally fails this sanity check.
Performance degrades substantially as the number of stores increases, with optimality gaps growing from under 1\% with 3 stores to over 8\% with 50 stores (using 128 training scenarios). This degradation directly stems from the architecture's structure-agnostic design: it must learn the independence of 50 stores from data, rather than encoding it architecturally.

This experiment crystallizes the need for structure-aware architectures and motivates our introduction of GNNs in the next section.

\begin{figure}[ht]
\begin{subfigure}{.32\textwidth}
  \centering
  \includegraphics[width=1\linewidth]{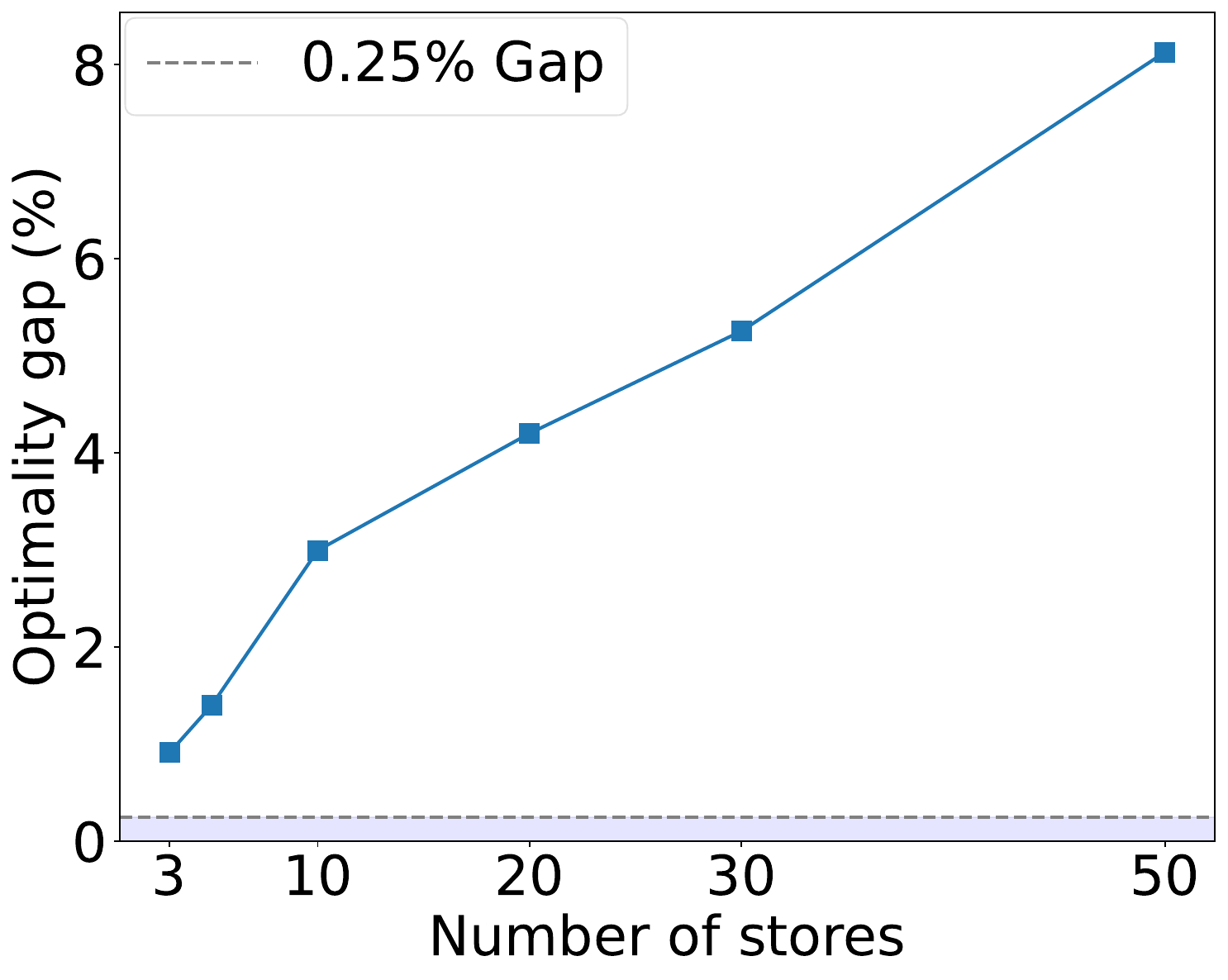}
  \caption{128 training scenarios.}
  \label{fig:weight-sharing-vanilla-small}
\end{subfigure}%
\hspace{0.015\textwidth}%
\begin{subfigure}{.32\textwidth}
  \centering
  \includegraphics[width=1\linewidth]{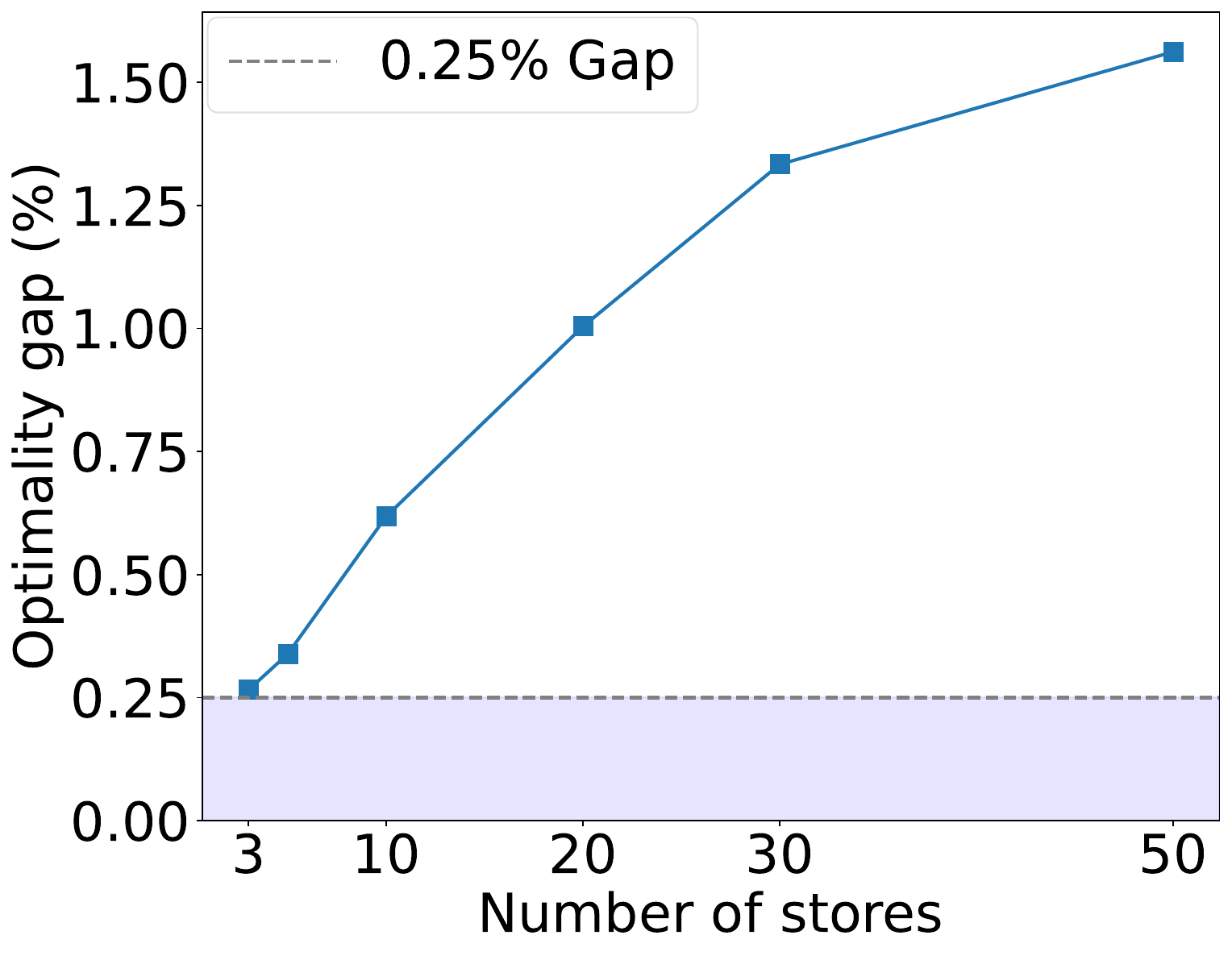}
  \caption{1024 training scenarios.}
  \label{fig:weight-sharing-vanilla-medium}
\end{subfigure}
\hspace{0.015\textwidth}%
\begin{subfigure}{.32\textwidth}
  \centering
  \includegraphics[width=1\linewidth]{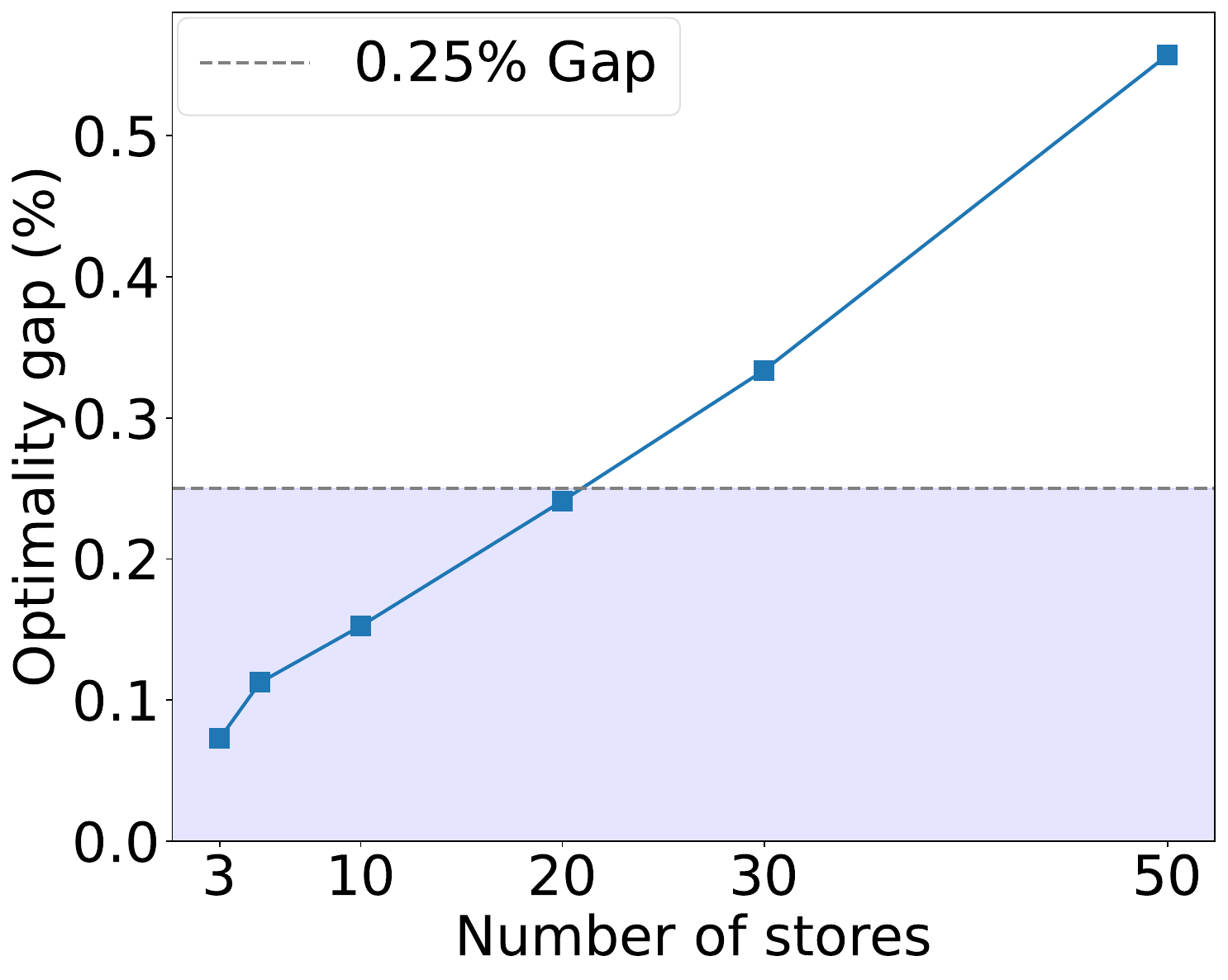}
  \caption{8192 training scenarios.}
  \label{fig:fig:weight-sharing-vanilla-large}
\end{subfigure}
\caption{ Optimality gaps (\%) for the Vanilla NN architecture in setting S10, which consists of many identical, independent stores. Results are shown across varying numbers of stores and sample sizes.}
\label{fig:weight_sharing-vanilla}
\end{figure}

\section{HDPO with GNN policies for large-scale inventory networks} 
\label{sec:gnn-section}

Section \ref{sec:vanilla-unrelated-stores} revealed that Vanilla NNs can struggle with sample efficiency because they ignore the inventory network's structure. We now introduce a GNN architecture in which the policy’s computation graph---the directed graph of computational operations---mirrors the physical inventory network along which goods flow. 

We introduce the desiderata motivating the architecture in Section \ref{sec: desiderata} and present our approach in Section \ref{sec: gnn-architecture}. Section \ref{sec:sample-efficiency-results} demonstrates the empirical impact: across six diverse inventory settings, GNNs are much more sample efficient than Vanilla NNs in large networks while matching Vanilla NN performance when data is abundant.
To understand these gains, Section \ref{sec: GNN_performance_drivers} investigates three complementary capabilities through focused experiments. We show how GNNs naturally decompose independent decisions and can enhance sample efficiency via weight sharing (Sec.\ref{sec:value-of-weight-sharing}), coordinate through message-passing when beneficial (Sec.\ref{sec:value-of-message-passing}), and dynamically exploit inventory network flexibility (Sec.~\ref{sec:value-of-flexibility}). Finally, Section \ref{sec: GNN_constructive} provides constructive demonstrations showing how GNNs can elegantly represent optimal policies in two stylized inventory settings.

\subsection{From desiderata to Graph Neural Networks \label{sec: desiderata} }
To motivate our proposed design, we enumerate the following desiderata for the policy architecture:
\begin{itemize}
    \item {\bf Desideratum 1 (decomposition across uncoupled locations):} If the inventory graph is disconnected, the allocation on each edge should depend only on state information available within its own connected component.
    \item {\bf Desideratum 2 (equivariance across repeated sub-structures):} If a node re-labeling maps one part of the inventory graph onto another, then applying that same re-labeling to the system state should produce correspondingly re-labeled policy actions—ensuring that learning on one sub-graph transfers to any isomorphic part of the network.
    \item {\bf Desideratum 3 (coordination across coupled locations)}: When locations are coupled through flows or constraints, the allocations involving them need to be coordinated so that decisions are jointly optimal.
\end{itemize}

\paragraph{\bf Why familiar architectures fall short.}
Before introducing our proposed design, it is helpful to see why two seemingly reasonable alternatives cannot satisfy all three desiderata.

\emph{(a) Vanilla NN — a fully centralized MLP.} As studied in Section~\ref{sec:vanilla-unrelated-stores}, the Vanilla NN feeds the \emph{entire} system state into one multilayer perceptron that outputs all replenishment decisions. Although expressive, it violates Desideratum 1: an action at an edge can be influenced by features from completely disconnected components. It also violates Desideratum 2: weights are tied to absolute node indices, so what the network learns at one location does not automatically transfer to isomorphic parts of the graph. These two shortcomings manifest as the pronounced sample-inefficiency documented in Section~\ref{sec:vanilla-unrelated-stores}.

\emph{(b) Fully–decentralized local rule.} At the opposite extreme, we could assign each edge a small neural module that looks only at its local state. This satisfies Desideratum 1 by construction. If all neural modules share common weights, then Desideratum 2 is also satisfied. But this architecture fails Desideratum 3: when flows or capacity constraints couple locations, the policy lacks a mechanism to incorporate upstream or downstream information, instead making decisions at each edge solely based on local state information. In RL terminology, this amounts to \emph{centralized training, decentralized execution} \citep{gronauer2022multi}: the policy is trained in a simulator to optimize a global reward function, but each edge is constrained to act on its own observations only, preventing the coordination required for joint optimality.

This reasoning leads us naturally to a GNN policy architecture. Our GNN involves NN modules with shared weights that perform computations for each node and edge in the inventory network, and use permutation invariant aggregation of messages. It naturally satisfies Desiderata 1 and 2. In fact, when locations are completely decoupled, the architecture reduces to the fully-decentralized local rule discussed above.

The key innovation of GNNs, however, is that through an iterative message-passing procedure, they also enable coordinated decision-making among connected nodes in the inventory graph, thereby satisfying Desideratum 3. Informally, GNNs enable coordinated decision-making while maintaining a ``relational inductive bias''---that is, a bias toward prioritizing information from neighboring locations in the network structure \citep[Chapter 13.4]{prince2023understanding}. This approach has proven valuable in other domains involving networked entities, such as molecular property prediction in drug discovery \citep{gilmer2017neural}, social influence modeling in social networks \citep{hamilton2017inductive}, and real-time ETA prediction in large-scale traffic systems \citep{derrow2021eta}.

\subsection{Our proposed GNN policy architecture \label{sec: gnn-architecture} }
We propose a concrete GNN architecture for sample efficient inventory control in networks. Our architecture is specified in Algorithm~\ref{alg:samp_gnn}, which we encourage the reader to review alongside the discussion below. We now discuss its features, including some key choices that are specific to our construction (directed edges, edge-level tasks, and aggregating information from neighbors by summing over their embeddings).
Our GNN starts (as usual) by embedding nodes and edges 
into a latent space, refines these embeddings through an iterative message-passing process, and finally applies a function to extract meaningful outputs for the task at hand. We adopt an architecture capturing node- and edge-level tasks in a \emph{directed} graph with vertices $V$ and directed edges $E$, with the direction of each edge capturing the direction of inventory flow (from supplier to receiver). (We note that typical GNNs live on \emph{undirected} graphs, with just a few exceptions \cite{emanu2023dirgnn}) 

\begin{algorithm}
\caption{GNN architecture for inventory networks}\label{alg:samp_gnn}
\begin{algorithmic}[1]
\State \textbf{Input:} Directed graph  $(\mathcal{N}, \mathcal{E})$ of the inventory network,
node features
$\{S^k_t\}_{k \in \mathcal{N}}$ including inventory level $I^j_t$, edge features $\{L^{(j,k)}\}_{(j,k) \in \mathcal{E}}$, number of message-passing layers $L^{\text{MP}}$, trainable functions \texttt{EmbedNode}, \texttt{EmbedEdge}, \texttt{UpdateNode}, \texttt{UpdateEdge} and \texttt{Readout}, and non-trainable function \texttt{FeasibilityEnforce}.
\State \textbf{Output:} Feasible allocation decisions $\{a^{(j,k)}\}_{(j,k) \in \mathcal{E}}$

\State Initialize node embeddings $h^k \gets$ \texttt{EmbedNode}$(S^k_t)$ for all $k \in \mathcal{N}$
\State Initialize edge embeddings $h^{(j,k)} \gets$ \texttt{EmbedEdge}$(h^j, h^k, L^{(j,k)})$ for all $(j, k) \in \mathcal{E}$
\For{$\ell = 1$ to $L^{\text{MP}}$} \Comment{Message-passing layers}
    \For{each node $k \in \mathcal{N}_{+}$}
        \State Aggregate messages from incoming edges: $z^{k}_{\text{in}} \gets \sum_{j \in \mathcal{N}_{\text{sup}}^k}  h^{(j,k)} / \sqrt{|\mathcal{N}_{\text{sup}}^k|}$
        \State Aggregate messages from outgoing edges: 
        $z^{k}_{\text{out}} \gets \sum_{j \in \mathcal{N}_{\text{rec}}^k}  h^{(k,j)} / \sqrt{|\mathcal{N}_{\text{rec}}^k|}$
        \State Update node embedding: $h^k \gets$ $h^k$ + \texttt{UpdateNode}$(h^k, z^{k}_{\text{in}}, z^{k}_{\text{out}})$
    \EndFor
    \For{each edge $(j, k) \in \mathcal{E}$}
        \State Update edge embedding: $h^{(j,k)} \gets$ $h^{(j,k)} + $ \texttt{UpdateEdge}$(h^{(j,k)}, h^j, h^k)$
    \EndFor
\EndFor
\For{each edge $(j, k) \in \mathcal{E}$}
    \State Compute intermediate output: $b^{(j,k)} \gets$ \texttt{Readout}$(h^{(j,k)})$
\EndFor
\For{each distribution center node $j \in \mathcal{N}_{\text{dc}}$}
    \State Enforce feasibility: $a^{(j,k)} \gets$ \texttt{FeasibilityEnforce}$(I^j_t, \{b^{(j,k)}\}_{k \in \mathcal{N}_{\text{rec}}^j})$ for all $k \in \mathcal{N}_{\text{rec}}^j$
\EndFor
\State \textbf{return} $\{a^{(j,k)}\}_{(j,k) \in \mathcal{E}}$
\end{algorithmic}
\end{algorithm}

Node features capture aspects such as inventory on hand, inventory pipelines, holding costs, underage costs and forecasting information. Edge features include, e.g., lead times. In each decision period, nodes and edges are initialized with trainable embeddings of the node features and edge features (lines 3 and 4). 
Embeddings are iteratively refined through multiple rounds of message passing (indexed by $\ell \in [L^{\text{MP}}]$), consisting of:
    \begin{enumerate}
        \item {\textbf{Edge Aggregation (lines 7 and 8):}}
        Each node separately sums over embeddings (``messages'') from its incoming (outgoing) edges, with results normalized to prevent gradient explosion.
        As a simple special case, this allows to capture summing over inventory shipments received (sent), though of course the edge embeddings are general and flexible.
       
        \item \textbf{Node updates (line 9):} The aggregated messages are used to update the node's representation. We choose to increment the current embedding of the node with the output of the trainable function \texttt{UpdateNode} (termed a \emph{skip connection} in machine learning). 
        \item \textbf{Edge Update (line 12):} Each edge updates its embedding based on its previous representation and the updated representation of the connected nodes, again via an increment over the current embedding.
\end{enumerate}
After message passing, a function extracts a tentative order/shipment quantity on each edge
(lines 16). These quantities are passed through non-trainable feasibility enforcement layers (FELs) to ensure that the inventory quantity at each DC node is respected (lines 19).   
 
A schematic of the message-passing process is shown in Figure~\ref{fig:gnn-architecture}, with further implementation details provided in Appendix~\ref{appendix:gnn-architecture}. One such detail, worth highlighting here, concerns how we encode the inventory graph in our implementation. Specifically, we include a fictitious self-loop edge at each DC, allowing the DC to output a tentative order that preserves its own inventory. In practice, this design choice improved training stability by preventing pathological cases where allocating more to a DC merely induced greater downstream flow without enabling the DC to retain stock.

\begin{figure}[htbp!]
  \centering
  \includegraphics[width=1.0\linewidth]{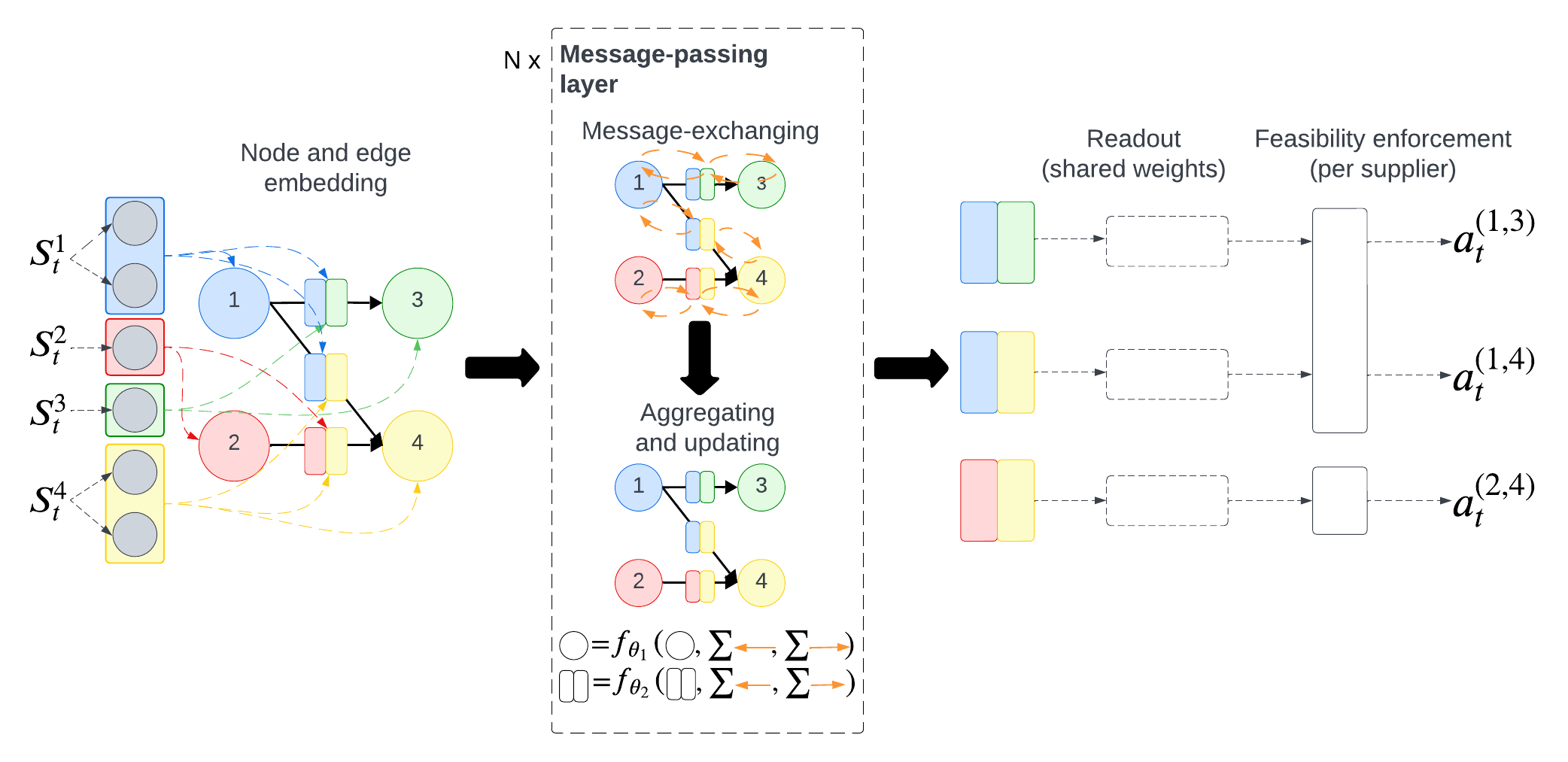}
  \caption{Overview of the GNN architecture in a setting with four locations. Location 1 supplies locations 3 and 4, while location 2 supplies location 4. For simplicity, we omit edges from the external supply node.}
  \label{fig:gnn-architecture}
\end{figure}

This completes our description of the GNN architecture and its motivation. It remains, of course, to see whether the architecture actually works as intended, and delivers the desired good performance and sample efficiency in larger inventory networks. We now pursue this investigation.

\subsection{Main Empirical Results: GNN-based policies are much more sample efficient than Vanilla
\label{sec:sample-efficiency-results}}

We hypothesize that GNN policies strike an effective balance between structure and flexibility: their relational inductive bias enhances sample efficiency when training data is limited, while retaining sufficient representational capacity to coordinate decisions optimally when samples are abundant. Our experiments test this hypothesis across a broad spectrum of structural features and data-availability regimes, motivated by real-world settings where retailers often have access to a limited number of scenarios (e.g., one per product offered) for training.

Our findings support this hypothesis. GNNs consistently outperform vanilla NNs when data is limited relative to the network scale, while matching vanilla performance when data is abundant. This outcome was not obvious a priori—GNNs could have (a) proven prone to overfitting or (b) potentially lost important information from distant parts of the network compared to the Vanilla NN's global view, risking suboptimal performance.
Our decision to treat vanilla NN architectures as the baseline stems from the fact that many settings we study are high-dimensional and not easily amenable to common heuristics in operations. Reassuringly, in two settings where the optimum can be computed or bounded, GNNs converge to optimal behavior, suggesting that the structured architecture does not fundamentally limit representational capacity.

The remainder of this section documents these findings across six inventory control settings. Section~\ref{sec: GNN_performance_drivers} then investigates the mechanisms driving these sample efficiency gains through focused  studies.

{\bf Benchmarks.} We evaluate performance across six distinct settings, summarized in Table~\ref{table:experiment-settings}. These settings vary in network topology, unmet demand assumptions, and demand distributions. For each, we vary the number of locations and consider a single meta-instance. In the synthetic demand scenarios (S3, S4, S6, and S8), we use the same problem instance across all scenarios while allowing parameters to vary by location. This fixed-instance approach follows established literature and aids in variance reduction. In contrast, settings with realistic demand (S7 and S9) use different instances across scenarios. These realistic-demand settings are based on the \textit{Favorita} dataset (see Appendix~\ref{appendix:realistic-demand-dataset}), from which we construct product-level demand traces spanning multiple stores.
In most settings, instance parameters are sampled independently from predefined distributions, and edge costs are set to zero. Exceptions are the settings involving MWMS networks (S8 and S9), where we introduce spatial structure: in S8, store locations are uniformly sampled in the plane, while in S9, they are set according to the geographic coordinates of the actual stores associated with each demand trace. In both settings, warehouses are positioned using K-means clustering to simulate strategic placement relative to store locations. Lead times and edge costs are then determined by Euclidean distances and warehouse-specific parameters, resulting in geographically coherent and heterogeneous inventory networks.

\begin{table}[h!]
\centering
\caption{Overview of settings and location counts used in the sample efficiency experiments.\newline
\footnotesize{In S4, the warehouse functions as a transshipment center$^\ast$ (\ie it cannot hold inventory).}}
\label{table:experiment-settings}
\begin{tabular}{@{}llllll@{}}
\toprule
\textbf{ID} & 
\shortstack{\textbf{Network} \\ \textbf{Structure}} & 
\shortstack{\textbf{Unmet Demand} \\ \textbf{Assumption}} & 
\shortstack{\textbf{Demand} \\ \textbf{Distribution}} & 
\shortstack{\textbf{Number of} \\ \textbf{Locations}} &
\shortstack{\textbf{Varied} \\ \textbf{Component}} \\ \midrule
S3 & Serial & Backlogged & Normal & 4--7 & Echelons \\
S4 & OWMS$^\ast$ & Backlogged & Normal & 4--51 & Stores \\
S6 & OWMS & Lost & Normal & 4--51 & Stores \\
S7 & OWMS & Lost & Realistic & 4--22 & Stores \\
S8 & MWMS & Lost & Normal & 12--56 & Stores and warehouses \\
S9 & MWMS & Lost & Realistic & 23--26 & Warehouses \\
\bottomrule
\end{tabular}
\end{table}

{\bf Experiment specifications.} For settings with synthetic demand (S3, S4, S6, and S8), we assess sample efficiency by varying the number of training scenarios across $\{128, 1024, 8192\}$, selecting models using early stopping on a dev set of the same size. For the two settings using realistic demand (S7 and S9), the training sets contain 64 and 288 product-level demand traces, respectively, with dev and test sets of equal size. Datasets are split temporally, following the specifications in Section~\ref{sec:vanilla-hdpo-realistic}.
Since both learning rate and model size significantly affect generalization \citep{jiang2019fantastic, allen2019convergence}, we consider three learning rates for each architecture. For the Vanilla NN, we also vary layer width across three values. Each hyperparameter configuration is run three times, and final results reflect the best-performing run for each architecture class, number of locations, and sample size.
Full architectural and experimental details, as well as complete results, are provided in Appendix~\ref{appendix:sample-efficiency-results}.

\subsubsection{Strong Sample Efficiency Gains with Limited Data.}
Figures~\ref{fig:sample-efficiency-small-sample} and~\ref{fig:sample-efficiency-large-sample} show that GNNs consistently outperform vanilla NNs when training data is limited, in settings S4, S6, and S8, which involve networks with repeated structural elements. With only 128 samples, the GNN's sample efficiency advantages are substantial across all network sizes—starting even with 3 stores—and increase dramatically as the number of locations grows. At 50 stores, the Vanilla NN suffers relative excess costs of approximately 31\%, 14\%, and 8\% in these three settings, respectively. Strikingly, in networks with 50 stores, GNN with only 128 samples performs comparably to (and even outperforms) Vanilla with 8,192 samples, implying a reduction in data requirements to achieve a given level of performance of up to 64 times (see Figures~\ref{fig:sample-efficiency-1dc-trans}, \ref{fig:sample-efficiency-1dc-lost} and~\ref{fig:sample-efficiency-manydcs} in Appendix~\ref{appendix:sample-efficiency-results}).  

\begin{figure}
\label{fig:sample-efficiency-synthetic}
\begin{subfigure}{.49\textwidth}
  \centering
  \includegraphics[width=1\linewidth]{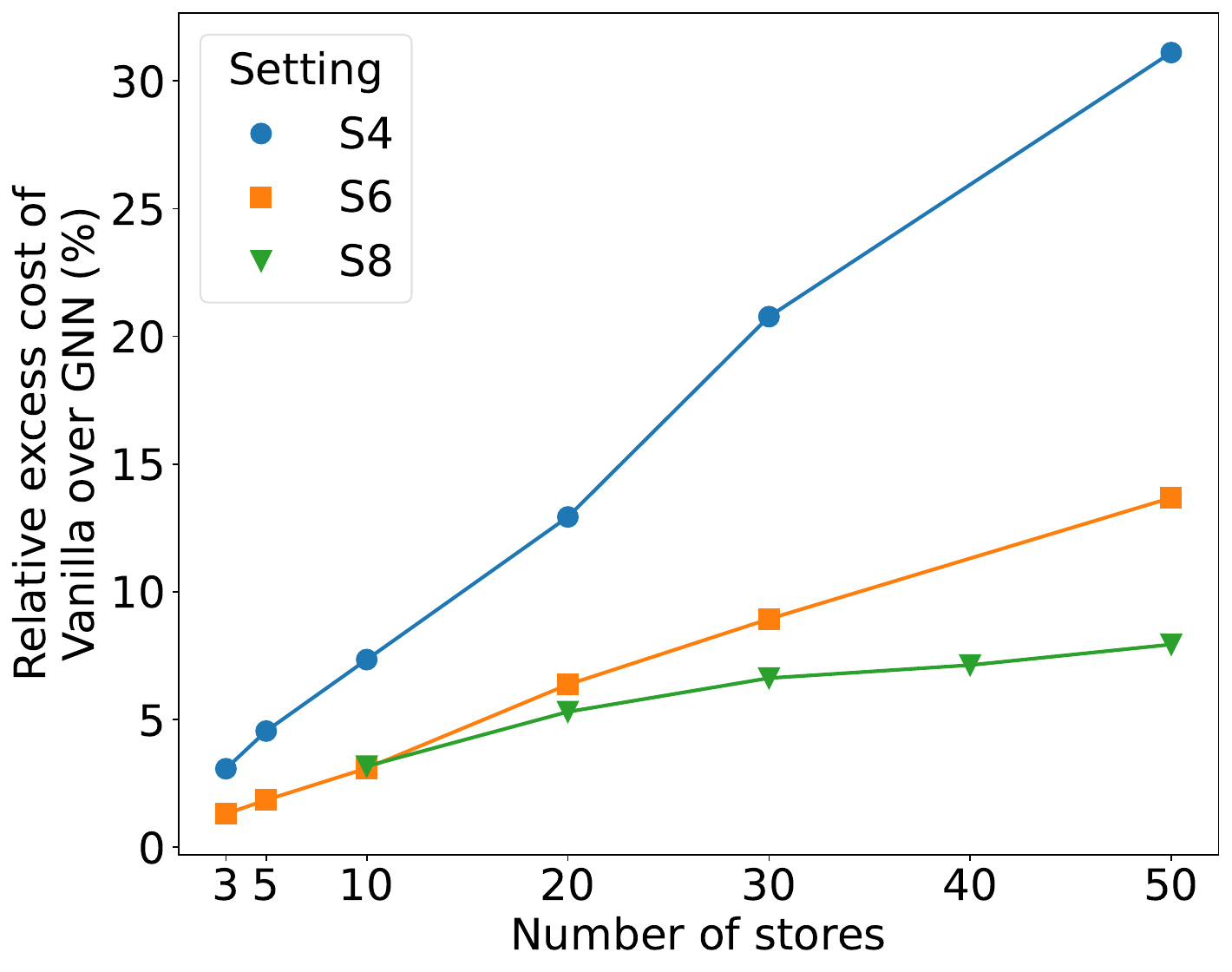}
  \caption{128 training scenarios.}
  \label{fig:sample-efficiency-small-sample}
\end{subfigure}%
\hspace{0.02\textwidth}
\begin{subfigure}{.49\textwidth}
  \centering
  \includegraphics[width=1\linewidth]{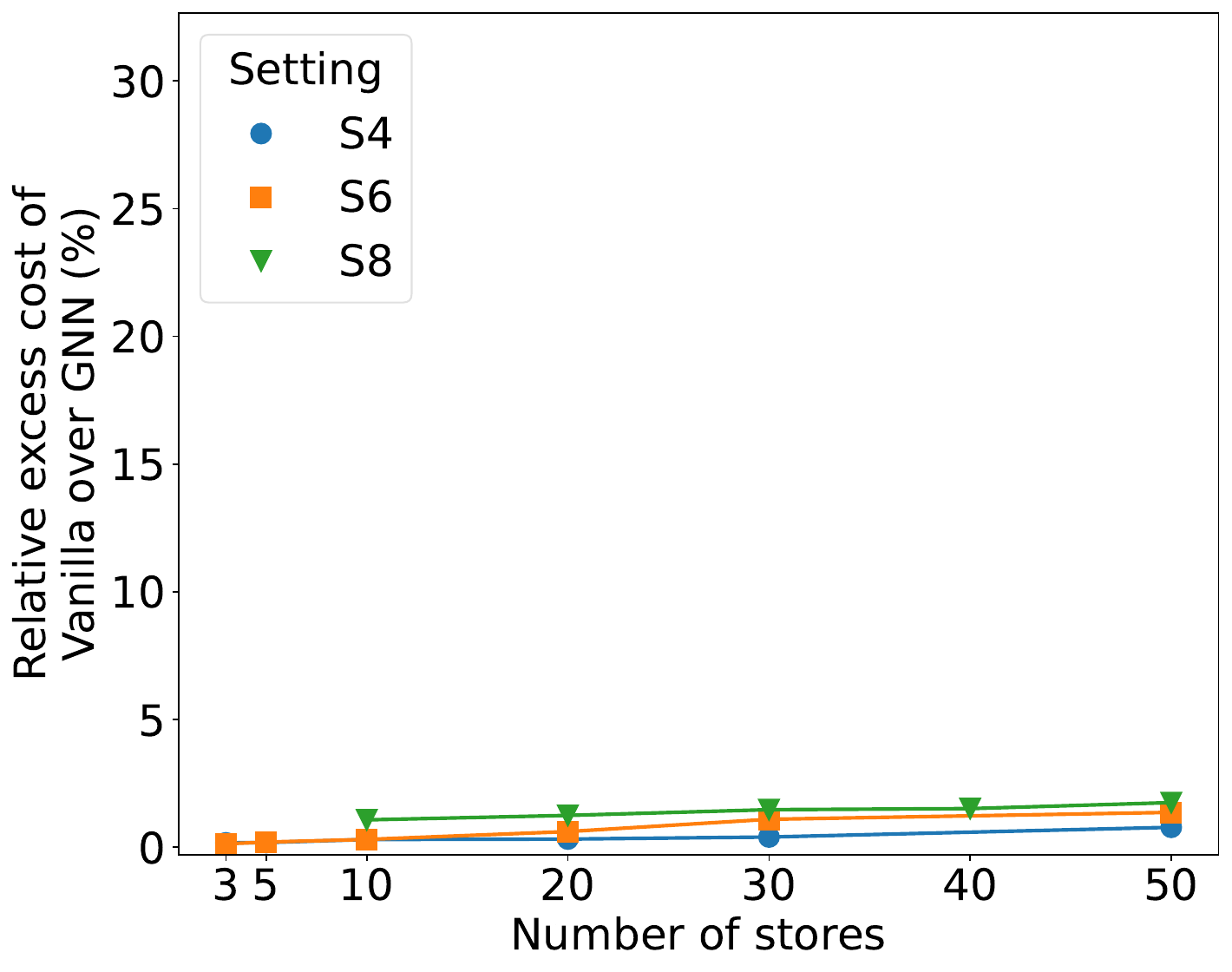}
  \caption{8192 training scenarios.}
  \label{fig:sample-efficiency-large-sample}
\end{subfigure}

\caption{Relative excess cost (\%) of the loss incurred by the Vanilla NN compared to the GNN, for 128 (left) and 8192 (right) training samples, in settings S4, S6, and S8. For S8,there are 2 to 6 warehouses for 10 to 50 stores, with 10 additional stores for each additional warehouse.}

\end{figure}

Remarkably, this sample efficiency pattern holds even in challenging network topologies. Figure \ref{fig:relative-excess-cost-serial} shows results for setting S3, which proved to be the most challenging among all six settings. Serial systems pose a particular structural difficulty for GNNs since information must propagate through many message-passing rounds to flow from one end of the chain to the other. Despite this structural challenge, the GNN still achieves significant performance gains over Vanilla with limited data (128 and 1024 training scenarios), and these advantages tend to widen for longer serial chains.

\begin{figure}[htbp]
\centering
\includegraphics[width=0.49\linewidth]{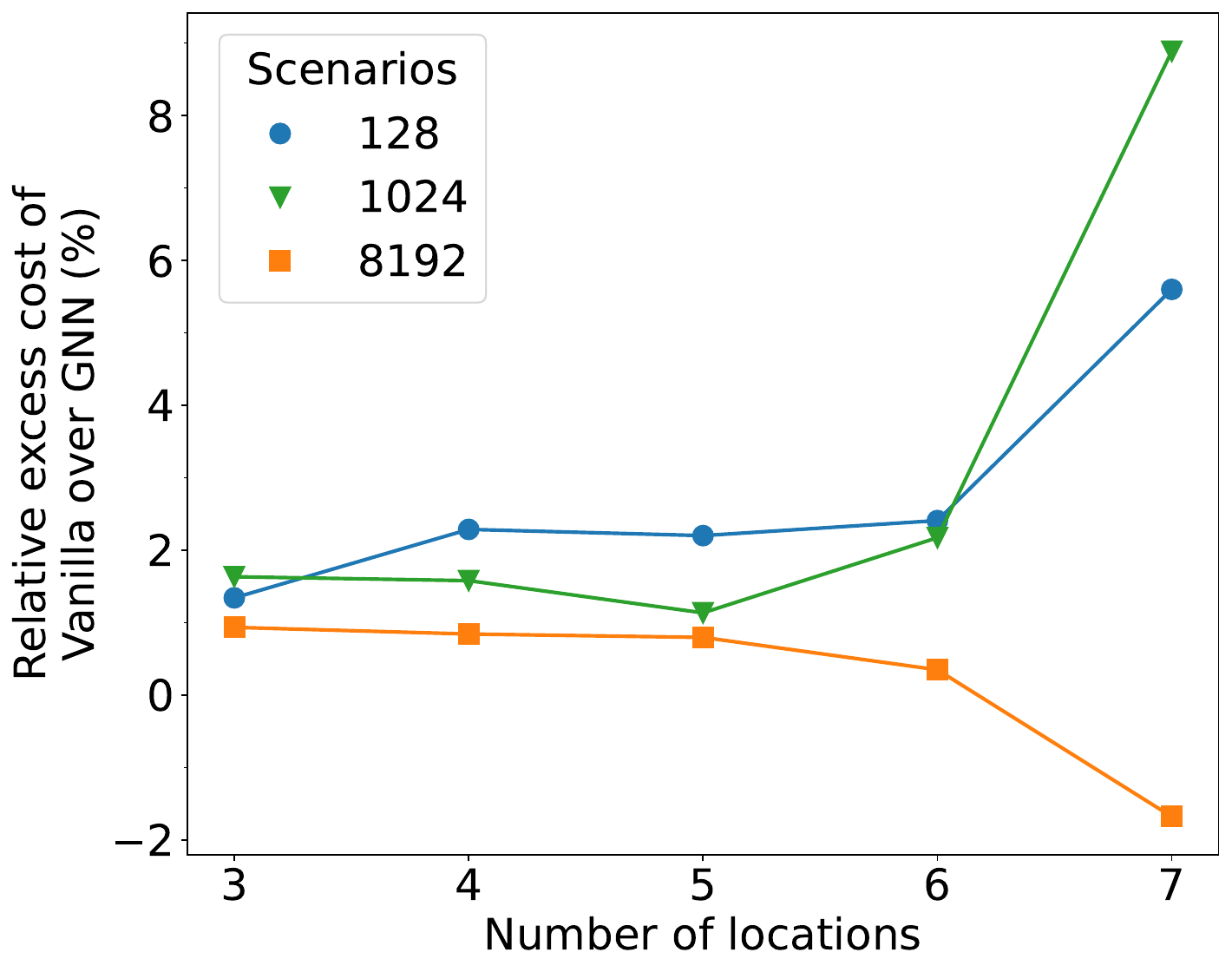}
\caption{Relative excess cost (\%) of the loss incurred by the Vanilla NN compared to the GNN in setting S3 for a varying number of training scenarios.}
\label{fig:relative-excess-cost-serial}
\end{figure}

\subsubsection{Convergence with Abundant Data Hints At  Sufficient  Representational Capacity.}
The second component of our hypothesis is that GNNs retain sufficient representational capacity when training data is abundant. To assess this, we examine optimality gaps in settings S3 and S4  where the true optimum can be obtained or bounded  (see Appendix \ref{appendix:sample-efficiency-results}). In the challenging serial network setting S3, GNNs achieve within 1\% of optimal performance with 8192 training scenarios for systems up to 6 locations. In setting S4, GNNs consistently achieve optimality gaps below 0.2\% with 8192 samples, even for 50-store systems with state representations up to 555 dimensions.

Figure~\ref{fig:sample-efficiency-large-sample} provides additional support for this hypothesis. With 8192 training samples, performance gaps shrink dramatically across all settings: the Vanilla NN's excess costs drop to roughly 1.0\%, 1.5\%, and 2.0\% for 50-store configurations in settings S4, S6, and S8, respectively.
Similarly, in the challenging serial networks (Figure~\ref{fig:relative-excess-cost-serial}), the substantial performance gaps observed with limited data largely disappear when abundant training samples are available. This convergence pattern holds consistently across network topologies and problem scales.

These results address a natural concern about structured architectures: that their inductive biases might fundamentally limit effective representational power. Our findings suggest otherwise: GNNs appear capable of learning the same complex coordination patterns as Vanilla NNs, but require fewer samples to do so effectively. In these settings, the structured architecture provides a more efficient path to the same representational capacity, rather than a ceiling on ultimate performance.

\subsubsection{Sample efficiency benefits persist in realistic settings.}
Here, we demonstrate that the patterns observed with synthetic demand extend to more realistic conditions. Figure~\ref{fig:sample-efficiency-real} shows results for settings S7 and S9, which use realistic demand traces from the Favorita dataset. Setting S9 additionally incorporates shipping costs and lead times that reflect the actual geographic placement of stores and warehouses.

For 64 training scenarios, the GNN achieves profits up to 
3.5\% and 7\% higher than the Vanilla NN in settings S7 and S9, respectively. While these gains are more modest than in synthetic-demand settings, the fundamental pattern persists: GNNs provide sample efficiency advantages 
that tend to widen as network scale increases (though the widening is less pronounced in S9, where the number of stores is fixed). 

These results provide important validation that the sample efficiency benefits we observe are not artifacts of simplified synthetic environments. Even when faced with realistic demand patterns, geographic constraints, and the complexities of real-world data, the structured inductive bias of GNNs continues to provide meaningful advantages when training data is limited relative to network scale.

\begin{figure}
\begin{subfigure}{.49\textwidth}
  \centering
  \includegraphics[width=1\linewidth]{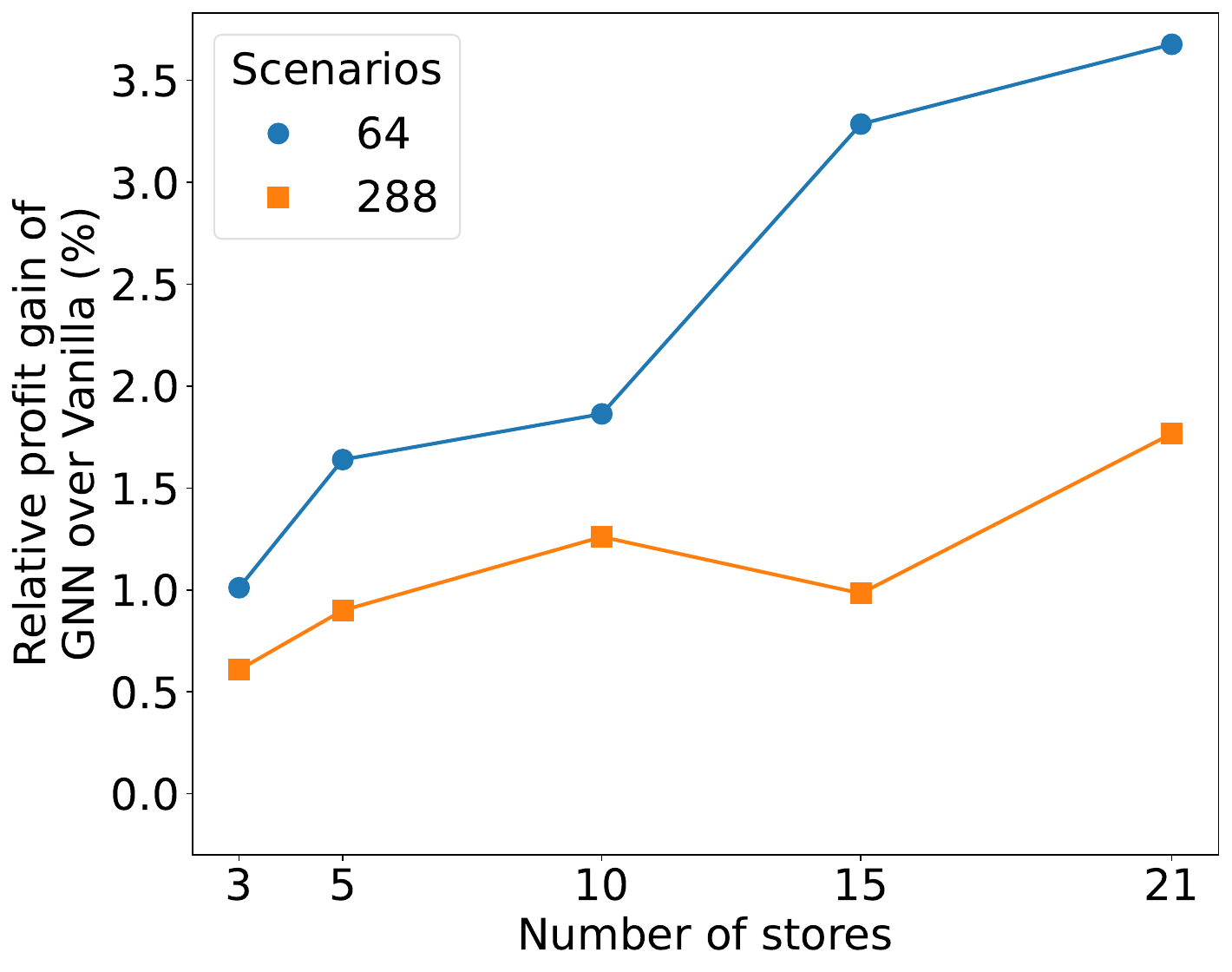}
  \caption{Setting S7.}
  \label{fig:sample-efficiency-owms-real}
\end{subfigure}%
\hspace{0.02\textwidth}
\begin{subfigure}{.49\textwidth}
  \centering
  \includegraphics[width=1\linewidth]{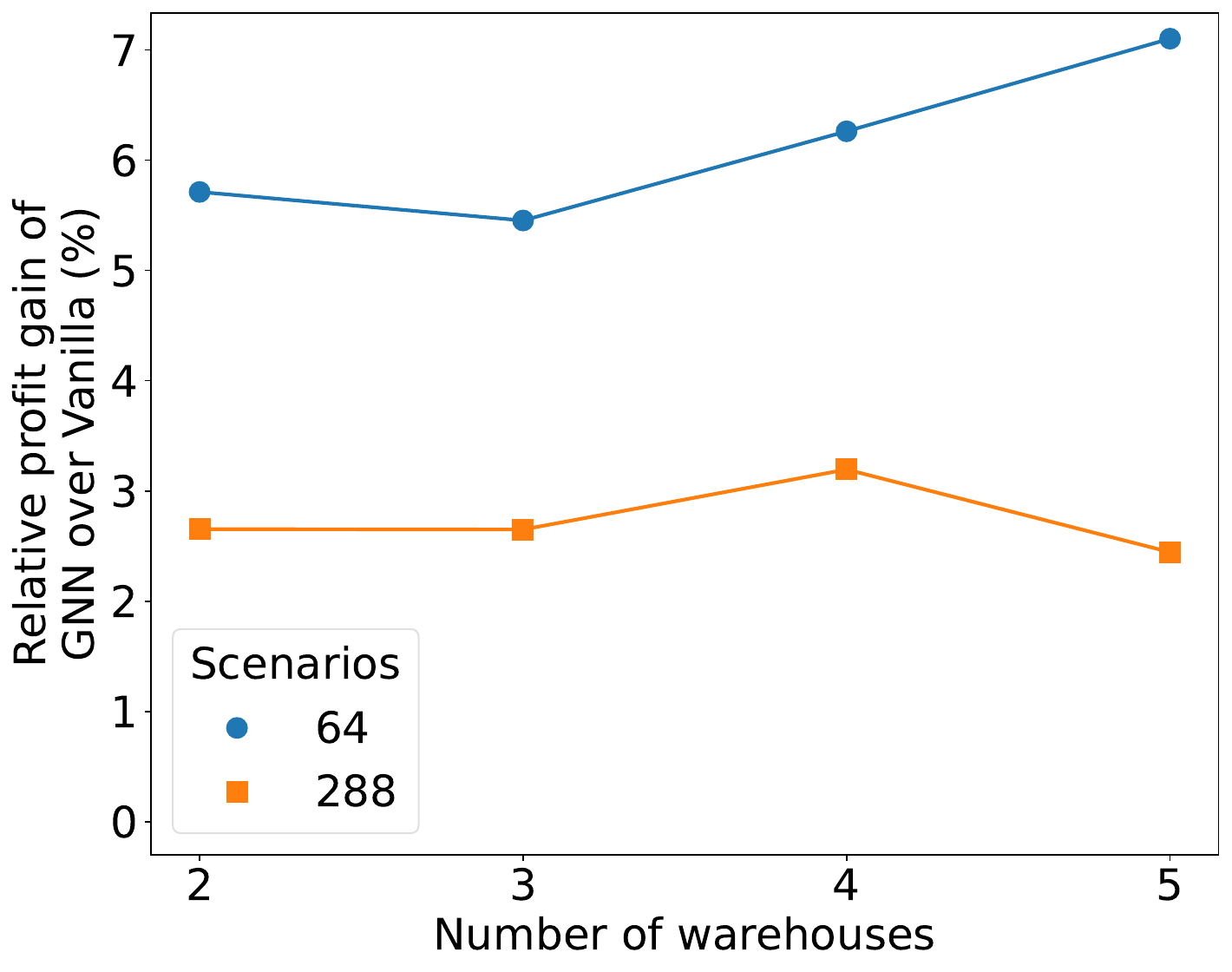}
  \caption{Setting S9.}
  \label{fig:sample-efficiency-mwms-real}
\end{subfigure}

\caption{Relative profit gain (\%) of the GNN compared to the Vanilla NN in Settings S7 (left) and S9 (right), for 64 and 288 training scenarios.
In S9, the number of stores is fixed at 21 across all warehouse counts.}
\label{fig:sample-efficiency-real}
\end{figure}

\subsection{Understanding the drivers of GNNs performance: a balance between decomposition and coordination}\label{sec: GNN_performance_drivers}

The strong empirical results in Section~\ref{sec:sample-efficiency-results} show that GNN policies require far fewer training scenarios than Vanilla NNs while matching their performance once data is plentiful. In this section, we investigate the mechanisms driving GNN performance through three focused experiments. We find that GNNs excel because they automatically strike a balance between \textit{decomposition} and \textit{coordination}: they decompose the global decision problem into local computations at nodes and edges, while coordinating these local decisions through message passing when the physical network creates dependencies.
To illustrate this, we provide numerical evidence showing that GNNs (i) decompose decisions when locations act independently (Section~\ref{sec:value-of-weight-sharing}) and (ii) coordinate via message passing when interactions matter (Section~\ref{sec:value-of-message-passing}). Moreover, we show that GNNs can adapt flexibly to changing network conditions (Section~\ref{sec:value-of-flexibility}). This ability to adaptively balance decomposition and coordination underlies the consistently strong performance of GNNs across our experiments.

\subsubsection{How GNNs decompose a problem: the unrelated stores problem.} \label{sec:value-of-weight-sharing}

Here we verify that  our GNN architecture addresses the fundamental limitation of vanilla NNs identified in Section \ref{sec:vanilla-unrelated-stores}. The unrelated stores setting (S10) provides an ideal sanity check: if stores are truly independent, the GNN should recognize this structure and avoid the degradation exhibited by the Vanilla NN.

{\bf Baselines.}
To isolate the benefits of recognizing the lack of a relationship between locations (under unrelated stores, there is \emph{no} relationship between any pair of locations), we introduce a variant called the \textit{Separate-Weights GNN (SW-GNN)}, in which there is no weight sharing across locations.

{\bf Experiment specifications.} See Appendix \ref{appendix:value-of-weight-sharing}

{\bf Results.}
Figure~\ref{fig:weight_sharing} presents the results. While the performance of the Vanilla NN deteriorates  as the number of stores increases, the performance of SW-GNN remains flat. This highlights the value of architectures that explicitly encode entity-level structure—e.g., recognizing that parts of the state space correspond to distinct locations—which supports better generalization. Moreover, the GNN (with shared weights) not only avoids degradation but improves as the number of stores increases, clearly outperforming the Separate-Weights baseline and highlighting the sample-efficiency gains enabled by weight sharing.

\begin{figure}[ht]
\begin{subfigure}{.32\textwidth}
  \centering
  \includegraphics[width=1\linewidth]{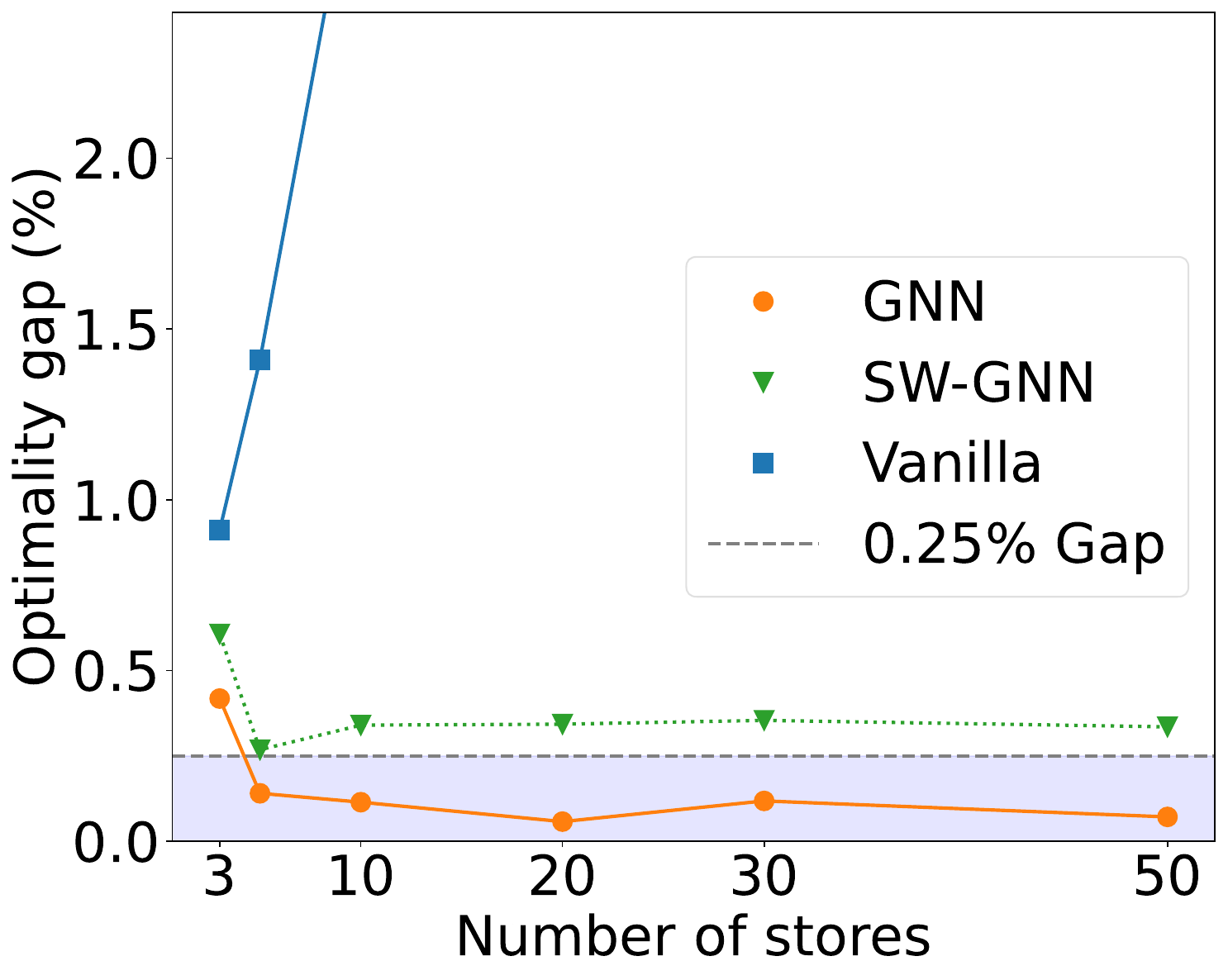}
  \caption{128 training scenarios.}
  \label{fig:sample-efficiency-weight-sharing-small}
\end{subfigure}%
\hspace{0.015\textwidth}%
\begin{subfigure}{.32\textwidth}
  \centering
  \includegraphics[width=1\linewidth]{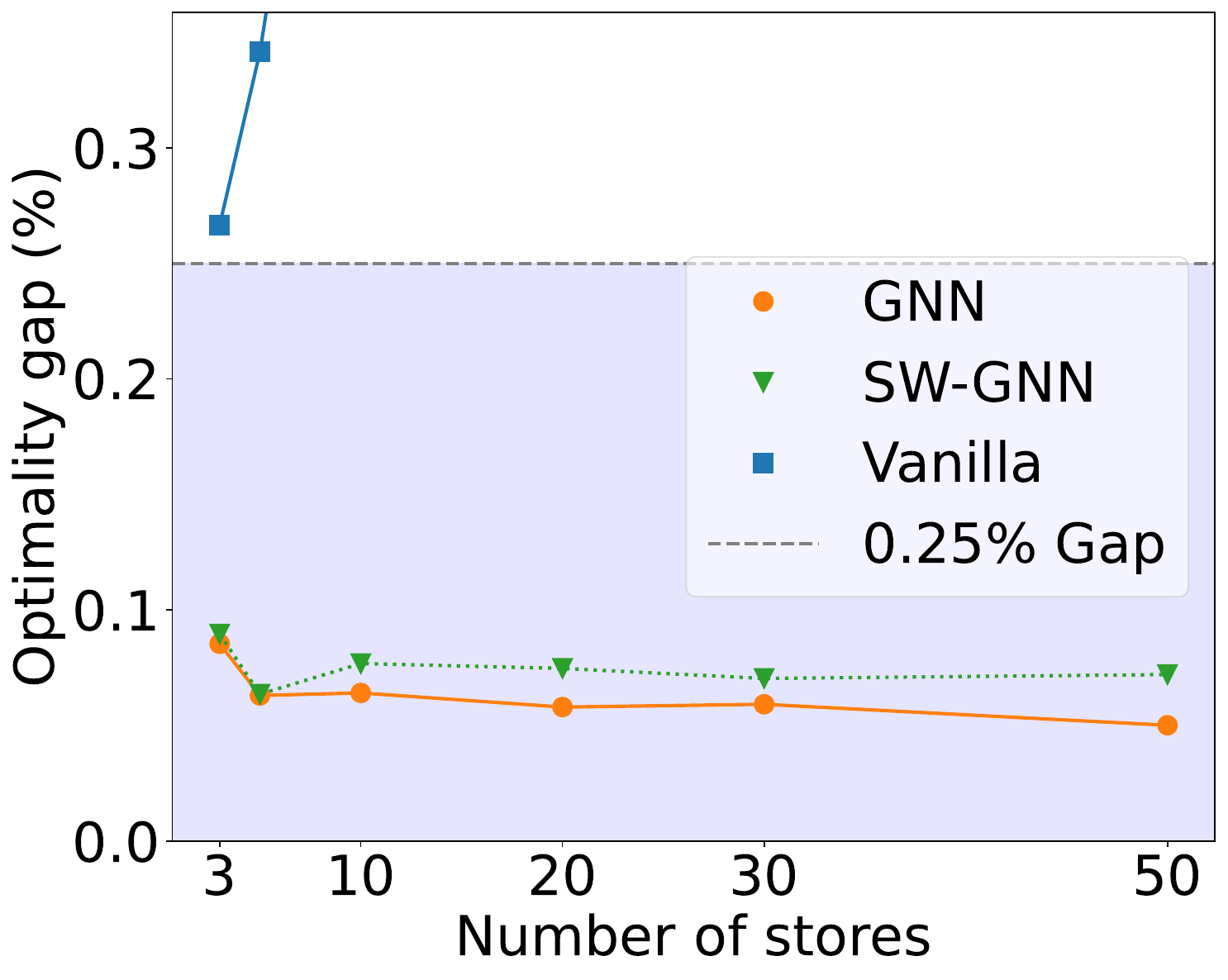}
  \caption{1024 training scenarios.}
  \label{fig:sample-efficiency-weight-sharing-medium}
\end{subfigure}
\hspace{0.015\textwidth}%
\begin{subfigure}{.32\textwidth}
  \centering
  \includegraphics[width=1\linewidth]{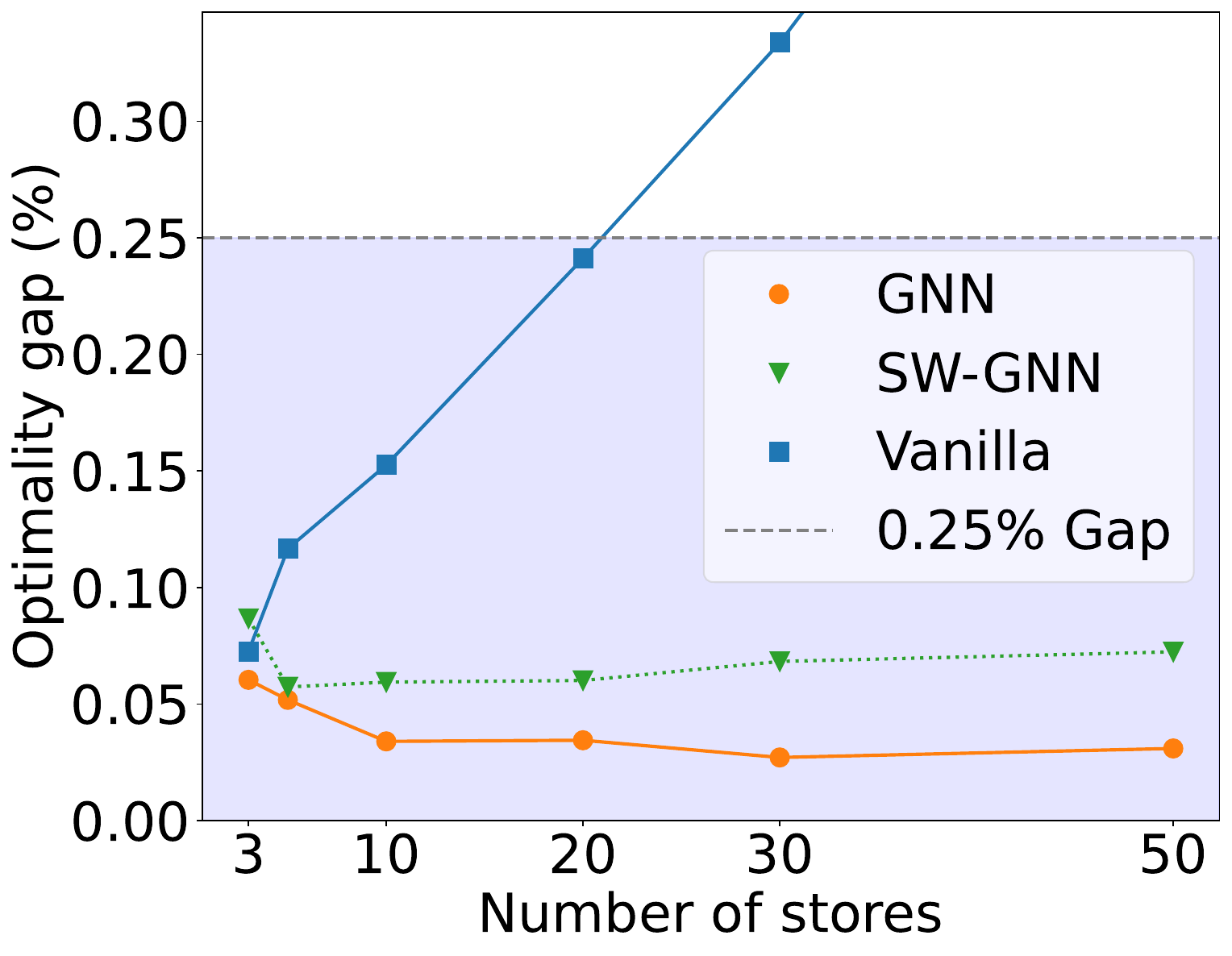}
  \caption{8192 training scenarios.}
  \label{fig:sample-efficiency-weight-sharing-large}
\end{subfigure}
\caption{ Optimality gaps (\%) for the SW-GNN, Vanilla NN, and GNN architectures in setting S10, which consists of many identical, independent stores. Results are shown across varying numbers of stores and sample sizes.}
\label{fig:weight_sharing}
\end{figure}

While this example lacks true network structure, the findings are closely mirrored in Section~\ref{sec:sample-efficiency-results} in the presence of network structure. Importantly, this behavior persists despite the inclusion of message passing, which does not appear to degrade generalization. We delve deeper into the role of message passing in the expressive power of the GNN architecture in Section \ref{sec:value-of-message-passing}.

\subsubsection{ GNNs enable coordination via message passing. \label{sec:value-of-message-passing}}

Section~\ref{sec:value-of-weight-sharing} explored how our structured representation and weight sharing enhance generalization. One possible approach to improve sample efficiency—and simplify computation—is to adopt a fully decentralized policy, where local decisions rely only on edge-specific information. In this section, we show that such decentralized policies can underperform in certain settings, even when trained centrally on large datasets—that is, even when decentralized-execution policies are trained to optimize a global cost function across a large dataset of demand trajectories.

To test whether GNNs can effectively coordinate, we focus on a setting where stores must share limited warehouse capacity and each receives private information about future demand. Optimal allocation requires stores to share their local forecasts—allocation to a store should be conservative today if higher-priority stores will face high demand tomorrow, which requires coordination across the network.

{\bf Benchmarks.}
We introduce setting S11, which features an OWMS network topology with a transshipment warehouse and three stores, under a lost demand assumption. A high-priority store faces a large unit underage cost of 10, while the medium- and low-priority stores have underage costs of
6 and 2, respectively. All lead times are one period, and each location incurs a holding cost equal to 25\% of its underage cost. Demand at each store follows a normal distribution, independent across stores and over time.
Crucially, at each period, each store observes the mean of its own demand distribution for the following period. Given the unit lead times, this allows stores to place informed orders. However, because stores are not informed of other stores' forecasts, their ability to coordinate is limited.

{\bf Baselines}
We compare the GNN with a \textit{decentralized NN}, which is a GNN without the message-passing component, reflecting a policy that relies solely on local, edge-level information.

{\bf Experiment specifications.} See Appendix \ref{appendix:value-of-message-passing}.

{\bf Results.}
In this setting, the GNN outperformed its decentralized counterpart by a $\edit{40}\%$ margin in total costs. Figure~\ref{fig:value-of-message-passing} plots the normalized allocations to the low-priority store as a function of the normalized sum of forecasted demand across the medium- and high-priority stores, under both policies. While allocations under the decentralized policy do decrease with forecasted demand at other stores, they are systematically higher than those of the GNN when demand elsewhere is high, and lower when it is low. This suggests that the decentralized policy is forced to “hedge” against missing information—choosing actions that perform reasonably on average but fail to adapt to specific realizations. These results underscore the importance of message passing in GNN-based architectures, which enables coordination by sharing relevant information across locations.

\begin{figure}[htbp!]
\centering
\captionsetup[subfigure]{justification=centering}

\begin{subfigure}{0.48\textwidth}
  \centering
  \includegraphics[width=\linewidth]{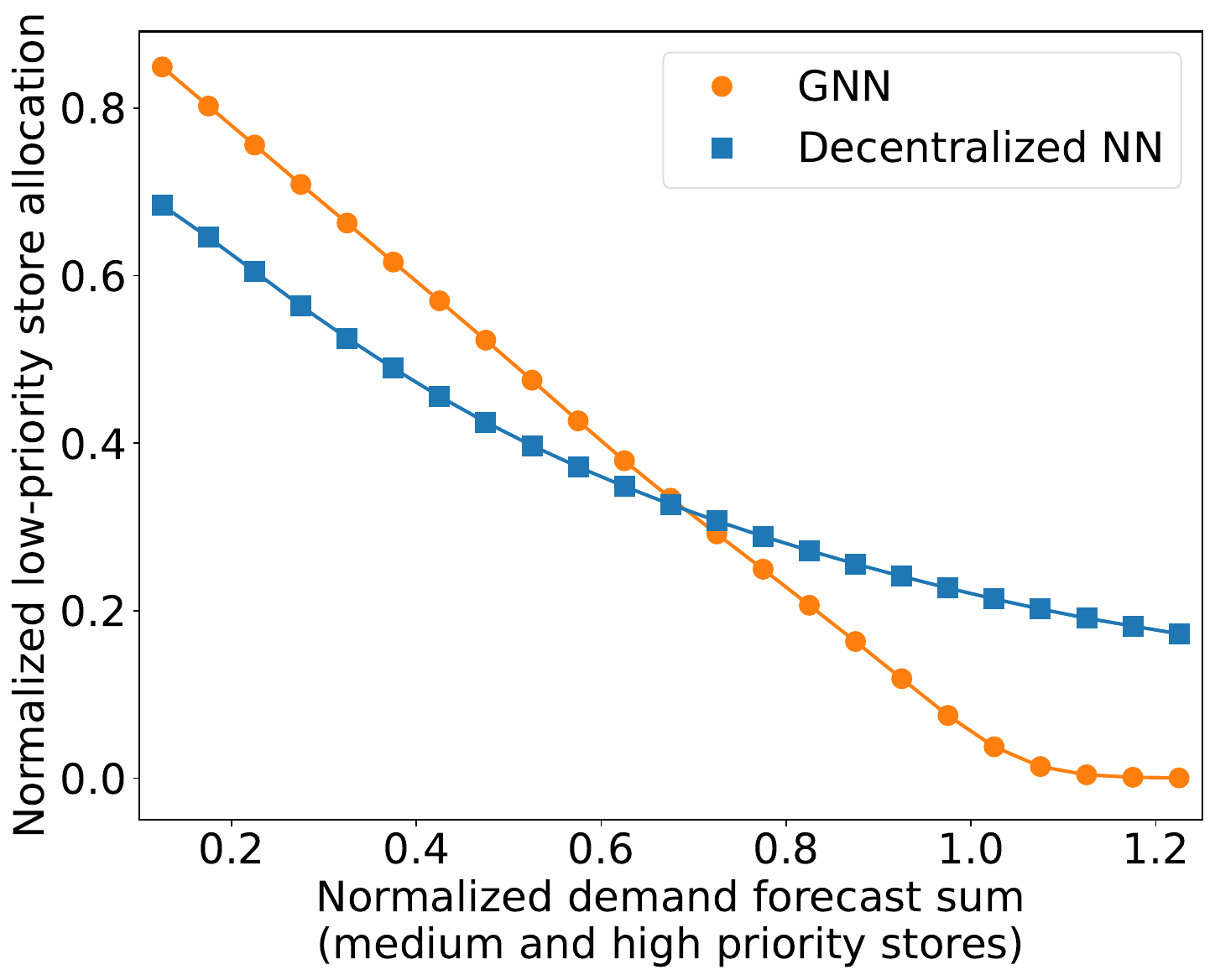}
  \caption{Setting S11: low-priority allocations.}
  \label{fig:value-of-message-passing}
\end{subfigure}
\hfill
\begin{subfigure}{0.48\textwidth}
  \centering
  \includegraphics[width=\linewidth]{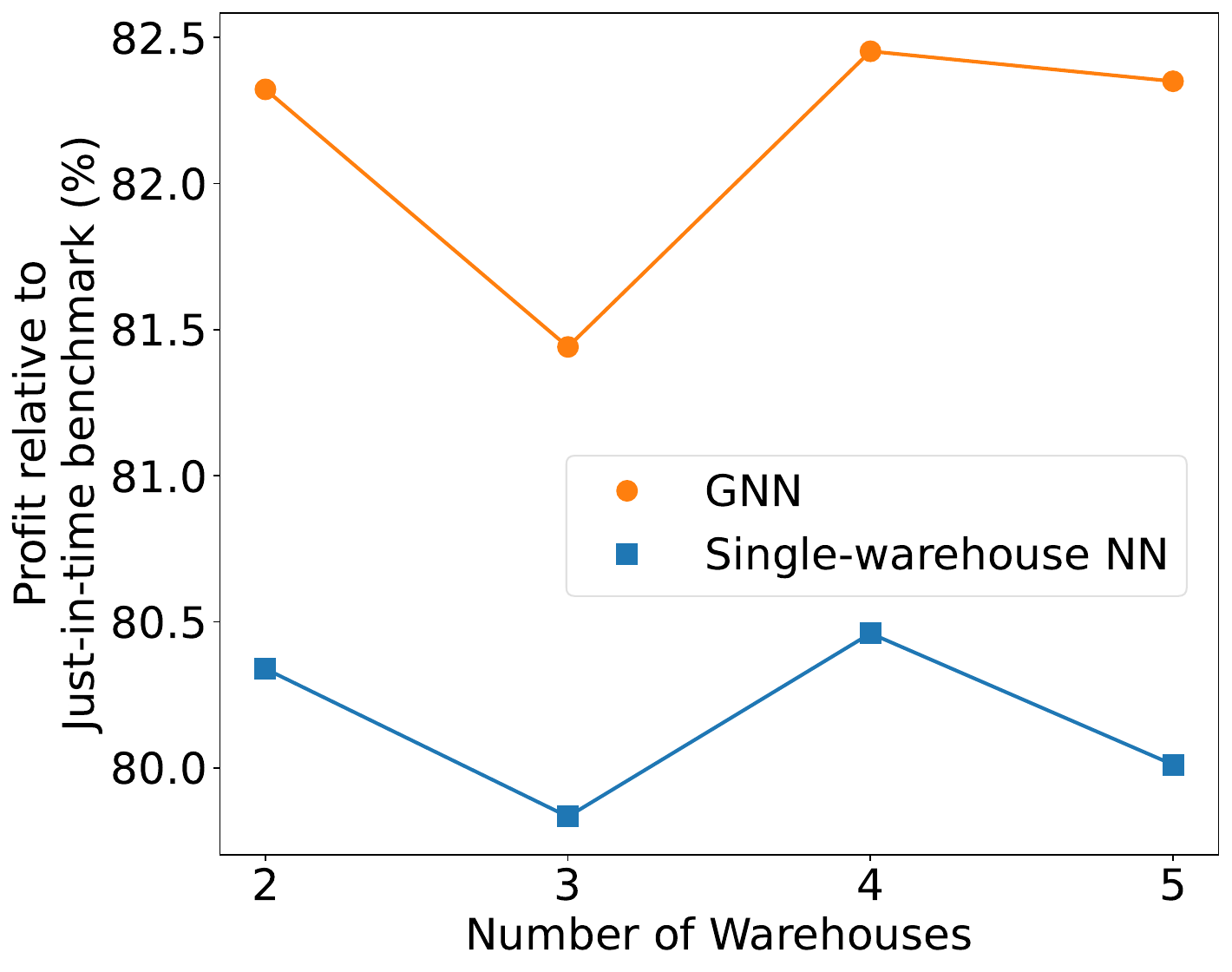}
  \caption{Setting S9: relative profit.}
  \label{fig:value-of-flexibility}
\end{subfigure}

\caption{\textbf{Left:} Allocations to the low-priority store in Setting S11, plotted against the sum of forecasted demand at the medium- and high-priority stores, with both quantities normalized by the warehouse's inventory on hand. \textbf{Right:} Profit relative to Just-in-time policy achieved by the GNN and the Single-warehouse NN in setting S9, across varying numbers of warehouses.}
\label{fig:value-of-experiments}
\end{figure}

\subsubsection{GNN policies  successfully exploit flexibility in the inventory network. \label{sec:value-of-flexibility} }

One of the key advantages of the GNN architecture is its ability to encode and exploit the structure of the inventory network. In this section, we investigate whether the policies induced by GNNs are able to selectively exploit this edge-level flexibility, adapting to the cost and lead-time tradeoffs available in the network.

{\bf Benchmarks.}
We consider setting S9, which corresponds to an MWMS network with realistic demand and a lost-demand assumption. We vary the number of warehouses and evaluate performance on a single meta-instance with heterogeneous parameters across locations.

{\bf Baselines.}
We construct a \textit{Single-warehouse NN} by assigning each store to the warehouse that minimizes expected cost, determined via a preliminary screening step that solves isolated two-node subproblems for each warehouse-store pair. The screening typically selected the warehouse with the lowest edge cost, and occasionally favored the one with the shortest lead time. The resulting static assignments define a simplified policy graph on which a GNN-based policy is trained.

{\bf Experiment specifications.} See Appendix \ref{appendix:value-of-flexibility}

{\bf Results.}
Figure~\ref{fig:value-of-flexibility} shows that the GNN consistently outperforms the Single-warehouse NN, achieving relative performance gains of 2.5\% and 2.9\% for networks with four and five warehouses, respectively. This highlights the GNN’s ability to leverage the flexibility of the full network, in contrast to the fixed-assignment baseline, which cannot adapt to local inventory or demand fluctuations. Since warehouses face no ordering limits and have low holding costs, alternative static assignments are unlikely to yield significantly better performance. Although combinatorial considerations can matter when capacity constraints induce competition across stores, such tensions are minimal in this setting. This supports the view that the GNN’s ability to dynamically adapt assignments based on realized conditions is the primary source of its advantage. These results underscore the value of an integrated policy that reasons jointly across the network.

\subsection{Constructive demonstrations: how GNNs capture optimal policy structure in two stylized inventory control problems.}\label{sec: GNN_constructive}  To further motivate our GNN architecture, we now briefly describe two settings in which our GNN architecture captures the optimal (or near optimal) policy in a simple way. 

{\bf Serial network.} First, we consider a serial network (see Figure \ref{fig:serial-system}) under a backlogged demand assumption, and label nodes as $0,\ldots,K$ from the special node $0$ to the only store $K$. In this setting, an echelon-stock policy is optimal (see Appendix \ref{appendix:echelon-stock-policy} for its definition). In this policy class, the shipment on each edge $(k-1, k)$ depends on the state $S_t$ only via the sum of inventory and outstanding orders across the ``downstream'' locations $k, k+1, \dots, K$ (the \textit{echelon inventory position}), given by $\sum_{i=k}^K(I^i_t + \sum_{l=1}^{\bar{L}^i - 1} Q_t(l))$. This can be captured with $L^{\textup{MP}}=K-1$ iterations of message passing via a simple choice of embeddings and functions in our GNN class. We provide the complete proof in Appendix \ref{appendix:gnn-serial-proof}.

{\bf Stylized OWMS setting.} Now, we consider a stylized example with an OWMS network (see Figure~\ref{fig:one-warehouse-many-stores}), with a warehouse lead time of 1 and a lead time of 0 for all stores, and demand that is correlated across stores and independent across time. Since the dimension of the state and action spaces increases as the number of stores grows, one might expect that the structure of an optimal policy grows increasingly complex. Instead, due to increasingly weak coupling between individual stores, the problem simplifies. \edit{We show that there exists a simple (asymptotically) optimal policy where allocation along any warehouse-to-store edge is determined by a single function applied to local information about the receiving store (its local state and unit economics) and a single global statistic (the total amount of inventory in the network).
Our GNN architecture naturally implements this through one round of message passing: the warehouse aggregates all store inventory levels to compute total network inventory, then propagates this global statistic to each warehouse-to-store edge. A shared readout function can then implement the optimal policy across all edges. 
Iterative message passing is important here: we conjecture that decentralized policies (such as the Decentralized NN described in Section \ref{sec:value-of-message-passing}) are generally suboptimal due to their inability to track total inventory in the network, and would incur an increase in expected cost which does not vanish with $K$. See Appendix \ref{appendix: proof of main theorem} for a description of the setting, our formal result characterizing a near optimal policy, and how the policy can be captured by a simple instantiation of our GNN architecture.}

\section{Conclusion 
\label{sec:conclusion} }

We begin with a thorough discussion of our modeling assumptions in Section~\ref{sec:discussion-on-main-assumptions}, and conclude in Section~\ref{sec:summary-and-outlook} by summarizing our findings and outlining directions for future research.

\subsection{Discussion on main assumptions 
\label{sec:discussion-on-main-assumptions} }

Here, we justify our main assumptions, building on the motivation outlined in Section~\ref{sec:scope-and-modeling-assumptions}.

{\bf Continuity of costs, transitions, and policy.} Our model treats ordering quantities as continuous variables and assumes smoothly varying costs, enabling the use of pathwise gradients. While this departs from classical settings with fixed ordering costs (where well-performing policies  often exhibit large discontinuities in the state, it is relevant for large retailers that receive shipments consisting of many products on a fixed schedule. We discuss possible extensions to discontinuous settings in Section~\ref{sec:summary-and-outlook}.

{\bf Differentiability in networked settings.} Even in complex networks involving shared resources, the cost and transition functions remain smooth under appropriate modeling choices. We achieve this by incorporating FELs (see Section~\ref{section:action-and-feasible-allocations}) that enforce constraints in a differentiable manner. This enables training policies that produce feasible actions end-to-end using pathwise gradients.

{\bf Uncensored demand observations and "backtestability".} Our approach relies on the key feature that the stochastic process is exogenous and observable, which allows for evaluating the performance of any policy using historical data. In particular, we assume access to a dataset of uncensored demand observations. In practice, one often only observes sales data, which reflects censored demand. Still, this assumption is common in the literature \citep{gijsbrechts2021can, xin2021understanding}, as it avoids the need to model uncertainty due to partial observability.
Alternatively, HDPO can be viewed as operating after a preprocessing step that imputes demand from sales and auxiliary signals. In Appendix \ref{appendix:censored-demand}, we justify this treatment by presenting: (i) evidence that de-censoring techniques allow HDPO to perform well even on small, censored datasets, and (ii) initial results suggesting that HDPO-trained neural policies are as robust to demand misspecification as classical heuristics. We also offer a brief tutorial on the challenges of censored data, explore possible solutions, and assess the performance impact of imperfect de-censoring.

{\bf Scope and extensions.} Our model and experiments focus on a subset of structural features commonly studied in the operations literature, including fixed lead times and single-product inventory. This setup facilitates benchmarking against alternative heuristic policies and enables clear performance insights. Nonetheless, our approach can readily accommodate additional complexities—such as stochastic lead times and random yields—through suitable redefinitions of the state, transition, and cost functions, as long as the problem retains the properties outlined in Section~\ref{sec:HD-structure}. Multi-product extensions are also feasible, though they may require additional development of FELs and a revised definition of scenario data. 

\subsection{Summary and Outlook \label{sec:summary-and-outlook}}

Previous research applying DRL to inventory management has largely relied on generic algorithms such as REINFORCE and standard NN architectures. While these approaches can often be made to work, \citet{gijsbrechts2021can} note that “initial tuning of the hyperparameters is computationally and time intensive and requires both art and science.” In contrast, we systematically investigate two directions for improving the reliability and efficiency of Deep RL in this domain.

The first is HDPO, a method that trains parameterized policies using low-variance pathwise gradients by leveraging key problem features. We demonstrate that even a simple “Vanilla” implementation of this approach can recover near-optimal policies across four inventory control settings, despite relying solely on raw state inputs. Experiments with real-world time-series data highlight the benefits of optimizing over a flexible policy class in an end-to-end manner, compared to common alternatives such as newsvendor-type policies.

The second is a policy network architecture based on GNNs, which we find significantly improves sample efficiency across a diverse set of inventory control settings. We identify its structured design and weight sharing as key enablers of this improved efficiency, and message passing as a mechanism that provides representational power despite its semi-decentralized structure. We validated this hypothesis through several experiments.

As a final contribution, we develop a comprehensive suite of benchmarks that span a variety of network topologies, demand processes, and structural assumptions. To support future work and promote reproducibility, we make our code publicly available.

This paper suggests several promising directions for future work. First, one might explore additional operations problems amenable to HDPO. Second, biased yet practical relaxations—such as Gumbel-softmax \citep{maddison2016concrete, jang2016categorical} and straight-through estimators \citep{bengio2013estimating}—enable the use of pathwise gradients even in discontinuous settings. These methods have shown success in related domains like queueing control \citep{che2024differentiable}, suggesting potential extensions of our approach to fixed-cost inventory problems.
In a similar vein, hybrid methods that combine score-function and pathwise gradients could support joint optimization over discrete and continuous decisions. Finally, GNNs have been shown to generalize across unseen topologies in other domains. Developing architectures that transfer across inventory networks would be practically valuable, enabling the use of Deep RL in settings where the network evolves over time (e.g., when new stores are added after training).

{\bf \large Acknowledgment} 

After beginning our numerical experiments, we learned of the work of \cite{madeka2022deep}. We thank them for their open and collaborative outlook, and their very helpful feedback.

\newpage

\bibliographystyle{informs2014}
\bibliography{references}

\begin{thebibliography}{75}
\providecommand{\natexlab}[1]{#1}
\providecommand{\url}[1]{\texttt{#1}}
\providecommand{\urlprefix}{URL }

\bibitem[{Abadi et~al.(2015)Abadi, Agarwal, Barham, Brevdo, Chen, Citro, Corrado, Davis, Dean, Devin, Ghemawat, Goodfellow, Harp, Irving, Isard, Jia, Jozefowicz, Kaiser, Kudlur, Levenberg, Man\'{e}, Monga, Moore, Murray, Olah, Schuster, Shlens, Steiner, Sutskever, Talwar, Tucker, Vanhoucke, Vasudevan, Vi\'{e}gas, Vinyals, Warden, Wattenberg, Wicke, Yu, \protect\BIBand{} Zheng}]{tensorflow2015-whitepaper}
Abadi M, Agarwal A, Barham P, Brevdo E, Chen Z, Citro C, Corrado GS, Davis A, Dean J, Devin M, Ghemawat S, Goodfellow I, Harp A, Irving G, Isard M, Jia Y, Jozefowicz R, Kaiser L, Kudlur M, Levenberg J, Man\'{e} D, Monga R, Moore S, Murray D, Olah C, Schuster M, Shlens J, Steiner B, Sutskever I, Talwar K, Tucker P, Vanhoucke V, Vasudevan V, Vi\'{e}gas F, Vinyals O, Warden P, Wattenberg M, Wicke M, Yu Y, Zheng X (2015) {TensorFlow}: Large-scale machine learning on heterogeneous systems. \urlprefix\url{https://www.tensorflow.org/}, software available from tensorflow.org.

\bibitem[{Allen-Zhu et~al.(2019)Allen-Zhu, Li, \protect\BIBand{} Song}]{allen2019convergence}
Allen-Zhu Z, Li Y, Song Z (2019) A convergence theory for deep learning via over-parameterization. \emph{International conference on machine learning}, 242--252 (PMLR).

\bibitem[{Almasan et~al.(2022)Almasan, Su{\'a}rez-Varela, Rusek, Barlet-Ros, \protect\BIBand{} Cabellos-Aparicio}]{almasan2022deep}
Almasan P, Su{\'a}rez-Varela J, Rusek K, Barlet-Ros P, Cabellos-Aparicio A (2022) Deep reinforcement learning meets graph neural networks: Exploring a routing optimization use case. \emph{Computer Communications} 196:184--194.

\bibitem[{Arrow et~al.(1958)Arrow, Karlin, Scarf et~al.}]{arrow1958studies}
Arrow KJ, Karlin S, Scarf HE, et~al. (1958) Studies in the mathematical theory of inventory and production .

\bibitem[{Bengio et~al.(2013)Bengio, L{\'e}onard, \protect\BIBand{} Courville}]{bengio2013estimating}
Bengio Y, L{\'e}onard N, Courville A (2013) Estimating or propagating gradients through stochastic neurons for conditional computation. \emph{arXiv preprint arXiv:1308.3432} .

\bibitem[{Bradbury et~al.(2018)Bradbury, Frostig, Hawkins, Johnson, Leary, Maclaurin, Necula, Paszke, Vander{P}las, Wanderman-{M}ilne, \protect\BIBand{} Zhang}]{jax2018github}
Bradbury J, Frostig R, Hawkins P, Johnson MJ, Leary C, Maclaurin D, Necula G, Paszke A, Vander{P}las J, Wanderman-{M}ilne S, Zhang Q (2018) {JAX}: composable transformations of {P}ython+{N}um{P}y programs. \urlprefix\url{http://github.com/google/jax}.

\bibitem[{Che et~al.(2024)Che, Dong, \protect\BIBand{} Namkoong}]{che2024differentiable}
Che E, Dong J, Namkoong H (2024) Differentiable discrete event simulation for queuing network control. \emph{arXiv preprint arXiv:2409.03740} .

\bibitem[{Clark \protect\BIBand{} Scarf(1960)}]{clark1960optimal}
Clark AJ, Scarf H (1960) Optimal policies for a multi-echelon inventory problem. \emph{Management science} 6(4):475--490.

\bibitem[{Corsten \protect\BIBand{} Gruen(2004)}]{corsten2004stock}
Corsten D, Gruen TW (2004) Stock-outs cause walkouts. \emph{Harvard Business Review} 82(5):26--28.

\bibitem[{Dai \protect\BIBand{} Gluzman(2022)}]{dai2022queueing}
Dai JG, Gluzman M (2022) Queueing network controls via deep reinforcement learning. \emph{Stochastic Systems} 12(1):30--67.

\bibitem[{De~Kok et~al.(2018)De~Kok, Grob, Laumanns, Minner, Rambau, \protect\BIBand{} Schade}]{de2018typology}
De~Kok T, Grob C, Laumanns M, Minner S, Rambau J, Schade K (2018) A typology and literature review on stochastic multi-echelon inventory models. \emph{European Journal of Operational Research} 269(3):955--983.

\bibitem[{Derrow-Pinion et~al.(2021)Derrow-Pinion, She, Wong, Lange, Hester, Perez, Nunkesser, Lee, Guo, Wiltshire et~al.}]{derrow2021eta}
Derrow-Pinion A, She J, Wong D, Lange O, Hester T, Perez L, Nunkesser M, Lee S, Guo X, Wiltshire B, et~al. (2021) Eta prediction with graph neural networks in google maps. \emph{Proceedings of the 30th ACM international conference on information \& knowledge management}, 3767--3776.

\bibitem[{Ding et~al.(2022)Ding, Feng, Liu, Jiang, Zhang, Zhao, Song, Li, Jin, \protect\BIBand{} Bian}]{ding2022multi}
Ding Y, Feng M, Liu G, Jiang W, Zhang C, Zhao L, Song L, Li H, Jin Y, Bian J (2022) Multi-agent reinforcement learning with shared resources for inventory management. \emph{arXiv preprint arXiv:2212.07684} .

\bibitem[{Favorita(2017)}]{kagglefavorita}
Favorita C (2017) Favorita grocery sales forecasting. \urlprefix\url{https://www.kaggle.com/competitions/favorita-grocery-sales-forecasting/overview}.

\bibitem[{Federgruen \protect\BIBand{} Zipkin(1984{\natexlab{a}})}]{federgruen1984approximations}
Federgruen A, Zipkin P (1984{\natexlab{a}}) Approximations of dynamic, multilocation production and inventory problems. \emph{Management Science} 30(1):69--84.

\bibitem[{Federgruen \protect\BIBand{} Zipkin(1984{\natexlab{b}})}]{federgruen1984computational}
Federgruen A, Zipkin P (1984{\natexlab{b}}) Computational issues in an infinite-horizon, multiechelon inventory model. \emph{Operations Research} 32(4):818--836.

\bibitem[{Feng et~al.(2021)Feng, Gluzman, \protect\BIBand{} Dai}]{feng2021scalable}
Feng J, Gluzman M, Dai JG (2021) Scalable deep reinforcement learning for ride-hailing. \emph{2021 American Control Conference (ACC)}, 3743--3748 (IEEE).

\bibitem[{Freeman et~al.(2021)Freeman, Frey, Raichuk, Girgin, Mordatch, \protect\BIBand{} Bachem}]{freeman2021brax}
Freeman CD, Frey E, Raichuk A, Girgin S, Mordatch I, Bachem O (2021) Brax--a differentiable physics engine for large scale rigid body simulation. \emph{arXiv preprint arXiv:2106.13281} .

\bibitem[{Gijsbrechts et~al.(2021)Gijsbrechts, Boute, Van~Mieghem, \protect\BIBand{} Zhang}]{gijsbrechts2021can}
Gijsbrechts J, Boute RN, Van~Mieghem JA, Zhang D (2021) Can deep reinforcement learning improve inventory management? performance on dual sourcing, lost sales and multi-echelon problems. \emph{Manufacturing \& Service Operations Management} .

\bibitem[{Gilmer et~al.(2017)Gilmer, Schoenholz, Riley, Vinyals, \protect\BIBand{} Dahl}]{gilmer2017neural}
Gilmer J, Schoenholz SS, Riley PF, Vinyals O, Dahl GE (2017) Neural message passing for quantum chemistry. \emph{International conference on machine learning}, 1263--1272 (PMLR).

\bibitem[{Glasserman(2004)}]{glasserman2004monte}
Glasserman P (2004) \emph{Monte Carlo methods in financial engineering}, volume~53 (Springer).

\bibitem[{Glasserman \protect\BIBand{} Tayur(1995)}]{glasserman1995sensitivity}
Glasserman P, Tayur S (1995) Sensitivity analysis for base-stock levels in multiechelon production-inventory systems. \emph{Management Science} 41(2):263--281.

\bibitem[{Greensmith et~al.(2004)Greensmith, Bartlett, \protect\BIBand{} Baxter}]{greensmith2004variance}
Greensmith E, Bartlett PL, Baxter J (2004) Variance reduction techniques for gradient estimates in reinforcement learning. \emph{Journal of Machine Learning Research} 5(9).

\bibitem[{Gronauer \protect\BIBand{} Diepold(2022)}]{gronauer2022multi}
Gronauer S, Diepold K (2022) Multi-agent deep reinforcement learning: a survey. \emph{Artificial Intelligence Review} 55(2):895--943.

\bibitem[{Haarnoja et~al.(2018)Haarnoja, Zhou, Hartikainen, Tucker, Ha, Tan, Kumar, Zhu, Gupta, Abbeel et~al.}]{haarnoja2018soft}
Haarnoja T, Zhou A, Hartikainen K, Tucker G, Ha S, Tan J, Kumar V, Zhu H, Gupta A, Abbeel P, et~al. (2018) Soft actor-critic algorithms and applications. \emph{arXiv preprint arXiv:1812.05905} .

\bibitem[{Hameed \protect\BIBand{} Schwung(2023)}]{hameed2023graph}
Hameed MSA, Schwung A (2023) Graph neural networks-based scheduler for production planning problems using reinforcement learning. \emph{Journal of Manufacturing Systems} 69:91--102.

\bibitem[{Hamilton et~al.(2017)Hamilton, Ying, \protect\BIBand{} Leskovec}]{hamilton2017inductive}
Hamilton W, Ying Z, Leskovec J (2017) Inductive representation learning on large graphs. \emph{Advances in neural information processing systems} 30.

\bibitem[{Harsha et~al.(2021)Harsha, Jagmohan, Kalagnanam, Quanz, \protect\BIBand{} Singhvi}]{harsha2021math}
Harsha P, Jagmohan A, Kalagnanam J, Quanz B, Singhvi D (2021) Math programming based reinforcement learning for multi-echelon inventory management. \emph{Available at SSRN 3901070} .

\bibitem[{Henderson et~al.(2018)Henderson, Islam, Bachman, Pineau, Precup, \protect\BIBand{} Meger}]{henderson2018deep}
Henderson P, Islam R, Bachman P, Pineau J, Precup D, Meger D (2018) Deep reinforcement learning that matters. \emph{Proceedings of the AAAI conference on artificial intelligence}, volume~32.

\bibitem[{Hu et~al.(2019)Hu, Anderson, Li, Sun, Carr, Ragan-Kelley, \protect\BIBand{} Durand}]{hu2019difftaichi}
Hu Y, Anderson L, Li TM, Sun Q, Carr N, Ragan-Kelley J, Durand F (2019) Difftaichi: Differentiable programming for physical simulation. \emph{arXiv preprint arXiv:1910.00935} .

\bibitem[{Huang et~al.(2022)Huang, Dossa, Raffin, Kanervisto, \protect\BIBand{} Wang}]{huang202237}
Huang S, Dossa RFJ, Raffin A, Kanervisto A, Wang W (2022) The 37 implementation details of proximal policy optimization. \emph{The ICLR Blog Track 2023} .

\bibitem[{Huh et~al.(2011)Huh, Levi, Rusmevichientong, \protect\BIBand{} Orlin}]{huh2011adaptive}
Huh WT, Levi R, Rusmevichientong P, Orlin JB (2011) Adaptive data-driven inventory control with censored demand based on kaplan-meier estimator. \emph{Operations Research} 59(4):929--941.

\bibitem[{Jang et~al.(2016)Jang, Gu, \protect\BIBand{} Poole}]{jang2016categorical}
Jang E, Gu S, Poole B (2016) Categorical reparameterization with gumbel-softmax. \emph{arXiv preprint arXiv:1611.01144} .

\bibitem[{Jiang et~al.(2018)Jiang, Dun, Huang, \protect\BIBand{} Lu}]{jiang2018graph}
Jiang J, Dun C, Huang T, Lu Z (2018) Graph convolutional reinforcement learning. \emph{arXiv preprint arXiv:1810.09202} .

\bibitem[{Jiang et~al.(2019)Jiang, Neyshabur, Mobahi, Krishnan, \protect\BIBand{} Bengio}]{jiang2019fantastic}
Jiang Y, Neyshabur B, Mobahi H, Krishnan D, Bengio S (2019) Fantastic generalization measures and where to find them. \emph{arXiv preprint arXiv:1912.02178} .

\bibitem[{Kakade(2001)}]{kakade2001natural}
Kakade SM (2001) A natural policy gradient. \emph{Advances in neural information processing systems} 14.

\bibitem[{Kaplan \protect\BIBand{} Meier(1958)}]{kaplan1958nonparametric}
Kaplan EL, Meier P (1958) Nonparametric estimation from incomplete observations. \emph{Journal of the American statistical association} 53(282):457--481.

\bibitem[{Kaynov et~al.(2024)Kaynov, van Knippenberg, Menkovski, van Breemen, \protect\BIBand{} van Jaarsveld}]{kaynov2024deep}
Kaynov I, van Knippenberg M, Menkovski V, van Breemen A, van Jaarsveld W (2024) Deep reinforcement learning for one-warehouse multi-retailer inventory management. \emph{International Journal of Production Economics} 267:109088.

\bibitem[{Konda \protect\BIBand{} Tsitsiklis(1999)}]{konda1999actor}
Konda V, Tsitsiklis J (1999) Actor-critic algorithms. \emph{Advances in neural information processing systems} 12.

\bibitem[{Kotecha \protect\BIBand{} Chanona(2024)}]{kotecha2024leveraging}
Kotecha N, Chanona AdR (2024) Leveraging graph neural networks and multi-agent reinforcement learning for inventory control in supply chains. \emph{arXiv preprint arXiv:2410.18631} .

\bibitem[{Krizhevsky et~al.(2012)Krizhevsky, Sutskever, \protect\BIBand{} Hinton}]{krizhevsky2012imagenet}
Krizhevsky A, Sutskever I, Hinton GE (2012) Imagenet classification with deep convolutional neural networks. \emph{Advances in neural information processing systems} 25.

\bibitem[{Levine et~al.(2016)Levine, Finn, Darrell, \protect\BIBand{} Abbeel}]{levine2016end}
Levine S, Finn C, Darrell T, Abbeel P (2016) End-to-end training of deep visuomotor policies. \emph{The Journal of Machine Learning Research} 17(1):1334--1373.

\bibitem[{Liu et~al.(2022)Liu, Hu, Peng, \protect\BIBand{} Yang}]{liu2022multi}
Liu X, Hu M, Peng Y, Yang Y (2022) Multi-agent deep reinforcement learning for multi-echelon inventory management. \emph{Available at SSRN} .

\bibitem[{Maddison et~al.(2016)Maddison, Mnih, \protect\BIBand{} Teh}]{maddison2016concrete}
Maddison CJ, Mnih A, Teh YW (2016) The concrete distribution: A continuous relaxation of discrete random variables. \emph{arXiv preprint arXiv:1611.00712} .

\bibitem[{Madeka et~al.(2022)Madeka, Torkkola, Eisenach, Foster, \protect\BIBand{} Luo}]{madeka2022deep}
Madeka D, Torkkola K, Eisenach C, Foster D, Luo A (2022) Deep inventory management. \emph{arXiv preprint arXiv:2210.03137} .

\bibitem[{Mania et~al.(2018)Mania, Guy, \protect\BIBand{} Recht}]{mania2018simple}
Mania H, Guy A, Recht B (2018) Simple random search provides a competitive approach to reinforcement learning. \emph{arXiv preprint arXiv:1803.07055} .

\bibitem[{Mnih et~al.(2016)Mnih, Badia, Mirza, Graves, Lillicrap, Harley, Silver, \protect\BIBand{} Kavukcuoglu}]{mnih2016asynchronous}
Mnih V, Badia AP, Mirza M, Graves A, Lillicrap T, Harley T, Silver D, Kavukcuoglu K (2016) Asynchronous methods for deep reinforcement learning. \emph{International conference on machine learning}, 1928--1937 (PMLR).

\bibitem[{Mnih et~al.(2015)Mnih, Kavukcuoglu, Silver, Rusu, Veness, Bellemare, Graves, Riedmiller, Fidjeland, Ostrovski et~al.}]{mnih2015human}
Mnih V, Kavukcuoglu K, Silver D, Rusu AA, Veness J, Bellemare MG, Graves A, Riedmiller M, Fidjeland AK, Ostrovski G, et~al. (2015) Human-level control through deep reinforcement learning. \emph{nature} 518(7540):529--533.

\bibitem[{Morton(1971)}]{morton1971near}
Morton TE (1971) The near-myopic nature of the lagged-proportional-cost inventory problem with lost sales. \emph{Operations Research} 19(7):1708--1716.

\bibitem[{Oda \protect\BIBand{} Joe-Wong(2018)}]{oda2018movi}
Oda T, Joe-Wong C (2018) Movi: A model-free approach to dynamic fleet management. \emph{IEEE INFOCOM 2018-IEEE Conference on Computer Communications}, 2708--2716 (IEEE).

\bibitem[{Oroojlooyjadid et~al.(2022)Oroojlooyjadid, Nazari, Snyder, \protect\BIBand{} Tak{\'a}{\v{c}}}]{oroojlooyjadid2022deep}
Oroojlooyjadid A, Nazari M, Snyder LV, Tak{\'a}{\v{c}} M (2022) A deep q-network for the beer game: Deep reinforcement learning for inventory optimization. \emph{Manufacturing \& Service Operations Management} 24(1):285--304.

\bibitem[{Park et~al.(2021)Park, Chun, Kim, Kim, \protect\BIBand{} Park}]{park2021learning}
Park J, Chun J, Kim SH, Kim Y, Park J (2021) Learning to schedule job-shop problems: representation and policy learning using graph neural network and reinforcement learning. \emph{International journal of production research} 59(11):3360--3377.

\bibitem[{Paszke et~al.(2019)Paszke, Gross, Massa, Lerer, Bradbury, Chanan, Killeen, Lin, Gimelshein, Antiga et~al.}]{paszke2019pytorch}
Paszke A, Gross S, Massa F, Lerer A, Bradbury J, Chanan G, Killeen T, Lin Z, Gimelshein N, Antiga L, et~al. (2019) Pytorch: An imperative style, high-performance deep learning library. \emph{Advances in neural information processing systems} 32.

\bibitem[{Powell(2007)}]{powell2007approximate}
Powell WB (2007) \emph{Approximate Dynamic Programming: Solving the curses of dimensionality}, volume 703 (John Wiley \& Sons).

\bibitem[{Prince(2023)}]{prince2023understanding}
Prince SJ (2023) \emph{Understanding deep learning} (MIT press).

\bibitem[{Qi et~al.(2023)Qi, Shi, Qi, Ma, Yuan, Wu, \protect\BIBand{} Shen}]{qi2023practical}
Qi M, Shi Y, Qi Y, Ma C, Yuan R, Wu D, Shen ZJ (2023) A practical end-to-end inventory management model with deep learning. \emph{Management Science} 69(2):759--773.

\bibitem[{Reiman(2004)}]{reiman2004new}
Reiman MI (2004) A new and simple policy for the continuous review lost sales inventory model. \emph{Unpublished manuscript} .

\bibitem[{Rossi et~al.(2023)Rossi, Charpentier, {Di Giovanni}, Frasca, G{\"u}nnemann, \protect\BIBand{} Bronstein}]{emanu2023dirgnn}
Rossi E, Charpentier B, {Di Giovanni} F, Frasca F, G{\"u}nnemann S, Bronstein M (2023) Edge directionality improves learning on heterophilic graphs. \emph{Proceedings of Machine Learning Research} 231:251--2527, ISSN 2640-3498, publisher Copyright: {\textcopyright} 2023 Proceedings of Machine Learning Research. All rights reserved.; 2nd Learning on Graphs Conference, LOG 2023 ; Conference date: 27-11-2023 Through 30-11-2023.

\bibitem[{Scarf et~al.(1960)Scarf, Arrow, Karlin, \protect\BIBand{} Suppes}]{scarf1960optimality}
Scarf H, Arrow K, Karlin S, Suppes P (1960) The optimality of (s, s) policies in the dynamic inventory problem. \emph{Optimal pricing, inflation, and the cost of price adjustment} 49--56.

\bibitem[{Schulman et~al.(2015)Schulman, Levine, Abbeel, Jordan, \protect\BIBand{} Moritz}]{schulman2015trust}
Schulman J, Levine S, Abbeel P, Jordan M, Moritz P (2015) Trust region policy optimization. \emph{International conference on machine learning}, 1889--1897 (PMLR).

\bibitem[{Schulman et~al.(2017)Schulman, Wolski, Dhariwal, Radford, \protect\BIBand{} Klimov}]{schulman2017proximal}
Schulman J, Wolski F, Dhariwal P, Radford A, Klimov O (2017) Proximal policy optimization algorithms. \emph{arXiv preprint arXiv:1707.06347} .

\bibitem[{Shar \protect\BIBand{} Jiang(2023)}]{shar2023weakly}
Shar IE, Jiang DR (2023) Weakly coupled deep q-networks. \emph{arXiv preprint arXiv:2310.18803} .

\bibitem[{Silver et~al.(2017)Silver, Schrittwieser, Simonyan, Antonoglou, Huang, Guez, Hubert, Baker, Lai, Bolton et~al.}]{silver2017mastering}
Silver D, Schrittwieser J, Simonyan K, Antonoglou I, Huang A, Guez A, Hubert T, Baker L, Lai M, Bolton A, et~al. (2017) Mastering the game of go without human knowledge. \emph{nature} 550(7676):354--359.

\bibitem[{Sun et~al.(2014)Sun, Wang, \protect\BIBand{} Zipkin}]{sun2014quadratic}
Sun P, Wang K, Zipkin P (2014) Quadratic approximation of cost functions in lost sales and perishable inventory control problems. \emph{Fuqua School of Business, Duke University, Durham, NC} .

\bibitem[{Tang et~al.(2019)Tang, Qin, Zhang, Wang, Xu, Ma, Zhu, \protect\BIBand{} Ye}]{tang2019deep}
Tang X, Qin Z, Zhang F, Wang Z, Xu Z, Ma Y, Zhu H, Ye J (2019) A deep value-network based approach for multi-driver order dispatching. \emph{Proceedings of the 25th ACM SIGKDD international conference on knowledge discovery \& data mining}, 1780--1790.

\bibitem[{van Hezewijk et~al.(2023)van Hezewijk, Dellaert, Van~Woensel, \protect\BIBand{} Gademann}]{van2023using}
van Hezewijk L, Dellaert N, Van~Woensel T, Gademann N (2023) Using the proximal policy optimisation algorithm for solving the stochastic capacitated lot sizing problem. \emph{International Journal of Production Research} 61(6):1955--1978.

\bibitem[{Van~Roy et~al.(1997)Van~Roy, Bertsekas, Lee, \protect\BIBand{} Tsitsiklis}]{van1997neuro}
Van~Roy B, Bertsekas DP, Lee Y, Tsitsiklis JN (1997) A neuro-dynamic programming approach to retailer inventory management. \emph{Proceedings of the 36th IEEE Conference on Decision and Control}, volume~4, 4052--4057 (IEEE).

\bibitem[{Vanvuchelen et~al.(2023)Vanvuchelen, De~Moor, \protect\BIBand{} Boute}]{vanvuchelen2023use}
Vanvuchelen N, De~Moor B, Boute R (2023) The use of continuous action representations to scale deep reinforcement learning for inventory control .

\bibitem[{Vanvuchelen et~al.(2020)Vanvuchelen, Gijsbrechts, \protect\BIBand{} Boute}]{vanvuchelen2020use}
Vanvuchelen N, Gijsbrechts J, Boute R (2020) Use of proximal policy optimization for the joint replenishment problem. \emph{Computers in Industry} 119:103239.

\bibitem[{Wang et~al.(2018)Wang, Liao, Ba, \protect\BIBand{} Fidler}]{wang2018nervenet}
Wang T, Liao R, Ba J, Fidler S (2018) Nervenet: Learning structured policy with graph neural networks. \emph{International conference on learning representations}.

\bibitem[{Watkins \protect\BIBand{} Dayan(1992)}]{watkins1992q}
Watkins CJ, Dayan P (1992) Q-learning. \emph{Machine learning} 8:279--292.

\bibitem[{Williams(1992)}]{williams1992simple}
Williams RJ (1992) Simple statistical gradient-following algorithms for connectionist reinforcement learning. \emph{Reinforcement learning} 5--32.

\bibitem[{Xin(2021)}]{xin2021understanding}
Xin L (2021) Understanding the performance of capped base-stock policies in lost-sales inventory models. \emph{Operations Research} 69(1):61--70.

\bibitem[{Ying et~al.(2018)Ying, He, Chen, Eksombatchai, Hamilton, \protect\BIBand{} Leskovec}]{ying2018graph}
Ying R, He R, Chen K, Eksombatchai P, Hamilton WL, Leskovec J (2018) Graph convolutional neural networks for web-scale recommender systems. \emph{Proceedings of the 24th ACM SIGKDD international conference on knowledge discovery \& data mining}, 974--983.

\bibitem[{Zipkin(2008)}]{zipkin2008old}
Zipkin P (2008) Old and new methods for lost-sales inventory systems. \emph{Operations research} 56(5):1256--1263.

\end{thebibliography}

\newpage
 
\renewcommand{\theHchapter}{A\arabic{chapter}}
\renewcommand{\theHsection}{appendix.\arabic{section}}  

\begin{APPENDIX}{}
\section*{Outline of the appendix}

We divide the appendix into four main sections. Appendix~\ref{appendix:core-ideas} elaborates on key components of our approach, including an extended literature review and an overview of the computation of zeroth-order gradient estimators. Appendix~\ref{appendix:baselines} describes the baselines used throughout the paper and explains how to compute the optimum or lower bounds for the settings in Section~\ref{subsec:optimal_exps}. In Appendix~\ref{appendix:implementation-and-experiments}, we detail our model implementation, hyperparameter tuning process, and numerical results. \edit{Finally, Appendix~\ref{appendix: proof of main theorem} provides a full proof of the asymptotic optimality of a GNN policy with one message-pass for a stylized setting involving one warehouse and a growing number of stores.}

\section{Core ideas \label{appendix:core-ideas}}

This section provides detail on several key components of our approach. Appendix~\ref{appendix:extended-lit-review} extends the literature review from Section~\ref{subsec: further-related-literature} by discussing additional machine learning methods relevant to our work. Appendix~\ref{appendix:model} gives a detailed formulation of inventory control problems in networks. Appendix~\ref{appendix:score-function-gradient} explains the computation of zeroth-order gradient estimators. Appendix~\ref{appendix:hdpo} presents the full HDPO algorithm, and Appendix~\ref{appendix:feasibility-enforcement} details the feasibility enforcement layers (FELs) used throughout our work. Appendix~\ref{appendix:vanilla-architecture} provides details on the Vanilla NN architecture design. Lastly, Appendix~\ref{appendix:gnn-serial-proof} contains the proof characterizing the optimality of message-passing policies in serial networks. The proof of the asymptotic optimality of GNNs for a setting with one warehouse and multiple stores is provided separately in Appendix~\ref{appendix: proof of main theorem} due to its length.

\subsection{Extended literature review \label{appendix:extended-lit-review}}

Here, we extend the literature review introduced in Section~\ref{subsec: further-related-literature}.

\paragraph{\bf \edit{Graph NNs in Reinforcement Learning.}}

GNNs have demonstrated success in supervised learning applications, including drug discovery \citep{gilmer2017neural}, social network analysis \citep{hamilton2017inductive}, and recommender systems \citep{ying2018graph}. However, applications to reinforcement learning (RL) are far more limited. Most existing works leverage GNNs to encode representations of high-dimensional states for use in policy-based \citep{park2021learning, hameed2023graph}, value-based \citep{almasan2022deep, kotecha2024leveraging}, or model-based approaches \citep{jiang2018graph}.

Our work instead focuses on network control problems where actions spaces are also very high dimensional. Specifically, we examine how GNN policies can automatically balance the decomposition and coordination of decisions across many network components. This focus has conceptual parallels to \cite{wang2018nervenet}, who used graph representations of an agent's physical structure (modeling body parts as nodes and joints as edges) to parameterize structured policies for simultaneous control. While our work differs in using pathwise gradients and addressing a different application domain, the core insight about structured coordination remains similar.

The work of \cite{kotecha2024leveraging} also applies GNNs to multi-agent inventory control but diverges significantly in methodology. They focus on \textit{decentralized-execution} policies where agents make decisions based solely on local information, using GNNs as value function architectures for centralized training. In contrast, our approach emphasizes coordination through message passing within the GNN policy itself.

\paragraph{\bf Other DRL methods in inventory management.}

To our knowledge, \cite{van1997neuro} were the first to apply DRL to inventory management. They introduce the notions of a ``pre-decision-state'' and ``post-decision-state'', which is the state of the inventory system before and after an order is placed.  See \cite{powell2007approximate} for further discussion. They  apply a specialized approximate policy iteration algorithm which approximates the cost-to-go from a ``post-decision-state'' by applying a two layer NN and hand-crafted features of the state. Recent successes of  DRL in various domains have sparked renewed interest in applications to inventory management.

Recent papers of \cite{oroojlooyjadid2022deep} and \cite{shar2023weakly} apply variants of  the deep Q-networks algorithm \citep{mnih2015human} to inventory management problems. \cite{oroojlooyjadid2022deep} extends deep Q-networks \citep{mnih2015human} to address the classical beer game. Their data-driven approach demonstrates the ability to recover near-optimal policies in a simplified setting and significantly outperforms simple heuristics when considering complicating problem features. \cite{shar2023weakly} introduce weakly coupled deep Q-networks, leveraging a Lagrangian-decomposition-based upper bound to accelerate the training of the Q-function approximation. Applied to an inventory control problem with multiple products and exogenous production rates, their approach demonstrates superiority over classical value-based methods. Our work deviates from this literature stream as HDPO focuses on directly approximating a policy rather than a value function.

\paragraph{\bf Multi-agent RL approaches in inventory management.}
In recent efforts to address inventory problems with a network structure, actor-critic multi-agent RL (MARL) techniques have been explored \citep{liu2022multi, ding2022multi}. Although our work has a loose connection to MARL, we take a distinct approach by considering a "central planner" who optimizes a system comprising many "weakly coupled" components. This approach avoids the need for sophisticated techniques to handle the flow of information among agents.

\paragraph{\bf{Other machine learning approaches in inventory management.}}
Alternative machine learning approaches, different from DRL, have also been investigated for inventory problems. \cite{qi2023practical} uses supervised learning for a single-location inventory problem with stochastic lead times, employing posterior optimal actions as training labels for a policy NN. Additionally, \cite{harsha2021math} proposed training ReLU-based NNs to approximate value functions, and then utilizing them in a Mixed Integer Program to derive actions.

\subsection{Detailed Inventory Network Formulation \label{appendix:model} }

This appendix presents a complete version of the model introduced in Section~\ref{sec: inventory_network_control}, including explicit definitions of the transition dynamics and cost functions.

The inventory network is modeled as a directed graph $\mathcal{G} = (\mathcal{N}, \mathcal{E})$, where $\mathcal{N}$ is the set of nodes (locations) and $\mathcal{E}$ contains a directed edge for each supplier-receiver relationship. We include a special node, labeled $0$, representing an \textit{external supply source} with unlimited inventory. Any edge of the form $(0, k) \in \mathcal{E}$ indicates that node $k$ can order from this external source. Node $0$ does not hold inventory or make decisions beyond supplying connected nodes. We define $\mathcal{N}_+ := \mathcal{N} \setminus \{0\}$ to represent the set of all physical locations in the network.

Each node $k \in \mathcal{N}_+$ is either a \textit{store}, where demand arises, or a \textit{distribution center} (DC), which supplies inventory to other nodes. This avoids cases in which a location must simultaneously fulfill demand and supply another node. However, the model can be readily extended to allow nodes to serve both roles. We denote the set of stores by $\mathcal{N}_{\text{st}} \subseteq \mathcal{N}_+$ and define DCs as $\mathcal{N}_{\text{dc}} = \{ k \in \mathcal{N}_+ \mid \mathcal{N}_{\text{rec}}^k \neq \emptyset \}$. For each node $k$, we define its set of \textit{suppliers} as $\mathcal{N}_{\text{sup}}^k = \{ j \in \mathcal{N} \mid (j,k) \in \mathcal{E} \}$, and its set of \textit{receivers} as $\mathcal{N}_{\text{rec}}^k = \{ j \in \mathcal{N} \mid (k,j) \in \mathcal{E} \}$.

We assume that each store sells the same product (meaning we do not model cross-product interactions). There are no constraints on storage or transfer quantities. 
All lead times are assumed to satisfy $L^e \geq 2$. Zero lead times allow inventory to flow through multiple edges within a single period, while lead times of up to $1$ period do not necessitate tracking outstanding orders explicitly. Slight adjustments to the formulation are required to handle lead times of zero or one.
Finally, we recall that $[x] = \{1,\ldots,x\}$ and define $[x]_0 = [x] \cup \{0\}$, $(x)^+ = \max\{0, x\}$, and $v(n)$ as the $n$-th entry of a vector $v$.

Each edge $e \in \mathcal{E}$ has a fixed lead time $L^e \geq 2$. For each node $k$, we define $\bar{L}^k = \max_{e = (j,k)} L^e$ as the longest incoming lead time. \edit{The cost structure includes holding costs $h^k$ and underage costs $p^k$ at each node, as well as edge-level costs $\lambda^e$, which represent per-unit shipment costs or procurement costs when the origin node is the outside supplier.}

{\bf State and action spaces.}
In this model, the trace of exogenous outcomes $(\xi_1, \ldots, \xi_T)$, with $\xi_t  = \left( \xi^k_t  \right)_{k \in \mathcal{N}_{\text{st}}}$, represents the sequence of (uncensored) demand quantities at each store across time. The system state is a concatenation $S_t = (S_t^k)_{k \in \mathcal{N}_+}$ of local states
\begin{align}
S_t^k = (I^k_t, Q^k_t, \mathcal{F}_t^k),
\end{align}
where $I^{k}_{t} \in \mathbb{R}$ is the current inventory on hand at location $k \in \mathcal{N}_+$, $Q^k_t$ represents the vector of outstanding orders, and $\mathcal{F}_t^k$ contains forecasting-relevant information.
The vector $Q^k_t$ tracks orders placed in the previous $\bar{L}^k - 1$ periods \edit{that have yet to arrive.}\footnote{We further note that random lead times can be modeled by letting each entry in $Q^k_t$ track orders that have not yet arrived and updating $I^{k}_t$ based on the realized lead time.}.

The forecasting information $\mathcal{F}_t^k$ may encode recent demand as well as exogenous contextual features such as weather.
To ensure that the system is Markovian, $\mathcal{F}_t^k$ must be rich enough\footnote{One can always ensure the Markov property by letting $\mathcal{F}_t^k = (\xi_1^k, \ldots, \xi_{t-1}^k)$ include the full demand history. If the demand process is an $m^{\text{th}}$-order Markov chain, then $\mathcal{F}_t^k = (\xi_{t-m}^k, \ldots, \xi_{t-1}^k)$ suffices.} to capture relevant historical demand information:
\begin{align}
\mathbb{P}\left( [\xi_{t}, \ldots, \xi_{T}] \in \cdot \mid (\mathcal{F}_t^k)_{k \in \mathcal{N}_{\text{st}}} \right)
= \mathbb{P}\left( [\xi_{t}, \ldots, \xi_{T}] \in \cdot \mid (\mathcal{F}_i^k)_{k \in \mathcal{N}_{\text{st}}, i \in [t]}, \xi_1, \ldots, \xi_{t-1} \right).
\end{align}

After observing $S_t$, a central planner must define the \textit{allocations} $a^e_t$ on each edge $e \in \mathcal{E}$. The action space at state $S_t$ is given by:
\begin{align}
\mathcal{A}(S_t) = \{ (a^e_t)_{e \in \mathcal{E}} \in \mathbb{R}_+^{|\mathcal{E}|} \mid \sum_{k \in \mathcal{N}_{\text{rec}}^j} a^{(j,k)}_t \leq I_t^j, \quad \forall j \in \mathcal{N}_{\text{dc}} \}.
\end{align}
This constraint ensures that no DC allocates more inventory than it has available.

The dimension of the state space grows with the number of locations, the maximum lead time on each edge, and the length of the demand forecast window. The dimension of demand realizations (\ie, $\xi_t$) grows with the number of locations, and the dimension of the action space grows with the number of edges.

{\bf Transition functions.}
Following the literature on inventory control, we separately consider two ways in which inventory on hand at each store evolves:
\begin{align}
    I^k_{t+1} &= I^k_{t} - \xi^k_t + Q^k_t(1) & k \in \mathcal{N}_{\text{st}}\, , \quad \text{OR} \label{eq: transition for backlogged} \\
    I^k_{t+1} &= (I^k_{t} - \xi^k_t)^+ + Q^k_t(1) & k \in \mathcal{N}_{\text{st}} \phantom{\quad \text{OR}} \label{eq: transition for lost demand}
\end{align}
Equation~\eqref{eq: transition for backlogged} corresponds to a \emph{backlogged demand} model, where customers wait if the product is unavailable. Equation~\eqref{eq: transition for lost demand} corresponds to a \emph{lost demand} model, where unmet demand is discarded.

Let $\{{\mathbf{e}_i}\}_{i=1}^{\bar{L}-1}$ denote the standard canonical basis for $\mathbb{R}^{\bar{L}-1}$, where $\bar{L} = \max_{(j,k) \in \mathcal{E}} L^{(j,k)}$ is the maximum lead time across all edges in the network.
We now define the remaining transitions:
\begin{align}
    I^k_{t+1} &= I^k_t + Q^k_t(1) - {\textstyle \sum_{(k,j) \in \mathcal{E}}} a^{(k,j)}_t, & k \in \mathcal{N}_{\text{dc}}  \label{eq: transition warehouse} \\
    Q^k_{t+1} &= (Q^k_t(2), \ldots, Q^k_t(\bar{L} - 2), 0)^{\top} + {\textstyle \sum_{(j,k) \in \mathcal{E}}} a^{(j,k)}_t \mathbf{e}_{L^{(j,k)}-1}, & k \in \mathcal{N}_+. \label{eq: transition pipeline}
\end{align}
Note that $Q^k_{t+1}$ is updated by shifting entries one position to the left and then adding each allocation in the position corresponding to one unit less than the associated lead time.

We omit an explicit formula for how forecasting information is updated, as this depends on the specific modeling assumptions.

{\bf Cost functions.}
The per-period cost function is defined as: 
\begin{align}
\label{eq: cost-per-period}
c(I_t, a_t, \xi_t) = \sum_{e \in \mathcal{E}} \underbrace{\lambda^e a^{e}_t}_{\text{cost at edge } e} 
+ \sum_{k \in \mathcal{N}_{\text{st}}} \underbrace{c^k(I^k_t, \xi^k_t)}_{\text{cost at store } k} 
+ \sum_{k \in \mathcal{N}_{\text{dc}}} \underbrace{c^k\left(I^k_t, \sum_{e=(k,j) \in \mathcal{E}} a^{e}_t \right)}_{\text{cost at DC } k}.
\end{align}

Here, $\lambda^e$ represents the per-unit cost of sending inventory along edge $e$. For edges where the special node 0 acts as the supplier (i.e., $(0,k) \in \mathcal{E}$), $\lambda^e$ is interpreted as a procurement cost. 

The cost incurred at store $k$ is given by:
\begin{align}
c^k(I^k_t, \xi^k_t) = p^k(\xi^k_t - I^k_t)^+ + h^k(I^k_t - \xi^k_t)^+,
\end{align}
where the first and second terms correspond to underage and overage costs, respectively. Per-unit underage costs $p^k$ and holding costs $h^k$ may vary across locations.

The cost incurred at DC $k$ is given solely by holding costs:
\begin{align}
c^k\left(I^k_t, \sum_{(k,j) \in \mathcal{E}} a^{(k,j)}_t\right) = h^k \left(I^k_t - \sum_{(k,j) \in \mathcal{E}} a^{(k,j)}_t \right).
\end{align}

\subsection{Score-function gradient estimation
\label{appendix:score-function-gradient}
}

Policy gradient algorithms in the style of REINFORCE \citep{williams1992simple} do not calculate gradients samples as in Algorithm \ref{alg:hdpo_full}, instead attempting to perform gradient based optimization without any knowledge of the model and without assuming any smoothness properties of it.  These algorithms enforce smoothness by employing  randomized policies under which action probabilities are smooth functions of policy parameters.  They derive gradient estimates by utilizing $T$-length sequences of states visited and costs incurred, and differentiating the log-probabilities associated with the sampled action in each observed state.

Zeroth-order estimators of gradients are typically computed as follows. Let $\pi_{\theta}$ be a randomized policy such that, for each $S \in \mathcal{S}$, $\pi_{\theta}(\cdot | S)$ defines a probability distribution over actions $a \in \mathcal{A}(S)$. Define the policy loss $J(\theta) = \mathbb{E}^{\pi_{\theta}} \left[\sum_{t \in [T]} c(S_t, a_t, \xi_t) \right]$, representing the expected cumulative cost when following policy $\pi_{\theta}$.
The REINFORCE algorithm uses the following policy gradient formula
\begin{align}
\label{eq:reinforce-gradient}
    \nabla_{\theta}J(\theta) = \E^{\pi_{\theta}} \left[\left(\sum_{t \in [T]} G_t \frac{\nabla_{\theta} \pi_{\theta}(a_t|S_t)}{\pi_{\theta}(a_t|S_t)} \right) \right],
\end{align}
where $G_t = \sum_{u= t,\ldots, T} c(S_u, a_u, \xi_u)$. 

To obtain an estimator for the gradient of the policy loss, a collection of $N$ trajectories, denoted as $\{\tau^n = (S^n_1, a^n_1, c^n_1, \ldots, S^n_T, a^n_T, c^n_T)\}_{n \in [N]}$, must be generated via policy rollouts following each gradient update. Here, $S^n_t$ represents the state, $a^n_t$ signifies the action sampled from $\pi_{\theta}(\cdot | S^n_t)$, and $c^n_t$ denotes the corresponding incurred cost at time step $t$ within policy rollout $n$. Subsequently, the gradient is estimated as:
\begin{align}
\label{eq:reinforce-estimator}
    \frac{1}{N} \sum_{n \in [N]} \left[\left(\sum_{t \in [T]} G^n_t \frac{\nabla_{\theta} \pi_{\theta}(a^n_t|S^n_t)}{\pi_{\theta}(a^n_t|S^n_t)} \right) \right],
\end{align}
with $G^n_t = \sum_{u= t,\ldots, T} c^n_u$ representing the cumulative cost from time period $t$ onwards for rollout $n$.

This method has at least two substantial drawbacks. First, estimating the policy gradient requires employing that policy for many periods to gather data. This is feasible when a high fidelity simulator is available, but may  be impractical otherwise.  Second, while the estimator \eqref{eq:reinforce-estimator} is unbiased, its variance grows with the time horizon $T$ and can explode if policies become nearly deterministic, since the inverse propensity weights $1/\pi_{\theta}(a^n_t|S^n_t)$ explode. The structure of hindsight differentiable RL problems alleviates these challenges, making it easy to backtest the performance of new polices without additional data gathering and enabling gradient estimation without inverse propensity weighting.

\subsection{HDPO algorithm \label{appendix:hdpo}}
Here, we provide a detailed description of the HDPO algorithm discussed in Section \ref{sec:hdpo}. To simplify the presentation, Algorithm \ref{alg:hdpo_full} demonstrates the forward and backward passes being performed one scenario at a time. However, in our implementation, we execute each pass simultaneously across a batch of scenarios. This is achieved by defining a tensor $\tilde{S}^H_t$ that describes the state across all scenarios within a batch and leveraging PyTorch's parallel computing capabilities to perform all operations in a parallel manner (see Appendix \ref{appendix:global-settings} for a detailed explanation). The algorithm also abstracts the policy update step through a generic function \texttt{UpdatePolicy}, which may correspond to stochastic gradient descent (SGD) or more sophisticated optimizers such as Adam—as used in our experiments.

\begin{algorithm}
\caption{Hindsight Differentiable Policy Optimization}\label{alg:hdpo_full}
\begin{algorithmic}
\Require Historical scenario pairs $(\bar{S}^1_1, \bar{\xi}^1_{1:T}), \ldots, (\bar{S}^H_1, \bar{\xi}^H_{1:T})$, initial policy parameters $\theta \in \Theta$, cost and transition functions $c$ and $f$, gradient-based policy update function \texttt{UpdatePolicy} (\eg Adam).
\While{\textup{not reached termination criteria}}
\State Randomly partition scenario indexes $[H]$ into $n_B$ index batches $(\hat{H}_1, \ldots, \hat{H}_{n_B})$.
\For{index batch $\hat{H} \in (\hat{H}_1, \ldots, \hat{H}_{n_B})$}
    \State $J \gets 0$, $\nabla J\gets 0$ \Comment{initialize batch cost and gradient}
    \For {$h \in \hat{H}$}
        \State $\ell \gets 0$, $S \gets \bar{S}^h_1$ \Comment{initialize cost and initial state}
        \For {$t \in [T]$} \Comment{forward pass to obtain cost on scenario}
        \State $\ell \gets  \ell + c(S, \pi_{\theta}(S), \bar{\xi}^h_t)/T$ \Comment{update cost on scenario}
        \State $S \gets f(S, \pi_{\theta}(S), \bar{\xi}^h_t)$ \Comment{get new state}
        \EndFor
        \State $J \gets J + \ell$ \Comment{update cost across batch}
        \State $\nabla J \gets \nabla J + \nabla_{\theta}\ell$ \Comment{get gradient sample (by backpropagation) and update}
    \EndFor
    \State $\theta \gets \texttt{UpdatePolicy} (\theta, \tfrac{1}{|\hat{H}|} \nabla J)$ \Comment{update parameters based on average sample gradient}
\EndFor
\EndWhile
\end{algorithmic}
\end{algorithm}

\subsection{Feasibility enforcement for neural policy classes}\label{appendix:feasibility-enforcement}

This appendix formalizes the feasibility enforcement mechanism outlined in Section \ref{section:action-and-feasible-allocations}. In our inventory control setting, NN policies must output feasible allocation decisions—i.e., allocations that do not exceed a distribution center's available inventory. We achieve this by structuring the network to output an unconstrained intermediate vector, and then applying deterministic post-processing functions that enforce feasibility.

As motivation, consider the conventional use of NNs in classification tasks. To ensure that the output is a valid probability vector over possible labels, the network typically produces an \emph{intermediate output}—a real-valued vector interpreted as ``logits’’—which is then passed through a softmax function to produce a feasible distribution.

Analogously, our policy maps a state to an intermediate edge output vector $(b^{(j,k)})_{(j,k) \in \mathcal{E}} \in \mathbb{R}^E$. For each distribution center $j \in \mathcal{N}_{\text{dc}}$, a \textit{feasibility enforcement layer} transforms $(b^{(j,k)})_{k \in \mathcal{N}_{\text{rec}}^j}$ into a feasible allocation vector $(a^{(j,k)})_{k \in \mathcal{N}_{\text{rec}}^j}$ that satisfies the inventory constraint:
\[
\sum_{k \in \mathcal{N}_{\text{rec}}^j} a^{(j,k)} \leq I^j.
\]

Formally, an FEL is any function $g$ that maps a DC’s inventory $I^j$ and its intermediate outputs $(b^{(j,k)})_{k \in \mathcal{N}_{\text{rec}}^j}$ to an allocation vector $(a^{(j,k)})_{k \in \mathcal{N}_{\text{rec}}^j} \in \mathbb{R}^{|\mathcal{N}_{\text{rec}}^j|}_+$ obeying the constraint above. The same function $g$ is applied independently for each DC $j$.

We evaluate two FELs:
\begin{align}
g_1(I^j, (b^{(j,k)})_{k \in \mathcal{N}_{\text{rec}}^j}) &= \left[[b^{(j,k)}]^+ \cdot \min\left\{1, \tfrac{I^j}{\sum_{i \in \mathcal{N}_{\text{rec}}^j} [b^{j,i}]^+} \right\}\right]_{k \in \mathcal{N}_{\text{rec}}^j} & (\text{Proportional Allocation}) \label{eq:proportional-allocation}\\
g_2(I^j, (b^{(j,k)})_{k \in \mathcal{N}_{\text{rec}}^j}) &= \left[I^j \cdot \tfrac{\exp(b^{(j,k)})}{1 + \sum_{i \in \mathcal{N}_{\text{rec}}^j} \exp(b^{j,i})}\right]_{k \in \mathcal{N}_{\text{rec}}^j} & (\text{Softmax}) \label{eq:softmax}
\end{align}

While FELs provide a principled way to guarantee valid actions, they are not the only approach. An alternative is to penalize constraint violations during training, encouraging the agent to learn to output feasible actions directly. However, this requires careful tuning of penalty coefficients and does not eliminate the need for projection at test time. In contrast, the explicit enforcement functions described above are simple to implement, perform well empirically across various inventory networks (see Sections~\ref{sec:vanilla-hdpo}--\ref{sec:gnn-section}), and eliminate the need for ad hoc constraint handling. 

Our actual implementations of these FELs may vary slightly from the definitions above. These modifications were introduced to enforce hard constraints in specific settings (e.g., full allocation at transshipment nodes), or to improve the optimization landscape by smoothing gradients or avoiding pathological behaviors. We describe these implementation variants in Appendix~\ref{appendix:fel-implementations}.

\subsection{Vanilla NN architecture design  
\label{appendix:vanilla-architecture}}

The Vanilla NN is a fully-connected multi-layer perceptron (MLP) that maps the full system state—flattened into a single long vector—to an output for each edge in the inventory network. We present a high-level depiction of the Vanilla NN in Figure~\ref{fig:vanilla nn}.

\begin{figure}
    \centering
    \includegraphics[scale=0.6]{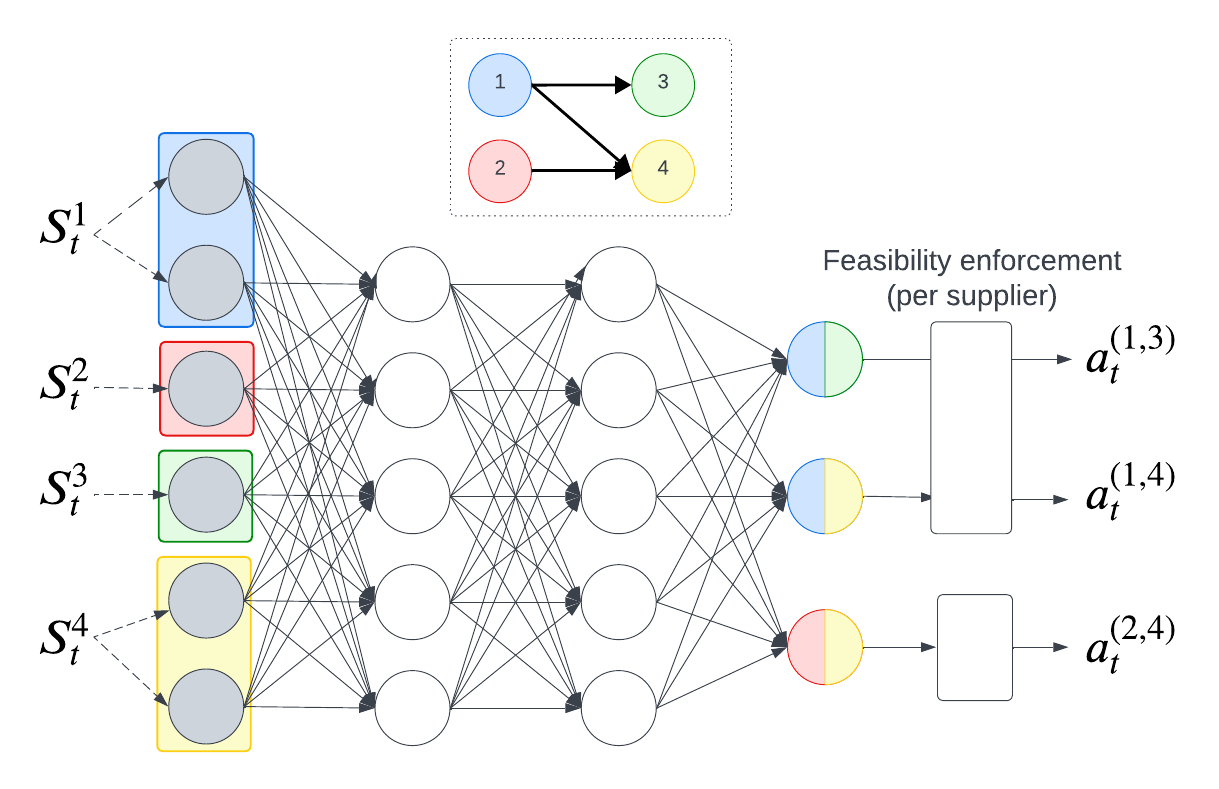}
    \caption{Overview of the Vanilla NN architecture in a setting with four locations. Location 1 supplies locations 3
and 4, while location 2 supplies location 4. For simplicity, we omit edges from the external supply node.}
    \label{fig:vanilla nn}
\end{figure}

\subsection{Proof of how GNNs capture the optimal policy structure in a serial system. \label{appendix:gnn-serial-proof} }

Here, we provide a constructive demonstration of how our proposed GNN architecture can represent the 
optimal policy for the setting of a serial network under stationary demand and a backlogged demand assumption 
(see setting S3 in Table~\ref{tab:opt-exp-settings-summary}). For a serial network with K echelons, an echelon-stock policy (see Equation \ref{eq:echelon-stock-for-serial} in Appendix~\ref{appendix:echelon-stock-policy}) is optimal. 
In what follows, we show how the GNN architecture can elegantly represent policies in this class with K-1 rounds of message-passing.

\edit{Edge embeddings have two coordinates: the first aims to track the echelon inventory position, and the second represents the message to be passed in the next round. They also track the \edit{echelon} base-stock levels of the receiving node, though we suppress this 
\edit{additional coordinate} for clarity in this construction. Let the inventory position at node $k$ be $X^k_t = I^k_t + \sum_{l=1}^{\bar{L}^k - 1} Q_t^l$. Set $\texttt{EmbedNode}(S^k_t) = X^k_t$ and $\texttt{EmbedEdge}(h^j, h^k, L^{(j,k)}) = (h^k, h^k)$, so that each edge initially receives two copies of the inventory position of its downstream node. Define $\texttt{UpdateNode}(h^k, z_{\textup{in}}, z_{\textup{out}}) = z_{\textup{out}}(2) - h^k$, \edit{so that the new embedding node of $k$, after a round of message passing, is $h^k + z_{\textup{out}}(2) - h^k = z_{\textup{out}}(2)$, i.e., the message received from the edge downstream of $k$ which it needs to send upstream}. 
Let $\texttt{UpdateEdge}(h^{(j,k)}, h^j, h^k) = (h^k, h^k - h^{(j,k)}(2))$. This update scheme ensures that, in each round, the second coordinate of the embedding acts as a message that carries the inventory position of a downstream node one step upstream. Node embeddings are not used for any direct decision-making; their sole purpose is to store and forward these messages. (Note that the store node receives $z^{K}_{\textup{out}} = (0, 0)$ in every round.) After $K-1$ rounds of message passing, the first entry of each $h^{(k-1,k)}$ equals the echelon inventory position at $k$, and a suitable $\texttt{Readout}$ function can implement the echelon stock policy at each $k$.}

\clearpage

\section{Description of baselines
\label{appendix:baselines}}

In this section, we describe the baselines and lower bounds used across all experimental settings. We highlight the modeling assumptions under which a known policy is optimal, and summarize the problem features that motivate the use of specific heuristic baselines. To streamline notation, we omit edge-specific superscripts in settings involving a single decision location. Table~\ref{tab:all-settings-summary} in Appendix~\ref{appendix:settings-summary} provides a summary of the structural features for each setting considered.

\subsection{Base stock policy (optimal for setting S1)
\label{appendix:base-stock-policy}}

A base-stock policy aims to maintain a target inventory position by ordering enough in each period to raise the current position to a fixed level. This simple structure is optimal in single-store settings with deterministic lead time, i.i.d. demand, linear underage and holding costs, and assuming all unmet demand is backlogged \citep{arrow1958studies} (setting S1).

We define the inventory position at time $t$ as
\begin{align}
\label{eq:inventory-position}
    X_t = I_t + \sum_{l=1}^{L - 1} Q_t(l),
\end{align}
where $L$ denotes the lead time from the outside supplier to the store, $I_t$ is the on-hand inventory, and $Q_t(l)$ denotes the quantity scheduled to arrive in $l$ periods, for $l = 1, \ldots, L{-}1$. The base-stock policy orders the amount needed to raise the inventory position to a constant target level $\hat{S} \in \mathbb{R}$:
\begin{align}
\label{eq:base-stock-policy}
    \pi(S_t) = \left( \hat{S} - X_t \right)^+.
\end{align}

When per-unit underage and holding costs are given by $p$ and $h$, and there are no ordering costs, the optimal base-stock level is
\begin{align}
\label{eq:base-stock-level}
    \hat{S} = \hat{F}^{-1} \left( \frac{p}{p + h} \right),
\end{align}
where $\hat{F}$ is the cumulative distribution of demand over $L + 1$ periods.

\subsection{Capped base stock (baseline for setting S2)
\label{appendix:capped-base-stock-policy}}

The Capped Base-Stock (CBS) policy extends the classical base-stock policy by imposing a cap on the maximum order size. It is defined by two parameters: a base-stock level $\hat{S}$ and an order cap $r$. The policy takes the form:
\begin{align}
\label{capped base stock}
    \pi(S_t) = \min \left\{ \left( \hat{S} - X_t \right)^+, r \right\},
\end{align}
where $X_t$ denotes the inventory position, as in Eq.~\eqref{eq:inventory-position}.

This policy class has been shown to perform particularly well under the structural features of setting S2, which involves a single store, lost sales, and Poisson demand.
In such settings, base-stock policies are no longer optimal for strictly positive lead times, and their performance typically deteriorates as the lead time increases. The optimal policy lacks a simple structure and may depend on the entire inventory pipeline, making exact computation intractable for most instances.

Nevertheless, CBS policies achieve strong empirical performance in this regime. In particular, \citet{xin2021understanding} demonstrate an average optimality gap of $0.71\%$ across the 32-instance benchmark introduced by \citet{zipkin2008old}. We adopt this benchmark as a reference in Section~\ref{subsec:hdrl-and-hdpo-vs-reinforce} and Appendix~\ref{appendix:single-store-lost-demand-exps}, where we compare CBS to our HDPO-based policies.

The lost-demand assumption is especially relevant in brick-and-mortar retail, where studies show that only 9–22\% of customers are willing to delay their purchase when an item is out of stock \citep{corsten2004stock}.

\subsection{Lower bound for setting S2 \label{appendix:one-store-lost-demand-lower-bound} }

We use the cost values reported by \citet{zipkin2008old} as a lower bound for benchmarking. These costs were obtained using value iteration on the exact dynamic program and are reported with two decimal places of precision.

Solving the dynamic program in this setting is computationally intensive, with complexity increasing rapidly in the lead time. As a result, \citet{zipkin2008old} were only able to compute exact costs for relatively short lead times and moderate problem sizes.

\subsection{Echelon stock (optimal for setting S3)
\label{appendix:echelon-stock-policy}}

This baseline corresponds to the optimal policy for Setting S3, which features a serial network topology, normally distributed demand, and a backlogged demand assumption. In this setting, the system consists of $K$ echelons, each comprising a single location. Locations are sequentially numbered from $1$ to $K$, starting from the upstream and progressing towards the downstream. Demand arises solely at the most downstream location (location $K$), which we refer to as the \textit{store}, and any unfilled demand is backlogged. Inventory costs, denoted as $h^k$, are incurred at each location, and there is an underage cost of $p$ per unit at the most downstream location. Each edge $(j,k)$ in the serial network has a lead time $L^{(j,k)} \in \mathbb{N}$. It is assumed that demand is i.i.d. over time.

The sequence of events is as follows: A central planner observes the state $S_t$ at time $t$ and jointly determines the order $a^{(0,1)}_t$ that the first location places with an external supplier with unlimited inventory and, for $k=2,\ldots,K$, specifies the quantity $a^{(k-1,k)}_t$ to transfer from location $k-1$ to location $k$, ensuring that it does not exceed the current inventory level at location $k-1$ (\ie, $a^{(k-1,k)}_t \leq I^{k-1}_t$).

We define the \emph{echelon inventory position} $Y^k_t$ of location $k$ at time $t$ as 
\begin{align}
\label{eq:echelon-position-definition}
    Y^k_t = \sum_{j=k}^{K} X^j_t,
\end{align}
with $X^k_t$ the inventory position of location $k$ as defined in \eqref{eq:inventory-position}. That is, it represents the sum over all inventory positions of downstream locations and its own. As shown by \citet{clark1960optimal}, the optimal policy $\pi$ for a serial system takes the form of an \emph{echelon base-stock policy}. Let $Y_t^k$ denote the echelon inventory position at node $k$, and $\hat{S}^1, \ldots, \hat{S}^K$ the corresponding echelon base-stock levels. The policy prescribes actions
\begin{align}
\label{eq:echelon-stock-for-serial}
    a^{(0,1)}_t &= [\hat{S}^1 - Y^1_t]^+ \\
    a^{(k-1,k)}_t &= \min\left\{ I^{k-1}_t,\; [\hat{S}^k - Y^k_t]^+ \right\} \quad \text{for } k = 2, \ldots, K,
\end{align}
which we denote compactly as $a^{(k-1,k)}_t = [\pi(S_t)]_k$ for all $k$.
We note that we will use the terms echelon base-stock levels and echelon-stock levels interchangeably. We refer to \eqref{eq:echelon-stock-for-serial} as an \emph{Echelon-stock policy}.

\subsection{Lower bound for setting S4
\label{appendix:transshipment-setting-lower-bound}}

We provide a lower bound for the setting in which a warehouse (indexed by $w$) operates as a transshipment center—i.e., it cannot hold inventory—and supplies multiple stores indexed by $1, \ldots, K$ under a backlogged demand assumption. Demand is i.i.d.\ across time but may be correlated across stores. Costs arise solely from holding and underage penalties at the stores. We slightly abuse notation by indexing the stores as $1, \ldots, K$ and the warehouse as $w$ to maintain consistency with the notation used in the proof of Theorem~\ref{theorem:1-round-message-passing-policy} in Appendix \ref{appendix: proof of main theorem}.

\citet{federgruen1984approximations} provide a clever method for obtaining a lower bound when demand across stores follows a joint normal distribution and store costs and lead times are identical. To apply their results, we assume all stores share the same underage cost $p$ and lead time $L$, while the transshipment center (warehouse) has lead time $L^w$. Let $\mu^k$ and $\sigma^k$ denote the mean and standard deviation of demand for store $k \in [K]$, and let $\Sigma$ denote the demand covariance matrix, where $[\Sigma]_{ij} = \textup{Cov}(\xi^i_1, \xi^j_1)$.

The \emph{echelon inventory position} at the warehouse is defined as
\begin{align}
\label{eq:echelon-stock-transshipment}
    Y^w_t = \sum_{k \in \mathcal{N}_+} X^k_t,
\end{align}
where $X^k_t$ denotes the inventory position of location $k$ as defined in \eqref{eq:inventory-position}. That is, $Y^1_t$ represents the sum of the inventory position across stores plus its own.

The relaxation allows inventory to flow from stores to the warehouse and across stores. In this relaxed setting, the Bellman equation can be written in terms of the echelon stock $Y^w_t$ alone, leading to an optimal policy of the form
\begin{align}
\label{eq:echelon-base-stock-policy-transshipment}
    a_t^{(0,w)} = [\hat{S}^w - Y^w_t]^+,
\end{align}
where $\hat{S}^w \in \mathbb{R}_+$ is the echelon base-stock level and $a_t^{(0,w)}$ denotes the order placed by the warehouse to the outside supplier (node $0$).

Since we do not consider purchase costs, $\hat{S}^w$ is given analytically by a newsvendor-type formula:
\begin{align}
\label{eq:echelon-base-stock-level-transshipment}
    \hat{S}^w = F_G^{-1}\left( \frac{p}{p + h} \right),
\end{align}
where $G$ is a normal distribution with:
\begin{align*}
    \hat{\mu}_G &= (L^w + L + 1) \sum_{k = 1}^K \mu^k, \\
    \hat{\sigma}_G &= \sqrt{L^w \sum_{i=1}^K \sum_{j=1}^K [\Sigma]_{ij} + (L + 1)\left(\sum_{k=1}^K \sigma^k \right)^2}.
\end{align*}

Letting $\hat{s} = \frac{\hat{S}^w - \hat{\mu}_G}{\hat{\sigma}_G}$ denote the standardized base-stock level, the resulting lower bound on per-period cost is:
\begin{align*}
    p(\hat{\mu}_G - \hat{S}^w) + (p + h) \hat{\sigma}_G \left( \hat{s} \Phi(\hat{s}) + \phi(\hat{s}) \right),
\end{align*}
where $\Phi(\cdot)$ and $\phi(\cdot)$ are the CDF and PDF of the standard normal distribution, respectively.

Note that the bound is not normalized by the number of stores $K$.

\subsection{Generalized Newsvendor Policies (baselines for setting S5) \label{appendix:generalized-newsvendor}}

These heuristics utilize a distribution forecaster trained offline, and aim to "raise" inventory levels to a level dictated by a given quantile. Denoting the distribution of the sum of the next $L+1$ demands, given lead time $L$ and forecasting information $\mathcal{F}_t$, as $H(L, \mathcal{F}_t)$, a \textit{generalized newsvendor policy} $\pi$ takes the form
\begin{equation}
\label{eq: quantile-policy}
    \pi(S_t) = \left(H(L, \mathcal{F}_t)^{-1}(\tau^{\pi}) - X_t \right)^+,
\end{equation}

where $X_t$ is the inventory position (see Equation~\eqref{eq:inventory-position}), and $\tau^{\pi}$ represents a time-invariant quantile, with the flexibility to depend on \edit{instance parameters} (e.g., underage and holding costs) to accommodate heterogeneity across products.

The rationale for adopting this class of policies stems from the notion that, assuming that unmet demand is backlogged, the inventory on-hand before the order placed at time $t$ arrives can be expressed as $I_{t + L} = X_t - \sum_{l=0}^{L - 1}\xi_{t + l}$. Hence, under the backlogged demand assumption, the action minimizing the expected cost $L$ periods into the future can be obtained as a quantile of the distribution of the sum of the next $L+1$ demands. Adjusting this quantile has the potential to account for variations in product economics (underage and holding costs) and underlying model misspecifications.

We trained a NN to forecast multiple quantiles for this metric using historical features. This allows us to approximate $H(L, \mathcal{F}_t)$ and hence estimate the specified quantiles. Additional information about the implementation and performance evaluation of the quantile forecaster can be found in Appendix \ref{appendix: quantile-forecaster}.

We considered several generalized newsvendor policies, which are equipped with the same quantile forecaster and have access to the problem primitives in each scenario. They only differ in the heuristic 
choice of the quantile $\tau^{\pi}$:

\begin{itemize}
    \item Newsvendor: Places orders up to the newsvendor quantile, given by $\frac{p}{p + h}$.
    \item Fixed Quantile: Considers a common quantile $\tau^{\pi}$ for all scenarios, with the quantile being a trainable parameter.
    \item Transformed Newsvendor: Utilizes a NN to flexibly learn a map from the newsvendor quantile $\frac{p}{p + h}$ to a new quantile.
\end{itemize}

For the last two heuristics, we optimize the parameters defining the policy by integrating the forecaster (with fixed parameters)  into our simulator. Employing linear interpolation across the predicted quantiles enables us to optimize the parameters in a differentiable manner. When predicting values outside the range of predicted quantiles, we employ linear interpolation using the slope of the nearest quantile range. For example, to predict an extremely low quantile, we interpolate with the same slope as the line connecting the first and second quantiles. It is worth noting that the vast majority of predicted quantiles fall within the range of predicted quantiles. Finally, it is crucial to highlight that our forecaster is not trained at this stage.

\subsection{Upper bound on profit for settings with realistic demand (S5, S7 and S9) \label{appendix:bound-just-in-time} }

We compute a loose upper bound on achievable profit using a \textit{Just-in-time} oracle policy that has perfect foresight of future demand. In the single-store case, this oracle places orders exactly $L$ periods in advance, ensuring that inventory arrives just in time to satisfy each week's demand. When a store is supplied by a DC, the DC places orders with enough anticipation to cover both its own lead time and the lead time from the DC to the store. In both cases, the oracle ensures that inventory arrives exactly when needed, with no holding or stockout costs incurred.

Because the oracle perfectly matches demand without incurring inventory costs, its profit equals the total revenue from satisfying all demand at prevailing prices. This serves as a natural upper bound for evaluating learned policies.

\clearpage

\section{Implementation details and numerical experiments 
\label{appendix:implementation-and-experiments}}

This appendix provides detailed implementation and experimental specifications for all results presented in the main text. Appendix~\ref{appendix:settings-summary} summarizes the key structural features that define each inventory control setting, including network topology, demand model, and assumptions regarding unmet demand. Appendices~\ref{appendix:realistic-demand-dataset} and~\ref{appendix:many-warehouse-many-store-placement} describe the realistic demand data and spatial configurations used in experiments involving many-warehouse many-store networks.

We then detail the architectural and implementation choices underpinning our numerical results. Appendix~\ref{appendix:fel-implementations} introduces implementation variants of the feasibility enforcement layers used to ensure valid allocations. Appendix \ref{appendix:gnn-architecture} gives implementation details for the GNN architecture. Appendix~\ref{appendix:nn-inputs} specifies the input features provided to the neural policies.
Next, we outline the global experimental setup (Appendix~\ref{appendix:global-settings}) and the construction of the quantile forecaster used in baseline heuristics (Appendix~\ref{appendix: quantile-forecaster}).

The remaining appendices report numerical specifications and results for each group of experiments. Appendix~\ref{appendix:optimal-exps} provides details for the experiments in Section~\ref{subsec:optimal_exps}, where we benchmark against known or tightly bounded optima (settings S1–S4). Appendix~\ref{appendix:censored-demand} assesses the impact of operating with censored demand data, outlines potential approaches to address it, and justifies the assumptions made in our work. Appendix~\ref{appendix:vanilla-hdpo-sample-efficiency} evaluates the sample efficiency of the Vanilla HDPO policy in a single-store setting (setting S2), while Appendix~\ref{appendix:one-store-real-setting} presents results for single-store problems with realistic demand (setting S5). Appendix~\ref{appendix:sample-efficiency-results} provides specifications and detailed results for the large-scale sample efficiency experiments in Section~\ref{sec:sample-efficiency-results}, covering settings S3, S4, and S6–S9. Finally, Appendix~\ref{appendix:value-of-experiments} describes the experimental design for the ablation studies presented in Sections~\ref{sec:value-of-weight-sharing},\ref{sec:value-of-message-passing}, and \ref{sec:value-of-flexibility}.

{\bf Terminology and reporting conventions.}
Below we define common terms used throughout our experimental setup and architectural descriptions
\begin{itemize}
    \item Hyperparameter: A parameter that is set before the learning process begins and affects the model's behavior or performance.
    \item NN architecture class: Overall topology and organization of the network, which determines how information flows through the network and how computations are performed.
    \item Module: A reusable subnetwork or component within a neural architecture that performs a specific function (e.g., embedding, message passing, readout).
    \item Run: A single execution of HDPO under a fixed setting, meta-instance and set of hyperparameters.
    \item Epoch: A complete pass through the entire training dataset during the training process.
    \item Batch: A subset of the training dataset used for updating model parameters.
    \item Batch size: The number of scenarios included in each batch.
    \item Gradient step: A step taken to update the model's parameters based on the gradient computed over one batch.
    \item Learning rate: A hyperparameter that determines the step size of the gradient step.
    \item Hidden layer: A layer in a NN that sits between the input and output layers and performs intermediate computations.
    \item Unit/neurons: Fundamental unit of a NN that receives input, performs a computation, and produces an output.
\end{itemize}

\subsection{Settings summary \label{appendix:settings-summary}}

Table~\ref{tab:all-settings-summary} provides a comprehensive overview of all experimental settings used throughout the paper. Each setting is defined by its network topology, unmet demand assumption, demand distribution, and structural constraints. The final column lists the sections and appendices where each setting is introduced or analyzed in detail. We use consistent setting identifiers (S1--S11) across figures and tables to aid navigation and comparison.

\begin{table}[h!]
\centering
\caption{Summary of all experimental settings used in the paper.\newline
\footnotesize{Abbreviations: OWMS = One-warehouse–many-stores; MWMS = Many-warehouses–many-stores; MS = Multiple independent stores.}}
\label{tab:all-settings-summary}
\begin{tabular}{@{}llllll@{}}
\toprule
\textbf{ID} &
\shortstack{\textbf{Network} \\ \textbf{Structure}} &
\shortstack{\textbf{Unmet} \\ \textbf{Demand} \\ \textbf{Assumption}} &
\shortstack{\textbf{Demand} \\ \textbf{Dist.}} &
\shortstack{\textbf{Structural} \\ \textbf{Constraints}} &
\textbf{Reference} \\ \midrule

S1 & Single-store & Backlogged & Normal & -- & \ref{subsec:optimal_exps}, \ref{appendix:single-store-backlogged-exps} \\
S2 & Single-store & Lost & Poisson & Discrete allocations & \ref{subsec:optimal_exps}, \ref{subsec:hdrl-and-hdpo-vs-reinforce}, \ref{appendix:single-store-lost-demand-exps}, \ref{appendix:censored-demand} \\
S3 & Serial & Backlogged & Normal & -- & \ref{subsec:optimal_exps}, \ref{sec:sample-efficiency-results}, \ref{appendix:serial-system-exps}, \ref{appendix:sample-efficiency-serial} \\
S4 & OWMS & Backlogged & Normal & Transshipment warehouse & \ref{subsec:optimal_exps}, \ref{sec:sample-efficiency-results}, \ref{appendix:transshipment-warehouse-exps}, \ref{appendix:sample-efficiency-transshipment} \\
S5 & Single-store & Lost & Realistic & -- & \ref{sec:vanilla-hdpo-realistic}, \ref{appendix:one-store-real-setting} \\
S6 & OWMS & Lost & Normal & -- & \ref{sec:sample-efficiency-results}, \ref{appendix:sample-efficiency-owms-synthetic} \\
S7 & OWMS & Lost & Realistic & -- & \ref{sec:sample-efficiency-results}, \ref{appendix:sample-efficiency-owms-realistic} \\
S8 & MWMS & Lost & Normal & -- & \ref{sec:sample-efficiency-results}, \ref{appendix:sample-efficiency-mwms-synthetic} \\
S9 & MWMS & Lost & Realistic & -- & \ref{sec:sample-efficiency-results}, \ref{sec:value-of-flexibility}, \ref{appendix:sample-efficiency-mwms-realistic}, \ref{appendix:value-of-flexibility} \\
S10 & MS & Lost & Poisson & Discrete allocations & \ref{sec:value-of-weight-sharing}, \ref{appendix:value-of-weight-sharing} \\
S11 & OWMS & Lost & Normal & \shortstack{Transshipment warehouse \\ + signaled demand mean} & \ref{sec:value-of-message-passing}, \ref{appendix:value-of-message-passing} \\
\bottomrule
\end{tabular}
\end{table}

\subsection{Realistic demand dataset \label{appendix:realistic-demand-dataset} }

In this section, we provide details on the dataset used for settings with realistic demand, including filtering criteria and data processing steps. We utilize sales data from the \textit{Corporación Favorita Grocery Sales Forecasting} competition \citep{kagglefavorita} hosted on Kaggle. Corporación Favorita, one of Ecuador's largest grocery retailers, operates stores in various formats, including hypermarkets, supermarkets, and convenience stores. The dataset includes around 200,000 daily sales time series for over 4,000 products across 54 physical store locations, spanning the years 2013 to 2018. A key challenge in using this dataset lies in addressing complex nonstationary patterns in the sales data.

We aggregate sales data to a weekly level, treating each time series—corresponding to a particular product at a given store—as an independent demand trace. We focus on a 171-week window from January 2013 to early April 2016, excluding data from mid-April 2016 onward to avoid distortions linked to the major 2016 Ecuador earthquake. Figure~\ref{fig: samples-sales} shows examples of weekly sales for five sample traces and their Gaussian smoothed counterparts, along with aggregate sales across the 32,768 demand traces used in our experiments. The figure highlights the large heterogeneity in both magnitude and seasonal patterns across traces.

\begin{figure}
    \centering
    \includegraphics[scale=0.6, trim={0.8cm 0 0 0}]{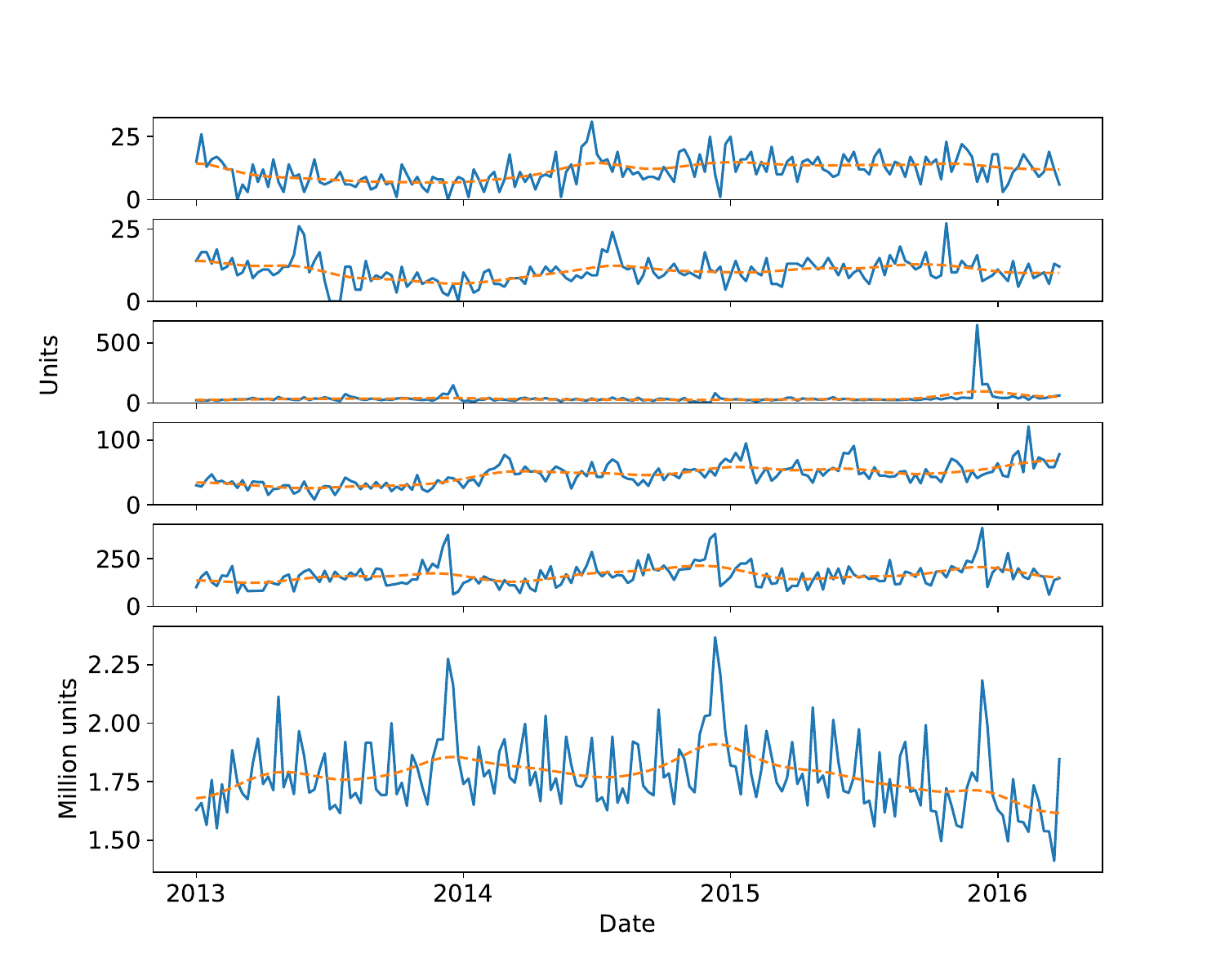}
    \caption{Weekly sales (blue line) and Gaussian smoothing (orange line) for 5 sample paths (first 5 plots) and summed across all 32,768 sample paths (last plot). Data corresponds to a filtered sub-sample of the dataset provided in the \textit{Corporación Favorita Grocery Sales Forecasting} competition.}
    \label{fig: samples-sales}
\end{figure}

As an initial pre-processing step, we discard demand traces that show no sales in the first 16 periods, to ensure meaningful signal in the early part of the trace. Since the dataset includes only sales (rather than true demand), and sales may be censored by stockouts, we attempt to reduce the influence of censoring by filtering out heavily affected traces. Specifically, we exclude any demand trace with zero sales in more than 10\% of the weeks, treating extended zero-sale periods as potential stockouts. Additionally, we remove all perishable products, which are likely to follow different replenishment schedules. From the remaining filtered set, we randomly sample 32,768 demand traces for use in our experiments. While we acknowledge that raw sales data may be an imperfect proxy for demand, we believe this real-world dataset still offers a challenging and realistic testbed for data-driven inventory control.

For single-store settings, we drop all store identifiers and treat each of the 32,768 selected traces as representing an independent product-level demand path. Although multiple traces may originate from the same product across different stores, we treat them independently. For multi-store settings, by contrast, we retain store identifiers and group data by product. A demand trace in this case corresponds to the weekly demand of a single product across all selected stores. We identify 21 high-volume stores and only retain product-level traces for which the demand series is valid across every one of these 21 stores—ensuring data consistency across locations.

Finally, upon inspecting the aggregate sales trends in Figure~\ref{fig: samples-sales}, we observed a clear spike around late December, indicating the strong impact of Christmas. To help the policy account for such demand surges, we include the number of weeks to Christmas as an input feature.

\subsection{Realistic location placement for many-warehouse many-store settings}
\label{appendix:many-warehouse-many-store-placement}
\edit{
Since actual network data is unavailable for scenarios S8 and S9, we employ a heuristic approach to configure their network topologies. The resulting network structures for these scenarios are illustrated in Figures \ref{fig:synthetic_MWMS_network_example} and \ref{fig:realistic_MWMS_network_example}, respectively.

\textbf{Store Placement:} In setting S8, which incorporates synthetic demand, store locations are sampled uniformly at random across a 2D plane. In setting S9, which incorporates realistic demand, stores are positioned at actual city coordinates from the Favorita dataset on the Ecuador map.

\textbf{Warehouse Positioning:} We employ a K-means clustering algorithm with the number of centroids equal to the number of warehouses. Each warehouse is positioned at the location of a resulting centroid. We define each cluster as being \textit{assigned} to its corresponding warehouse, a definition that will be used in the subsequent edge creation step.

\textbf{Edge Creation:} We adopted different approaches for each setting due to the geographical clustering of stores in S9, which would result in an overly sparse network if the same criteria were applied to both settings. We adjusted parameters to achieve networks that are relatively dense while avoiding complete connectivity.
We define $d_{\text{max}}^{w}$ as the maximum distance from warehouse $w \in \mathcal{N}_{\text{dc}}$ to any store $s \in \mathcal{N}_{\text{st}}$ in its assigned cluster. For S8, we use L2 norm distance and connect warehouse $w \in \mathcal{N}_{\text{dc}}$ to store $s \in \mathcal{N}_{\text{st}}$ when $d(w,s) < 1.3 \times d_{\text{max}}^{w}$. For S9, we use haversine distance and connect when $d(w,s) < 1.8 \times \max_w d_{\text{max}}^{w}$.

\textbf{Lead time assignment:} We define a lead time range $[L_{\min}^w, L_{\max}^w]$ for each warehouse $w$, which specifies the minimum and maximum lead times that will depend on the distance to stores. These ranges are specified in Appendix~\ref{appendix:sample-efficiency-mwms-synthetic} and Appendix~\ref{appendix:sample-efficiency-mwms-realistic} for S8 and S9 respectively. 

The lead time for each warehouse-to-store edge, $L_{s}^w$, is calculated as:
\begin{equation}
L_s^w = L_{\min}^w + (L_{\max}^w - L_{\min}^w) \times \min\left(\frac{d(w,s)}{d_{\text{max}}^{w}}, 1.0\right)
\end{equation}

This formula linearly interpolates lead times based on distance: stores closer to the warehouse receive lead times closer to $L_{\min}^w$, while stores at maximum distance $d_{\text{max}}^{w}$ or beyond receive $L_{\max}^w$.
}

\textbf{Other Parameter Settings:} We assign each warehouse a fixed holding cost and fixed edge costs to all connected stores (meaning they do not depend on lead time). This assumes that variable costs per unit primarily depend on handling and packaging, with distance-based transportation costs being negligible. We believe this is a reasonable approximation for a small country like Ecuador.
The procedures for sampling demand parameters in scenario S8 follow those of other scenarios (see Appendix \ref{appendix:sample-efficiency-mwms-synthetic}).

\begin{figure}[htbp!]
\centering
\captionsetup[subfigure]{justification=centering}
\begin{subfigure}[t]{0.49\textwidth}
  \centering
  \includegraphics[width=\linewidth,height=9cm,keepaspectratio]{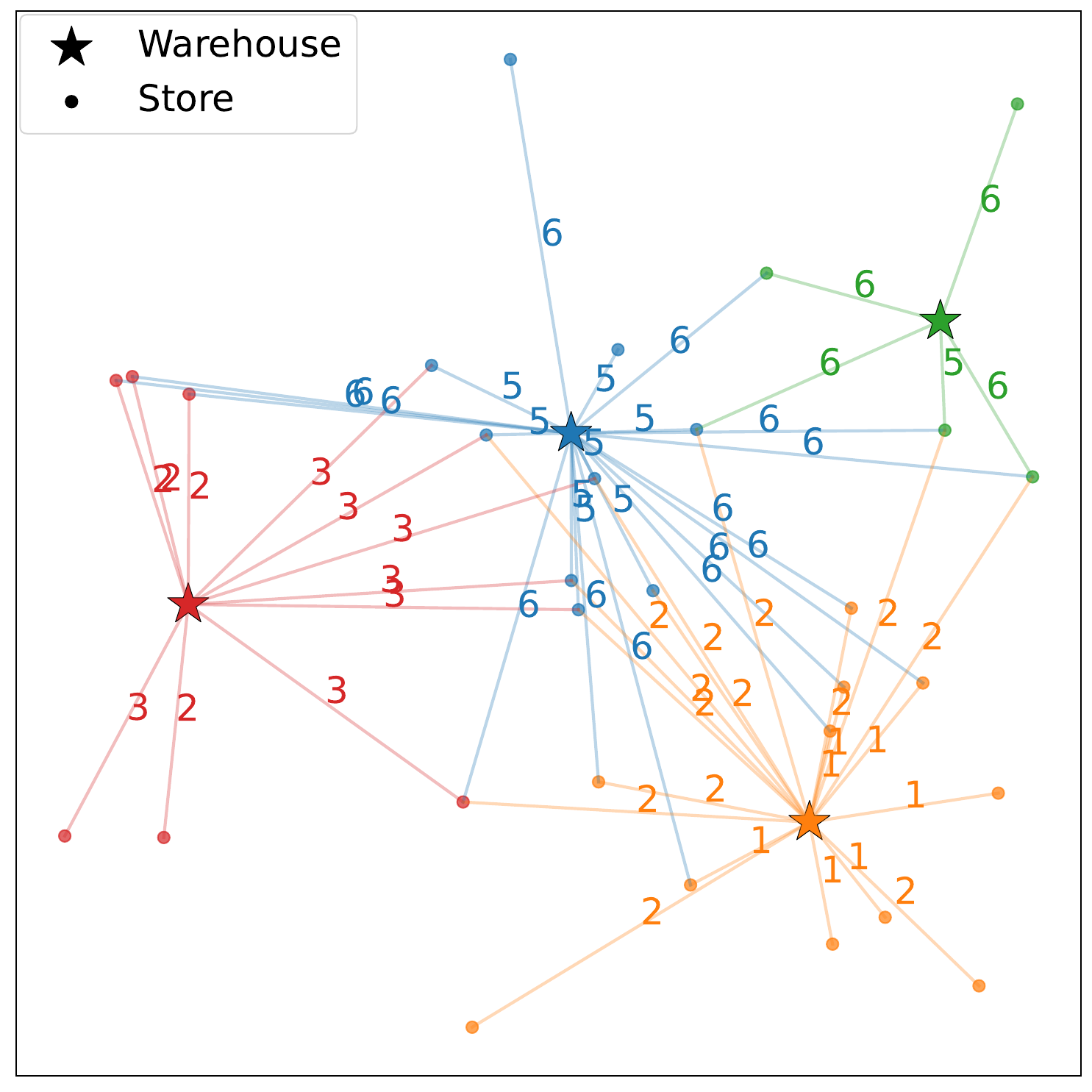}
  \caption{Example network configuration for Setting S8 with 4 warehouses and 30 stores.}
  \label{fig:synthetic_MWMS_network_example}
\end{subfigure}
\hfill
\begin{subfigure}[t]{0.49\textwidth}
  \centering
  \includegraphics[width=\linewidth,height=8.11cm,keepaspectratio]{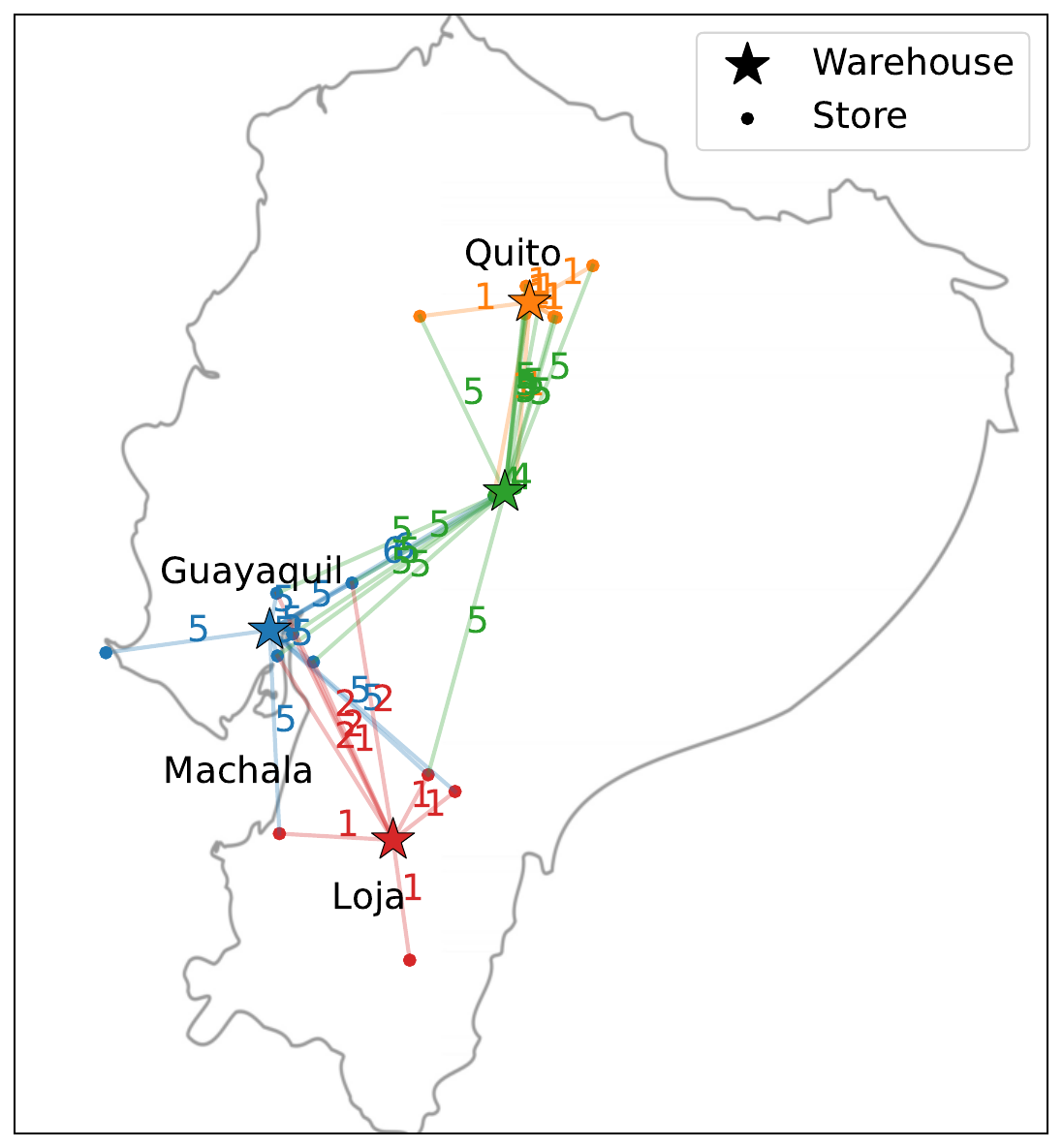}
  \caption{Example network configuration for Setting S9 with 4 warehouses and 21 stores.}
  \label{fig:realistic_MWMS_network_example}
\end{subfigure}
\caption{Examples of network configurations used for Settings S8 (left) and S9 (right). Each color indicates the cluster group associated with each warehouse. Each number above an edge denotes the lead time from the corresponding warehouse to the store.}
\label{fig:network-configurations}
\end{figure}

\subsection{Implementation Variants of Feasibility Enforcement Layers}
\label{appendix:fel-implementations}

Here we elaborate on the implementation of FELs used in our experiments. We designed several variants of the FELs introduced in Appendix~\ref{appendix:feasibility-enforcement}, tailoring them to enforce constraints in specific settings—such as ensuring full inventory allocation at transshipment warehouses—and to improve the optimization landscape.

Although HDPO is generally reliable, we observed that it occasionally converged to poor local minima in some instances. We believe certain FEL modifications helped alleviate this issue by shaping the optimization landscape more favorably and preventing gradients from vanishing.

For the GNN policy architecture, we use the same FEL across all settings—$g_{1b}$, described below—except when transshipment warehouses are present, in which case we apply a full-allocation variant. For the Vanilla architecture, we experimented with several FELs, but always used a consistent choice within each setting. That is, for each architecture-setting pair, the same FEL is applied across all instances.

We describe the FEL variants used below:

\begin{itemize}
    \item $g_{1a}$ (\textbf{Full Allocation Proportional}):  
    This variant ensures that a DC allocates \emph{all} of its on-hand inventory:
    \begin{equation}
        g_{1a}(I^j, (b^{(j,k)})_{k \in \mathcal{N}_{\text{rec}}^j}) = \left[ [b^{(j,k)}]_+ \cdot \frac{I^j}{\sum_i [b^{j,i}]_+} \right]_{k \in \mathcal{N}_{\text{rec}}^j}
        \label{eq:full-proportional}
    \end{equation}
    This is used in settings involving transshipment warehouses, where holding inventory is infeasible and full allocation must be enforced.

    \item $g_{2a}$ (\textbf{Full Allocation Softmax}):  
    This variant also ensures that all of a DC’s inventory is allocated and is defined as:
    \begin{equation}
        g_{2a}(I^j, (b^{(j,k)})_{k \in \mathcal{N}_{\text{rec}}^j}) = \left[ I^j \cdot \frac{\exp(b^{(j,k)})}{\sum_i \exp(b^{j,i})} \right]_{k \in \mathcal{N}_{\text{rec}}^j}
        \label{eq:full-softmax}
    \end{equation}

    \item $g_{1b}$ (\textbf{DC-Bid Proportional Allocation}):  
    A variant of $g_1$ that includes an additional logit $b^{j,j}$ for each DC $j$, representing the DC's own “bid” to retain inventory. This output is appended to the logits sent to the FEL and included in the normalization, even though its corresponding allocation $a^{j,j}$ is discarded after the FEL. This design helped the policy avoid pathological cases where increasing allocation toward a DC simply triggered more downstream flow without allowing the DC to retain inventory. The FEL is equivalent to Equation~\eqref{eq:proportional-allocation}, applied over the augmented vector $\left((b^{(j,k)})_{k \in \mathcal{N}^j_{\text{rec}}}, \{b^{j,j}\}\right)$. We detail the necessary corresponding adjustments for the GNN architecture in Appendix \ref{appendix:gnn-architecture}.

    \item $g_{2b}$ (\textbf{Extended Softmax}):  
    A modification of $g_2$ designed for Vanilla architectures. In addition to applying softmax at standard DCs, we also apply it to the outside supplier. For the outside supplier, we replace the inventory on hand $I^j$ in Equation~\eqref{eq:softmax} with a scalar upper bound on feasible order quantities. The application remains unchanged for DCs. This approach was only used for the Vanilla NN and improved training stability in some settings.
\end{itemize}

We provide specific implementation choices for each experiment in the corresponding subsections.

{\bf Additional Layers Used.}
For FELs that treat the intermediate outputs $(b^{(j,k)})$ as allocation amounts (\ie $g_{1a}$ (Full Allocation Proportional), $g_{1b}$ (DC-Bid Proportional), and $g_1$ (Proportional Allocation)), we apply a softplus transformation followed by the addition of a small positive constant before passing the values to the FEL. This ensures that all inputs to the allocation function are strictly positive, preventing negative allocations while maintaining differentiability. We use softplus instead of ReLU to avoid regions with zero gradient, which can impede optimization. The bias term further shifts the outputs away from zero, helping to prevent initialization in flat regions of the loss landscape. The precise values used are detailed in Appendix \ref{appendix:global-settings}.

\subsection{Implementation details for GNN architecture 
\label{appendix:gnn-architecture}}

Here, we provide further details on our specific implementation of the GNN architecture described in Algorithm \ref{alg:samp_gnn}.

\begin{itemize}
\item \textbf{Outside supply representation:} Rather than using the \texttt{EmbedNode} module for the outside supply node, we represent it with a fixed zero-vector of appropriate dimensions. This embedding remains unchanged during message-passing but serves as input for edge updates.
\item \textbf{Demand region representation:} To indicate that a node faces demand, we represent each "demand region" as a node with a zero-vector embedding and create an edge from each store to its corresponding demand region. The demand region's embedding remains fixed during message-passing but provides input for edge updates, while the edge embeddings themselves are updated during message-passing. The \texttt{Readout} module is not applied to these edges.
\item \textbf{State padding:} The \texttt{EmbedNode} module creates node embeddings $h^k$ using fixed-size vector inputs. However, location states can have different dimensionalities for several reasons: i) stores face demands while DCs do not, ii) locations can have different lead times, resulting in outstanding orders with different dimensions, and iii) some parameters (such as underage cost) may apply only to certain node types. To address this, we pad lower-dimensional vectors with zeros for each state component that can cause dimensional differences, ensuring uniform dimensionality across all locations.
\item \textbf{Warehouse "bidding":} In Appendix \ref{appendix:fel-implementations}, we explain that in practice we use a "DC-Bid Proportional Allocation" FEL, in which each DC places its own "bid" to retain inventory. To implement this, we create a self-loop edge for each DC (where the origin and destination nodes are the same DC) and set its lead time to 0. These edges participate in message-passing, and their embeddings are updated during message-passing.
\end{itemize}

\subsection{Input features for neural policies \label{appendix:nn-inputs}}

This section specifies the input features provided to the neural policies. These inputs fall into three categories: (i) inventory state and pipeline, (ii) forecasting and contextual features, and (iii) instance parameters. While both the Vanilla and GNN policies consume the same high-level inputs, the encoding and representation differ by architecture. At the end of this section, we provide a detailed specification of the inputs and outputs of the GNN modules.

{\bf Inventory state and pipeline.}

All policies receive the on-hand inventory $I^k_t$ and the outstanding orders $Q^k_t \in \mathbb{R}_+^{\bar{L}^k - 1}$ for each node $k$, where $\bar{L}^k$ denotes the maximum incoming lead time. This captures the inventory pipeline under deterministic lead times.

{\bf Forecasting and contextual features.}
These features are included only in settings with realistic demand. For each location $k$, we input a fixed-length \emph{demand lookback window} of recent demand realizations, typically covering the past 16 periods. We also include the number of days remaining until the next Christmas, which captures the strong seasonal fluctuations observed in the data. These features are omitted in synthetic experiments, where demand is stationary and independent over time.

{\bf Instance parameters.}
Each node $k$ is associated with static parameters such as the holding cost $h^k$, underage cost $p^k$, and maximum incoming lead time $\bar{L}^k$. For each edge $e = (j,k)$, we provide the shipment cost $\lambda^e$ and lead time $L^e$.

These instance-level parameters are always provided to the GNN architecture, even when they do not vary across scenarios, as doing so allows it to learn across locations. For the Vanilla NN, instance parameters are only included when they vary across scenarios. We assessed passing the instance parameters as input in such settings. As performance was similar, we decided not to include them.

{\bf Omitted features.}
For settings with realistic demand, although product-specific attributes such as category could potentially enhance performance, we intentionally omit them to simplify the development and interpretation of baseline heuristics used for comparison. This ensures that all policy classes operate under comparable information assumptions. Additionally, we do not consider any store-specific features. This follows the intuition that all the relevant information can be inferred from the demand time series and parameters such as costs and lead times.

\subsection{Global experimental settings \label{appendix:global-settings}}

This section details implementation and design choices that are consistent across most experiments. Deviations are specified in the setting-specific subsections.

{\bf Differentiable simulator implementation.}
We developed a differentiable simulator using PyTorch and conducted all experiments on an NVIDIA A40 GPU with 48GB of memory. Even though the pseudo-code for HDPO shown in Algorithm \ref{alg:hdpo_full} shows scenarios being processed sequentially, in practice we implemented an efficient parallel computation scheme to expedite the training process. For a given mini-batch of $H$ scenarios, we simultaneously executed the forward pass first, followed by the backward pass, across all scenarios. To achieve this, we utilized an initial mini-batch "state" matrix denoted as $\tilde{S}_1^H$. This matrix was obtained by stacking the initial states $\bar{S}_1^h$ for each scenario $h \in H$. At each time period, we input the matrix $\tilde{S}_t^H$ into the NN, enabling us to obtain the outputs for every scenario in a highly parallelizable manner. Subsequently, we computed the costs $c_t$ and updated the mini-batch state matrix $\tilde{S}^H_{t+1}$ through efficient matrix computations. This approach allowed PyTorch to efficiently estimate the gradients during the backward pass. 

{\bf Data Splits and Evaluation Windows.} Each experiment uses three data splits: train, dev, and test. For settings with synthetic demand, each split contains 32{,}768 independently generated demand scenarios, unless otherwise noted. Training and dev episodes span 50 and 100 periods, respectively, with evaluation based on the final 20 and 40 periods. Test episodes span 5{,}000 periods, with costs computed over the final 2{,}000 to approximate steady-state performance.

For settings involving realistic demand, we evaluate performance over 79 periods for training (periods 33–111), 21 periods for development (periods 121–141), and 21 periods for testing (periods 151–171). Each evaluation window is preceded by a 16-period warm-up phase and a 16-period initialization phase used to construct lagged input features. For example, development performance is computed over periods 121–141, following a warm-up on periods 105–120 and initialization over periods 89–104. This setup ensures there is no overlap between train, dev, and test windows and that reported costs reflect behavior after a realistic adaptation period.

{\bf Training Protocol.} Training was halted if performance on the dev set did not improve for 500 epochs or reached 20,000 epochs. The training set is used exclusively for updating model weights. Early stopping is used to select the final model for evaluation. During training, we monitor performance on the dev set at regular intervals and save the weights of the model that achieves the lowest dev loss. After training concludes, this model—frozen at its best dev-set performance—is used both for model selection and for final evaluation on the test set.

{\bf Performance Evaluation.} All reported performance metrics are computed on a clean test set that is held out throughout training and model selection. To reduce sensitivity to the initial state, we discard a fixed number of warm-up periods before computing average costs, as described above.

{\bf Initialization of inventory state.}
We set inventory on hand $I^k_1$ and every entry of $Q^k_1$ to 0 for every location $k \in \mathcal{N}_{+}$.

{\bf Fixed hyperparameters.} We use the Adam optimizer with PyTorch's default parameters, $(\beta_1, \beta_2) = (0.9, 0.999)$.

{\bf Hyperparameter Tuning.}
We found our method to be relatively robust to the choice of hyperparameters, especially in single-location settings. Initial experiments explored various combinations of learning rates, batch sizes, hidden layer counts, and layer widths. For each architecture, we selected a reasonable set of hyperparameters that performed well across early settings. In multi-location settings, we occasionally performed ad-hoc tuning using small grid searches to refine performance.

{\bf Other fixed parameters.} In settings with realistic demand, we use a demand lookback window of length 16; this history is omitted in synthetic-demand settings.

\subsubsection{GNN default hyperparameters. \label{appendix:gnn-hyperparams} }

The values in Table \ref{tab:gnn-defaults} summarize the default hyperparameter configuration used for all GNN-based architectures, including the Decentralized NN, Single-warehouse NN, and Separate-weights NN. These variants typically use a subset of the full set of modules listed in the table, in which case the listed hyperparameters apply only to the modules present in the architecture. The defaults are applied across settings unless explicitly overridden. For instance, when tuning a specific hyperparameter, we replace only the corresponding default. As noted in Section~\ref{appendix:fel-implementations}, when using the $g_{1b}$ (DC-Bid Proportional) FEL, we apply a softplus transformation followed by a small positive bias to the outputs of the \texttt{Readout} module before passing them to the feasibility enforcement layer. This ensures non-negativity and helps avoid flat regions during early training. 
Demand lookback windows are only used in settings involving realistic demand.

\begin{table}[h!]
\centering
\caption{Default hyperparameters used for GNN-based architectures. Sets \{·\} indicate grid search values. Exceptions are noted below.}
\label{tab:gnn-defaults}
\begin{tabular}{@{}ll@{}}
\toprule
\textbf{Component} & \textbf{Default Setting} \\
\midrule
Batch size & 1024 \\
Learning rates & $\{10^{-2}, 10^{-3}, 10^{-4}\}$ \\
Softplus bias & 5 \\
FEL & $g_{1b}$ (DC-Bid Proportional) \\
\texttt{EmbedNode} layers & 2 (width 32) \\
\texttt{EmbedEdge} layers & 2 (width 32) \\
\texttt{UpdateNode} layers & 2 (width 32) \\
\texttt{UpdateEdge} layers & 2 (width 32) \\
\texttt{Readout} layers & 2 (width 32) \\
Initialization & Zero \\
\bottomrule
\end{tabular}
\end{table}

\vspace{1em}
\noindent
\textbf{Exceptions.} The default hyperparameters in Table~\ref{tab:gnn-defaults} apply to all GNN-based architectures unless otherwise specified. The following deviations are made in particular settings:
\begin{itemize}
    \item \textbf{Softplus bias.} In settings with realistic demand (S7 and S9), we reduce the softplus bias from 5 to 1 to improve numerical stability.
    \item \textbf{FEL choice.} In settings with a transshipment warehouse (S4 and S10), we use $g_{1a}$ (Full Allocation Proportional) instead of the default $g_{1b}$ (DC-Bid Proportional) to enforce full allocation of on-hand inventory.
    \item \textbf{Layer width.} In settings with realistic demand (S7 and S9), we increase all of the a NN layers' width from 32 to 64.

    \item \textbf{Batch size.} In settings with realistic demand (S7 and S9), we decrease batch size to 72 since the datasets are significantly smaller than those for settings with synthetic demand.
\end{itemize}

\subsubsection{Vanilla NN hyperparameters. \label{appendix:vanilla-hyperparams} }

The Vanilla NN architecture maps the full state of the inventory system to a vector of allocation logits for each edge. Unlike the GNN-based architectures, the design of the Vanilla NN varies across settings, reflecting differences in action dimensionality and network structure. Table~\ref{tab:vanilla-defaults} summarizes the key hyperparameter choices used in each setting.

\begin{table}[h!]
\centering
\caption{Vanilla NN hyperparameter configuration by setting. Sets \{·\} indicate grid search values. For Section \ref{sec:vanilla-hdpo} experiments using fixed hyperparameters, S3 uses LR=$10^{-2}$ and width=32; S4 uses LR=$10^{-3}$ and width=256.}
\label{tab:vanilla-defaults}
\begin{tabular}{@{}lcccccc@{}}
\toprule
\textbf{Setting} & \textbf{LR} & \textbf{Batch Size} & \textbf{NN layers} & \textbf{Layer width} & \textbf{FEL} & \textbf{Softplus} \\
\midrule
S1 & $10^{-3}$ & 8192 & 3 & 32 & - & Yes \\
S2 & $10^{-3}$ & 1024 & 3 & 32 & - & Yes \\
S3 & $\{10^{-2}, 10^{-3}, 10^{-4}\}$ & 8192 & 2 & $\{128, 256, 512\}$ & $g_{2b}$ & No \\
S4 & $\{10^{-2}, 10^{-3}, 10^{-4}\}$ & 1024 & 3 & $\{128, 256, 512\}$ & $g_{2a}$ & No \\
S5 & $10^{-3}$ & 8192 & 2 & 64 & - & Yes \\
S6 & $\{10^{-2}, 10^{-3}, 10^{-4}\}$ & 1024 & 3 & $\{128, 256, 512\}$ & $g_{2b}$ & No \\
S7 & $\{10^{-2}, 10^{-3}, 10^{-4}\}$ & 72 & 3 & $\{128, 256, 512\}$ & $g_1$ & Yes \\
S8 & $\{10^{-2}, 10^{-3}, 10^{-4}\}$ & 1024 & 3 & $\{128, 256, 512\}$ & $g_{2b}$ & No \\
S9 & $\{10^{-2}, 10^{-3}, 10^{-4}\}$ & 72 & 3 & $\{128, 256, 512\}$ & $g_1$ & Yes \\
S10 & $\{10^{-2}, 10^{-3}, 10^{-4}\}$ & 1024 & 3 & $\{128, 256, 512\}$ & - & Yes \\
\bottomrule
\end{tabular}
\end{table}

{\bf Feasibility enforcement layers.}
For networked settings with synthetic demand, we primarily use \textbf{Extended Softmax} ($g_{2b}$), which incorporates heuristic upper bounds into the softmax normalization. These upper bounds are calculated consistently across settings by summing the mean demands of all downstream stores and multiplying by 4. While the standard softmax ($g_2$) yielded similar performance, we found that explicitly incorporating these upper bounds significantly improved training stability and reduced the likelihood of convergence to poor local optima.
We apply this Extended Softmax to most synthetic networked settings (S3, S6, S8). The only exception is the transshipment setting (S4), where the structural requirement that on-hand inventory must be fully allocated each period necessitates using \textbf{Full Allocation Softmax} ($g_{2a}$) instead.

For realistic demand settings (S7, S9), meaningful upper bounds are not obvious due to the variability of real demand patterns. Therefore, we opted for using the \textbf{Proportional Allocation} rule ($g_1$) with these settings.

{\bf Softplus transformation and biases.} For FELs that interpret logits as allocation amounts (specifically, $g_1$) or in architectures without an FEL, we apply a softplus transformation to ensure strictly positive inputs and prevent invalid behavior during early training. We add a positive bias of 1 to shift outputs away from zero and avoid flat gradient regions. Softmax-based FELs ($g_{2a}$, $g_{2b}$) do not require softplus or bias since they inherently map unconstrained logits to valid allocations.

\subsection{Building a quantile forecaster \label{appendix: quantile-forecaster} }

In this subsection, we provide detailed information on the offline training and performance evaluation of the quantile forecaster used in the generalized newsvendor policies defined in Appendix~\ref{appendix:generalized-newsvendor} and evaluated in Section~\ref{sec:vanilla-hdpo-realistic}. The quantile forecaster is designed to estimate the distribution of the sum of $m$ demand terms (considering various values of $m$ simultaneously), given the sequence of the previous $s_1 \in \mathbb{N}$ demands and the number of days until the next occurrence of Christmas, denoted as $d_{\textup{Christmas}} \in \mathbb{N}$. Specifically, given a tuple ($\xi_1, \ldots, \xi_{s_1}, d_{\textup{Christmas}}$), the quantile forecaster predicts the value of the $\tau$-quantile of the sum of the next $m$ demands concurrently for each $\tau$ and $m$ within specified sets $\mathcal{Q}$ and $\mathcal{M}$, respectively. This is, for one input tuple, it produces $|\mathcal{Q}| |\mathcal{M}|$ terms.

We consider a dataset comprising $N$ pairs in the form of $(x_i, y_i)$, where $x_i \in \R^{s_1 + 1}$ represents a feature vector, and $y_i \in \R^{|\mathcal{M}|}$ is a target vector. Given the prediction $\hat{y}_i \in \R^{|\mathcal{Q}| |\mathcal{M}|}$ for feature vector $x_i$, we define the multi-horizon loss for the $\tau$-quantile as
\begin{equation}
\label{eq: multi-horizon-loss}
    \ell_{\tau}(y_i, \hat{y}_i) = \sum_{m \in \mathcal{M}} \max\{\tau(y_{im} - \hat{y}_{im\tau}), (1 - \tau)(\hat{y}_{im\tau} - y_{im})\},
\end{equation}
and the multi-horizon multi-quantile loss as
\begin{equation}
\label{eq: multi-horizon-multi-quantile-loss}
    \ell(y_i, \hat{y}_i) = \sum_{\tau \in \mathcal{Q}} \ell_{\tau}(y_i, \hat{y}_i).
\end{equation}
The objective is to minimize the sample average of $\ell(y_i, \hat{y}_i)$, given by $1/N \sum_{i=1}^{N}\ell(y_i, \hat{y}_i)$.

We implemented the quantile forecaster using PyTorch, considering quantiles ranging from $0.05$ to $0.95$ in steps of $0.05$. Given that the numerical experiments in Section \ref{sec:vanilla-hdpo-realistic} consider lead times of $4$, $5$, and $6$, we set $\mathcal{M} = \{5, 6, 7\}$ (recall that generalized newsvendor policies take into account the distribution of the sum of $L + 1$ demands). The train set and dev set sizes were approximately $4$ million and $1.5$ million samples, respectively. These sets were generated by extracting subsequences of demands with a length of $16$ from the datasets outlined in Section \ref{appendix:realistic-demand-dataset}.

The architecture of the implemented Multilayer Perceptron (MLP) includes $2$ hidden layers, each with $128$ neurons. We considered a batch size of $1,048,576$ samples and a learning rate of $0.01$. In practice, input features for all samples within a batch are fed in a compact matrix form, and the model outputs a matrix of predictions with dimensions $(n_{\textup{batch}}, |\mathcal{Q}|, |\mathcal{M}|)$, where $n_{\textup{batch}}$ is the batch size. This design allows us to leverage the parallel computing capabilities of GPUs and take advantage of an efficient implementation in PyTorch. 

We validate the efficacy of the quantile forecaster by observing the strong performance of the Transformed Newsvendor policy in the lost demand setting with high average unit underage costs (depicted in Figure \ref{fig: real_data_lost_demand} in Section \ref{sec:vanilla-hdpo-realistic}). Recall that for high unit underage costs, stockouts are infrequent, allowing policies optimized for backlogged demand settings to perform well. To further assess the performance of our quantile forecaster, we conducted two separate analyses.

First, we evaluated the model's calibration, which measures the agreement between the predicted values for each quantile and the empirical probability of a target lying below each of them. Ideally, for each $\tau \in \mathcal{Q}$, we would observe an empirical proportion of $\tau$ of the samples lying below the predicted $\tau$-quantile.
Therefore, in Figure \ref{fig:calibration}, we plotted each $\tau \in \mathcal{Q}$ against the proportion of targets lying below the predicted $\tau$-quantile, calculated as $\sum_{i \in [N]} \sum_{m \in \mathcal{M}} \mathbbm{1}\left(y_{im} \leq \hat{y}_{im\tau}\right) / (N |\mathcal{M}|)$. The figure demonstrates that our forecaster is nearly perfectly calibrated in the train set. However, in the dev set, there is a slight overestimation, resulting in larger-than-ideal proportions of observed targets lying below each predicted quantile value, although the calibration remains generally good.

Next, our objective is to provide an estimation of the magnitude of prediction errors relative to the targets. Illustrated in Figure \ref{fig:quantile-loss}, we depict, for each $\tau \in \mathcal{Q}$, the cumulative multi-horizon loss (see Eq. \ref{eq: multi-horizon-loss}) divided by the sum of targets, given by $\sum_{i \in [N]} \ell_{\tau}(y_i, \hat{y}_i) / \sum_{i \in [N]} \sum_{m \in \mathcal{M}}y_{im}$. To enhance comprehension, consider that when $\tau=0.5$, the multi-horizon quantile loss is equivalent to one half of the sum of mean absolute errors (MAEs) across time horizons $m \in \mathcal{M}$. The sum of MAEs is therefore less than $18\%$ of the sum of the targets, indicating a relatively low value. As the analyzed ratio is nearly maximized for $\tau=0.5$, this suggests that the quantile forecaster adeptly predicts the conditional distribution of cumulative demands. 

While we recognize the importance of comparing our forecaster with other methodologies to evaluate its performance, the three evaluations conducted in this section collectively indicate its effectiveness. 

\begin{figure}
\begin{subfigure}{.49\textwidth}
  \centering
  \includegraphics[width=1\linewidth]{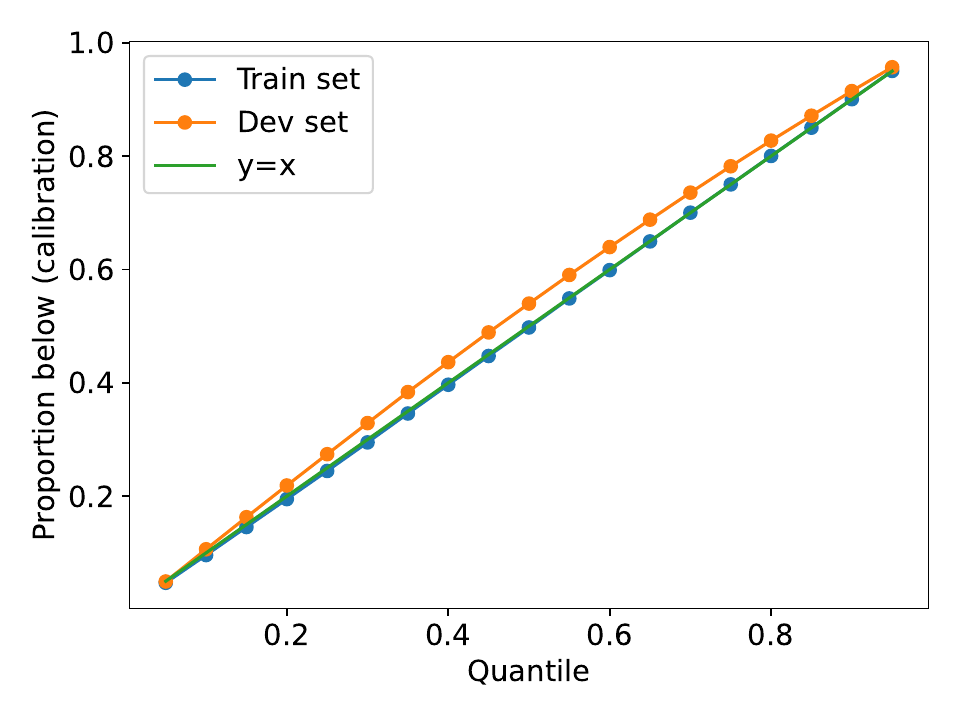}
  \caption{Calibration for each quantile, calculated as \\the proportion of targets that lie below the \\quantile predicted by the forecaster.}
  \label{fig:calibration}
\end{subfigure}%
\begin{subfigure}{.49\textwidth}
  \centering
  \includegraphics[width=1\linewidth]{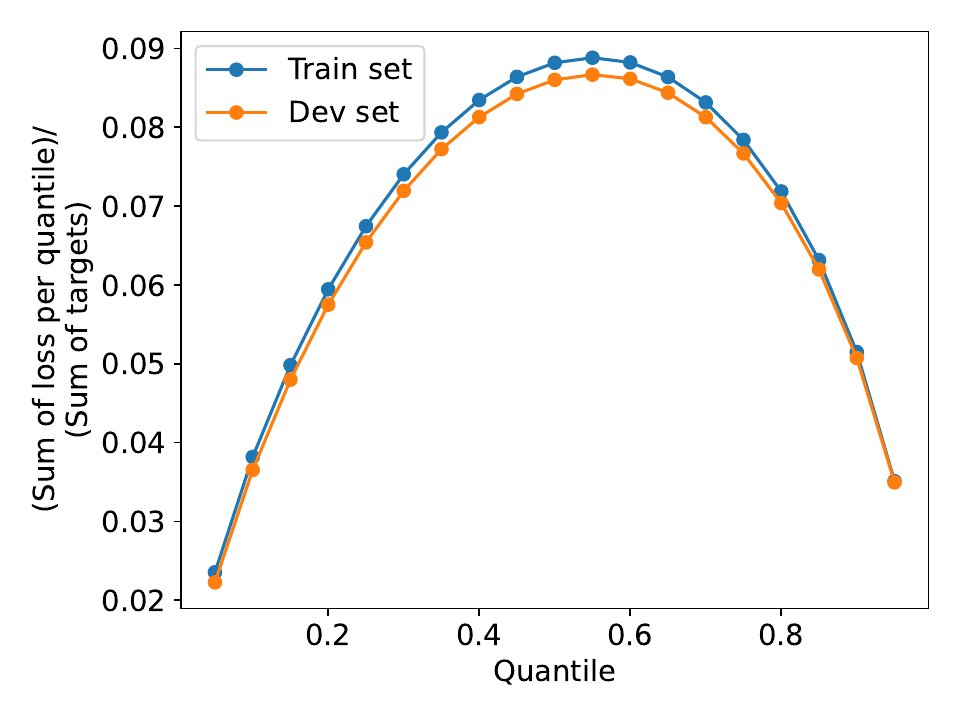}
  \caption{Sum of multi-horizon quantile loss (see Equation \ref{eq: multi-horizon-loss}) divided by the sum of the targets, for each quantile.}
  \label{fig:quantile-loss}
\end{subfigure}

\caption{Performance indicators for the quantile forecaster.}
\label{fig:forecaster-plots}
\end{figure}

\subsection{Specifications and results for numerical experiments in Section \ref{subsec:optimal_exps} \label{appendix:optimal-exps}}

In this section, we provide detailed benchmark definitions, baselines, and results for the experiments with known bounds on the optimal cost from Section \ref{subsec:optimal_exps}. Since we use the same experiment specifications outlined in that section, we omit them here to avoid redundancy. The hyperparameters used are given in Table \ref{tab:vanilla-defaults}.

\subsubsection{Vanilla HDPO in Setting S1. \label{appendix:single-store-backlogged-exps}} \hfill\\
{\bf Benchmarks.} We generate demand traces by sampling from a normal distribution with a mean of $\mu = 5.0$ and a standard deviation of $\sigma = 1.6$, truncating it from below at $0$. We created $24$ instances by setting $h = 1$, $p = 4, 9, 19, 39$ and $L = 1, 4, 7, 10, 15, 20$.  

{\bf Baselines} We compare our model with the optimal base-stock policy computed according to \eqref{eq:base-stock-policy} and \eqref{eq:base-stock-level} in Appendix \ref{appendix:base-stock-policy}.

{\bf Results.}
Figures \ref{Opt gap one-store backlogged} and \ref{Time to opt one-store backlogged}, respectively, show the optimality gap on the test set and time to reach $1\%$ of optimality gap on the dev set. Our approach achieves an average gap of $0.03\%$ across instances and takes less than $4$ minutes, on average, to obtain a $1\%$ gap. Further, gaps are consistently below $\edit{0.2}\%$ across instances, even for long lead times of up to $20$ periods. We report additional performance indicators in Table \ref{table: backlogged detailed}.

\begin{figure}
\begin{subfigure}{.49\textwidth}
  \centering
  \includegraphics[width=1\linewidth]{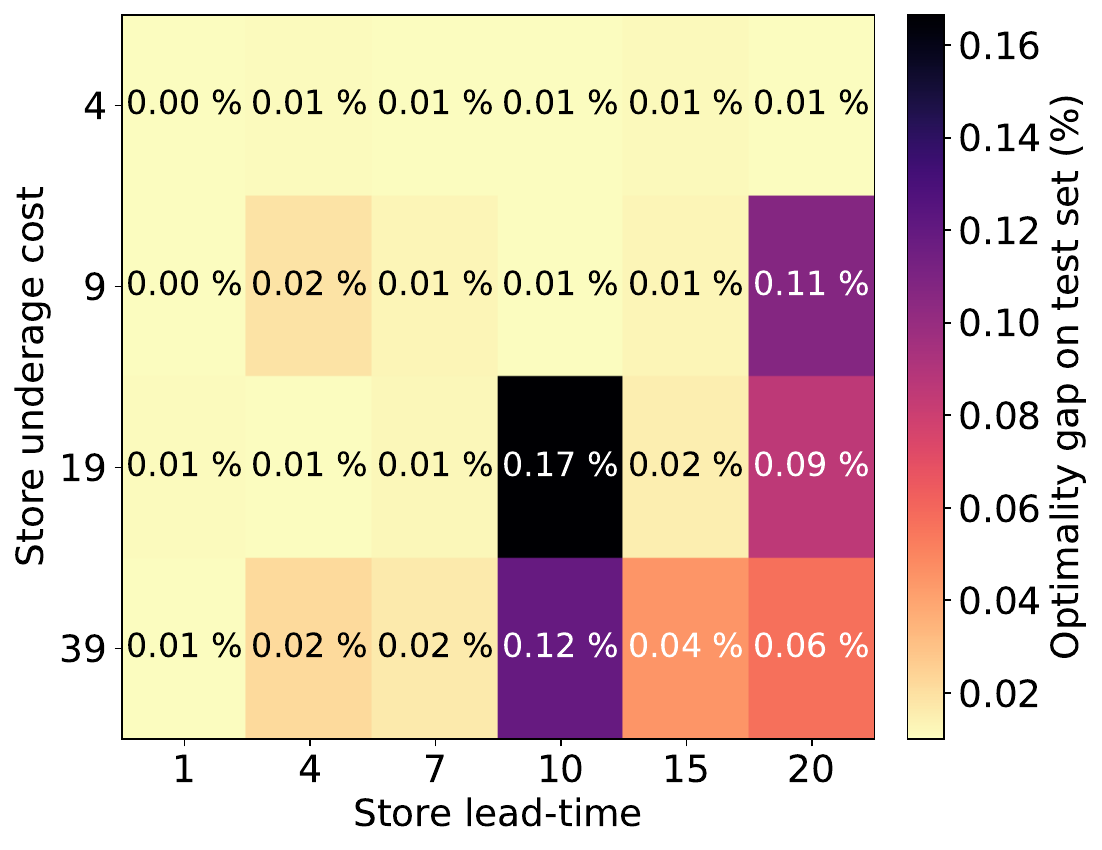}
  \caption{Optimality gap on the test set.}
  \label{Opt gap one-store backlogged}
\end{subfigure}%
\hspace{0.02\textwidth}
\begin{subfigure}{.49\textwidth}
  \centering
  \includegraphics[width=1\linewidth]{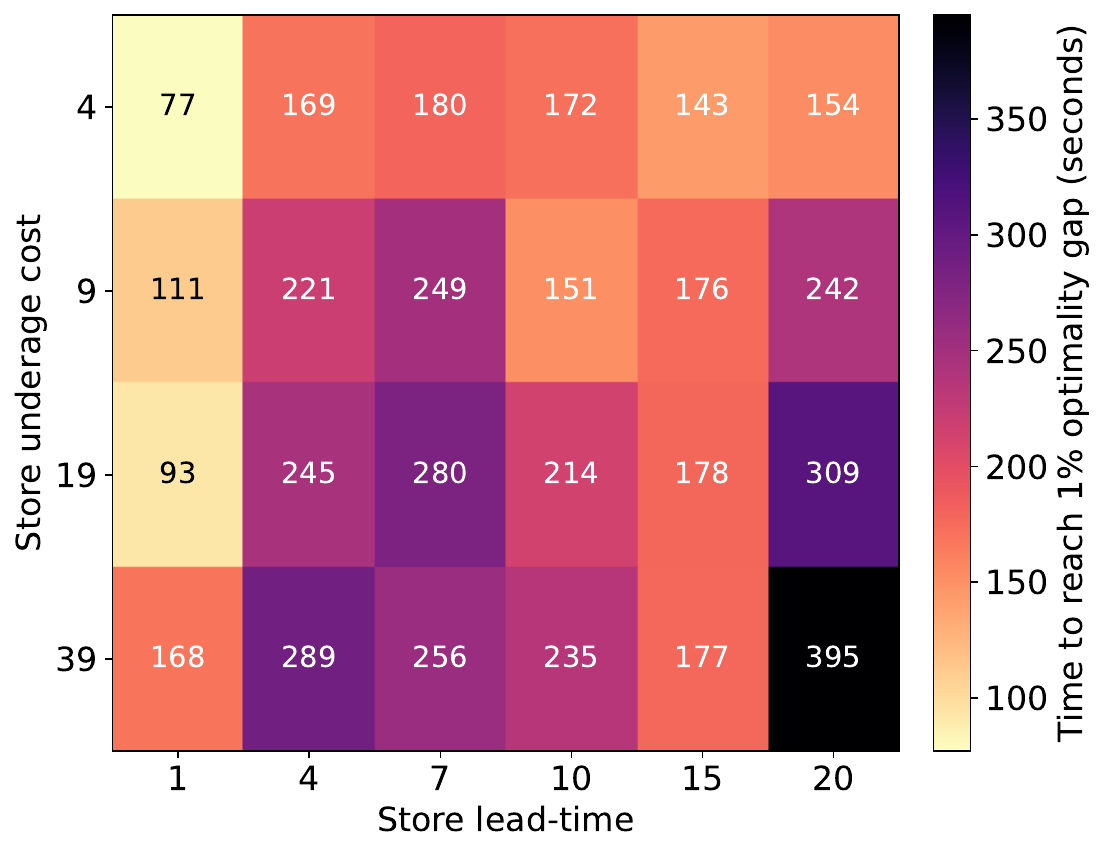}
  \caption{Time to reach 1\% opt. gap on the dev set.}
  \label{Time to opt one-store backlogged}
\end{subfigure}

\caption{Optimality gap on the test set and time to reach 1\% optimality gap on the dev set for setting S1 across different underage costs and lead times.}
\end{figure}

\begin{table}[h!]
\begin{center}
\caption{Performance metrics of the Vanilla  NN for each instance of setting S1. \label{table: backlogged detailed}}

\begin{tabular}{>{\raggedleft}p{1.4cm}>{\raggedleft}p{1.5cm}>{\raggedleft}p{1.4cm}>{\raggedleft}p{1.4cm}>{\raggedleft}p{1.4cm}>{\raggedleft}p{1.4cm}>{\raggedleft}p{1.4cm}>{\raggedleft}p{1.4cm}>{\raggedleft\arraybackslash}p{1.7cm}}
\toprule
 \raggedright Store leadtime &  \raggedright Store  underage cost &  \raggedright Train loss &  \raggedright Dev  loss &  \raggedright Test  loss &  \raggedright Train  gap (\%) &  \raggedright Dev  gap (\%) &  \raggedright Test  gap (\%) &  \raggedright\arraybackslash Time to 1\% dev gap (s) \\
\midrule
1  & 4  & 3.17  & 3.17  & 3.17  & 0.01 & 0.02 & 0.00 & 77 \\
1  & 9  & 3.97  & 3.97  & 3.97  & 0.01 & 0.03 & 0.00 & 111 \\
1  & 19 & 4.68  & 4.67  & 4.67  & 0.02 & 0.04 & 0.01 & 93 \\
1  & 39 & 5.29  & 5.29  & 5.29  & 0.01 & 0.05 & 0.01 & 168 \\
4  & 4  & 5.01  & 5.01  & 5.01  & 0.01 & 0.05 & 0.01 & 169 \\
4  & 9  & 6.30  & 6.27  & 6.28  & 0.02 & 0.06 & 0.02 & 221 \\
4  & 19 & 7.37  & 7.39  & 7.38  & 0.03 & 0.07 & 0.01 & 245 \\
4  & 39 & 8.36  & 8.37  & 8.37  & 0.00 & 0.08 & 0.02 & 289 \\
7  & 4  & 6.33  & 6.34  & 6.33  & 0.01 & 0.06 & 0.01 & 180 \\
7  & 9  & 7.94  & 7.92  & 7.94  & 0.06 & 0.08 & 0.01 & 249 \\
7  & 19 & 9.40  & 9.36  & 9.33  & 0.09 & 0.09 & 0.01 & 280 \\
7  & 39 & 10.61 & 10.57 & 10.58 & 0.03 & 0.10 & 0.02 & 256 \\
10 & 4  & 7.42  & 7.44  & 7.43  & 0.05 & 0.07 & 0.01 & 172 \\
10 & 9  & 9.29  & 9.32  & 9.31  & 0.02 & 0.08 & 0.01 & 151 \\
10 & 19 & 11.01 & 10.94 & 10.96 & 0.09 & 0.10 & 0.17 & 214 \\
10 & 39 & 12.42 & 12.41 & 12.42 & 0.07 & 0.12 & 0.12 & 235 \\
15 & 4  & 8.94  & 8.94  & 8.95  & 0.04 & 0.09 & 0.01 & 143 \\
15 & 9  & 11.22 & 11.26 & 11.23 & 0.01 & 0.11 & 0.01 & 176 \\
15 & 19 & 13.21 & 13.20 & 13.21 & 0.02 & 0.12 & 0.02 & 178 \\
15 & 39 & 14.90 & 14.92 & 14.97 & 0.10 & 0.11 & 0.04 & 177 \\
20 & 4  & 10.29 & 10.29 & 10.26 & 0.06 & 0.10 & 0.01 & 154 \\
20 & 9  & 12.86 & 12.85 & 12.88 & 0.02 & 0.13 & 0.11 & 242 \\
20 & 19 & 15.16 & 15.08 & 15.13 & -0.04 & 0.14 & 0.09 & 309 \\
20 & 39 & 17.11 & 17.17 & 17.14 & 0.09 & 0.17 & 0.06 & 395 \\
\bottomrule
\end{tabular}

\end{center}
\end{table}

\subsubsection{Vanilla HDPO in Setting S2. \label{appendix:single-store-lost-demand-exps}} \hfill\\

{\bf Benchmarks, Baselines and Experiment specifications.} See Section \ref{subsec:hdrl-and-hdpo-vs-reinforce}

{\bf Results.} In Table \ref{table: lost demand detailed}, we report the performance under the hyperparameter setting that minimizes loss on the dev set (see Table \ref{tab:vanilla-defaults} in Appendix \ref{appendix:vanilla-hyperparams}). This table illustrates the reliability of HDPO, achieving results within $1\%$ of optimality in under $\edit{90}$ seconds for all but one instance. Moreover, Figures \ref{fig: learning curve u9l4_small_0.0001_10to13} and \ref{fig: learning curve u9l4_large_0.01_10to10} show the learning curves on one instance for two "extreme" choices of hyperparameters, revealing stable learning and rapid convergence to near-optimal solutions across different hyperparameter settings (note that an epoch corresponds to 4 and 32 gradients steps for the plots on the left and right, respectively).

\begin{table}[h!]
\begin{center}
\caption{Performance metrics of the Vanilla NN for each instance of setting S2 for the best hyperparameter setting. For the last two columns we consider the performance of the NN with continuous allocation as proxy for the performance in the discrete-allocation setting, as costs changed, on average, by less than 1\% after discretizing the allocation.\label{table: lost demand detailed}}

\begin{tabular}{>{\raggedleft}p{1.8cm}>{\raggedleft}p{1.8cm}>{\raggedleft}p{1.8cm}>{\raggedleft}p{1.8cm}>{\raggedleft}p{2.1cm}>{\raggedleft\arraybackslash}p{1.8cm}}
\toprule
 \raggedright Store lead time &  \raggedright Store  underage cost &  \raggedright Test loss & \raggedright Test gap (\%) &  \raggedright Gradient steps to  1\% dev gap &  \raggedright\arraybackslash Time to 1\% dev gap (s) \\
\midrule
                           1 &                                  4 &                    4.04 &                      $<$0.25 &                                         640 &                                                 44 \\
                           1 &                                  9 &                    5.44 &                      $<$0.25 &                                         640 &                                                 45 \\
                           1 &                                 19 &                    6.68 &                      $<$0.25 &                                         320 &                                                 24 \\
                           1 &                                 39 &                    7.84 &                      $<$0.25 &                                         320 &                                                 23 \\
                           2 &                                  4 &                    4.40 &                      $<$0.25 &                                         320 &                                                 27 \\
                           2 &                                  9 &                    6.09 &                      $<$0.25 &                                        1280 &                                                111 \\
                           2 &                                 19 &                    7.67 &                      $<$0.25 &                                         960 &                                                 80 \\
                           2 &                                 39 &                    9.11 &                      $<$0.25 &                                         640 &                                                 56 \\
                           3 &                                  4 &                    4.60 &                      $<$0.25 &                                         960 &                                                 82 \\
                           3 &                                  9 &                    6.53 &                      $<$0.25 &                                         960 &                                                 88 \\
                           3 &                                 19 &                    8.36 &                      $<$0.25 &                                         640 &                                                 57 \\
                           3 &                                 39 &                   10.04 &                      $<$0.25 &                                         640 &                                                 56 \\
                           4 &                                  4 &                    4.73 &                      $<$0.25 &                                         640 &                                                 53 \\
                           4 &                                  9 &                    6.84 &                      $<$0.25 &                                         960 &                                                 82 \\
                           4 &                                 19 &                    8.88 &                      $<$0.25 &                                         320 &                                                 29 \\
                           4 &                                 39 &                   10.79 &                      $<$0.25 &                                         320 &                                                 27 \\
\bottomrule
\end{tabular}

\end{center}
\end{table}
\begin{figure}
\captionsetup[subfigure]{justification=centering}
\begin{subfigure}{.49\textwidth}
  \centering
  \includegraphics[width=1\linewidth]{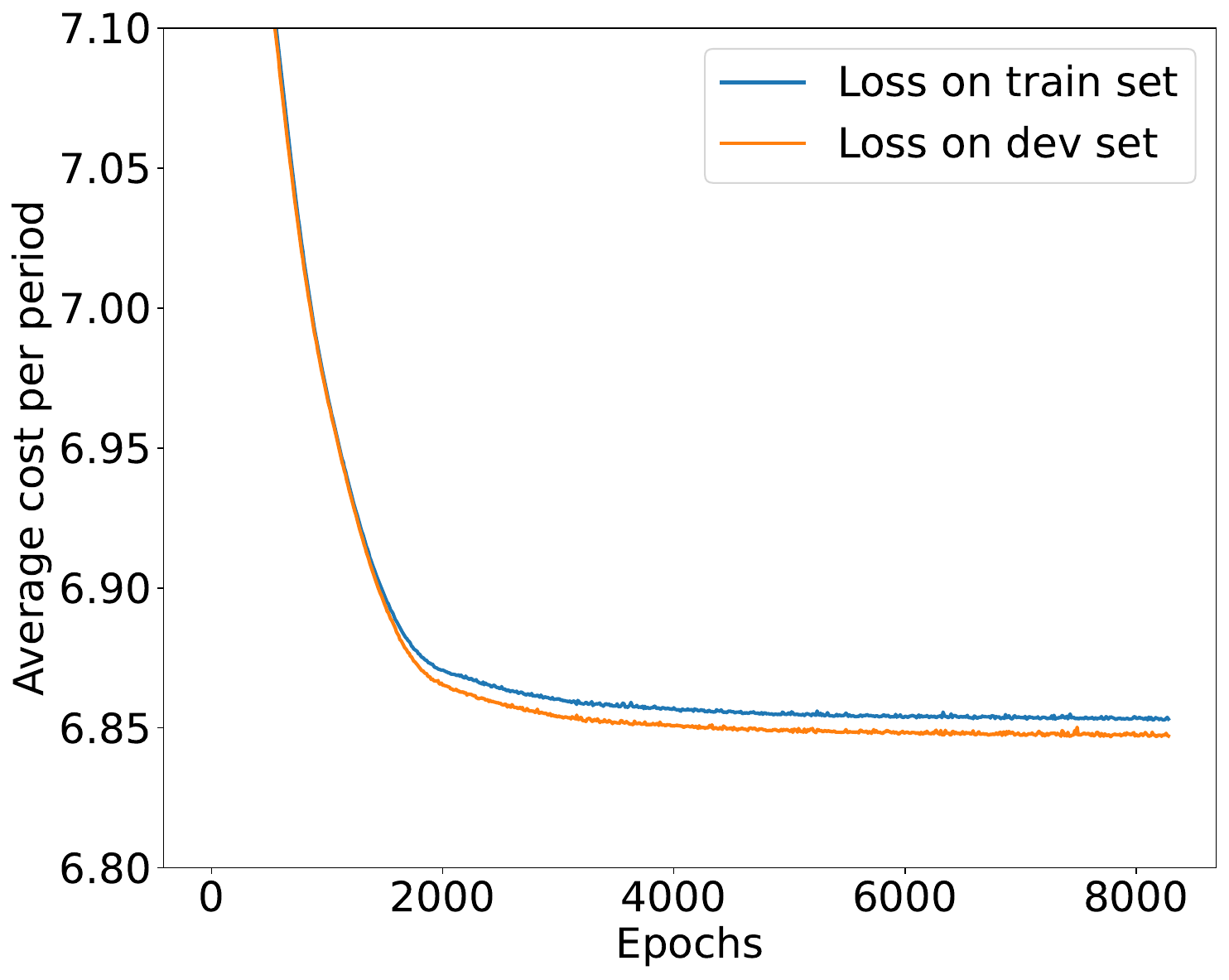}
  \caption{$2$ hidden layers, learning rate of $10^{-4}$ \newline and batch size of $8192$.}
  \label{fig: learning curve u9l4_small_0.0001_10to13}
\end{subfigure}%
\hspace{0.02\textwidth}
\begin{subfigure}{.49\textwidth}
  \centering
  \includegraphics[width=1\linewidth]{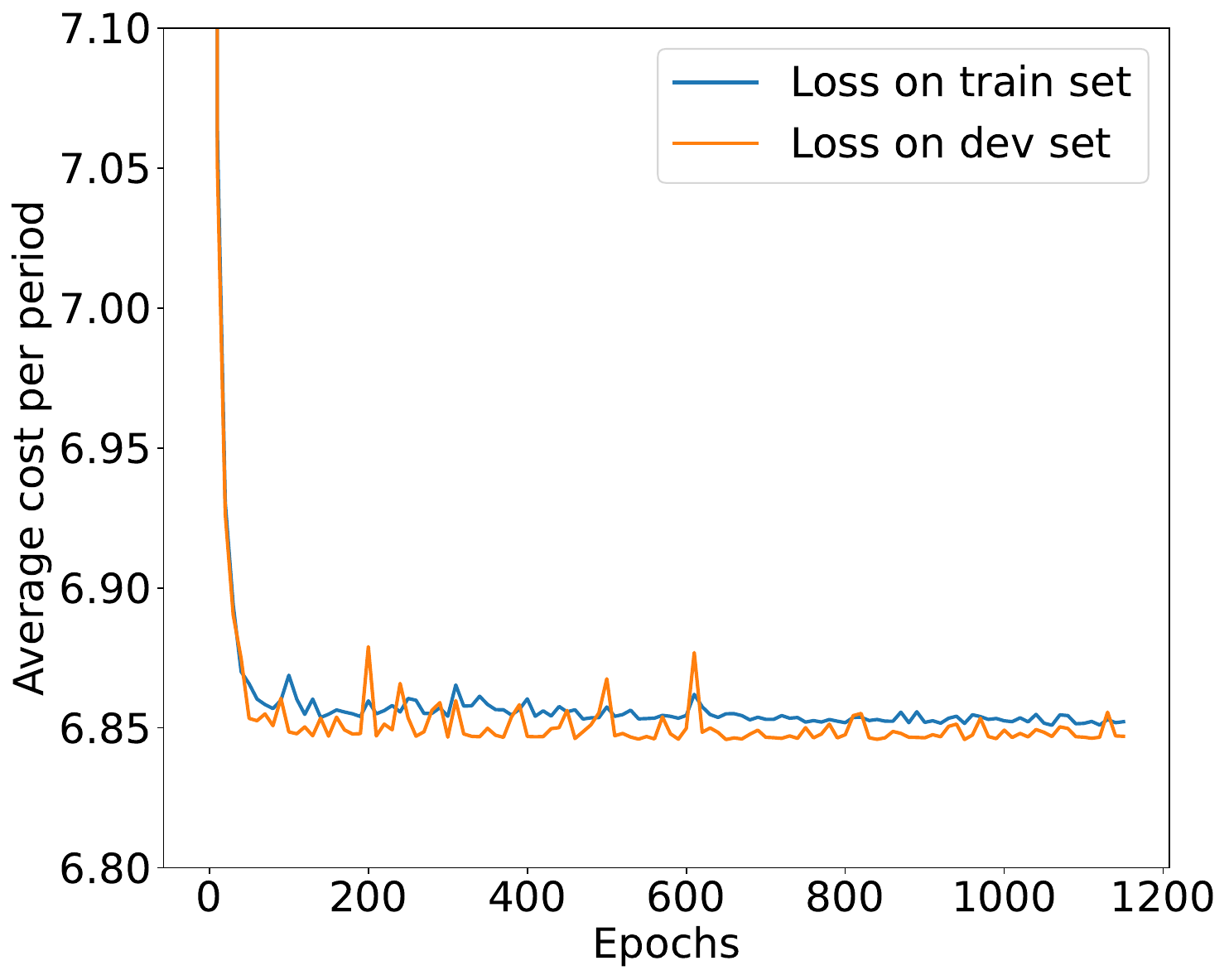}
  \caption{$3$ hidden layers, learning rate of $10^{-2}$ \newline and batch size of $1024$.}
  \label{fig: learning curve u9l4_large_0.01_10to10}
\end{subfigure}

\caption{Epochs vs average cost per period on train and dev sets with continuous allocation for different hyperparameter choices. Setting S2 with unit underage cost of 9 and lead time of 4. Note that an epoch corresponds to 4 and 32 gradient steps for the left and right plots, respectively.}
\label{fig: learning curves lost demand}
\end{figure}

In Figure \ref{fig: scatter one store lost demand} we plot the inventory position (see Eq. \eqref{eq:inventory-position}) and allocation under the Vanilla NN policy for $2$ settings and compare it to the allocation under the optimal CBS policy (red line). We randomly jitter points for visibility, and color points according to the current inventory on-hand. We observe that the structure of the policy learned by our Vanilla NN somewhat resembles a CBS policy, but our learned policy is able to use additional information in the state space to achieve lower costs. For a fixed inventory position, the NN tends to order less for lower inventory on hand, as a stock-out is more likely under such a scenario in which case less inventory will actually be depleted.

\begin{figure}
\begin{subfigure}{.5\textwidth}
  \centering
  \includegraphics[width=1\linewidth]{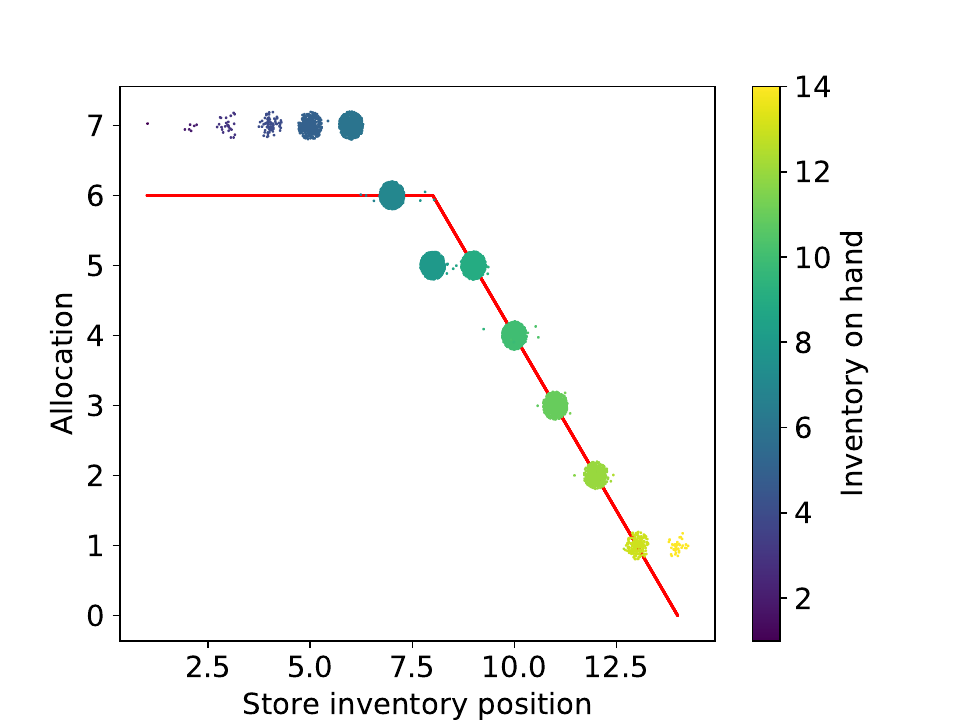}
  \caption{Lead time of $1$ period}
  \label{scatter l1 u9}
\end{subfigure}%
\begin{subfigure}{.5\textwidth}
  \centering
  \includegraphics[width=1\linewidth]{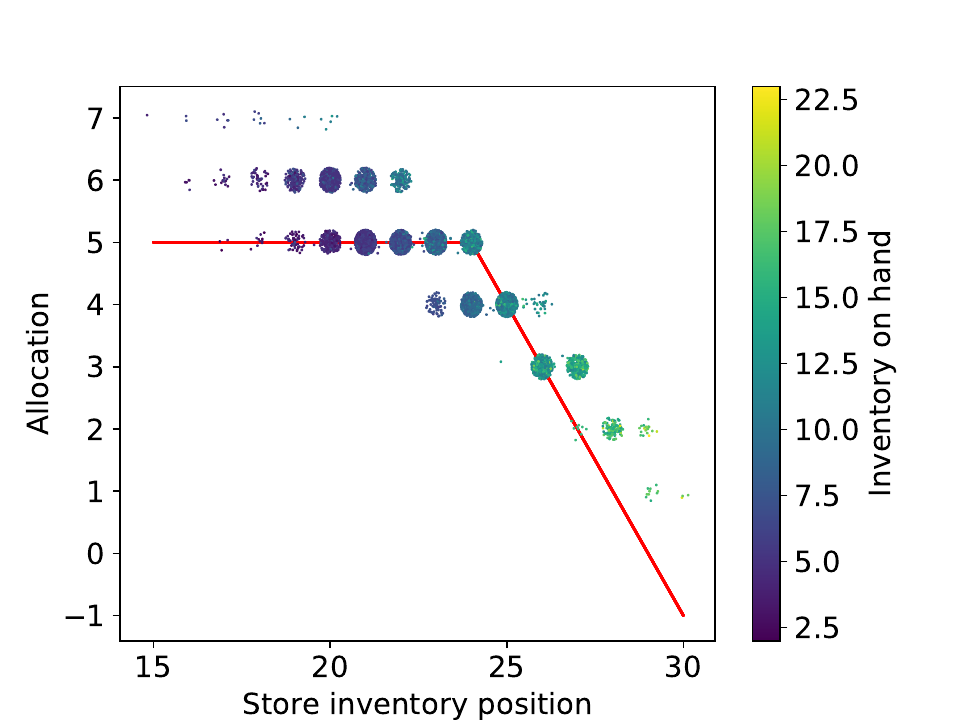}
  \caption{Lead time of $4$ periods}
  \label{scatter l4 u9}
\end{subfigure}

\caption{Inventory position vs allocation under Vanilla NN for setting S2 with a unit underage cost of $9$. Points are randomly jittered for clarity. Point colors correspond to current inventory on-hand, and the red line captures the allocation under the optimal CBS policy (reported in \cite{xin2021understanding}).}
\label{fig: scatter one store lost demand}
\end{figure}

Table \ref{table: lost demand by hyperparams} reports the average number of gradient steps and wall-clock time required to reach a 1\% optimality gap on the dev set across the 16 instances from \cite{zipkin2008old}, evaluated under a range of hyperparameter settings. The results highlight HDPO’s computational efficiency and robustness: in all runs but one, the method produced solutions indistinguishable from the optimum. Moreover, across all hyperparameter choices, solutions within 1\% of optimality were found in under \edit{14} minutes on average, indicating consistently strong performance. Under the best-performing configuration, this time dropped to just \edit{55} seconds on average, demonstrating the method’s speed and reliability.
\begin{table}[h!]
\begin{center}
\caption{Summary of performance metrics of the Vanilla NN for different hyperparameter settings for setting S2, across the 16-instance test-bed in \cite{zipkin2008old} for Poisson demand with mean $5$. We consider that an instance is solved to optimality if the gap is smaller than $0.25\%$. For the last $2$ columns we consider the performance of the NN with continuous allocation as proxy for the performance in the discrete-allocation setting, as costs changed, on average, by less than 1\% after discretizing the allocation.\label{table: lost demand by hyperparams}}

\begin{tabular}{@{}rrrrrr@{}}
\toprule
\multicolumn{1}{l}{\begin{tabular}[c]{@{}l@{}}Hidden  \\ layers\end{tabular}} & \multicolumn{1}{l}{\begin{tabular}[c]{@{}l@{}}Batch\\ size\end{tabular}} & \multicolumn{1}{l}{\begin{tabular}[c]{@{}l@{}}Learning \\ rate\end{tabular}} & \multicolumn{1}{l}{\begin{tabular}[c]{@{}l@{}}Instances solved \\ to optimality (\#)\end{tabular}} & \multicolumn{1}{l}{\begin{tabular}[c]{@{}l@{}}Average gradient steps \\ to  1\% dev gap\end{tabular}} & \multicolumn{1}{l}{\begin{tabular}[c]{@{}l@{}}Average time to \\ 1\%  dev gap (s)\end{tabular}} \\ 
\midrule
2 & 1024 & 0.0001 & 16 & 6120 & 467 \\
2 & 1024 & 0.0010 & 16 & 1740 & 133 \\
2 & 1024 & 0.0100 & 16 & 1420 & 105 \\
2 & 8192 & 0.0001 & 15 & 5795 & 819 \\
2 & 8192 & 0.0010 & 16 & 1412 & 203 \\
2 & 8192 & 0.0100 & 16 & 1250 & 171 \\
3 & 1024 & 0.0001 & 16 & 4140 & 351 \\
3 & 1024 & 0.0010 & 16 & 940 & 81 \\
3 & 1024 & 0.0100 & 16 & 660 & 55 \\
3 & 8192 & 0.0001 & 16 & 3840 & 554 \\
3 & 8192 & 0.0010 & 16 & 825 & 118 \\
3 & 8192 & 0.0100 & 16 & 395 & 57 \\
\bottomrule
\end{tabular}

\end{center}
\end{table}

\subsubsection{Vanilla HDPO in Setting S3. \label{appendix:serial-system-exps}} \hfill

{\bf Benchmarks.} We analyzed a serial network structure with $4$ echelons. We considered Normal demand with mean $\mu = 5.0$ and standard deviation $\sigma = 2.0$, truncated at $0$ from below. We fixed holding costs as $h^1 = 0.1$, $h^2 = 0.2$, $h^3 = 0.5$ and $h^4 = 1.0$. The lead times for edges are fixed as $L^{(0,1)} = 2$, $L^{(1,2)} = 4$ and $L^{(2,3)} = 3$. We created $16$ instances by setting $p = 4, 9, 19, 39$ and $L^{(3,4)} = 1, 2, 3, 4$.

{\bf Baselines.} Within this setting, there exists an optimal echelon-stock policy (see Eq. \eqref{eq:echelon-stock-for-serial} in Appendix \ref{appendix:echelon-stock-policy}). We employ our differentiable simulator to search for the best-performing base-stock levels $\hat{S}^1, \ldots, \hat{S}^4$ through multiple runs. We compare our NNs with the best-performing echelon-stock policy obtained, which we take as the optimal cost.

{\bf Results.} Table \ref{table:serial} describes the performance of the Vanilla NN for each instance of the serial network structure (see Figure \ref{fig:serial-system} in Section \ref{sec: inventory_network_control}). The Vanilla NN achieved an average optimality gap of \edit{$0.54\%$}, and needed around \edit{$5000$} gradient steps, on average, to achieve a gap smaller than $1\%$ in the dev set.

\begin{table}[h!]
\begin{center}
\caption{Performance metrics for the Vanilla NN for each instance of setting S3.
\label{table:serial}}

\begin{tabular}{>{\raggedleft}p{1.1cm}>{\raggedleft}p{1.4cm}>{\raggedleft}p{0.8cm}>{\raggedleft}p{0.8cm}>{\raggedleft}p{0.8cm}>{\raggedleft}p{1.1cm}>{\raggedleft}p{0.8cm}>{\raggedleft}p{0.8cm}>{\raggedleft\arraybackslash}p{2.1cm}>{\raggedleft\arraybackslash}p{2.1cm}}
\toprule
Store lead time & Store underage cost & Train loss & Dev loss & Test loss & Train gap (\%) & Dev gap (\%) & Test gap (\%) & Gradient steps to 1\% dev gap & Time to 1\% dev gap (s) \\
\midrule
1 & 4 & 6.93 & 6.94 & 6.93 & 0.42 & 0.51 & 0.46 & 5040 & 1410 \\
1 & 9 & 8.42 & 8.41 & 8.41 & 0.66 & 0.50 & 0.47 & 4320 & 1255 \\
1 & 19 & 9.69 & 9.65 & 9.65 & 0.88 & 0.40 & 0.44 & 2680 & 733 \\
1 & 39 & 10.77 & 10.74 & 10.75 & 0.70 & 0.43 & 0.50 & 4640 & 1331 \\
2 & 4 & 7.63 & 7.64 & 7.64 & 0.43 & 0.53 & 0.44 & 7200 & 2185 \\
2 & 9 & 9.33 & 9.32 & 9.32 & 0.72 & 0.69 & 0.66 & 6280 & 1858 \\
2 & 19 & 10.77 & 10.72 & 10.72 & 0.93 & 0.45 & 0.50 & 4240 & 1237 \\
2 & 39 & 11.96 & 11.96 & 11.97 & 0.44 & 0.44 & 0.50 & 4040 & 1251 \\
3 & 4 & 8.28 & 8.27 & 8.27 & 0.83 & 0.73 & 0.65 & 4400 & 1231 \\
3 & 9 & 10.10 & 10.08 & 10.08 & 0.60 & 0.46 & 0.40 & 3000 & 867 \\
3 & 19 & 11.66 & 11.66 & 11.66 & 0.62 & 0.58 & 0.63 & 4040 & 1197 \\
3 & 39 & 13.04 & 13.03 & 13.03 & 0.62 & 0.49 & 0.54 & 7080 & 2165 \\
4 & 4 & 8.83 & 8.84 & 8.83 & 0.82 & 0.88 & 0.78 & 7000 & 2167 \\
4 & 9 & 10.79 & 10.81 & 10.80 & 0.54 & 0.67 & 0.60 & 6040 & 1729 \\
4 & 19 & 12.55 & 12.48 & 12.48 & 1.02 & 0.48 & 0.49 & 3120 & 916 \\
4 & 39 & 14.03 & 13.97 & 13.99 & 0.90 & 0.46 & 0.55 & 5160 & 1468 \\
\bottomrule
\end{tabular}

\end{center}
\end{table}

\subsubsection{Vanilla HDPO in Setting S4 \label{appendix:transshipment-warehouse-exps}} \hfill

{\bf Benchmarks.} We consider the setting introduced in \cite{federgruen1984approximations}, where a warehouse $w$ operates as a transshipment center (\ie cannot hold inventory) and there are multiple stores $1, \ldots, K$, under a backlogged demand assumption (fourth row in Table \ref{table:optimal-exps-summary}). Demand is i.i.d. across time but may exhibit correlation across stores. In this setting, demand may take on negative values (corresponding to the possibility of products being returned directly to a store). To apply a known analytical lower bound (see Appendix \ref{appendix:transshipment-setting-lower-bound}) on the optimum, we assume uniform per-unit costs and lead time across all stores. The warehouse has a constant lead time $L^{(0,w)} = 3$ periods from the outside source. We generated $24$ instances by fixing holding costs at $1$, considering number of stores $K=3, 5, 10$, lead times $L^{(w,k)}=2, 6$ from warehouse to stores, underage costs $p^k=4, 9$, and pairwise correlation in demands of $0.0$ and $0.5$. The store-level marginal demand distributions are assumed to be normal, with the mean and coefficient of variation sampled uniformly between $2.5$ to $7.5$ and \edit{$0.25$ to $0.5$}, respectively.

{\bf Baselines.} We compare against the analytical lower bound for this transshipment setting, as detailed in Appendix \ref{appendix:transshipment-setting-lower-bound}.

{\bf Results.} Table \ref{table: trans_shipment detailed} summarizes the performance results. We obtain near-optimal performance across all instances, with a maximum gap of \edit{$0.24\%$} and an average gap of \edit{$0.13\%$}. These results demonstrate the reliability of HDPO in effectively addressing the network problem studied, particularly when the existing constraints can be represented in a differentiable manner. HDPO required approximately \edit{5,300} gradient steps and \edit{14} minutes, on average, to achieve $1\%$ of optimality gap on the dev set.

\begin{table}[h!]
\begin{center}
\caption{Performance metrics of the Vanilla NN for each instance of setting with one "transshipment" warehouse under backlogged demand assumption. Results consider sampling means and coefficients of variation uniformly between $2.5-7.5$ and $0.25-0.5$, respectively. \label{table: trans_shipment detailed}}

\begin{tabular}{>{\raggedleft}p{1.1cm}>{\raggedleft}p{1.1cm}>{\raggedleft}p{1.4cm}>{\raggedleft}p{1.3cm}>{\raggedleft}p{1.1cm}>{\raggedleft}p{0.7cm}>{\raggedleft}p{0.7cm}>{\raggedleft}p{0.7cm}>{\raggedleft}p{1.1cm}>{\raggedleft}p{2.1cm}>{\raggedleft\arraybackslash}p{1.6cm}}
\toprule
Number of stores & Store leadtime & Store underage cost & Pairwise correlation & Lower bound & Train loss & Dev loss & Test loss & Test gap $\leq$(\%) & Gradient steps to 1\% dev gap & Time to 1\% dev gap (s) \\
\midrule
3 & 2 & 4 & 0.0 & 5.15 & 5.15 & 5.15 & 5.16 & 0.10 & 2880 & 408 \\
3 & 2 & 4 & 0.5 & 5.75 & 5.76 & 5.75 & 5.76 & 0.08 & 2880 & 445 \\
3 & 2 & 9 & 0.0 & 6.46 & 6.45 & 6.45 & 6.47 & 0.12 & 3200 & 457 \\
3 & 2 & 9 & 0.5 & 7.21 & 7.23 & 7.21 & 7.22 & 0.09 & 3520 & 493 \\
3 & 6 & 4 & 0.0 & 7.28 & 7.27 & 7.27 & 7.29 & 0.14 & 3200 & 466 \\
3 & 6 & 4 & 0.5 & 7.72 & 7.73 & 7.72 & 7.73 & 0.16 & 2880 & 435 \\
3 & 6 & 9 & 0.0 & 9.13 & 9.11 & 9.12 & 9.14 & 0.16 & 3520 & 587 \\
3 & 6 & 9 & 0.5 & 9.68 & 9.70 & 9.67 & 9.69 & 0.17 & 3520 & 483 \\
\midrule
5 & 2 & 4 & 0.0 & 4.59 & 4.59 & 4.59 & 4.59 & 0.07 & 4800 & 698 \\
5 & 2 & 4 & 0.5 & 5.30 & 5.30 & 5.29 & 5.30 & 0.08 & 4160 & 689 \\
5 & 2 & 9 & 0.0 & 5.76 & 5.75 & 5.76 & 5.76 & 0.08 & 4160 & 617 \\
5 & 2 & 9 & 0.5 & 6.64 & 6.64 & 6.63 & 6.65 & 0.09 & 4160 & 598 \\
5 & 6 & 4 & 0.0 & 6.67 & 6.67 & 6.67 & 6.67 & 0.09 & 5120 & 714 \\
5 & 6 & 4 & 0.5 & 7.17 & 7.19 & 7.17 & 7.18 & 0.15 & 5120 & 731 \\
5 & 6 & 9 & 0.0 & 8.36 & 8.36 & 8.36 & 8.37 & 0.12 & 5440 & 821 \\
5 & 6 & 9 & 0.5 & 8.99 & 9.01 & 8.98 & 9.00 & 0.18 & 5120 & 861 \\
\midrule
10 & 2 & 4 & 0.0 & 4.67 & 4.68 & 4.68 & 4.68 & 0.11 & 7040 & 1035 \\
10 & 2 & 4 & 0.5 & 5.54 & 5.55 & 5.54 & 5.54 & 0.13 & 7040 & 1048 \\
10 & 2 & 9 & 0.0 & 5.86 & 5.87 & 5.87 & 5.87 & 0.18 & 7680 & 1267 \\
10 & 2 & 9 & 0.5 & 6.94 & 6.96 & 6.95 & 6.95 & 0.15 & 9280 & 1348 \\
10 & 6 & 4 & 0.0 & 6.94 & 6.94 & 6.94 & 6.95 & 0.13 & 7040 & 987 \\
10 & 6 & 4 & 0.5 & 7.55 & 7.57 & 7.56 & 7.56 & 0.23 & 7360 & 1127 \\
10 & 6 & 9 & 0.0 & 8.70 & 8.71 & 8.71 & 8.71 & 0.17 & 8320 & 1202 \\
10 & 6 & 9 & 0.5 & 9.46 & 9.50 & 9.48 & 9.48 & 0.24 & 9600 & 1430 \\
\bottomrule
\end{tabular}

\end{center}
\end{table}

\subsection{Impact of demand censoring in setting S2
\label{appendix:censored-demand} }

In practice, demand can only be observed when sufficient inventory is available to meet it, so a realistic dataset might only include sales data. HDPO is designed to operate after, and independently of, the preprocessing stage via which demand is imputed from sales and other available data. This section aims to justify our treatment by presenting the following empirical evidence: (i) early experiments suggest that even with small, censored datasets, existing de-censoring techniques enable the effective application of HDPO with minimal performance degradation, and (ii) promisingly, initial results indicate that the robustness of NNs trained with HDPO to distribution shifts caused by misspecified demand is comparable to that of simple heuristics. Additionally, this section serves as a concise tutorial, discussing the challenges of applying HDPO with censored data, exploring potential solutions, and evaluating the performance impact of incorrect de-censoring.

There are several methods to estimate demand from offline censored datasets. When demand is assumed to be stationary, the Kaplan-Meier estimator \citep{kaplan1958nonparametric}, a classical method for estimating survival functions in the presence of right-censored data, can be applied. We refer readers to \cite{huh2011adaptive} for its formal definition (see Section 3) and for an explanation of its use in inventory control, including general conditions under which the estimator converges to the true distribution (see Theorem 2).

In practice, retailers might also have access to additional information that facilitates demand estimation utilizing other ad-hoc methods. For example, \cite{madeka2022deep} leveraged web glance views to infer demand from sales data by exploiting the assumption that the conversion rate (\ie the ratio of glance views that turn into a sale) is constant for each product and period, thus allowing to directly impute demand in the presence of stock-outs. Using this de-censoring approach, they trained HDPO and reported promising results in a real-world deployment, outperforming a Newsvendor benchmark.

\subsubsection{Performance under successful de-censoring. \label{appendix:succesful-decensoring} }

To assess the robustness of the Vanilla NN on de-censored data, we analyze performance in a setting with lost demand and stochastic censoring. Our goal is to demonstrate that HDPO remains effective even when operating on imputed demand distributions, showing comparable robustness to classical heuristics. In the absence of realistic data, we simulate an environment that aims to reflect the natural variation of stocking levels due to policy decisions and random lead times, allowing us to analyze performance under controlled conditions while capturing the key challenges of demand censoring.

For demand imputation, we employ the classical Kaplan-Meier estimator, a well-established non-parametric method for handling censored data. When the largest sales observation is censored, we extend the estimated distribution using a simple exponential tail fit via Maximum Likelihood Estimation. This straightforward approach avoids sophisticated or contrived techniques, relying instead on standard statistical methods to reconstruct the demand distribution from censored observations.

{\bf Benchmarks.} We consider setting S2 with a Poisson demand distribution with mean 5, censored by a Poisson process with mean 6. We set underage and holding costs to $p=4$ and $h=1$ respectively, with a deterministic lead time of $2$ periods.

{\bf Baselines.} We consider the CBS heuristic (see Equation \eqref{capped base stock} in Section \ref{appendix:capped-base-stock-policy}.

{\bf Experiment specifications.} To ensure robust evaluation, we considered sample sizes of $10^2, 10^3, 10^4$ and $10^5$ samples and $32$ random seeds for the demand and censoring process used for estimating the CDF. For each combination of seed and sample size, we estimate a CDF, then generate training and dev data from that estimated CDF, while testing is performed using data generated from the true CDF. We use a large sample of 32,768 samples for train, dev, and test sets (the aforementioned sample sizes refer to those used to estimate the CDF). For each combination of seed and sample size, we trained the Vanilla NN model once with its fixed hyperparameters and the CBS policy once for each of three learning rates (1.0, 0.5, 0.1), selecting the one that performed best on a dev set generated using the estimated CDF.

{\bf Results.} Our results, illustrated in Figures \ref{fig: KM CDF estimation} and \ref{fig: KM results}, demonstrate strong robustness to demand censoring. The performance loss is comparable between the Vanilla NN and Capped Base Stock heuristic, with gaps below $0.5\%$ when the estimated distribution closely approximates the true one (for $10^3$, $10^4$, and $10^5$ demand observations). Even with limited data ($10^2$ observations), the gaps remain below $2\%$. These findings suggest that HDPO can be applied directly with small censored datasets when using appropriate demand imputation, maintaining its effectiveness despite the additional preprocessing step.

\begin{figure}
\captionsetup[subfigure]{justification=centering}

\begin{subfigure}{\textwidth}
  \centering
  \includegraphics[width=1.0\linewidth]{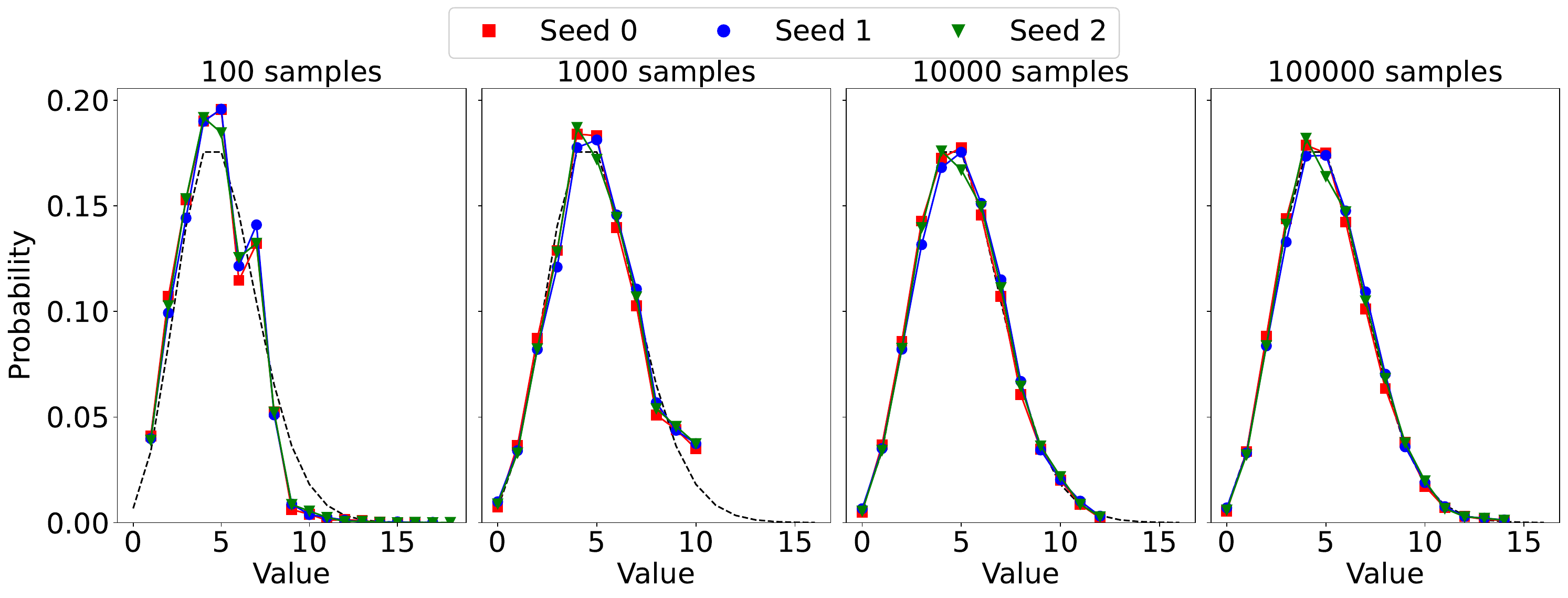}
  \caption{Demand densities estimated using the Kaplan-Meier estimator for varying sample sizes ($10^2$, $10^3$, $10^4$, and $10^5$ observations) for $3$ different random seeds. Dashed black line corresponds to the true distribution (Poisson with a mean of $5$).}
  \label{fig: KM CDF estimation}
\end{subfigure}

\begin{subfigure}{\textwidth}
  \centering
  \includegraphics[width=0.8\linewidth]{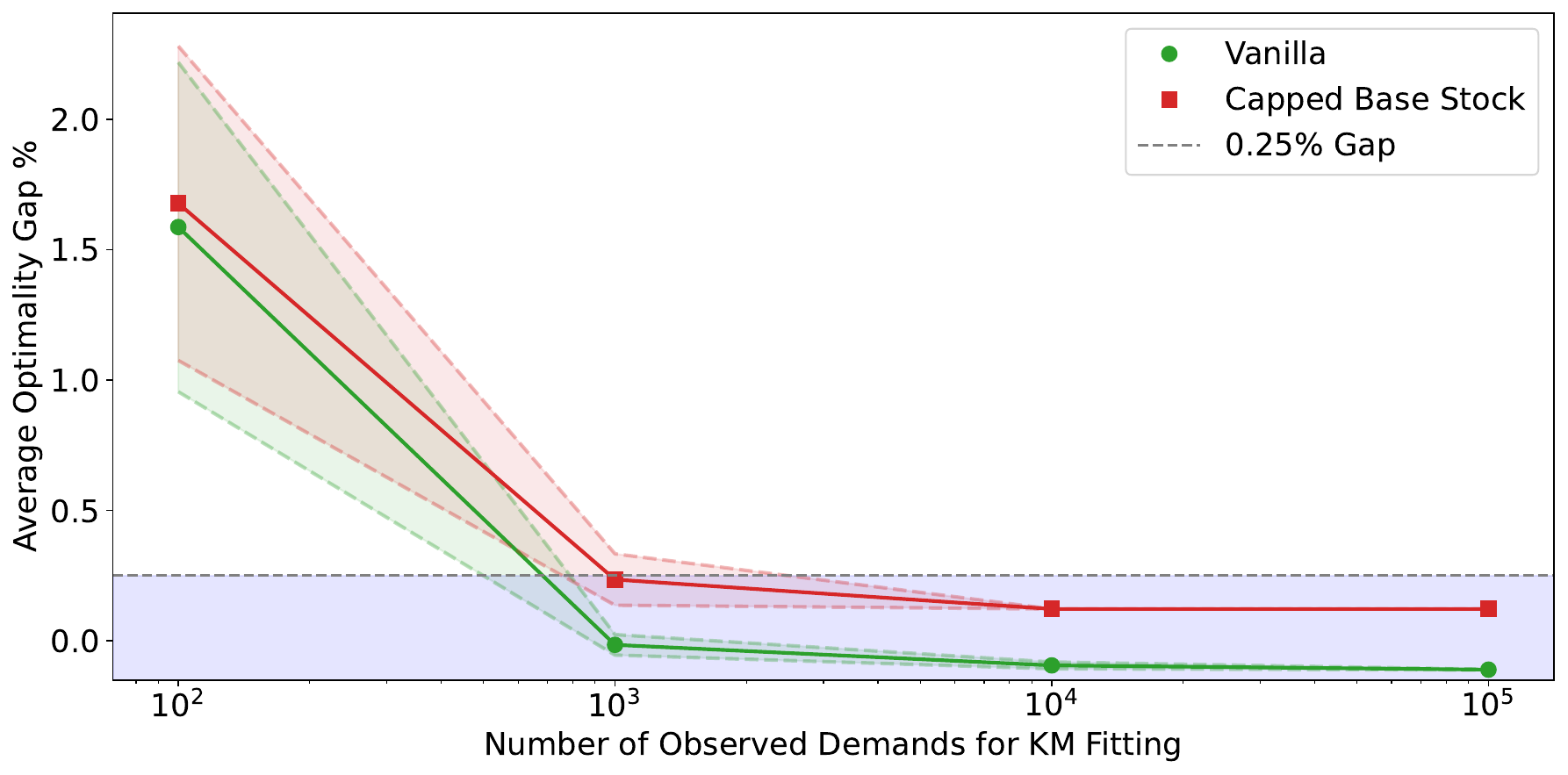}
  \caption{Average optimality gap (mean and 95\% confidence interval) for the Vanilla NN and Capped Base Stock heuristic across different sample sizes; performances are comparable. Results are averaged over multiple random seeds, with each model selected based on its best performance in a dev set generated using the estimated CDF. The performance of each model is evaluated on a test set generated from the true
Poisson demand distribution with a mean of 5. We consider that gaps below $0.25\%$ cannot be "detected", given that optimal costs are reported in \cite{zipkin2008old} with two decimal places.}
  \label{fig: KM results}
\end{subfigure}

\caption{
Evaluation of demand estimation using the Kaplan-Meier estimator (a) and policy performance (b) under varying sample sizes used for demand estimation. Setting of a single retailer operating under a lost demand assumption, with an underage cost of $4$, a holding cost of $1$, and a lead time of $2$ periods.}
\label{fig: KM}
\end{figure}

\subsubsection{Evidence of robustness of the Vanilla NN to unavoidable de-censoring errors. \label{appendix:unsuccesful-decensoring} }

In practice, a retailer might encounter situations where the dataset lacks the information required to accurately impute demand, even as the number of observations grows. One such scenario arises when systematic under-ordering prevents large demand values from being observed. To analyze HDPO's robustness in such settings, we consider scenarios where the tail of the distribution must be imputed using methods such as specifying a family of distributions and estimating its parameters via Maximum Likelihood Estimation, which may lead to performance degradation in downstream models.

For our analysis, we employ a straightforward censoring and imputation approach. We introduce a censoring threshold $M$, above which all demands are censored, and assume accurate demand estimation below this threshold due to sufficient data availability. For censored observations, we simply truncate demand at $M$ and add a sample from a predefined tail distribution. Formally, if $D_i$ is a demand sampled from the true distribution, we estimate demand as $\hat{D}_i = \min\{D_i, M\} + B_i \mathbbm{1}(D_i \geq M)$, where $B_i$ is drawn from the tail distribution.

{\bf Benchmarks.} We consider setting S2 with a single retailer under a lost demand assumption, with true demand following a Poisson distribution with mean 5. We set underage and holding costs to $p=4$ and $h=1$, with a lead time of $2$ periods.

{\bf Baselines.} We consider the CBS heuristic (see Equation \eqref{capped base stock} in Section \ref{appendix:capped-base-stock-policy}.

{\bf Experiment specifications.} For tail imputation, we use a Weibull distribution with varying scale ($\lambda = 1.0, 2.0, 3.0$) and shape ($k = 0.8, 1.0, 1.2$) parameters, testing censoring thresholds $M = 5, 6, 7, 8$. The resulting distributions for these de-censoring configurations are shown in Figure \ref{fig: Weibull CDF estimation}, spanning a wide range of tail behaviors from significantly lighter to substantially heavier than the original distribution. For each configuration, we estimate a CDF using the Weibull imputation, then generate training and development data from that estimated CDF, while testing is performed using data generated from the true Poisson distribution. We use a large sample of 32,768 samples for train, dev, and test sets. For each configuration, we consider one run for the Vanilla NN and one run for each of learning rates \{1.0, 0.5, 0.1\} for the Capped Base Stock heuristic, select the best-performing model on a dev set of de-censored demand data, and evaluate on a test set considering the true demand distribution.

{\bf Results.} Our results, illustrated in Figure \ref{fig: Weibull results}, demonstrate that while performance depends significantly on the chosen imputation scheme, the scale of degradation is comparable between the Vanilla NN and the CBS heuristic. This suggests that NNs trained via HDPO can achieve robustness on par with simple heuristics, even when operating with imperfectly estimated demand distributions.

\begin{figure}
\captionsetup[subfigure]{justification=centering}

\begin{subfigure}{\textwidth}
  \centering
  \includegraphics[width=1.0\linewidth]{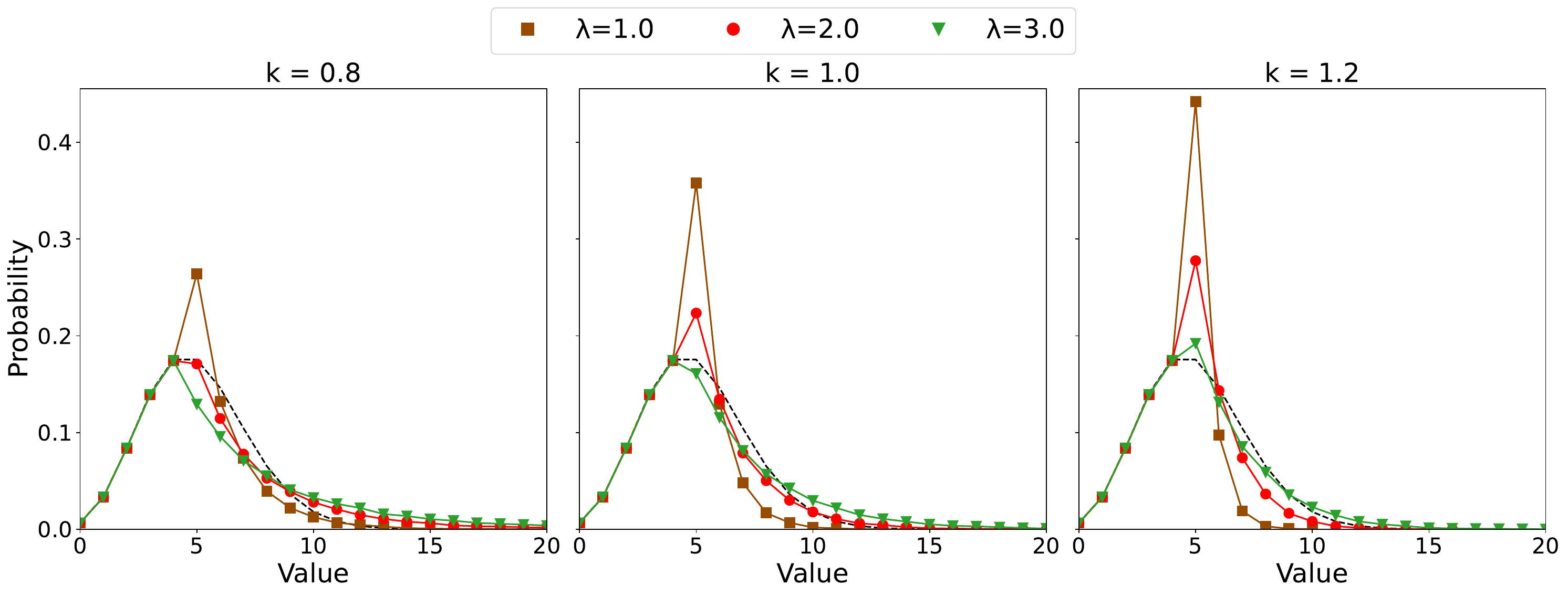}
  \caption{Estimated demand densities obtained by imputing the censored tail using a Weibull distribution with various scale ($\lambda = 1.0, 2.0, 3.0$) and shape ($k = 0.8, 1.0, 1.2$) parameters. The dashed black line corresponds to the true distribution (Poisson with a mean of 5).}
  \label{fig: Weibull CDF estimation}
\end{subfigure}

\begin{subfigure}{\textwidth}
  \centering
  \includegraphics[width=1.0\linewidth]{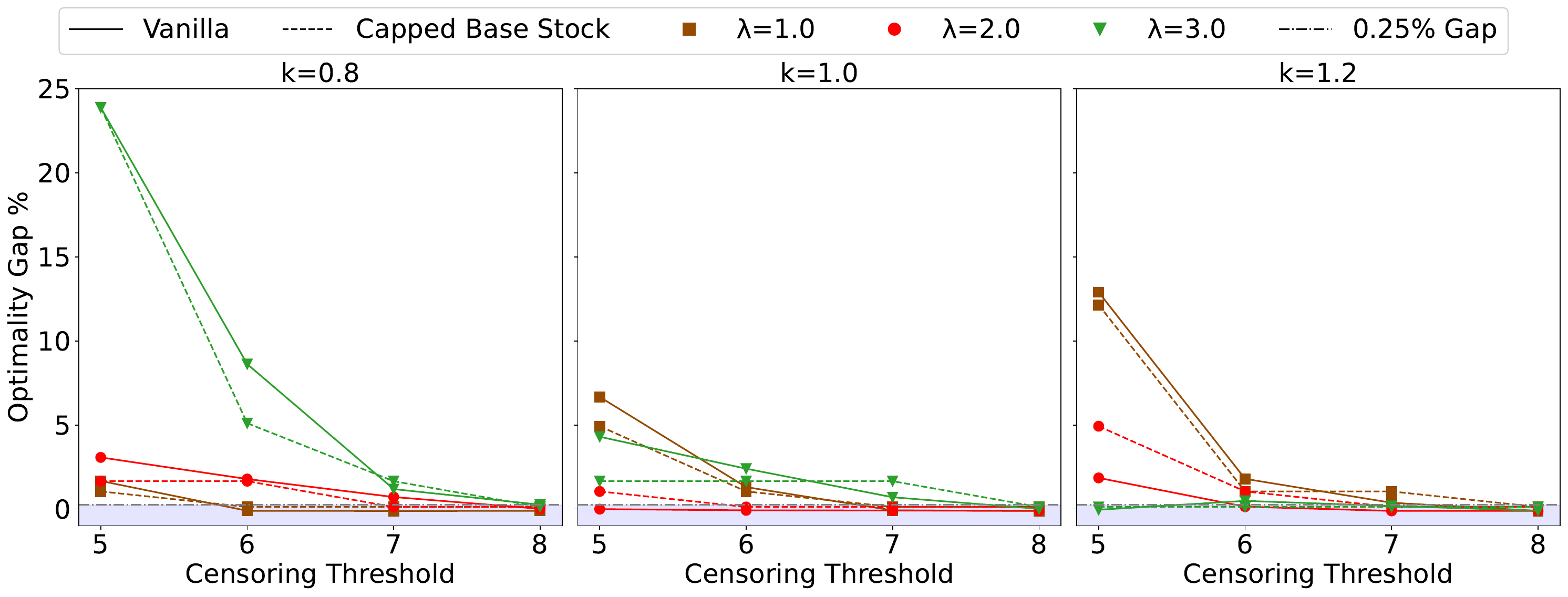}
  \caption{Optimality gap for the Vanilla NN and CBS heuristic across different parameters of the Weibull distribution used to impute the tail; performances are comparable. Each model is selected based on its best performance in a dev set generated using the estimated CDF. The performance of each model is evaluated on a test set generated from the true
Poisson demand distribution with a mean of 5. We consider that gaps below $0.25\%$ cannot be "detected", given that optimal costs }
  \label{fig: Weibull results}
\end{subfigure}

\caption{Evaluation of demand estimation (a) and policy performance (b) using a Weibull distribution under different parameters to impute the tail. Setting S2, with an underage cost of 4, a holding cost of 1, and a lead time of 2 periods.}
\label{fig: Weibull}
\end{figure}

\subsection{Sample efficiency experiment in Setting S2 \label{appendix:vanilla-hdpo-sample-efficiency} }

In this section, we evaluate the sample efficiency of Vanilla HDPO in setting S2, which considers a single store with Poisson demand and lost sales. 

{\bf Benchmarks.} We consider setting S2, fixing holding cost to $1$, the mean of Poisson demand to $5$. We consider lead times of 3 and 4 periods and underage costs of $9.0$ and $19.0$.

{\bf Baselines.} We compare against the optimal costs as reported by \cite{zipkin2008old}.

{\bf Experiment specifications.} We consider 4 meta-instances and 10 runs (each with a different seed for generating demand samples) for each meta-instance. We trained the Vanilla NN on varying numbers of training scenarios, with the dev set size matching the training set size. We train considering continuous actions and round to the nearest integer at test time. Each scenario corresponds to a synthetic demand trace for a product across 50 periods—approximately one year of sales. We average the results over the 10 independent training runs and evaluate on a common test set of 32,768 scenarios.

{\bf Results.} Figure~\ref{fig:sample-efficiency-one-store} reports the mean and 95\% confidence intervals of the optimality gap across the 40 training runs. HDPO achieves optimality gaps below 1\% with as few as 128 training scenarios (where each scenario typically corresponds to one product's demand trace), and performance nearly indistinguishable from the optimum with 512 scenarios. These results suggest that even retailers with a relatively small number of products may benefit from applying HDPO.

\begin{figure}[ht]
    \centering
    
    \includegraphics[width=0.8\textwidth]{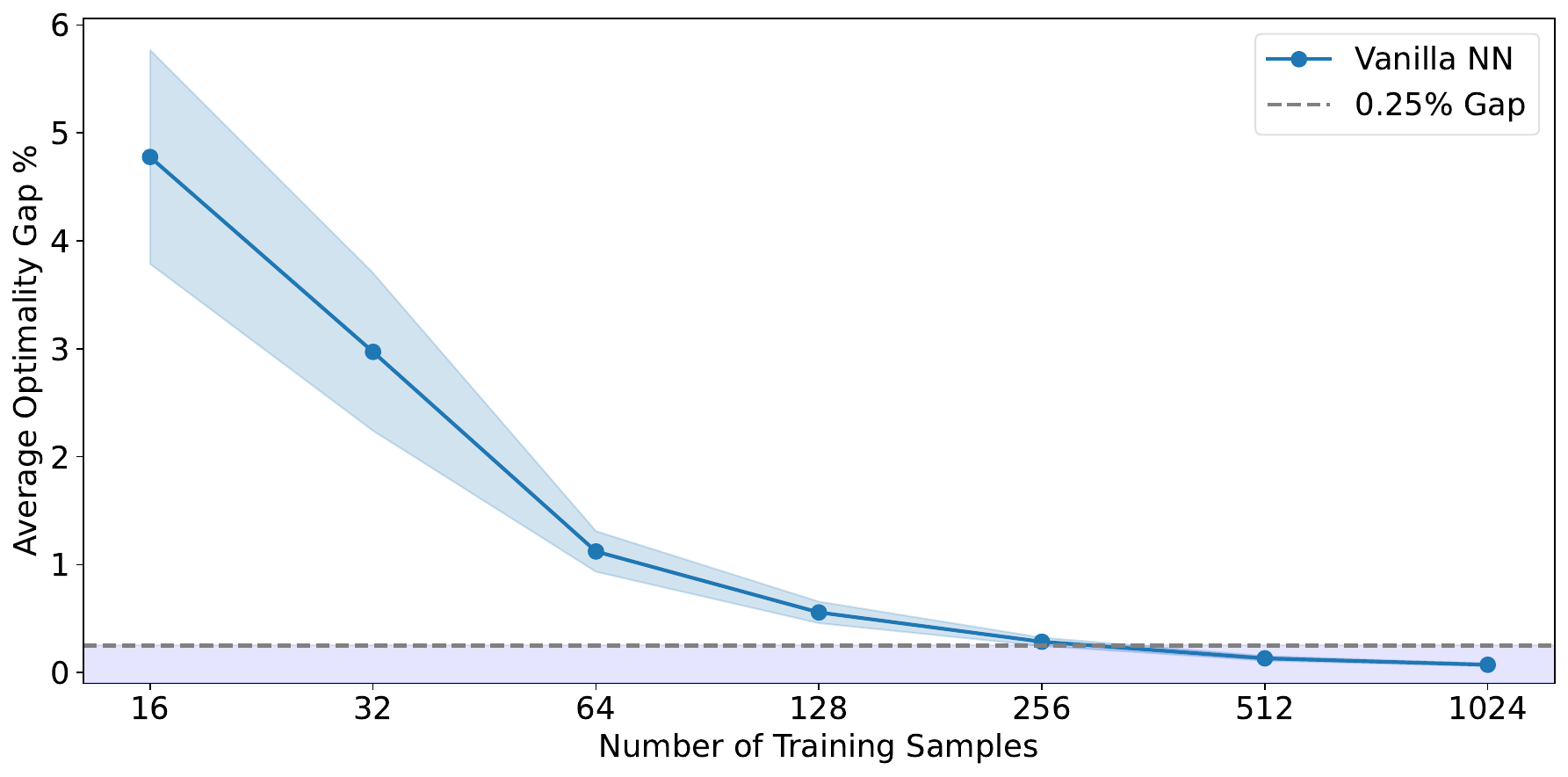}
    \caption{Sample efficiency of HDPO in setting S2. We tested across 4 meta-instances (with a fixed instance per meta-instance) and 10 independent training runs (each with a different seed) per meta-instance. Each scenario represents a synthetic demand trace over $50$ periods. The average optimality gaps are shown with 95\% confidence intervals.}
    \label{fig:sample-efficiency-one-store}
\end{figure}

\subsection{Vanilla HDPO in Setting S5 \label{appendix:one-store-real-setting}}

{\bf Benchmarks, Baselines and Experiment specifications.} See Section \ref{sec:vanilla-hdpo-realistic}.

{\bf Results.}
Table \ref{table:real-data-lost-demand} showcases the performance of our approach across different average unit underage costs. It is important to note that, despite including the elapsed time and the number of gradient steps required to achieve a $1\%$ performance gap relative to a baseline, our approach was not specifically optimized for speed. These metrics would likely be considerably smaller if the hyperparameters were selected with speed optimization as the primary objective.

Table \ref{table:real-data-lost-demand} demonstrates that our approach exhibits robust generalization, as relative profits remain consistent across train, dev, and test sets. Furthermore, our approach consistently achieves profits within $1\%$ of the Transformed Newsvendor baseline in just $\edit{10}$ seconds per instance. Notably, as shown in Figure \ref{fig: real_data_lost_demand} in Section \ref{sec:vanilla-hdpo-realistic}, the Transformed Newsvendor policy achieves profits within $5\%$ of our approach for instances with unit underage costs of $9$ or greater. This indicates that our approach delivers strong performance with minimal computational time.

\begin{table}[h!]
\begin{center}
\caption{Performance metrics of the Vanilla NN for each instance of the setting with one store considering sales data from the \textit{Corporación Favorita Grocery Sales Forecasting} competition, under a lost demand assumption. Evaluation on dev set was performed every $\edit{10}$ epochs (equivalent to $\edit{40}$ gradient steps), so time and gradient steps represent upper bounds. Profits are relative to the Just-in-time policy, and the Transformed Newsvendor is considered as the baseline.
\label{table:real-data-lost-demand}}
\begin{tabular}{>{\raggedleft}p{1.5cm}>{\raggedleft}p{1.2cm}>{\raggedleft}p{1.2cm}>{\raggedleft}p{1.2cm}>{\raggedleft}p{1.5cm}>{\raggedleft}p{1.5cm}>{\raggedleft}p{1.5cm}>{\raggedleft}p{2.0cm}>{\raggedleft\arraybackslash}p{2.0cm}}
\toprule
\raggedright Average unit underage cost & \raggedright Train profit & \raggedright Dev profit & \raggedright Test profit & \raggedright Relative train profit (\%) & \raggedright Relative dev profit (\%) & \raggedright Relative test profit (\%) & \raggedright Time to 1\% gap from baseline (s) & \raggedright\arraybackslash Gradient steps to 1\% gap from baseline \\
\midrule
2  & 69.35  & 70.53  & 64.90  & 64.83 & 67.83 & 66.01 & 10 & 40 \\
3  & 112.55 & 113.06 & 104.26 & 70.14 & 72.49 & 70.69 & 10 & 40 \\
4  & 157.60 & 157.26 & 144.89 & 73.66 & 75.62 & 73.68 & 10 & 40 \\
6  & 250.42 & 248.41 & 228.64 & 78.02 & 79.63 & 77.51 & 10 & 40 \\
9  & 394.70 & 389.39 & 359.59 & 81.98 & 83.21 & 81.27 & 10 & 40 \\
13 & 590.35 & 582.13 & 537.57 & 84.89 & 86.12 & 84.11 & 10 & 40 \\
19 & 890.04 & 876.13 & 809.90 & 87.56 & 88.68 & 86.70 & 10 & 40 \\
\bottomrule
\end{tabular}

\end{center}
\end{table}

Lastly, Figure \ref{fig:profit-period-lost-demand} displays the weekly profit, averaged across scenarios, for each week and policy in the test set. The illustration highlights that our approach consistently achieves higher profits almost every week, emphasizing the reliability of our approach.

\begin{figure}
    \centering
    \includegraphics[width=0.8\textwidth]{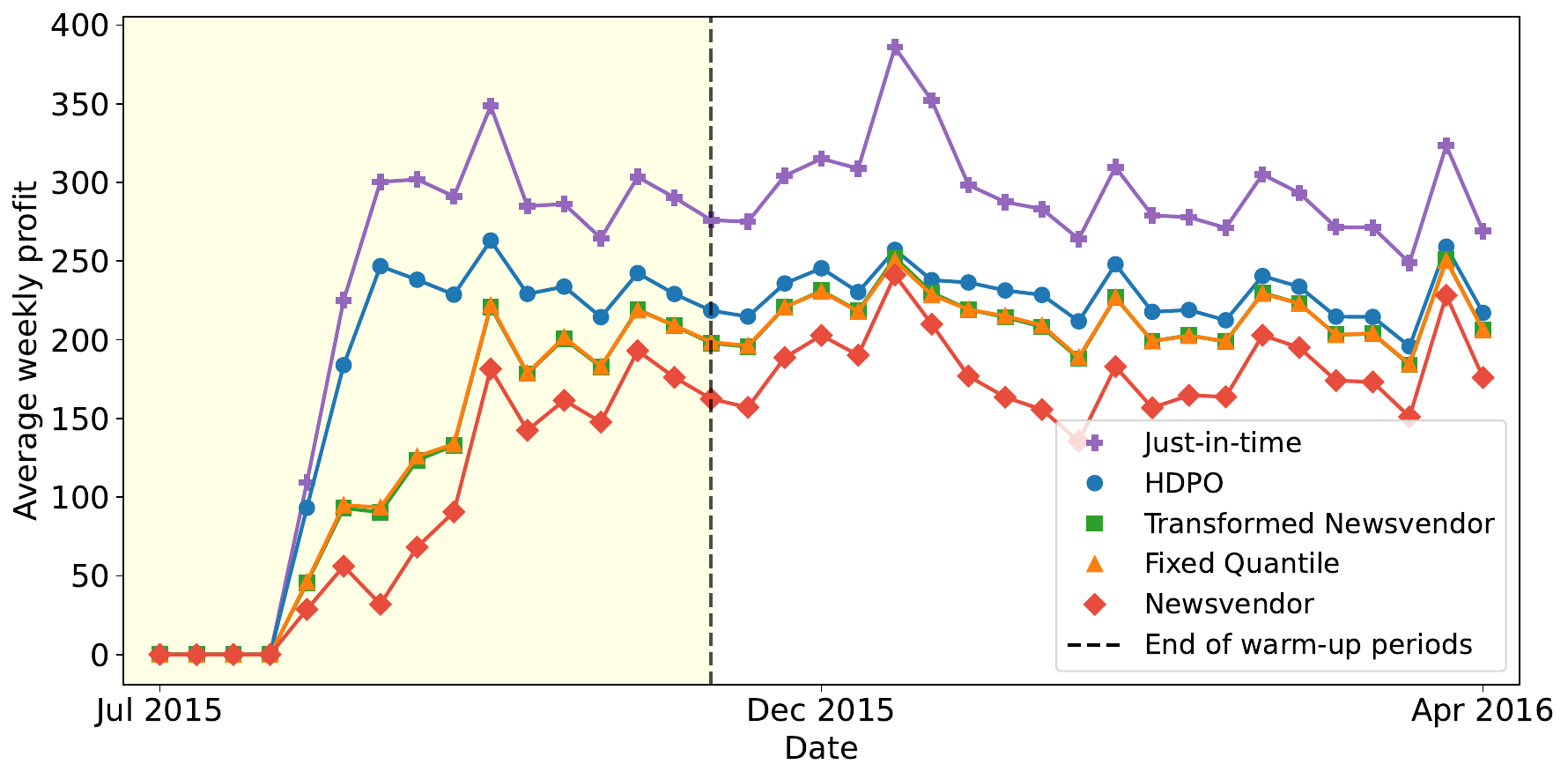}
    \caption{Average weekly profit on the \edit{test} set obtained by each policy in the setting with one store and realistic demand, for an average unit underage cost of $6$ and under a lost demand assumption. \edit{The yellow shaded area denotes warm-up periods omitted from evaluation to minimize the influence of initial conditions.}}
    \label{fig:profit-period-lost-demand}
\end{figure}

\subsubsection{What do generalized newsvendor policies miss? 
\label{sec: what-heuristics-miss} }

Quantile policies are designed to be effective in a backlogged demand setting, where all incoming demand depletes the existing inventory. 
In contrast, in a lost demand setting demand does not deplete inventory upon a stock-out, posing a challenge for policies optimized for backlogged demand settings. While adjusting the target quantile can partially address this issue, the underestimation of inventory may depend on the current inventory state in a non-trivial manner. As anticipated, in the lost demand setting, the performance of generalized newsvendor policies improves with an increase in average underage cost (see Figure \ref{fig: real_data_lost_demand} in Section \ref{sec:vanilla-hdpo-realistic}), as unmet demand is expected to decrease. This observation aligns with previous analyses in stylized settings considering a lost demand assumption, where base-stock policies become asymptotically optimal as the underage cost grows large \citep{sun2014quadratic}.

To conduct a more detailed analysis, we calculate the \textit{implied quantile} that the agent orders up to by "inverting" the target level $(a_t + X_t)$ using our quantile forecaster. In other words, we find the $\tau^{\pi}$ that solves $H(L, \mathcal{F}_t)(\tau^{\pi}) = (a_t + X_t)$ following Equation \eqref{eq: quantile-policy} (with the implied quantile set to $0$ when $a_t = 0$ since we cannot solve for $\tau^{\pi}$ in that case). We standardize this quantity by subtracting the mean and dividing by the standard deviation for each scenario. Subsequently, we define the \textit{stock-out ratio} as the cumulative unmet demand until the order at time $t$ arrives (\ie $L$ periods into the future), divided by the scenario's average cumulative demand across $L$ periods.

In Figure \ref{fig:quantile_vs_lost_demand}, we group the standardized implied quantile into buckets of width $0.1$ and compute the sample average stock-out ratio for each bucket, for the meta-instance with $\hat{u} = 9$. We interpret this as the expected stock-out ratio observed by the agent at time $t$ when defining each quantile. The analysis reveals that our agent tends to place orders corresponding to lower quantiles when predicting significant stock-outs in subsequent periods, highlighting the potential underestimation of future inventory by a generalized newsvendor policy in such cases.

\begin{figure}
    \centering
    \includegraphics[scale=0.6]{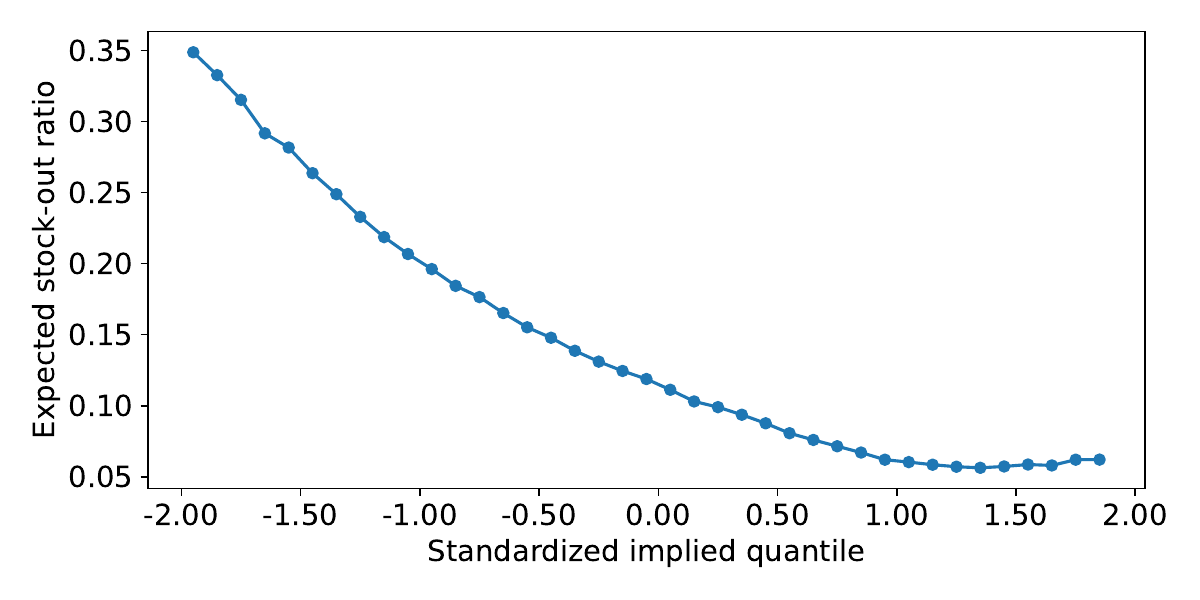}
    \caption{Standardized implied quantile versus expected stock-out ratio for setting S5 considering an average unit underage cost of $9$ and a lost demand assumption. The standardized implied quantile is binned with a width of $0.1$, and the expected stock-out ratio is computed by averaging the sample stock-out ratio within each bin.}
    \label{fig:quantile_vs_lost_demand}
\end{figure}

An additional drawback of generalized newsvendor policies lies in their failure to incorporate value functions. For products with pronounced seasonality, such as Christmas trees, excessive ordering may lead to incurring additional holding costs across multiple periods beyond the one in which orders arrive. \edit{In our experiments involving large unit underage costs (in which stock-outs are rare and therefore considering demand to be backlogged is a good approximation),} the over-ordering indicated by the Newsvendor quantile can be partially mitigated by statically adjusting the target quantile, as demonstrated by the favorable performance of the Transformed Newsvendor policy (refer to Figure \ref{fig: real_data_lost_demand}). However, it is reasonable to anticipate that in a scenario featuring predictable short-term demand surges (such as planned promotions or national holidays), the effectiveness of generalized newsvendor policies might decline.

Finally, it is crucial to emphasize that constructing a generalized newsvendor policy, which involves building and training a forecaster for multiple time horizons and quantiles, requires significantly more computational resources and time compared to building and training an agent directly on the downstream task. Furthermore, determining how to divide a task into prediction and optimization stages becomes increasingly challenging as the problem incorporates more realistic features, such as the inclusion of random lead times. This underscores additional advantages of an end-to-end approach over strategies that separate prediction and decision-making.

\subsection{Sample efficiency experiments in Settings S3, S4 and S6--S9 \label{appendix:sample-efficiency-results}}

This appendix provides detailed specifications for the six experimental settings evaluated in Section~\ref{sec:sample-efficiency-results}. These settings are designed to evaluate the sample efficiency of different policy architectures across a range of inventory control problems, which vary in terms of network structure, demand model, and the number of locations. Each setting is described in detail in the following sub-sections.

{\bf Experiment specifications.}
Across all settings with synthetic demand (Settings S3, S4, S6 and S8), we assess sample efficiency by varying the number of training scenarios over ${128, 1024, 8192}$. For the two settings that use realistic demand data (Settings S7 and S9), the training set considers $64$ and $288$ product-level demand traces, with the dev and test sets having the same size (see Appendix \ref{appendix:global-settings} for the temporal split specifications).

To ensure a fair comparison between architectures, we conduct a grid search over the most relevant hyperparameters. While more sophisticated methods such as Bayesian optimization could be used, our goal is to standardize the tuning process across models. Specifically, we vary the learning rate and the size of each NN layer. For the Vanilla NN architecture, we fix the number of hidden layers to $3$, and consider hidden layer widths of $128, 256, 512$ neurons and learning rates of $10^{-4}, 10^{-3}, 10^{-2}$. The only exception is setting S3, where we consider 2 hidden layers and widths of $32, 64, 128$. This treatment follows the best-performing parameters previously identified, which are shown in Table \ref{tab:vanilla-defaults}. 
For the GNN-based architecture, we vary learning rates over the same set and consider other hyperparameters to be fixed, as specified in Table \ref{tab:gnn-defaults}. Each configuration is run three times for every architecture.

For settings involving synthetic demand, the batch size for training is set to the minimum of 1024 and the total number of training scenarios. For settings with realistic demand, we set the batch size to 72 samples since the datasets are significantly smaller. In early experiments, we also evaluated the use of weight decay to mitigate overfitting, but it yielded no consistent improvements and was excluded from final experiments.

For each fixed meta-instance, sample size and architecture class, we report the metrics corresponding to the run that minimizes the dev loss among all runs and hyperparameters
For settings with a known bound on the optimal cost (Settings S3 and S4), we report percentage optimality gaps. In settings with synthetic demand but no such bound (Settings S6 and S8), we report the relative excess cost (\%)—measured against the best-performing run across all architectures, hyperparameters, and sample sizes for a fixed number of locations. For the settings involving realistic demand (S7 and S9), we report the percentage of achievable profit relative to the Just-in-time benchmark (see Appendix~\ref{appendix:bound-just-in-time}), which serves as an upper bound on achievable performance.

{\bf Results.}
We summarize the performance of the Vanilla and GNN architectures across all six settings and training sample sizes in Figure~\ref{fig:sample_efficiency_subfigure}. More detailed results for each setting are shown in Tables \ref{table:sample-efficiency-serial}-\ref{table:sample-efficiency-mwms-realistic}. These tables report, for each meta-instance, number of training scenarios, and architecture class, the hyperparameters that achieved the best dev performance, alongside the corresponding train loss, dev loss, and test loss. The final column presents a relative performance metric (optimality gaps, relative excess costs, or percentage of achievable profit, depending on the setting as explained in the previous paragraph). We highlight the best test performance for each meta-instance in bold.

\begin{figure}[htbp!]
\captionsetup[subfigure]{justification=centering}
\centering

\begin{subfigure}[t]{.32\textwidth}
  \centering
  \includegraphics[width=\linewidth]{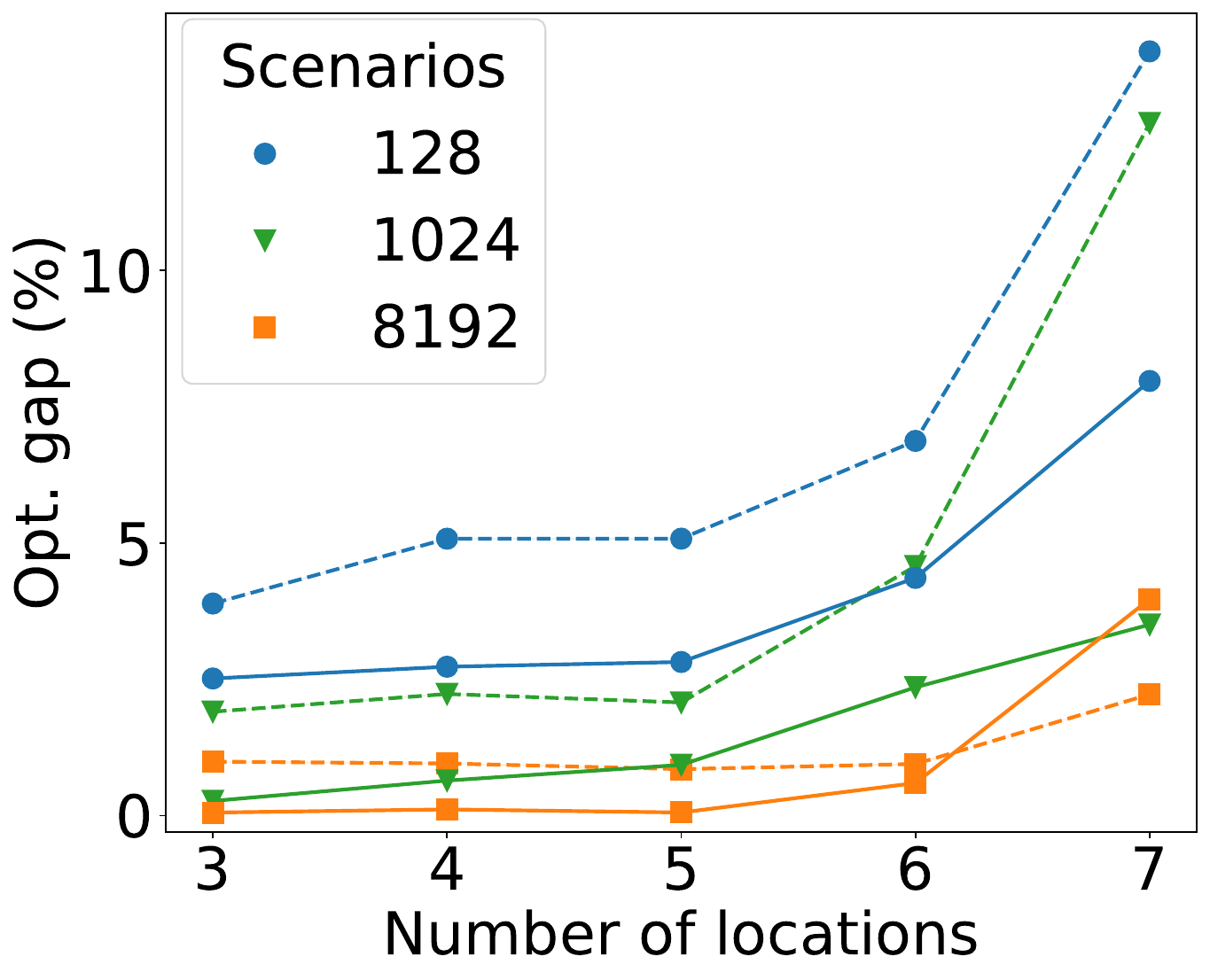}
  \caption{Setting S3. \textbf{Lower is better.}}
  \label{fig:sample-efficiency-serial}
\end{subfigure}%
\hspace{0.015\textwidth}%
\begin{subfigure}[t]{.32\textwidth}
  \centering
  \includegraphics[width=\linewidth]{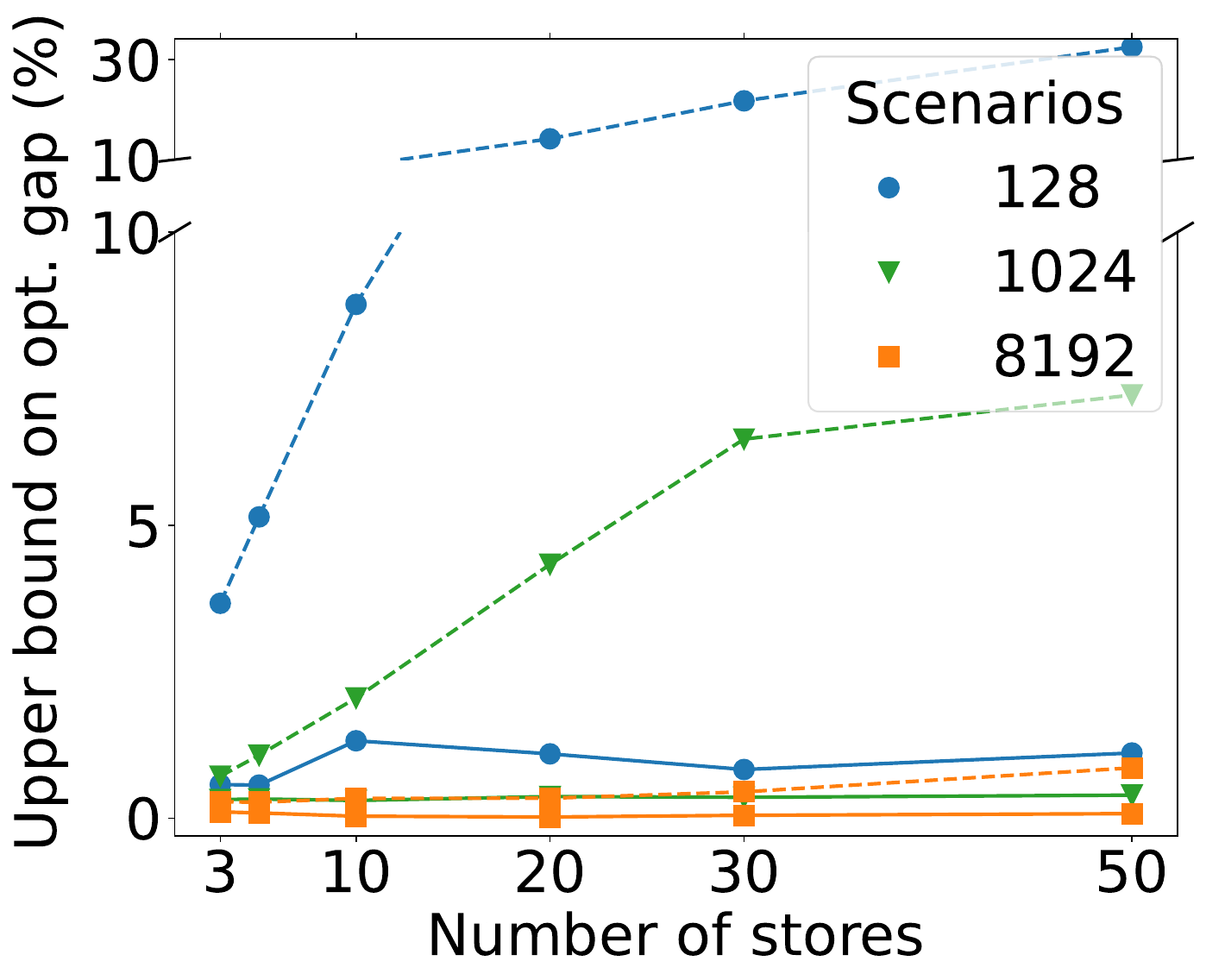}
  \caption{Setting S4. \textbf{Lower is better.}}
  \label{fig:sample-efficiency-1dc-trans}
\end{subfigure}%
\hspace{0.015\textwidth}%
\begin{subfigure}[t]{.32\textwidth}
  \centering
  \includegraphics[width=\linewidth]{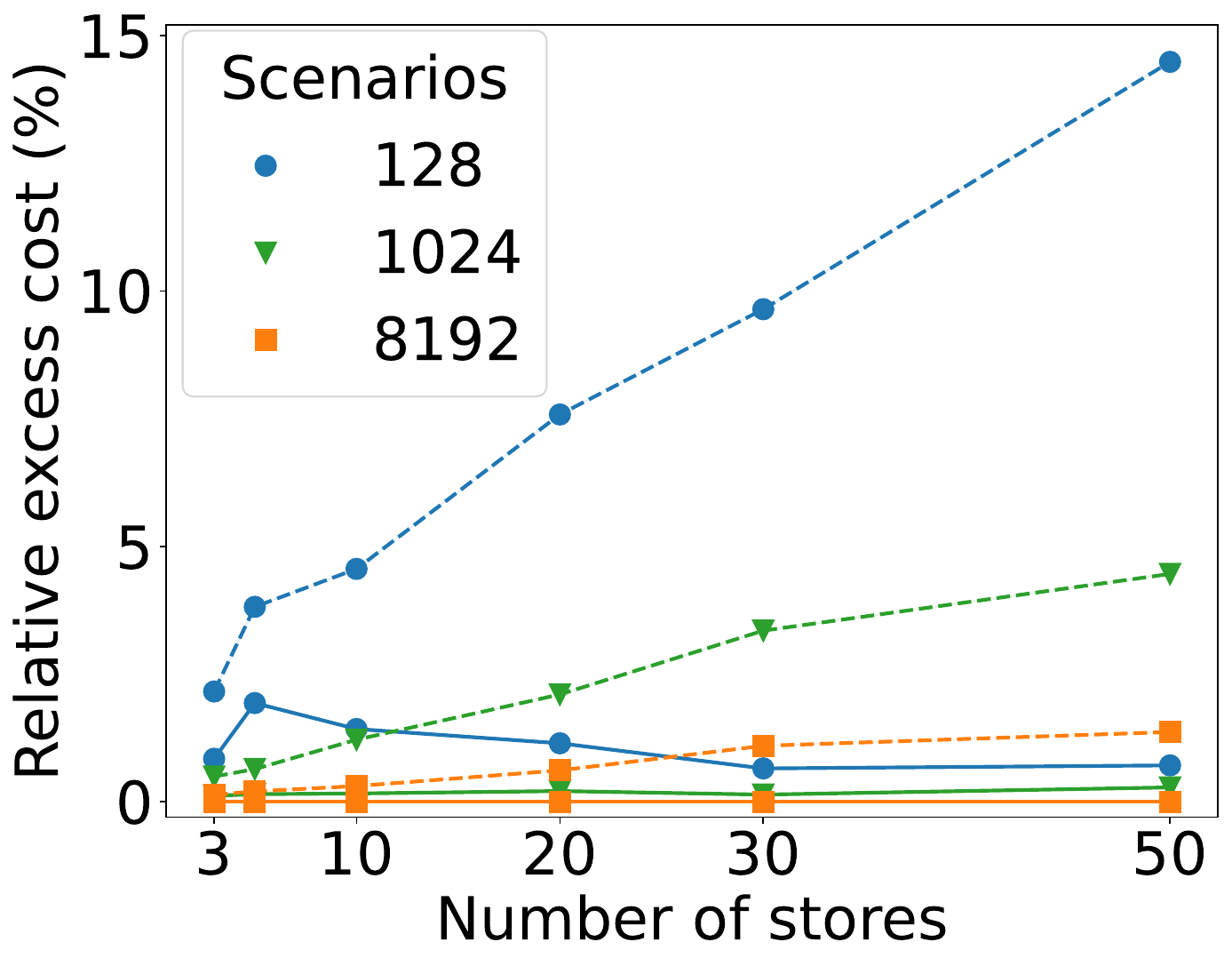}
  \caption{Setting S6. \textbf{Lower is better.}}
  \label{fig:sample-efficiency-1dc-lost}
\end{subfigure}
\vspace{0.25cm}

\begin{subfigure}[t]{.32\textwidth}
  \centering
  \includegraphics[width=\linewidth]{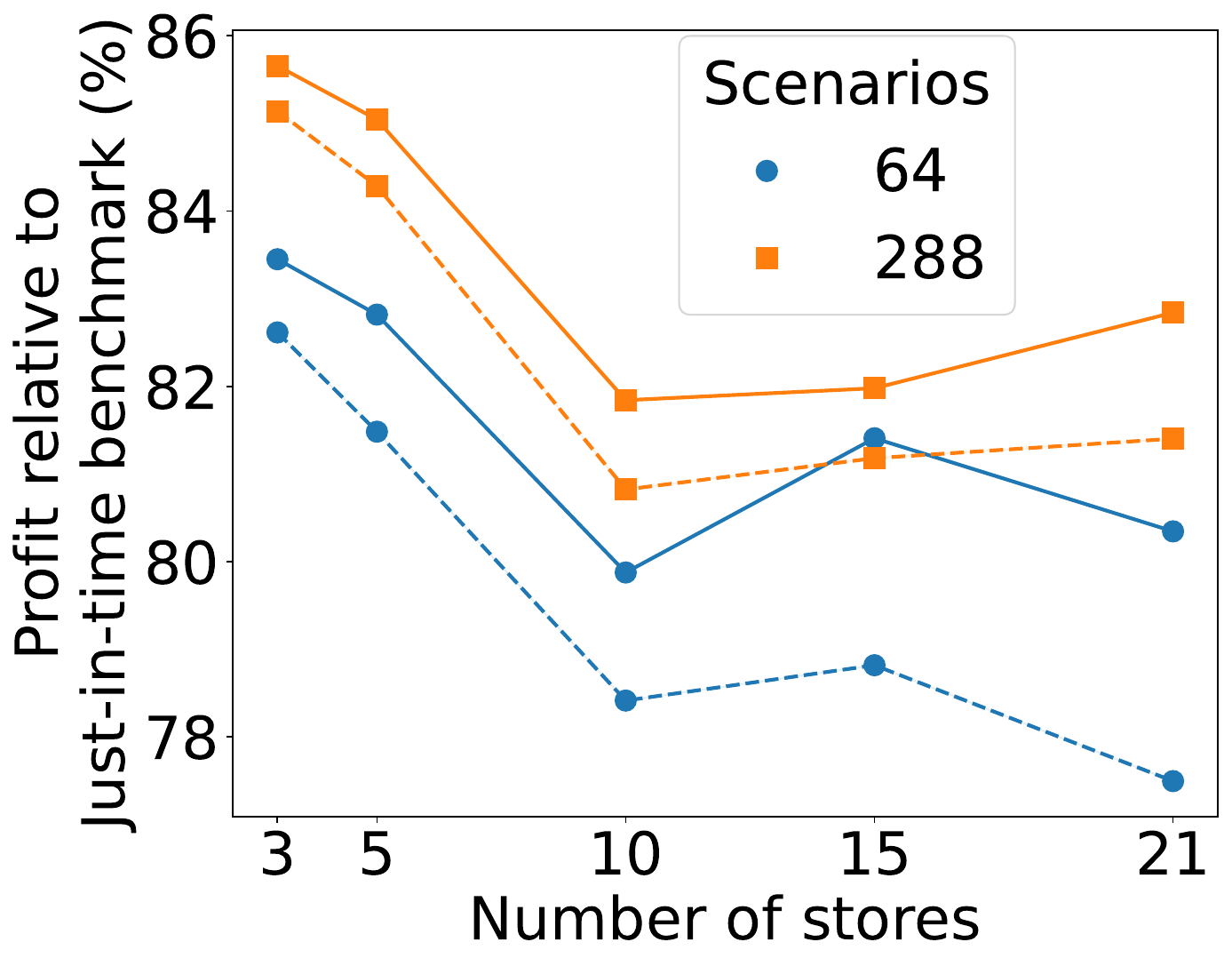}
  \caption{Setting S7. \textbf{Higher is better.}}
  \label{fig:sample-efficiency-1dc-lost-real}
\end{subfigure}%
\hspace{0.015\textwidth}%
\begin{subfigure}[t]{.32\textwidth}
  \centering
  \includegraphics[width=\linewidth]{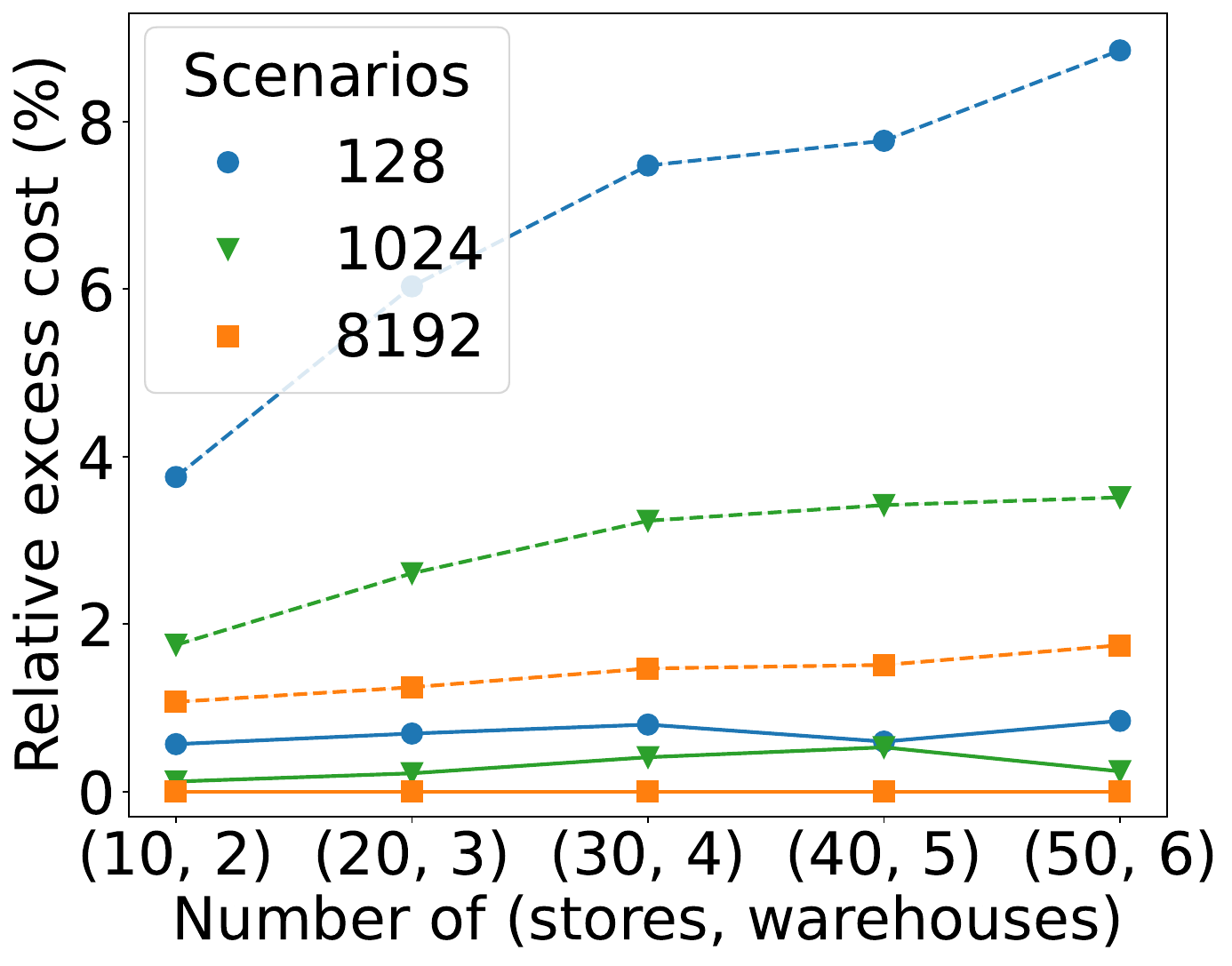}
  \caption{Setting S8. \textbf{Lower is better.}}
  \label{fig:sample-efficiency-manydcs}
\end{subfigure}%
\hspace{0.015\textwidth}%
\begin{subfigure}[t]{.32\textwidth}
  \centering
  \includegraphics[width=\linewidth]{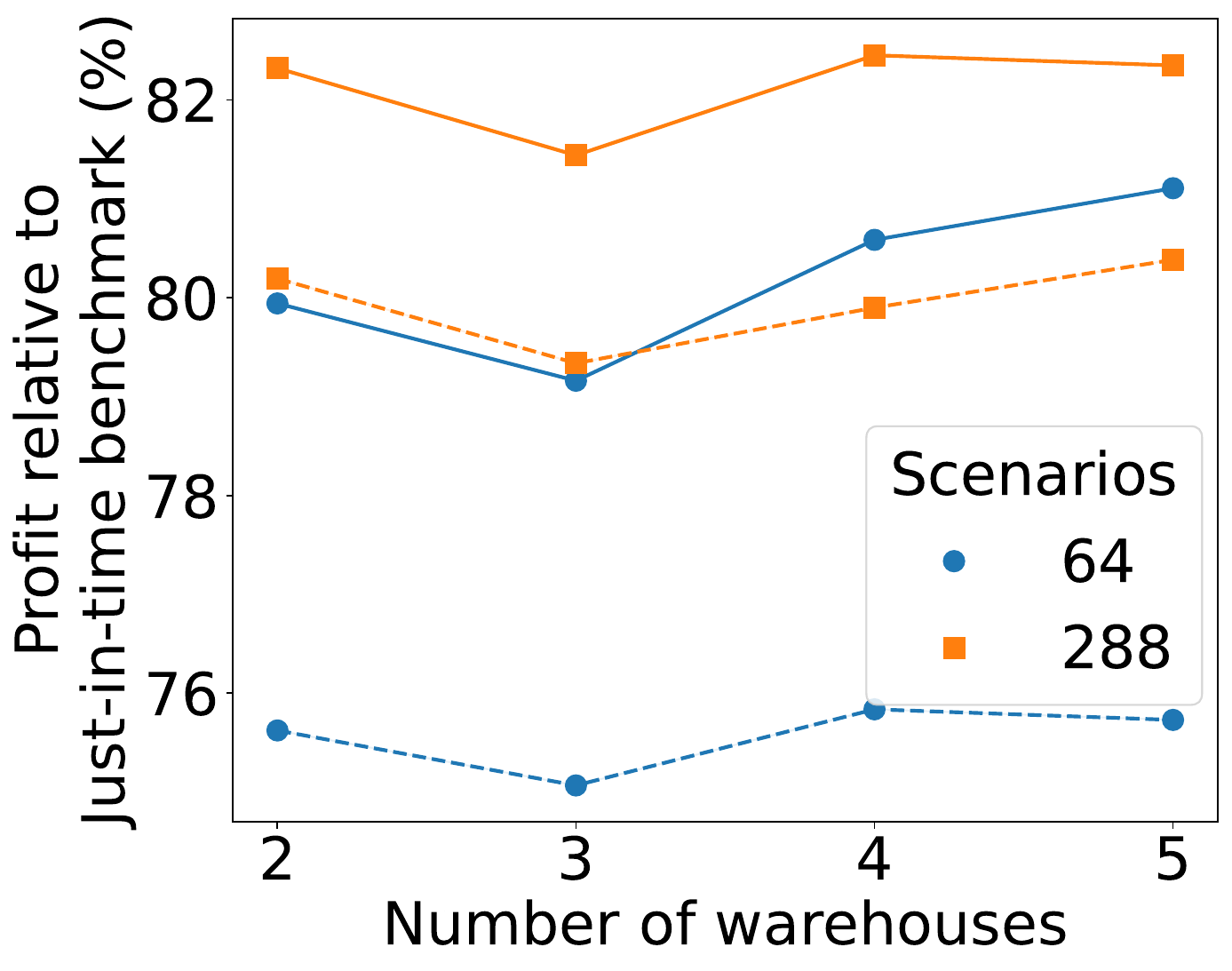}
  \caption{Setting S9. \textbf{Higher is better.}}
  \label{fig:sample-efficiency-manydcs-real}
\end{subfigure}

\caption{Sample efficiency results across the six settings outlined in Table~\ref{table:experiment-settings}. Each plot shows relative performance as a function of the number of training scenarios and locations. Costs are normalized by the minimum value achieved within each architecture class, number of scenarios, and location count. Solid lines represent the SA-MP NN, dashed lines represent the Vanilla NN, and line colors correspond to the number of scenarios used for train and dev sets. Some legends are omitted from subplots for clarity.}
\label{fig:sample_efficiency_subfigure}
\end{figure}

\subsubsection{Recovering near-optimal performance in stylized settings.}  Figures~\ref{fig:sample-efficiency-serial_GNN} and~\ref{fig:sample-efficiency-transshipment_GNN} report the optimality gaps achieved by the GNN in settings S3 and S4, which involve serial and OWMS network topologies, respectively. Setting S3 proved to be the most challenging among the six, likely due to the long serial inventory flow, which may induce sharp local optima. Despite this, the GNN performs within 1\% of optimal with 8192 training scenarios for systems with up to 6 locations, and achieves gaps below 3\% with just 128 training scenarios for systems with 5 or fewer locations. While performance degrades for 6 and 7 locations, we view these instances as unrealistic due to the complexity introduced by long inventory chains. Still, these results highlight the GNN’s ability to recover near-optimal policies in structurally challenging environments.

Figure~\ref{fig:sample-efficiency-transshipment_GNN} further demonstrates the GNN’s striking sample efficiency in setting S4: with only 128 samples, it consistently achieves optimality gaps below 1.5\%, even for systems with up to 50 stores. With 8192 samples, the gap drops below 0.2\%. \edit{These instances involve state representations with up to 303 dimensions—or 555 when including instance parameters—highlighting the GNN’s ability to scale effectively to high-dimensional problems.}

\begin{figure}
\label{fig:sample-efficiency-serial-and-transshipment_GNN}
\begin{subfigure}{.49\textwidth}
  \centering
  \includegraphics[width=1\linewidth]{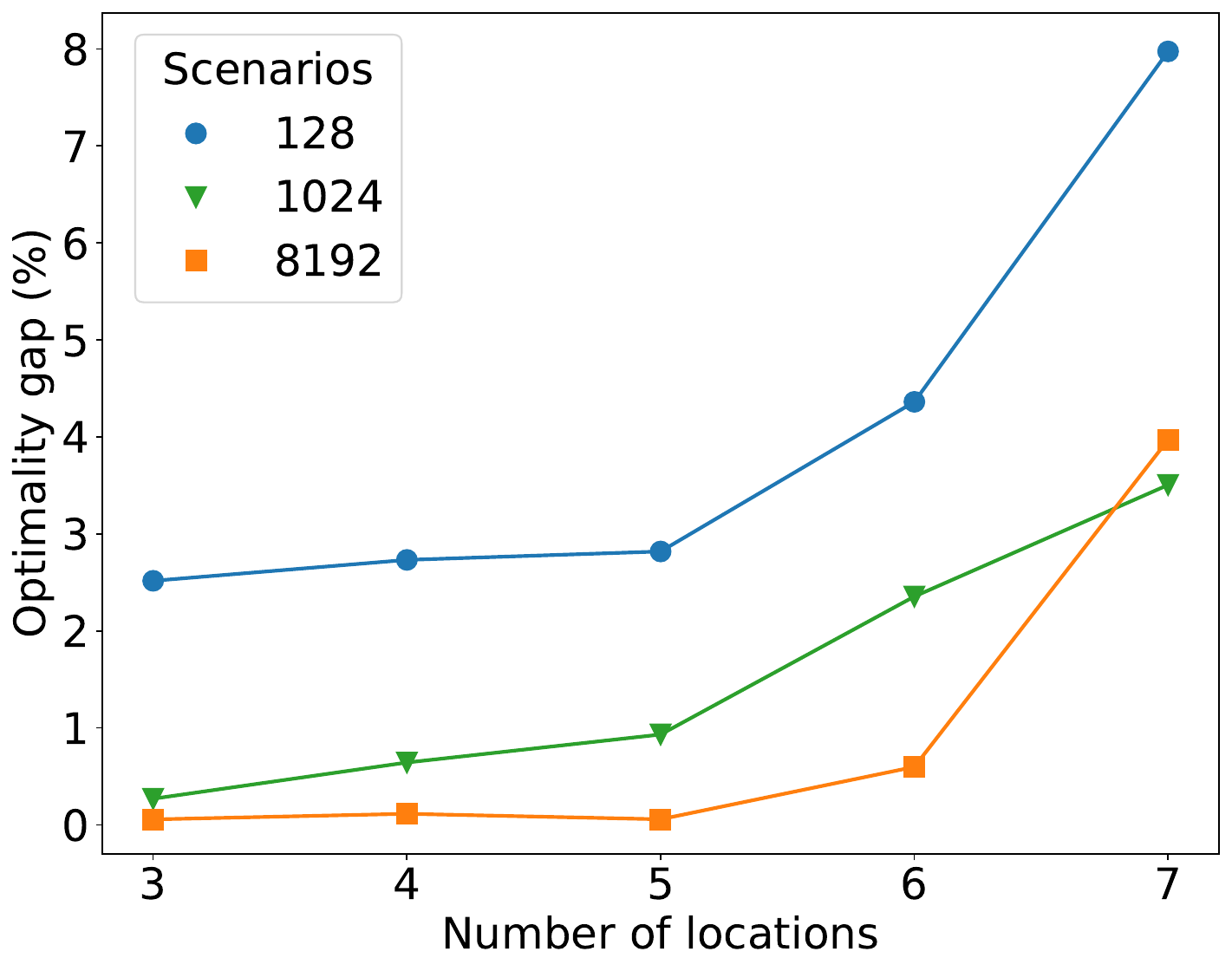}
  \caption{Setting S3}
  \label{fig:sample-efficiency-serial_GNN}
\end{subfigure}%
\hspace{0.02\textwidth}
\begin{subfigure}{.49\textwidth}
  \centering
  \includegraphics[width=1\linewidth]{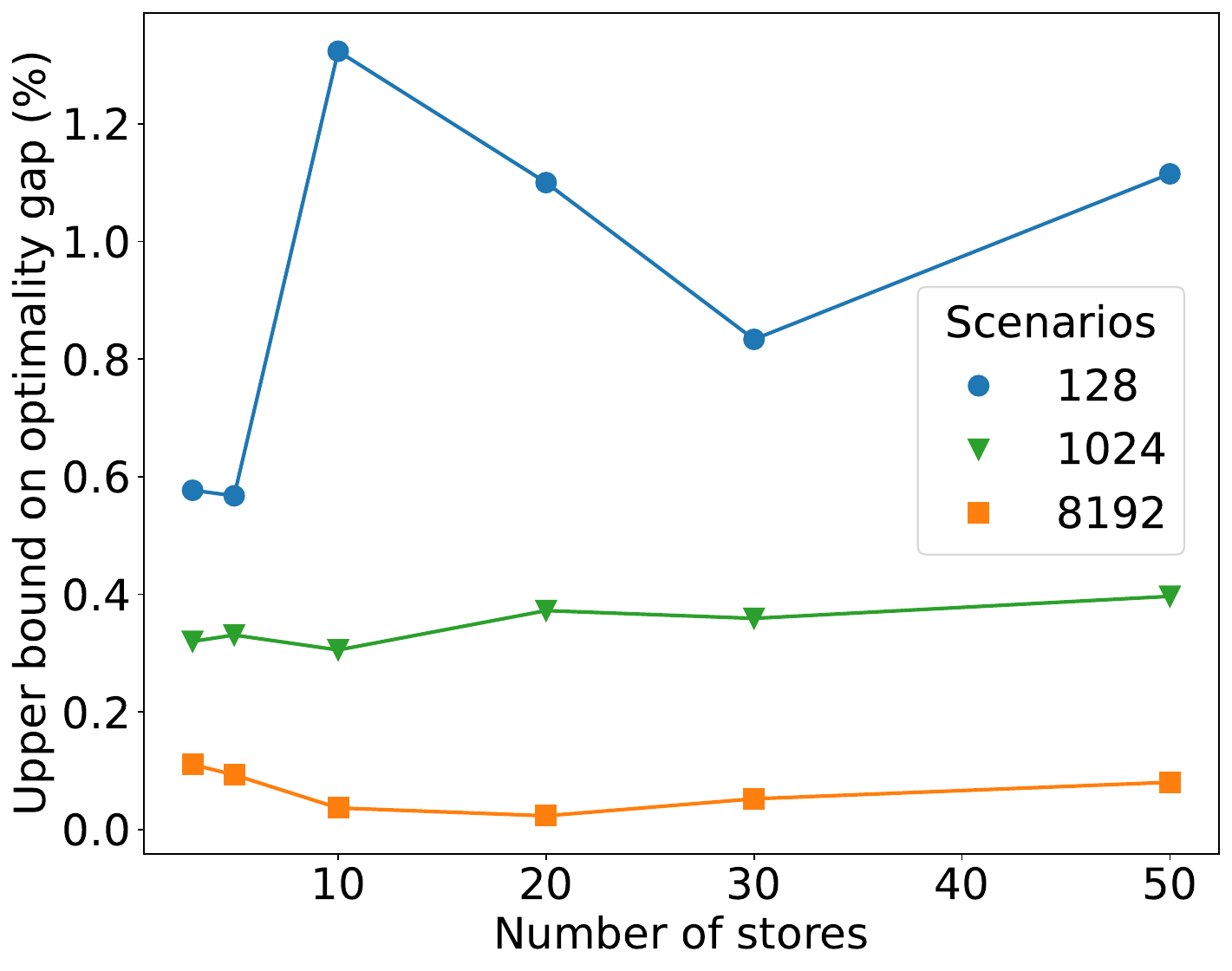}
  \caption{Setting S4}
  \label{fig:sample-efficiency-transshipment_GNN}
\end{subfigure}

\caption{Optimality gap (\%) of the GNN in settings S3 (left) and S4 (right) across multiple sample sizes.}
\end{figure}

\subsubsection{Setting S3.} \hfill
\label{appendix:sample-efficiency-serial}

{\bf Benchmarks.} Demand follows a normal distribution with mean 5 and standard deviation 2. We vary the number of locations in $\{3, 4, 5, 6, 7\}$, numbering them sequentially from the outside supplier (location 0) to the final store. The underage cost is set to 9.

The location parameters (holding cost, lead time from supplier) are assigned from a master list: $(0.1, 2)$, $(0.2, 4)$, $(0.25, 1)$, $(0.3, 3)$, $(0.4, 2)$, $(0.5, 3)$, $(1.0, 4)$. For each network size of $n$ locations (excluding the outside supplier), we use the last $n$ parameter pairs and re-number them as locations 1 through $n$ (\eg for $n=4$ locations, we use the last 4 parameter pairs corresponding to $(0.3, 3)$, $(0.4, 2)$, $(0.5, 3)$, $(1.0, 4)$ and assign them to locations 1, 2, 3, 4 respectively).

For each number of locations, we consider one meta-instance and one fixed instance.

{\bf Results.} Results are shown in Table~\ref{table:sample-efficiency-serial}.

\begin{table}[h!]
\begin{center}
\scriptsize
\caption{Metrics for Setting S3.}
\label{table:sample-efficiency-serial}

\begin{tabular}{>{\raggedleft}p{1.9cm}>{\raggedleft}p{2.3cm}>{\raggedleft}p{2.0cm}>{\raggedleft}p{1.3cm}>{\raggedleft}p{1.6cm}>{\raggedleft}p{0.7cm}>{\raggedleft}p{0.7cm}>{\raggedleft}p{0.7cm}>{\raggedleft\arraybackslash}p{1.6cm}}
\toprule
Number of echelons & Training scenarios (\#) & Architecture Class & Learning rate & Units per layer (\#) & Train loss & Dev loss & Test loss & Test opt. gap (\%) \\
\midrule
3 & 128  & GNN     & 0.0010 & -   & 9.76 & 10.10 & 10.11 & 2.43 \\
3 & 128  & Vanilla & 0.0100 & 128 & 10.02 & 10.13 & 10.25 & 3.85 \\
3 & 1024 & GNN     & 0.0100 & -   & 9.94 & 9.94  & 9.89  & 0.20 \\
3 & 1024 & Vanilla & 0.0100 & 64  & 10.10 & 10.10 & 10.05 & 1.82 \\
3 & 8192 & GNN     & 0.0010 & -   & 9.92 & 9.89  & \textbf{9.87} & 0.00 \\
3 & 8192 & Vanilla & 0.0100 & 32  & 9.98 & 9.98  & 9.96  & 0.91 \\
\midrule
4 & 128  & GNN     & 0.0010 & -   & 10.55 & 11.00 & 11.02 & 2.70 \\
4 & 128  & Vanilla & 0.0100 & 32  & 11.12 & 11.12 & 11.27 & 5.03 \\
4 & 1024 & GNN     & 0.0010 & -   & 10.87 & 10.84 & 10.80 & 0.65 \\
4 & 1024 & Vanilla & 0.0100 & 64  & 10.99 & 11.01 & 10.97 & 2.24 \\
4 & 8192 & GNN     & 0.0010 & -   & 10.82 & 10.76 & \textbf{10.74} & 0.09 \\
4 & 8192 & Vanilla & 0.0100 & 32  & 10.86 & 10.86 & 10.83 & 0.93 \\
\midrule
5 & 128  & GNN     & 0.0010 & -   & 10.49 & 11.16 & 11.29 & 2.82 \\
5 & 128  & Vanilla & 0.0100 & 64  & 11.19 & 11.46 & 11.54 & 5.10 \\
5 & 1024 & GNN     & 0.0010 & -   & 11.21 & 11.11 & 11.09 & 1.00 \\
5 & 1024 & Vanilla & 0.0100 & 128 & 11.17 & 11.30 & 11.21 & 2.09 \\
5 & 8192 & GNN     & 0.0001 & -   & 11.10 & 11.01 & \textbf{10.99} & 0.09 \\
5 & 8192 & Vanilla & 0.0100 & 64  & 11.15 & 11.11 & 11.08 & 0.91 \\
\midrule
6 & 128  & GNN     & 0.0010 & -   & 11.60 & 12.11 & 12.22 & 4.36 \\
6 & 128  & Vanilla & 0.0100 & 32  & 11.58 & 12.38 & 12.52 & 6.92 \\
6 & 1024 & GNN     & 0.0010 & -   & 12.27 & 12.03 & 11.99 & 2.39 \\
6 & 1024 & Vanilla & 0.0010 & 64  & 12.05 & 12.34 & 12.25 & 4.61 \\
6 & 8192 & GNN     & 0.0010 & -   & 11.77 & 11.79 & \textbf{11.78} & 0.60 \\
6 & 8192 & Vanilla & 0.0100 & 32  & 12.00 & 11.86 & 11.82 & 0.94 \\
\midrule
7 & 128  & GNN     & 0.0010 & -   & 12.06 & 13.07 & 13.29 & 7.96 \\
7 & 128  & Vanilla & 0.0010 & 64  & 15.16 & 13.97 & 14.03 & 13.97 \\
7 & 1024 & GNN     & 0.0010 & -   & 12.92 & 12.73 & 12.74 & 3.49 \\
7 & 1024 & Vanilla & 0.0010 & 64  & 13.28 & 13.90 & 13.87 & 12.67 \\
7 & 8192 & GNN     & 0.0010 & -   & 12.50 & 12.81 & 12.80 & 3.98 \\
7 & 8192 & Vanilla & 0.0100 & 64  & 12.53 & 12.59 & \textbf{12.58} & 2.19 \\
\bottomrule
\end{tabular}

\end{center}
\end{table}

\subsubsection{Setting S4.}
\label{appendix:sample-efficiency-transshipment} \hfill

{\bf Benchmarks.}
The lead time from the supplier to the transshipment warehouse is fixed at $L^{(0,1)} = 3$, and the warehouse does not have holding cost since it cannot hold inventory. The number of stores varies over $\{3, 5, 10, 20, 30, 50\}$. Store-level lead times $L^{(1,k)}$ are set to $4$ periods, store underage costs to $9$, and store holding costs to $1$. Fixing these values allow us to use the lower bounds in \cite{federgruen1984approximations}.
Demand at the stores follows a multivariate normal distribution. For each store, the demand mean is sampled uniformly from $[2.5, 7.5]$ and the coefficient of variation from $[0.25, 0.5]$. Pairwise correlation between stores is fixed at $0.5$.

For each number of stores, we consider one meta-instance and one fixed instance.

{\bf Results.} Results are presented in Table~\ref{table:sample-efficiency-transshipment}.

\begin{table}[h]
\centering
\scriptsize
\caption{Metrics for Setting S4.}
\label{table:sample-efficiency-transshipment}
\begin{tabular}{>{\raggedleft}p{1.9cm}>{\raggedleft}p{2.3cm}>{\raggedleft}p{2.0cm}>{\raggedleft}p{1.3cm}>{\raggedleft}p{1.6cm}>{\raggedleft}p{0.7cm}>{\raggedleft}p{0.7cm}>{\raggedleft}p{0.7cm}>{\raggedleft\arraybackslash}p{2.2cm}}
\toprule
Number of stores & Training scenarios (\#) & Architecture Class & Learning rate & Units per layer (\#) & Train loss & Dev loss & Test loss & Upper bound on test opt. gap (\%) \\
\midrule
3 & 128 & GNN & 0.0010 & - & 8.60 & 8.54 & 8.58 & 0.58 \\
3 & 128 & Vanilla & 0.0010 & 256.0 & 8.46 & 8.81 & 8.85 & 3.67 \\
3 & 1024 & GNN & 0.0010 & - & 8.71 & 8.54 & 8.56 & 0.32 \\
3 & 1024 & Vanilla & 0.0010 & 256.0 & 8.67 & 8.58 & 8.59 & 0.72 \\
3 & 8192 & GNN & 0.0010 & - & 8.55 & 8.52 & \textbf{8.54} & 0.11 \\
3 & 8192 & Vanilla & 0.0010 & 128.0 & 8.54 & 8.53 & 8.56 & 0.28 \\
\midrule
5 & 128 & GNN & 0.0001 & - & 7.83 & 7.86 & 7.95 & 0.57 \\
5 & 128 & Vanilla & 0.0010 & 256.0 & 7.92 & 8.23 & 8.31 & 5.14 \\
5 & 1024 & GNN & 0.0001 & - & 7.82 & 7.89 & 7.93 & 0.33 \\
5 & 1024 & Vanilla & 0.0010 & 256.0 & 7.82 & 7.96 & 7.99 & 1.08 \\
5 & 8192 & GNN & 0.0010 & - & 7.89 & 7.90 & \textbf{7.91} & 0.09 \\
5 & 8192 & Vanilla & 0.0010 & 128.0 & 7.94 & 7.91 & 7.92 & 0.28 \\
\midrule
10 & 128 & GNN & 0.0100 & - & 7.89 & 8.35 & 8.41 & 1.32 \\
10 & 128 & Vanilla & 0.0010 & 256.0 & 7.92 & 8.97 & 9.02 & 8.77 \\
10 & 1024 & GNN & 0.0010 & - & 8.25 & 8.39 & 8.32 & 0.31 \\
10 & 1024 & Vanilla & 0.0010 & 256.0 & 8.26 & 8.54 & 8.47 & 2.05 \\
10 & 8192 & GNN & 0.0001 & - & 8.27 & 8.33 & \textbf{8.30} & 0.04 \\
10 & 8192 & Vanilla & 0.0010 & 128.0 & 8.32 & 8.35 & 8.33 & 0.34 \\
\midrule
20 & 128 & GNN & 0.0010 & - & 8.62 & 8.64 & 8.83 & 1.10 \\
20 & 128 & Vanilla & 0.0010 & 256.0 & 9.17 & 9.75 & 9.97 & 14.17 \\
20 & 1024 & GNN & 0.0001 & - & 8.72 & 8.66 & 8.76 & 0.37 \\
20 & 1024 & Vanilla & 0.0010 & 128.0 & 8.91 & 9.00 & 9.11 & 4.33 \\
20 & 8192 & GNN & 0.0001 & - & 8.70 & 8.73 & \textbf{8.73} & 0.02 \\
20 & 8192 & Vanilla & 0.0010 & 128.0 & 8.72 & 8.76 & 8.76 & 0.34 \\
\midrule
30 & 128 & GNN & 0.0010 & - & 8.33 & 8.38 & 8.45 & 0.83 \\
30 & 128 & Vanilla & 0.0010 & 256.0 & 8.94 & 10.15 & 10.2 & 21.78 \\
30 & 1024 & GNN & 0.0001 & - & 8.33 & 8.29 & 8.41 & 0.36 \\
30 & 1024 & Vanilla & 0.0010 & 128.0 & 8.56 & 8.81 & 8.92 & 6.47 \\
30 & 8192 & GNN & 0.0010 & - & 8.41 & 8.36 & \textbf{8.38} & 0.05 \\
30 & 8192 & Vanilla & 0.0010 & 128.0 & 8.42 & 8.39 & 8.42 & 0.45 \\
\midrule
50 & 128 & GNN & 0.0010 & - & 8.30 & 8.14 & 8.21 & 1.11 \\
50 & 128 & Vanilla & 0.0010 & 256.0 & 8.54 & 10.57 & 10.76 & 32.57 \\
50 & 1024 & GNN & 0.0010 & - & 8.12 & 8.15 & 8.15 & 0.40 \\
50 & 1024 & Vanilla & 0.0010 & 256.0 & 8.39 & 8.71 & 8.70 & 7.22 \\
50 & 8192 & GNN & 0.0001 & - & 8.14 & 8.12 & \textbf{8.12} & 0.08 \\
50 & 8192 & Vanilla & 0.0010 & 128.0 & 8.19 & 8.18 & 8.19 & 0.86 \\
\bottomrule
\end{tabular}

\end{table}

\subsubsection{Setting S6.}
\label{appendix:sample-efficiency-owms-synthetic} \hfill

{\bf Benchmarks.} 
The lead time from the supplier to the warehouse is fixed at $L^{(0,1)} = 3$, and the warehouse holding cost is set to $h^1 = 0.3$. The number of stores varies over $\{3, 5, 10, 20, 30, 50\}$. Store-level lead times $L^{(1,k)}$ are sampled uniformly at random from $\{2, 3, 4, 5\}$. Store holding and underage costs are uniformly sampled from [0.7, 1.3] and [6.3, 11.7], respectively.

Demand at the stores follows a multivariate normal distribution. For each store, the demand mean is sampled uniformly from $[2.5, 7.5]$ and the coefficient of variation from $[0.25, 0.5]$. Pairwise correlation between stores is fixed at $0.5$.

For each configuration, we consider one meta-instance and one fixed instance.

{\bf Results.} Results appear in Table~\ref{table:sample-efficiency-owms-synthetic}.

\begin{table}[h]
\centering
\scriptsize
\caption{Metrics for Setting S6.}
\label{table:sample-efficiency-owms-synthetic}
\begin{tabular}{>{\raggedleft}p{1.9cm}>{\raggedleft}p{2.3cm}>{\raggedleft}p{2.0cm}>{\raggedleft}p{1.3cm}>{\raggedleft}p{1.6cm}>{\raggedleft}p{0.7cm}>{\raggedleft}p{0.7cm}>{\raggedleft}p{0.7cm}>{\raggedleft\arraybackslash}p{1.6cm}}
\toprule
Number of stores & Training scenarios (\#) & Architecture Class & Learning rate & Units per layer (\#) & Train loss & Dev loss & Test loss & Relative excess test loss (\%) \\
\midrule
3 & 128 & GNN & 0.0100 & - & 5.43 & 5.63 & 5.64 & 0.84 \\
3 & 128 & Vanilla & 0.0001 & 512 & 5.36 & 5.70 & 5.72 & 2.16 \\
3 & 1024 & GNN & 0.0001 & - & 5.59 & 5.62 & 5.60 & 0.12 \\
3 & 1024 & Vanilla & 0.0010 & 128 & 5.59 & 5.64 & 5.63 & 0.49 \\
3 & 8192 & GNN & 0.0001 & - & 5.59 & 5.59 & \textbf{5.60} & 0.00 \\
3 & 8192 & Vanilla & 0.0001 & 128 & 5.59 & 5.60 & 5.61 & 0.13 \\
\midrule
5 & 128 & GNN & 0.0001 & - & 5.22 & 5.31 & 5.36 & 1.93 \\
5 & 128 & Vanilla & 0.0010 & 256 & 5.34 & 5.42 & 5.46 & 3.82 \\
5 & 1024 & GNN & 0.0001 & - & 5.24 & 5.26 & 5.27 & 0.15 \\
5 & 1024 & Vanilla & 0.0010 & 256 & 5.26 & 5.29 & 5.29 & 0.64 \\
5 & 8192 & GNN & 0.0001 & - & 5.27 & 5.26 & \textbf{5.26} & 0.00 \\
5 & 8192 & Vanilla & 0.0010 & 128 & 5.27 & 5.27 & 5.27 & 0.20 \\
\midrule
10 & 128 & GNN & 0.0001 & - & 5.58 & 5.82 & 5.80 & 1.42 \\
10 & 128 & Vanilla & 0.0010 & 256 & 5.67 & 6.00 & 5.98 & 4.56 \\
10 & 1024 & GNN & 0.0010 & - & 5.68 & 5.76 & 5.73 & 0.16 \\
10 & 1024 & Vanilla & 0.0010 & 256 & 5.83 & 5.83 & 5.79 & 1.21 \\
10 & 8192 & GNN & 0.0010 & - & 5.71 & 5.74 & \textbf{5.72} & 0.00 \\
10 & 8192 & Vanilla & 0.0010 & 256 & 5.73 & 5.75 & 5.74 & 0.31 \\
\midrule
20 & 128 & GNN & 0.0010 & - & 5.89 & 5.87 & 5.96 & 1.14 \\
20 & 128 & Vanilla & 0.0010 & 256 & 6.02 & 6.23 & 6.34 & 7.58 \\
20 & 1024 & GNN & 0.0100 & - & 5.87 & 5.88 & 5.91 & 0.21 \\
20 & 1024 & Vanilla & 0.0010 & 256 & 5.96 & 5.99 & 6.02 & 2.10 \\
20 & 8192 & GNN & 0.0010 & - & 5.89 & 5.90 & \textbf{5.90} & 0.00 \\
20 & 8192 & Vanilla & 0.0010 & 256 & 5.91 & 5.93 & 5.93 & 0.61 \\
\midrule
30 & 128 & GNN & 0.0100 & - & 5.58 & 5.67 & 5.64 & 0.65 \\
30 & 128 & Vanilla & 0.0010 & 256 & 5.80 & 6.16 & 6.15 & 9.65 \\
30 & 1024 & GNN & 0.0001 & - & 5.58 & 5.57 & 5.61 & 0.14 \\
30 & 1024 & Vanilla & 0.0010 & 512 & 5.71 & 5.74 & 5.79 & 3.35 \\
30 & 8192 & GNN & 0.0010 & - & 5.60 & 5.59 & \textbf{5.61} & 0.00 \\
30 & 8192 & Vanilla & 0.0001 & 512 & 5.65 & 5.65 & 5.67 & 1.09 \\
\midrule
50 & 128 & GNN & 0.0010 & - & 5.47 & 5.47 & 5.46 & 0.71 \\
50 & 128 & Vanilla & 0.0010 & 256 & 5.58 & 6.15 & 6.20 & 14.49 \\
50 & 1024 & GNN & 0.0010 & - & 5.44 & 5.44 & 5.43 & 0.28 \\
50 & 1024 & Vanilla & 0.0010 & 256 & 5.54 & 5.67 & 5.66 & 4.46 \\
50 & 8192 & GNN & 0.0010 & - & 5.43 & 5.42 & \textbf{5.42} & 0.00 \\
50 & 8192 & Vanilla & 0.0010 & 512 & 5.50 & 5.49 & 5.49 & 1.37 \\
\bottomrule
\end{tabular}

\end{table}

\subsubsection{Setting S7.}
\label{appendix:sample-efficiency-owms-realistic} \hfill

{\bf Benchmarks.} 
The warehouse has a fixed lead time of 3 periods and a holding cost of $0.3$. The number of stores ranges over $\{3, 5, 10, 15, 21\}$. Store-level holding and penalty costs are sampled uniformly at random from $[0.7, 1.3]$ and $[6.3,11.7]$, respectively. Store lead times are drawn uniformly from $\{2, 3, 4, 5\}$. We generate a separate instance for each scenario. We fix holding costs and lead times across scenarios, but underage costs vary across scenarios, being sampled independently for each store and instance. Realistic demand samples are obtained as described in Appendix \ref{appendix:realistic-demand-dataset}.

{\bf Results.} Results are reported in Table~\ref{table:sample-efficiency-owms-realistic}.

\begin{table}[h]
\centering
\scriptsize
\caption{Metrics for Setting S7.}
\label{table:sample-efficiency-owms-realistic}
\begin{tabular}{>{\raggedleft}p{1.9cm}>{\raggedleft}p{2.3cm}>{\raggedleft}p{2.0cm}>{\raggedleft}p{1.3cm}>{\raggedleft}p{1.6cm}>{\raggedleft}p{0.7cm}>{\raggedleft}p{0.7cm}>{\raggedleft}p{0.7cm}>{\raggedleft\arraybackslash}p{2.2cm}}
\toprule
Number of stores & Scenarios (\#) & Architecture Class & Learning rate & Units per layer (\#) & Train & Dev & Test & Percentage of Achievable profit (\%) \\
\midrule
3  & 64  & GNN     & 0.0001 & -   & 639.67 & 619.41 & 543.28 & 83.45 \\
3  & 64  & Vanilla & 0.0100 & 128 & 638.78 & 616.28 & 537.84 & 82.62 \\
3  & 288 & GNN     & 0.0001 & -   & 742.65 & 703.97 & \textbf{641.76} & 85.66 \\
3  & 288 & Vanilla & 0.0001 & 128 & 745.89 & 702.55 & 637.88 & 85.14 \\
\midrule
5  & 64  & GNN     & 0.0001 & -   & 574.85 & 551.12 & 487.07 & 82.82 \\
5  & 64  & Vanilla & 0.0010 & 256 & 579.27 & 544.95 & 479.22 & 81.48 \\
5  & 288 & GNN     & 0.0010 & -   & 659.61 & 624.99 & \textbf{578.32} & 85.04 \\
5  & 288 & Vanilla & 0.0001 & 256 & 665.26 & 623.16 & 573.17 & 84.29 \\
\midrule
10 & 64  & GNN     & 0.0010 & -   & 471.95 & 454.94 & 418.83 & 79.88 \\
10 & 64  & Vanilla & 0.0010 & 128 & 478.47 & 446.60 & 411.16 & 78.42 \\
10 & 288 & GNN     & 0.0010 & -   & 526.10 & 501.21 & \textbf{468.13} & 81.84 \\
10 & 288 & Vanilla & 0.0001 & 128 & 524.66 & 494.03 & 462.30 & 80.82 \\
\midrule
15 & 64  & GNN     & 0.0001 & -   & 506.26 & 494.58 & 445.07 & 81.41 \\
15 & 64  & Vanilla & 0.0010 & 256 & 514.70 & 481.37 & 430.91 & 78.82 \\
15 & 288 & GNN     & 0.0010 & -   & 555.05 & 532.73 & \textbf{484.03} & 81.98 \\
15 & 288 & Vanilla & 0.0001 & 128 & 561.16 & 524.06 & 479.31 & 81.18 \\
\midrule
21 & 64  & GNN     & 0.0001 & -   & 595.42 & 582.30 & 511.89 & 80.35 \\
21 & 64  & Vanilla & 0.0001 & 512 & 615.35 & 563.95 & 493.73 & 77.50 \\
21 & 288 & GNN     & 0.0001 & -   & 664.63 & 637.44 & \textbf{580.45} & 82.84 \\
21 & 288 & Vanilla & 0.0001 & 256 & 668.73 & 624.17 & 570.37 & 81.40 \\
\bottomrule
\end{tabular}

\end{table}

\subsubsection{Setting S8.}
\label{appendix:sample-efficiency-mwms-synthetic} \hfill

{\bf Benchmarks.} We consider paired configurations of (warehouses, stores): $\{(2, 10), (3, 20), (4, 30), (5, 40), (6, 50)\}$. \edit{Store-level holding and penalty costs are sampled uniformly at random from $[0.7, 1.3]$ and $[6.3, 11.7]$, respectively. Demand at the stores follows a multivariate normal distribution. For each store, the demand mean is sampled uniformly from $[2.5, 7.5]$ and the coefficient of variation from $[0.25, 0.5]$. Pairwise correlation between stores is fixed at $0.5$. Store parameters remain fixed across scenarios.} 

Each warehouse has a different cost and lead time profile. Specifically, we define six types of warehouse primitives, described by $(h^i, \edit{[L_{min}, L_{max}]}, \text{edge cost})$ triplets: $(0.3, [5,6], 0.5)$, $(0.4, [1,2], 1.5)$, $(0.2, [4,6], 0.7)$, $(0.5, [1,3], 1.3)$, $(0.1, [3,5], 0.9)$, and $(0.6, [2,4], 1.1)$. For each configuration, we use the first $n$ warehouse types (e.g., for 3 warehouses, we use the first 3 triplets from this list). These values create heterogeneous warehouses, with faster warehouses generally being more expensive. 

The placement of warehouses and stores and their network configuration follows the methodology described in Appendix~\ref{appendix:many-warehouse-many-store-placement}, which also details how lead time and edge cost parameters are translated into actual lead times and edge costs. 

{\bf Results.} Results are shown in Table~\ref{table:sample-efficiency-mwms-synthetic}.

\begin{table}[h]
\centering
\scriptsize
\caption{Metrics for Setting S8.}
\label{table:sample-efficiency-mwms-synthetic}
\begin{tabular}{>{\raggedleft}p{1.9cm}>{\raggedleft}p{2.3cm}>{\raggedleft}p{2.0cm}>{\raggedleft}p{1.3cm}>{\raggedleft}p{1.6cm}>{\raggedleft}p{0.7cm}>{\raggedleft}p{0.7cm}>{\raggedleft}p{0.7cm}>{\raggedleft\arraybackslash}p{1.6cm}}
\toprule
Number of (warehouses, stores)  & Training scenarios (\#) & Architecture Class  & Learning rate & Units per layer (\#) & Train loss & Dev loss & Test loss & Relative excess test loss (\%) \\
\midrule
(2, 10) & 128 & GNN & 0.0010 & -     & 9.51 & 9.73 & 9.65 & 0.57 \\
(2, 10) & 128 & Vanilla & 0.0010 & 256.0 & 9.65 & 10.01 & 9.96 & 3.76 \\
(2, 10) & 1024 & GNN & 0.0010 & -     & 9.56 & 9.65 & 9.61 & 0.12 \\
(2, 10) & 1024 & Vanilla & 0.0010 & 256.0 & 9.67 & 9.80 & 9.77 & 1.75 \\
(2, 10) & 8192 & GNN & 0.0010 & -     & 9.59 & 9.61 & \textbf{9.60} & 0.00 \\
(2, 10) & 8192 & Vanilla & 0.0010 & 256.0 & 9.68 & 9.71 & 9.70 & 1.07 \\
\midrule
(3, 20) & 128 & GNN & 0.0010 & -     & 9.32 & 9.36 & 9.46 & 0.69 \\
(3, 20) & 128 & Vanilla & 0.0010 & 128.0 & 9.55 & 9.81 & 9.96 & 6.03 \\
(3, 20) & 1024 & GNN & 0.0010 & -     & 9.37 & 9.38 & 9.42 & 0.22 \\
(3, 20) & 1024 & Vanilla & 0.0001 & 512.0 & 9.53 & 9.60 & 9.64 & 2.61 \\
(3, 20) & 8192 & GNN & 0.0010 & -     & 9.39 & 9.40 & \textbf{9.40} & 0.00 \\
(3, 20) & 8192 & Vanilla & 0.0010 & 512.0 & 9.47 & 9.51 & 9.51 & 1.25 \\
\midrule
(4, 30) & 128 & GNN & 0.0010 & -     & 9.12 & 9.17 & 9.14 & 0.80 \\
(4, 30) & 128 & Vanilla & 0.0010 & 256.0 & 9.35 & 9.75 & 9.74 & 7.48 \\
(4, 30) & 1024 & GNN & 0.0010 & -     & 9.10 & 9.05 & 9.10 & 0.41 \\
(4, 30) & 1024 & Vanilla & 0.0010 & 256.0 & 9.26 & 9.30 & 9.36 & 3.23 \\
(4, 30) & 8192 & GNN & 0.0010 & -     & 9.06 & 9.06 & \textbf{9.06} & 0.00 \\
(4, 30) & 8192 & Vanilla & 0.0010 & 512.0 & 9.19 & 9.19 & 9.20 & 1.47 \\
\midrule
(5, 40) & 128 & GNN & 0.0010 & -     & 9.37 & 9.40 & 9.47 & 0.60 \\
(5, 40) & 128 & Vanilla & 0.0010 & 256.0 & 9.71 & 10.07 & 10.14 & 7.77 \\
(5, 40) & 1024 & GNN & 0.0010 & -     & 9.45 & 9.46 & 9.46 & 0.53 \\
(5, 40) & 1024 & Vanilla & 0.0001 & 512.0 & 9.62 & 9.73 & 9.73 & 3.42 \\
(5, 40) & 8192 & GNN & 0.0010 & -     & 9.43 & 9.41 & \textbf{9.41} & 0.00 \\
(5, 40) & 8192 & Vanilla & 0.0010 & 512.0 & 9.56 & 9.55 & 9.55 & 1.51 \\
\midrule
(6, 50) & 128 & GNN & 0.0010 & -     & 9.29 & 9.28 & 9.25 & 0.85 \\
(6, 50) & 128 & Vanilla & 0.0001 & 512.0 & 9.60 & 9.97 & 9.98 & 8.85 \\
(6, 50) & 1024 & GNN & 0.0010 & -     & 9.22 & 9.21 & 9.20 & 0.24 \\
(6, 50) & 1024 & Vanilla & 0.0010 & 256.0 & 9.42 & 9.51 & 9.50 & 3.51 \\
(6, 50) & 8192 & GNN & 0.0010 & -     & 9.17 & 9.17 & \textbf{9.17} & 0.00 \\
(6, 50) & 8192 & Vanilla & 0.0010 & 256.0 & 9.32 & 9.33 & 9.33 & 1.75 \\
\bottomrule
\end{tabular}

\end{table}

\subsubsection{Setting S9.}
\label{appendix:sample-efficiency-mwms-realistic} \hfill

{\bf Benchmarks.} 
We vary the number of warehouse over \{2, 3, 4, 5\}, and fix the number of stores to 21. \edit{Store-level holding and underage costs are sampled uniformly at random from $[0.7, 1.3]$ and $[6.3, 11.7]$, respectively. We fix holding costs across scenarios, but underage costs vary across scenarios, being sampled independently for each store and instance}. 

Each warehouse has a different cost and lead time profile. Specifically, we define five types of warehouse primitives, described by $(h^i, \edit{[L_{min}, L_{max}]}, \text{edge cost})$ triplets: $(0.3, [5,6], 0.5)$, $(0.4, [1,2], 1.5)$, $(0.2, [4,6], 0.7)$, $(0.5, [1,3], 1.3)$ and $(0.1, [3,5], 0.9)$. For each configuration, we use the first $n$ warehouse types (e.g., for 3 warehouses, we use the first 3 triplets from this list). These values create heterogeneous warehouses, with faster warehouses generally being more expensive. 

The placement of warehouses and stores and their network configuration follows the methodology described in Appendix~\ref{appendix:many-warehouse-many-store-placement}, which also details how lead time and edge cost parameters are translated into actual lead times and edge costs. Realistic demand samples are obtained as described in Appendix \ref{appendix:realistic-demand-dataset}.

{\bf Results.} Results are shown in Table~\ref{table:sample-efficiency-mwms-realistic}.

\begin{table}[h]
\centering
\scriptsize
\caption{Metrics for Setting S9}
\label{table:sample-efficiency-mwms-realistic}

\begin{tabular}{>{\raggedleft}p{1.9cm}>{\raggedleft}p{2.3cm}>{\raggedleft}p{2.0cm}>{\raggedleft}p{1.3cm}>{\raggedleft}p{1.6cm}>{\raggedleft}p{0.7cm}>{\raggedleft}p{0.7cm}>{\raggedleft}p{0.7cm}>{\raggedleft\arraybackslash}p{2.1cm}}
\toprule
Number of warehouses & Training scenarios (\#) & Architecture Class & Learning rate & Units per layer (\#) & Train profit & Dev profit & Test profit & Percentage of Achievable profit (\%) \\
\midrule
2  & 64  & GNN     & 0.0010 & -     & 560.20 & 533.23 & 473.43 & 79.94 \\
2  & 64  & Vanilla & 0.0001 & 512.0 & 557.01 & 510.17 & 447.85 & 75.62 \\
2  & 288 & GNN     & 0.0010 & -     & 614.07 & 588.22 & \textbf{534.90} & 82.32 \\
2  & 288 & Vanilla & 0.0001 & 512.0 & 607.67 & 572.48 & 521.07 & 80.19 \\
\midrule
3  & 64  & GNN     & 0.0010 & -     & 546.44 & 529.71 & 475.31 & 79.16 \\
3  & 64  & Vanilla & 0.0001 & 512.0 & 563.54 & 509.83 & 450.73 & 75.07 \\
3  & 288 & GNN     & 0.0010 & -     & 622.05 & 590.60 & \textbf{537.06} & 81.44 \\
3  & 288 & Vanilla & 0.0001 & 512.0 & 619.12 & 575.39 & 523.19 & 79.34 \\
\midrule
4  & 64  & GNN     & 0.0010 & -     & 556.15 & 533.32 & 477.68 & 80.59 \\
4  & 64  & Vanilla & 0.0001 & 512.0 & 555.83 & 503.88 & 449.53 & 75.84 \\
4  & 288 & GNN     & 0.0010 & -     & 617.12 & 589.20 & \textbf{536.53} & 82.45 \\
4  & 288 & Vanilla & 0.0001 & 128.0 & 614.42 & 572.82 & 519.92 & 79.90 \\
\midrule
5  & 64  & GNN     & 0.0010 & -     & 555.05 & 534.88 & 481.44 & 81.11 \\
5  & 64  & Vanilla & 0.0001 & 512.0 & 559.06 & 502.28 & 449.52 & 75.73 \\
5  & 288 & GNN     & 0.0010 & -     & 619.03 & 587.66 & \textbf{536.57} & 82.35 \\
5  & 288 & Vanilla & 0.0001 & 512.0 & 617.83 & 571.70 & 523.76 & 80.38 \\
\bottomrule
\end{tabular}

\end{table}

\subsection{Ablation study experimental design \label{appendix:value-of-experiments} }
This appendix details the experimental setup and provides results for the ablation studies corresponding to Sections~\ref{sec:value-of-weight-sharing}, \ref{sec:value-of-message-passing}, and~\ref{sec:value-of-flexibility}. First, we give an overview of the specifications for experiments in this Appendix section.

{\bf Experiment specifications.} Hyperparameters for GNNs and Vanilla NNs are specified in Tables \ref{tab:gnn-defaults} and \ref{tab:vanilla-defaults}, respectively. The SW-GNN, Decentralized NN, and Single-warehouse NN use the same hyperparameters as the GNN (see Table \ref{tab:gnn-defaults}), excluding hyperparameters not applicable to their architecture (\eg the Decentralized NN does not use message-passing parameters). For each hyperparameter configuration, we conducted 3 independent runs and selected the best-performing run for evaluation. Other specifications follow those outlined in Appendix \ref{appendix:global-settings}.

\subsubsection{Setting S10 (the unrelated stores problem).} \hfill
\label{appendix:value-of-weight-sharing}

{\bf Benchmarks.} We consider a setting with multiple identical, unrelated stores. We vary the number of store copies over $\{3, 5, 10, 20, 30, 50\}$. Each store faces stationary Poisson demand with mean 5, independent across time and locations. Holding cost is set to $1$, underage cost to $9$ and lead time to $3$ periods. These values are the same across stores, scenarios and time.

{\bf Baselines.} We consider the Vanilla NN and the Separate Weight GNN (SW-GNN). The SW-GNN considers a separate copy of each module (\ie \texttt{EmbedNode}, \texttt{EmbedEdge}, \texttt{UpdateNode}, \texttt{UpdateEdge} and \texttt{Readout}) for each store copy. For the GNN and SW-GNN, we consider one round of message-passing. Here, the only edges connect the outside supplier to each store. Detail on how message-passing operates in this network can be found in Appendix \ref{appendix:gnn-architecture}.

{\bf Results.} Results are reported in Table~\ref{table:value-of-weight-sharing}.

\begin{table}[h!]
\begin{center}
\centering
\scriptsize
\caption{Metrics for setting S10. Numbers in bold represent the best test loss for a fixed number of stores.}
\label{table:value-of-weight-sharing}

\begin{tabular}{>{\raggedleft}p{1.9cm}>{\raggedleft}p{2.3cm}>{\raggedleft}p{2.0cm}>{\raggedleft}p{1.3cm}>{\raggedleft}p{1.6cm}>{\raggedleft}p{0.7cm}>{\raggedleft}p{0.7cm}>{\raggedleft}p{0.7cm}>{\raggedleft\arraybackslash}p{1.6cm}}
\toprule
Number of stores & Training scenarios (\#) & Architecture Class & Learning rate & Units per layer (\#) & Train loss & Dev loss & Test loss & Test gap (\%) \\
\midrule
3 & 128 & GNN & 0.0010 & 0 & 6.55 & 6.60 & 6.56 & 0.42 \\
3 & 128 & SW-GNN & 0.0100 & 0 & 6.55 & 6.61 & 6.57 & 0.61 \\
3 & 128 & Vanilla & 0.0001 & 128 & 6.50 & 6.64 & 6.59 & 0.91 \\
3 & 1024 & GNN & 0.0010 & 0 & 6.49 & 6.56 & 6.54 & $<$0.25 \\
3 & 1024 & SW-GNN & 0.0100 & 0 & 6.50 & 6.57 & 6.54 & $<$0.25 \\
3 & 1024 & Vanilla & 0.0100 & 128 & 6.51 & 6.58 & 6.55 & 0.27 \\
3 & 8192 & GNN & 0.0001 & 0 & 6.52 & 6.55 & \textbf{6.53} & $<$0.25 \\
3 & 8192 & SW-GNN & 0.0001 & 0 & 6.52 & 6.55 & 6.54 & $<$0.25 \\
3 & 8192 & Vanilla & 0.0001 & 128 & 6.52 & 6.55 & 6.53 & $<$0.25 \\
\midrule
5 & 128 & GNN & 0.0100 & 0 & 6.53 & 6.56 & 6.54 & $<$0.25 \\
5 & 128 & SW-GNN & 0.0010 & 0 & 6.51 & 6.59 & 6.55 & 0.27 \\
5 & 128 & Vanilla & 0.0100 & 128 & 6.48 & 6.64 & 6.62 & 1.41 \\
5 & 1024 & GNN & 0.0100 & 0 & 6.51 & 6.57 & 6.53 & $<$0.25 \\
5 & 1024 & SW-GNN & 0.0100 & 0 & 6.52 & 6.57 & 6.53 & $<$0.25 \\
5 & 1024 & Vanilla & 0.0010 & 128 & 6.51 & 6.58 & 6.55 & 0.34 \\
5 & 8192 & GNN & 0.0010 & 0 & 6.53 & 6.54 & \textbf{6.53} & $<$0.25 \\
5 & 8192 & SW-GNN & 0.0010 & 0 & 6.53 & 6.54 & 6.53 & $<$0.25 \\
5 & 8192 & Vanilla & 0.0001 & 128 & 6.53 & 6.55 & 6.54 & $<$0.25 \\
\midrule
10 & 128 & GNN & 0.0100 & 0 & 6.46 & 6.55 & 6.54 & $<$0.25 \\
10 & 128 & SW-GNN & 0.0010 & 0 & 6.46 & 6.57 & 6.55 & 0.34 \\
10 & 128 & Vanilla & 0.0001 & 256 & 6.43 & 6.69 & 6.73 & 2.99 \\
10 & 1024 & GNN & 0.0010 & 0 & 6.53 & 6.55 & 6.53 & $<$0.25 \\
10 & 1024 & SW-GNN & 0.0100 & 0 & 6.54 & 6.56 & 6.54 & $<$0.25 \\
10 & 1024 & Vanilla & 0.0001 & 128 & 6.51 & 6.59 & 6.57 & 0.62 \\
10 & 8192 & GNN & 0.0010 & 0 & 6.53 & 6.54 & \textbf{6.53} & $<$0.25 \\
10 & 8192 & SW-GNN & 0.0001 & 0 & 6.53 & 6.55 & 6.53 & $<$0.25 \\
10 & 8192 & Vanilla & 0.0001 & 128 & 6.53 & 6.56 & 6.54 & $<$0.25 \\
\midrule
20 & 128 & GNN & 0.0100 & 0 & 6.51 & 6.57 & 6.53 & $<$0.25 \\
20 & 128 & SW-GNN & 0.0001 & 0 & 6.49 & 6.59 & 6.55 & 0.34 \\
20 & 128 & Vanilla & 0.0001 & 512 & 6.53 & 6.80 & 6.80 & 4.20 \\
20 & 1024 & GNN & 0.0010 & 0 & 6.52 & 6.54 & 6.53 & $<$0.25 \\
20 & 1024 & SW-GNN & 0.0010 & 0 & 6.52 & 6.55 & 6.53 & $<$0.25 \\
20 & 1024 & Vanilla & 0.0001 & 128 & 6.53 & 6.61 & 6.60 & 1.00 \\
20 & 8192 & GNN & 0.0010 & 0 & 6.53 & 6.54 & \textbf{6.53} & $<$0.25 \\
20 & 8192 & SW-GNN & 0.0001 & 0 & 6.53 & 6.55 & 6.53 & $<$0.25 \\
20 & 8192 & Vanilla & 0.0001 & 128 & 6.53 & 6.56 & 6.55 & $<$0.25 \\
\midrule
30 & 128 & GNN & 0.0100 & 0 & 6.51 & 6.57 & 6.54 & $<$0.25 \\
30 & 128 & SW-GNN & 0.0010 & 0 & 6.49 & 6.59 & 6.55 & 0.35 \\
30 & 128 & Vanilla & 0.0010 & 512 & 6.50 & 6.86 & 6.87 & 5.25 \\
30 & 1024 & GNN & 0.0010 & 0 & 6.53 & 6.55 & 6.53 & $<$0.25 \\
30 & 1024 & SW-GNN & 0.0001 & 0 & 6.53 & 6.56 & 6.53 & $<$0.25 \\
30 & 1024 & Vanilla & 0.0010 & 128 & 6.53 & 6.64 & 6.62 & 1.33 \\
30 & 8192 & GNN & 0.0010 & 0 & 6.53 & 6.54 & \textbf{6.53} & $<$0.25 \\
30 & 8192 & SW-GNN & 0.0001 & 0 & 6.54 & 6.55 & 6.53 & $<$0.25 \\
30 & 8192 & Vanilla & 0.0001 & 128 & 6.53 & 6.57 & 6.55 & 0.33 \\
\midrule
50 & 128 & GNN & 0.0100 & 0 & 6.53 & 6.57 & 6.53 & $<$0.25 \\
50 & 128 & SW-GNN & 0.0010 & 0 & 6.52 & 6.58 & 6.55 & 0.33 \\
50 & 128 & Vanilla & 0.0010 & 512 & 6.46 & 7.04 & 7.06 & 8.12 \\
50 & 1024 & GNN & 0.0100 & 0 & 6.54 & 6.55 & 6.53 & $<$0.25 \\
50 & 1024 & SW-GNN & 0.0001 & 0 & 6.53 & 6.55 & 6.53 & $<$0.25 \\
50 & 1024 & Vanilla & 0.0010 & 128 & 6.53 & 6.64 & 6.63 & 1.56 \\
50 & 8192 & GNN & 0.0010 & 0 & 6.53 & 6.54 & \textbf{6.53} & $<$0.25 \\
50 & 8192 & SW-GNN & 0.0001 & 0 & 6.53 & 6.54 & 6.53 & $<$0.25 \\
50 & 8192 & Vanilla & 0.0010 & 128 & 6.54 & 6.58 & 6.57 & 0.56 \\
\bottomrule
\end{tabular}

\end{center}
\end{table}

\subsubsection{Setting S11 (GNNs enable coordination via message passing).} \hfill
\label{appendix:value-of-message-passing} 

{\bf Benchmarks.} We consider a OWMS network topology with a transshipment warehouse (meaning it cannot hold inventory) and three stores under a lost demand assumption. Stores are prioritized by underage costs: high-priority (10), medium-priority (6), and low-priority (2). All lead times are one period, and each location has a holding cost equal to 25\% of its underage cost.
Demand follows a two-stage process. First, at each period, mean demands for all stores are sampled independently from a Normal distribution with mean 10 and standard deviation 2.5, truncated at 0. Each store observes its own demand mean for the following period. Second, actual demands are sampled from Normal distributions with the observed means and coefficient of variation 0.1, also truncated at 0.

{\bf Baselines.} We compare the GNN with a \textit{decentralized NN}, which removes the message-passing component from the GNN architecture. This baseline uses the same trainable \texttt{EmbedNode}, \texttt{EmbedEdge}, and \texttt{Readout} modules but sets $L^{\text{MP}} = 0$ in Algorithm \ref{alg:samp_gnn}.

{\bf Results.} Results are shown in Table \ref{table:value-of-message-passing}.

\begin{table}[h]
\centering
\scriptsize
\caption{Metrics for setting S11. The number in bold represent the best test loss attained.}
\label{table:value-of-message-passing}

\begin{tabular}{>{\raggedleft}p{2.00cm}>{\raggedleft}p{2.52cm}>{\raggedleft}p{2.00cm}>{\raggedleft}p{0.7cm}>{\raggedleft}p{0.7cm}>{\raggedleft}p{0.7cm}>{\raggedleft\arraybackslash}p{4.0cm}}
\toprule
Number of stores & Architecture Class & Learning rate & Train loss & Dev loss & Test loss & Relative
excess test
loss (\%) \\
\midrule
1 & GNN & 0.001 & 3.68 & 3.69 & \textbf{3.70} & 0.00 \\
1 & Decentralized NN & 0.001 & 5.18 & 5.17 & 5.17 & 39.91 \\
\bottomrule
\end{tabular}
\end{table}

\subsubsection{Setting S9 (GNN policies successfully exploit flexibility in the inventory network).} \hfill
\label{appendix:value-of-flexibility}

{\bf Benchmarks.} See Appendix \ref{appendix:sample-efficiency-mwms-realistic}.

{\bf Baselines.} We construct a \textit{Single-warehouse NN} through a two-step process. First, for each edge in the network, we train a Vanilla NN to completion on an isolated two-node subproblem consisting of only that specific warehouse and store, using the same parameters (holding costs, underage costs, lead times, and demand distributions) as sampled for the original multi-warehouse problem. Second, for each store, we select the warehouse that achieves the minimum test loss among all trained warehouse-store pairs. The resulting Single-warehouse NN is equivalent to the full GNN architecture, except that we "trim" all edges except those corresponding to the cost-minimizing warehouse for each store. This creates a simplified policy graph with static warehouse-store assignments on which the GNN-based policy is trained. The screening process typically selected warehouses with the lowest edge costs, and occasionally favored those with the shortest lead times.

{\bf Results.} Results are shown in Table \ref{table:value-of-flexibility}.

\begin{table}[h]
\centering
\scriptsize
\caption{Metrics for ablation study using Setting S9. Numbers in bold represent the best test loss for a fixed number of stores.}
\label{table:value-of-flexibility}

\begin{tabular}{>{\raggedleft}p{1.9cm}>{\raggedleft}p{2.3cm}>{\raggedleft}p{3.0cm}>{\raggedleft}p{1.2cm}>{\raggedleft}p{0.8cm}>{\raggedleft}p{0.8cm}>{\raggedleft}p{0.8cm}>{\raggedleft\arraybackslash}p{2.7cm}}
\toprule
Number of warehouses & Training scenarios (\#) & Architecture Class & Learning rate & Train profit & Dev profit & Test profit & Test profit relative to Just-in-time (\%) \\
\midrule
2 & 288 & GNN & 0.0010 & 614.07 & 588.22 & \textbf{534.90} & 82.32 \\
2 & 288 & Single-warehouse NN & 0.0010 & 596.70 & 573.89 & 522.01 & 80.34 \\
3 & 288 & GNN & 0.0010 & 622.05 & 590.60 & \textbf{537.06} & 81.44 \\
3 & 288 & Single-warehouse NN & 0.0001 & 603.98 & 578.60 & 527.01 & 79.92 \\
4 & 288 & GNN & 0.0010 & 617.12 & 589.20 & \textbf{536.53} & 82.45 \\
4 & 288 & Single-warehouse NN & 0.0001 & 596.71 & 574.41 & 523.56 & 80.46 \\
5 & 288 & GNN & 0.0010 & 619.03 & 587.66 & \textbf{536.57} & 82.35 \\
5 & 288 & Single-warehouse NN & 0.0001 & 595.48 & 573.10 & 521.33 & 80.01 \\
\bottomrule
\end{tabular}

\end{table}

\clearpage

\section{Asymptotic optimality of GNN policy for a one-warehouse many-store setting}
\label{appendix: proof of main theorem}

This section presents a complete proof that our GNN architecture can represent asymptotically optimal policies in the stylized setting—introduced in Section~\ref{sec: GNN_constructive}—where a warehouse supplies multiple stores. We begin by describing the setting and stating a general version of the main theorem, which we specialize in subsequent sections. The formal setup is defined below.

\begin{example}
\label{example: 1d context}
\edit{Consider one warehouse (denoted by $w$) and $K$ stores (denoted by $1,\ldots,K)$ under a backlogged demand assumption.} The underage and holding costs of each store $k \in \mathcal{N}_{\text{st}}$ satisfy $p^k \in [\underline{p}, \overline{p}]$ and $h^k \in [\underline{h}, \overline{h}]$ {for some $\underline{p}, \overline{p}, \underline{h}, \overline{h} \in \R_+$}. 
Demand is of the form $\xi_t^k = B_t U_t^k$, where $U_t^k \sim \textup{Uniform}(\ul^k, \uh^k)$ is drawn independently across stores and time (with $\uh^k \in [1/\kappa, \kappa] \, \forall k \in \mathcal{N}_{\text{st}}$ for some $\kappa \in \R$), and $B_t$ is drawn independently across time; it takes a \emph{high} value $\gammah > 0$ with probability $q$ and a \emph{low} value $\gammal  \in (0, \gammah)$ otherwise.
The lead time from the outside supplier (node $0$) to the warehouse is $1$ and the lead time from warehouse to stores is $0$.
The system starts with zero inventory. 
We make two extra assumptions. First, $ \gammal  \ul^k \geq \gammah \uh^k - \gammal  \ul^k$ for every store $k$, which ensures that, if store $k$ starts with an inventory no larger than $\gammah \uh^k$, the remaining inventory after demand is realized is below $\gammal  \ul^k$. (Note that this assumption implies that $\ul^k > 0$ for every $k$, ensuring $U^k_t>0$ w.p. 1 for every $k$ and $t$.) Second, $q \underline{p} \geq (1-q)h^w$, which ensures that it is worth acquiring an incremental unit of inventory at the warehouse if that is sure to prevent a lost sale in a period of high demand. $\halmos$
\end{example}
We comment briefly on the structure of this example setting. Here, due to the aggregate demand being either ``high'' or ``low'' in each period, and the lead time from the outside supplier to the warehouse being a single period, in each period there is either a scarcity or an abundance of inventory in the network, due to the aggregate demand having been ``high'' or ``low'', respectively, in the previous period. As a result, edge-level decisions must rely on summary statistics that reflect the global inventory state in order to place appropriate orders. We believe it is possible to prove that relying only on local information at each edge leads to an $\Omega(1)$ increase in expected cost under a range of problem primitives. By contrast, we expect that no information about the global state of the system is needed when demand realizations are independent across stores. 
We establish the following guarantee of asymptotic optimality in per-period costs of a $1$-round message-passing policy, as the number of stores $K$ grows. In this asymptotic limit, an optimal allocation towards each store can be represented by a function that depends only on the local state and unit economics of the store, along with a summary statistic of the overall system's inventory. \edit{This function is the same across all stores and allows allocations to differ only through the values of local parameters and state; this function can be easily represented by our GNN architectures presented in Section \ref{sec:gnn-section} (see Theorem \ref{appendix:proof overview} for a description).}

\begin{definition}
    A policy $\pi$ is an $R$-round GNN policy if it can be represented by the GNN architecture of Algorithm \ref{alg:samp_gnn} with $L^{\text{MP}} = R$ message-passing rounds.
\end{definition}

\begin{theorem}
\label{theorem:1-round-message-passing-policy}
 
In Example \ref{example: 1d context}, let $J_1^{\pi}$ be the expected total cost incurred by policy $\pi$ from the initial state. There exists $C=C(\overline{p}, \overline{h}, \underline{p}, \underline{h}, q, \kappa, \gammal , \gammah )< \infty$ such that the following occurs. There exists a $1$-round message-passing policy $\tilde{\pi}$ which satisfies
$\frac{J^{\tilde{\pi}}_1}{ \inf_{\pi} J^{\pi}_1} \leq 1 + \frac{C}{\sqrt{K}}$.
\end{theorem}
In the proof presented in the remainder of this section, we construct an asymptotically optimal policy~$\tilde{\pi}$ in which the allocation from the outside supplier to the warehouse is obtained by following an echelon-stock policy (see Equation~\ref{eq:echelon-stock-transshipment} in Appendix~\ref{appendix:transshipment-setting-lower-bound}) at the warehouse, and the allocations from the warehouse to each store are determined by following a base-stock policy (see Equation~\ref{eq:base-stock-policy} in Appendix~\ref{appendix:base-stock-policy}) at each store. The current base-stock level depends only on the total inventory $\sum_{k \in \mathcal{N}_+} I_t^k$, which can be computed at each edge after nodes pass a message that signals each location’s inventory on hand.

\subsection{Problem instantiation.}
We now review the relevant notation and introduce auxiliary definitions to facilitate a \emph{relaxed} formulation of our problem. Since stores have zero lead time, the transition and cost functions differ slightly from those in Section~\ref{sec: problem formulation} and Appendix~\ref{appendix:model}. For clarity and brevity, we overload notation in a few ways. We index stores by $1, \ldots, K$ and denote the warehouse by $w$, yielding $K$ stores, $K{+}1$ locations, and $K{+}2$ nodes including the external supply node. We write $[K]$ to denote the set of store indices, and $\mathcal{N}_+$ to represent the set of all physical locations excluding the external supplier—i.e., the $K$ stores and the warehouse. We simplify action notation by writing $a_t^k$ for the quantity allocated to node $k$, omitting the sender index—this is unambiguous since each node receives from at most one upstream node. Throughout, we interpret $a_t^k$ as the order placed by node $k$ to its supplier, aligning with standard inventory control conventions.

{\bf State and action spaces.}
The state of the system is given by $S_t = I_t = (I_t^w, I_t^1, \ldots, I_t^K)$. We emphasize that $I_t^k$ represents the inventory on-hand for location $k \in \mathcal{N}_+$ {\bf before} orders are placed.  As store lead times are zero, there are no outstanding orders. Further, for the warehouse, we can update inventory by replacing $Q^k_{t}(1)$ by $a_t^w$ in \eqref{eq: transition warehouse}, avoiding the need for keeping track of $Q_t^1$. We further let $Z_t = \sum_{k \in \mathcal{N}_+} I_t^k$ track the system-wide inventory in the system. The action space at state $S_t$ is given by $\mathcal{A}(I_t) = \{a_t \in \mathbb{R}_+^{K+1} \ | \ \sum_{k \in [K]} a_t^k \leq I_t^w \}$.

{\bf Random variables.}
As described in Example \ref{example: 1d context}, demand takes the form $\xi_t^k = B_t U_t^k$, where $U_t^k \sim \textup{Uniform}(\ul^k, \uh^k)$ and $B_t$ takes a \emph{high} value $\gammah > 0$ with probability $q$ and a \emph{low} value $\gammal  \in (0, \gammah)$. We denote the mean $\E[U_t^k]$ of each uniform as $\mu^k = (\uh^k + \ul^k)/2$, and denote the aggregate demand by $\Xi_t = \sum_{k \in [K]} \xi_t^k$. Further, let $\hat{U}_t = \sum_{k \in [K]} U_t^k$ be the sum of the $K$ uniforms at time $t$, $\hat{\mu} = \sum_{k \in [K]}\mu^k$ its mean, and $\tilde{U}_{t}^k = \sum_{k \in [K]} (U_{t}^k - \mu^k)$ be the sum of zero-mean uniforms. Note that we can represent aggregate demand as $\Xi_t = B_t \hat{U}_t$. 

We define $\Dh = \E[\Xi_t | B_t = \gammah] = \gammah  \hat{\mu}$ and $\Dl = \E[\Xi_t | B_t = \gammal] = \gammal  \hat{\mu}$ as the expectation of aggregate demand conditional on $B_t$ taking \emph{high} and \emph{low} values, respectively, {and let $D_t$ be a random variable such that $D_t = D^H$ if $B_t = \gamma^H$ and $D_t = D^L$ otherwise}.  Finally, let $W_t = \mathbbm{1}(B_{t-1} = \gammal )$ be an indicator that the demand was low in the previous period, and $\tilde{W}_t(Z_t)$ be an estimator of the previous quantity, which will be defined shortly. For brevity, we define the shorthand notation $\tilde{W}_t \equiv \tilde{W}_t(Z_t)$.

{\bf Transition functions}

Given that stores have no lead time and that the warehouse lead time is equal to $1$, inventory evolves as
\begin{align}
    I^k_{t+1} &= I^k_{t} + a^k_{t} - \xi^k_t & k \in [K] \label{eq: transition backlogged for proof} \\
    I^w_{t+1} &= I^w_t + a^w_t - {\textstyle \sum_{k \in [K]}} a^k_t \label{eq: transition warehouse for proof}.
\end{align}

Even though \eqref{eq: transition backlogged for proof} would be identical for store lead times of one period, there is a distinction in the cost function between lead times of zero and one, as shown later on. 

Meanwhile, the system-wide level of inventory evolves as
\begin{align}
    Z_{t+1} = Z_{t} + a^w_t - \Xi_t\, ,
    \label{eq:system-wide-inventory}
\end{align}
and note that $Z_t$ does not depend on $a^1_t, \ldots, a^K_t$. 

{\bf Cost functions.} \emph{The case when store lead times are zero is an edge-case where the cost incurred at stores is different from what was written in the generic problem formulation.} In particular, the cost at store $k$  in some period is
given by $c^k(I^k_t, a^k_t, \xi^k_t) = p^k(\xi^k_t - (I^k_t + a^k_t))^+ + h^k((I^k_t + a^k_t) - \xi^k_t)^+$, where the first and second terms correspond to underage and overage costs, respectively. Note that, as store lead times are zero, the inventory $a_t^k$ allocated in the current period arrives to the store \emph{before} demand occurs, so the store incurs the same overage/underage costs it would if lead time were nonzero and it had inventory $I^k_t + a^k_t$ on hand at the start of the period. 

The cost at the warehouse is solely composed of holding costs and is given by $c^w(I^w_t, a_t) = h^w\left(I^w_t - \sum_{k \in [K]}a^k_t\right)$. Total costs in period $t$ are finally $c(I_t, a_t, \xi_t) = c^w(I^w_t, a_t)  + \sum_{k \in [K]} c^k(I^k_t, a^k_t, \xi^k_t)$. We will define costs in a more convenient way by defining
\begin{align}
\label{eq: immediate cost}
    v(Z, y^1, \ldots, y^K) = h^w Z + \sum_{k \in [K]} \E_{\xi^k_t}\left[\left( p^k(\xi^k_t - y^k)^+ + h^k(y^k - \xi_t^k)^+ - h^w y^k\right) \right]
\end{align}
and noting that $c(I_t, a_t, \xi_t)=v\left(\sum_{k \in  \mathcal{N}_+} I^k_t, I_t^1 + a_t^1, \ldots, I_t^K + a_t^K\right)$.

{\bf Cost-to-go.}
We define $J^{\pi}_t(I_t) = \mathbb{E}^{\pi} \left [ \, {\textstyle \sum_{s = t}^T} c(I_s, \pi_{s}(I_s), \xi_s) \big | I_t \right ]$ as the expected cost-to-go under policy $\pi \in \Pi$ from period $t$ and initial state $I_t$, and $J_t(I_t)$ as the optimal cost-to-go, that is, $J_t(I_t) = \inf_{\pi \in \Pi} J^{\pi}_t(I_t)$.

Recall that the setting in Example~\ref{example: 1d context} makes two assumptions. 
\begin{assumption}
\label{assumption:base-level-above}
$ \gammal  \ul^k \geq \gammah \uh^k - \gammal  \ul^k$ for every $k$.
\end{assumption}

\begin{assumption}
\label{assumption:sold-unit-worth}
$q \underline{p} \geq (1-q)h^w$.
\end{assumption}

These assumptions will facilitate our analysis.
Assumption~\ref{assumption:base-level-above} ensures that, if store $k$ starts with an inventory no larger than $\gammah \uh^k$, the remaining inventory after demand is realized is below $\gammal  \ul^k$. This condition can be slightly relaxed relative to the form stated, but we avoid presenting the most general form in the interest of simplicity. Assumption~\ref{assumption:sold-unit-worth} ensures that it is worth acquiring an incremental unit of inventory at the warehouse if that is sure to prevent a lost sale in a period of high demand.

\subsection{Proof overview 
\label{appendix:proof overview}}

We now present a detailed version of Theorem \ref{theorem:1-round-message-passing-policy}. At the end of this section, we will show how this asymptotically optimal policy can be represented as a GNN policy with one round of message-passing.
Let $$\kappa \equiv \max_{k \in [K]} \max (\uh^k, 1/ \uh^k) \quad \Rightarrow \quad \uh^k \in [1/\kappa, \kappa] \, \forall k \in [K]. $$
Recall that the Proportional Allocation feasibility enforcement layer (see Equation \ref{eq:proportional-allocation} in Section \ref{appendix:feasibility-enforcement}) is designed to ensure that the total outbound allocation from a DC does not exceed its inventory on hand. In this setting, we apply it to the warehouse node, which supplies all $K$ stores. The resulting allocation rule is given by $g_1(I^w, b^{1}, \ldots, b^{K}) = \left[ [b^k]^+ \cdot \min\left\{1, \frac{I^w}{\sum_{j \in [K]} [b^{j}]^+} \right\} \right]_{k \in [K]}$.
Moreover, denote $\mathcal{R}^k = (\underline{u}^k, \overline{u}^k, p^k, h^k)$ as the store-specific primitives of location $k$, and let $\mathcal{H} \subset \R^4$ denote the set of possible values for store-specific primitives according to the assumptions outlined in Example \ref{example: 1d context}.

\begin{theorem}
\label{theorem:message-passing-base-stock}
Under the Proportional Allocation FEL, there exists a constant $C = C(\overline{p}, \overline{h}, \underline{p}, \underline{h}, q, \kappa, \gammal , \gammah) < \infty$, a scalar $y^w \in \R$, continuous functions $\underline{G}, \overline{G}: \mathcal{H} \to \R_+$, and a stationary \edit{1-round GNN policy} $\tilde{\pi}$ with
\begin{align}
\begin{split}
    \tilde{\pi}^{\text{dc}} \left(I_t^w, Z_t \right) &= y^w - Z_t, \\
    \tilde{\pi}^{\text{st}} \left(I_t^k, \mathcal{R}^k, Z_t \right) &= \max\left(0, \left[ \underline{G}(\mathcal{R}^k) + \tilde{W}_t(Z_t) \left( \overline{G}(\mathcal{R}^k) - \underline{G}(\mathcal{R}^k) \right) \right] - I^k_t \right) \quad \text{for all } k \in [K],
\label{eq:optimal-message-passing-policy}
\end{split}
\end{align}
where $\tilde{W}_t(Z_t) = \mathbbm{1}\left\{ y^w - Z_t < (\Dh + \Dl)/2 \right\}$ and $Z_t = \sum_{k \in \mathcal{N}_+} I^k_t$. The expected cumulative cost $J^{\tilde{\pi}}_1(I_1)$ satisfies
\begin{align*}
    \frac{J^{\tilde{\pi}}_1(I_1)}{J_1(I_1)} \leq 1 + \frac{C}{\sqrt{K}}.
\end{align*}
\end{theorem}

{\bf Policy structure and intuition.} Throughout this section, we use the shorthand \(\underline{y}^k = \underline{G}(\mathcal{R}^k)\) and \(\overline{y}^k = \overline{G}(\mathcal{R}^k)\), where \(\underline{G}\) and \(\overline{G}: \mathcal{H} \to \mathbb{R}_+\) are continuous functions constructed in the proof of Lemma~\ref{lemma:relaxed base-stock fills each store} (see Appendix~\ref{appendix:relaxed base-stock fills each store}).
For the asymptotically optimal policy $\tilde{\pi}$, allocations toward the warehouse are determined by an echelon-stock policy (see Equations ~\ref{eq:echelon-base-stock-policy-transshipment} and ~\ref{eq:echelon-base-stock-level-transshipment} in Appendix~\ref{appendix:transshipment-setting-lower-bound}). Specifically, the policy uses $Z_t$ to allocate inventory from the external source in order to raise the system-wide inventory level up to the target $y^w$. 
Allocations toward the stores also depend on $Z_t$, which is used to compute the estimator $\tilde{W}_t(Z)$—an indicator of whether the previous period’s demand was low. Based on this estimate, each store computes an intermediate output that targets either the lower or upper base-stock level, $\underline{y}^k$ or $\overline{y}^k$, respectively.
The intuition behind $\tilde{W}_t(Z)$ is as follows. Under policy $\tilde{\pi}$, the system-wide inventory is brought to $y^w$ before demand is realized. Therefore, the quantity $(y^w - Z_t)$ corresponds exactly to the cumulative demand in the previous period. If this value is less than the midpoint $(\Dh + \Dl)/2$ of the conditional expectations for high and low demand periods, the policy interprets this as evidence that the previous demand was low.

\edit{{\bf Representation as a GNN policy.} We now show how the policy above can be represented under our GNN architecture. In the setting of the theorem, the warehouse node participates in every edge as either sender or receiver, which allows aggregate inventory to be communicated to each edge's embedding after one round of message-passing. Edge embeddings track aggregate inventory levels, the receiving node's type (store or warehouse), and the local state and parameters; for notational convenience, we will denote embeddings as only tracking aggregate inventory levels, since the other coordinates remain fixed across updates.

We define $\textup{NodeEmbed}(S^k_t) = I^k_t$ and $\textup{EdgeEmbed}(h^j, h^k, L^{(j,k)}) = I^k_t$. The key insight is to use the warehouse node as an aggregator: we let $\textup{UpdateNode}(h^k, z^k_{\textup{in}}, z^k_{\textup{out}}) = z^k_{\textup{out}}$, so that the warehouse node's embedding equals aggregate inventory $Z_t = \sum_{k \in \mathcal{N}_+} I^k_t$ after one message pass. 
To propagate this aggregate information back to the edges, we let $\textup{UpdateEdge}(h^{(j,k)}, h^j, h^k) = h^j - h^{(j,k)}$ if $k$ is a store, and $\textup{UpdateEdge}(h^{(j,k)}, h^j, h^k) = h^k - h^{(j,k)}$ otherwise. Since the warehouse node embedding becomes the aggregate inventory $Z_t$, this update rule ensures that all edge embeddings capture $Z_t$. Finally, a suitable $\textup{Readout}$ function can implement the policy from Theorem \ref{theorem:message-passing-base-stock} by applying $\tilde{\pi}^{\text{dc}}$ for edges terminating at the warehouse and $\tilde{\pi}^{\text{st}}$ for edges terminating at stores, since all edge embeddings contain the required aggregate inventory $Z_t$ and local parameters $\mathcal{R}^k$.}

{\bf Proof strategy.} Our proof of Theorem~\ref{theorem:message-passing-base-stock} consists of two main parts, which we now summarize:

\begin{enumerate}
    \item {\bf Converse bound}: We start by constructing a relaxed setting in two steps. First, we allow inventory to "flow back" from stores to the warehouse, and show that the cost-to-go only depends on the current state through the system-wide level of inventory. We then analyze a setting, which we refer to as the \emph{fully relaxed setting}, in which the system-wide level of inventory evolves as if $\hat{U}_t = \hat{\mu}$, but in which the costs incurred at stores do account for the randomness coming from $U^k_t$, and show that the expected cost in this relaxed system is a lower bound on that of the original one. 
    
    \item {\bf Achievability}: Our strategy for establishing achievability begins with a detour in which we consider the fully relaxed setting and we characterize the optimal policy for that system: We will demonstrate the optimality of a stationary policy for the fully relaxed setting, where the warehouse follows an echelon stock policy and the stores follow a base-stock policy with one of two base-stock levels. Additionally, we will establish that the inventory at the stores prior to demand realization is kept below both base-stock levels. This indicates that the optimal policy for the fully relaxed setting does not involve transferring inventory from the stores to other stores or the warehouse. We will then derive an asymptotically optimal policy for the original system by following the warehouse's echelon stock level and store's base-stock levels from the optimal policy in the relaxed setting, and employing a Proportional Allocation feasibility enforcement layer (Eq. \ref{eq:proportional-allocation} in Section \ref{appendix:feasibility-enforcement}). We will use a simple upper bound on the the per-unit cost of "scarcity" (\ie units of unsatisfied store's orders) and show that scarcity is upper bounded by the deviation of the sum of Uniforms $\hat{U}_t$ from its mean. We will conclude that per-period scarcity costs scale as $\sqrt{K}$, implying a relative excess cost of order $1/\sqrt{K}$ as compared to that incurred in the fully relaxed setting.
\end{enumerate}

\subsection{Converse bound}

We begin by presenting the structure of the optimal cost-to-go in the original setting, and then introduce a series of relaxations that will allow us to obtain a lower bound on costs.

In Appendix \ref{appendix:Bellman Equation for Optimal Cost-to-go}, we show the following expression for the cost-to-go.

\begin{lemma}
The cost-to-go (for the original setting) in period $t$ takes the form
\begin{align}
    J_t(I_t)
    &= \inf_{a_t \in \mathcal{A}(I_t)} \left\{\E_{\xi_t} \left[J_{t+1}(f\left(I_{t}, a_t, \xi_t\right))\right] +  v\left(\sum_{k \in \mathcal{N}_+} I^k_t, I_t^1 + a_t^1, \ldots, I_t^K + a_t^K\right) \right\},
\label{eq:cost-to-go-original}
\end{align}
where the immediate cost $v(\cdot)$ was defined in \eqref{eq: immediate cost}.
\end{lemma}

We construct a relaxed system, which we denote as the \emph{partially relaxed system}, following a procedure similar to that of \cite{federgruen1984approximations}. We will allow $a^1, \ldots, a^K$ to be negative, thus allowing inventory to flow from stores to the warehouse and be re-allocated to other stores. The action space for the partially relaxed system is given by $\hat{\mathcal{A}}(I_t) = \{a_t \in \mathbb{R}^+ \times \mathbb{R}^{K} \ | \ \sum_{k \in [K]} a_t^k \leq I_t^w \}$. As formalized in Appendix \ref{appendix:cost-to-go on z}, since we allow inventory to flow back from stores, the cost-to-go for the partially relaxed system $\breve{J}_{t+1}(I_{t+1})$ depends on $I_{t+1}$ solely through the sum of its components $Z_{t+1} = \sum_{k \in \mathcal{N}_+}I^k_{t+1} = Z_{t} + a^w_t - \Xi_t$ (from \eqref{eq:system-wide-inventory}).  This leads to the following Bellman equation for the partially relaxed cost-to-go
\begin{align}
    \breve{J}_t(I_t)
    &=\inf_{a_t \in \hat{\mathcal{A}}(I_t)} \left\{\E_{\xi_t} \left[\breve{J}_{t+1}(Z_{t} + a^w_t - \Xi_t)\right] +  v\left(\sum_{k \in \mathcal{N}_+} I^k_t, I_t^1 + a_t^1, \ldots, I_t^K + a_t^K\right) \right\},\label{eq:cost-to-go} \\
    &\stackrel{}{=}  \inf_{a_t^w \in \mathbb{R}^+} \E_{\xi_t} \left[\breve{J}_{t+1}(Z_{t} + a^w_t - \Xi_t)\right] +  \inf_{\sum_{k \in [K]} a_t^k \leq I_t^w } v\left(\sum_{k \in \mathcal{N}_+} I^k_t, I_t^1 + a_t^1, \ldots, I_t^K + a_t^K\right) \, .  \label{eq:cost-to-go-2}
\end{align}
Here, \eqref{eq:cost-to-go-2} captures that the minimization over $a_t$ is separable across the warehouse action $a_t^w$ versus the store allocations $(a_t^k)_{k \in [K]}$, because 
the cost-to-go term in \eqref{eq:cost-to-go} does not depend on store allocations, and the immediate cost term $v(\cdot)$ in \eqref{eq:cost-to-go}  does not depend on the warehouse order $a_t^w$. (We abuse notation in writing $\breve{J}_t(I_t) = \breve{J}_t(Z_t)$.)

Letting $y_t^k = I_t^k + a_t^k$ for $k \in [K]$, the second minimization in \eqref{eq:cost-to-go-2} can be rewritten as a minimization over $y_t = (y_t^k)_{k \in [K]}$ as follows:
\begin{align}
\label{eq:R-relaxed}
    \begin{split}
    \hat{R}(Z_t) &\equiv \inf_{y_t \in \R^K} v \left(Z_t, y_t^1, \ldots, y_t^K \right)\\
    & \quad \textrm{s.t.} \sum_{k \in [K]} y_t^k\leq Z_t\, .
    \end{split}
\end{align}
As a result, the Bellman equation \eqref{eq:cost-to-go-2} for the partially relaxed system can be written as 
\begin{align}
    \label{eq:bellman partially relaxed}
    \breve{J}_t(Z_t) &= \hat{R}(Z_t) + \min_{a^w_t \in \mathbb{R}^+} \E_{\Xi_t}\left[ \breve{J}_{t+1}(Z_t + a^w_t - \Xi_t) \right].
\end{align}

We introduce an additional relaxation referred to as the \emph{fully relaxed system}, where the state evolves as if $\hat{U}_t = \hat{\mu}$, while maintaining the same cost incurred in each period as in the partially relaxed system. In the fully relaxed setting, we denote  the system-wide inventory level by $\hat{Z}_t$, and define that it evolves according to $\hat{Z}_{t+1} = \hat{Z}_t + a_t^w - D_t$. It is important to note that this quantity is, in general, different from the system-wide inventory levels in the original and partially relaxed systems, even when fixing a policy and demand trace. However, we will demonstrate later that the optimal expected cost in all three systems are close to each other. Let $\hat{J}_t(\hat{Z}_t)$ represent the cost-to-go for the fully relaxed system, starting from period $t$ with a system-wide inventory level of $\hat{Z}_t$. The Bellman equation for the fully relaxed system is 
\begin{align}
    \label{eq:bellman fully relaxed}
    \hat{J}_t(\hat{Z}_t) &= \hat{R}(\hat{Z}_t) + \min_{a^w_t \in \mathbb{R}^+} \E_{D_t}\left[ \hat{J}_{t+1}(\hat{Z}_t + a^w_t - D_t) \right]\, .
\end{align}

We consider $\hat{Z}_1 = Z_1$ as the initial state for the fully relaxed system. 

In Appendix \ref{appendix:R-hat-convex-proof}, we show the following key convexity property.
\begin{lemma}
\label{lemma:convexity}
$\hat{R}(\cdot)$ is a convex function. For all $t = 1, 2, \dots, T$, $\breve{J}_t(\cdot)$ and $\hat{J}_t(\cdot)$ are convex functions. 
\end{lemma} 
We use this convexity to prove that our relaxations indeed lead to a lower bound on costs for the original setting.
\begin{proposition}
    \label{relaxation is lower bound}
    For any starting inventory state $I_1 \in (\mathbb{R}^+)^{K+1}$ and $Z_1 = \sum_{k \in \mathcal{N}_+} I_1^k$, we have
    $$\hat{J}_1({Z}_1) \leq \breve{J}_1(Z_1) \leq J_1(I_1)\, ,$$ 
    \ie the expected cost in the fully relaxed system is weakly smaller than the expected cost in the partially relaxed system, which, in turn, is weakly smaller than the expected cost in the original system.
\end{proposition}

We show Proposition \ref{relaxation is lower bound} in Appendix \ref{proof: relaxation is lower bound}. The first inequality is obtained by using the convexity of $\hat{R}$ and applying Jensen's inequality on $\hat{U}_t$. The second inequality is immediate since $\breve{J}$ is obtained by relaxing the original setting.

\subsection{Achievability}

We begin by unveiling the structure of the optimal policy for the fully relaxed setting which, in turn, will allow us to construct an asymptotically optimal policy for the original setting.

Let $y_t^k = I_t^k + a_t^k$, for every $k \in [K]$, be the inventory level at store $k$ before demand is realized.
\begin{lemma}
    \label{lemma:relaxed policy optimal structure}
    There exists an optimal policy (which is stationary) for the fully relaxed setting such that the warehouse follows an echelon-stock policy. Furthermore, the optimal echelon-stock level $\hat{S}$ satisfies  
    \begin{align}
    \label{eq:optimal-echelon-fully-relaxed}
    \begin{split}
    \hat{S} \in \arg\min_{S \in \R} \left[ q \hat{R}(S - \Dh)  +  (1-q) \hat{R}(S - \Dl) \right] \, ,
    \end{split}
    \end{align}
    where $\hat{R}(\cdot)$ is as defined in \eqref{eq:R-relaxed}.
\end{lemma}

We prove Lemma \ref{lemma:relaxed policy optimal structure} in Appendix \ref{appendix:relaxed policy optimal structure}. The policy for the warehouse relies on the fact that \eqref{eq:bellman fully relaxed} takes the form of a single-location inventory problem with inventory level $\hat{Z}_t$ and convex cost function $\hat{R}(\hat{Z}_t)$. Following 
\cite{federgruen1984approximations}, an echelon-stock policy must be optimal. Since there are no procurement costs involved, we demonstrate that the echelon-stock level in each period should minimize the expected cost incurred in the subsequent period, as indicated in \eqref{eq:optimal-echelon-fully-relaxed}. Consequently, the echelon-stock level remains constant across all periods.

Starting from $t=2$ onwards, it is important to observe that $\hat{Z}_t$ can only assume one of two distinct values under the optimal policy: either $\hat{S} - \Dh$ or $\hat{S} - \Dl$. Since store inventory levels solve \eqref{eq:R-relaxed} in each period, at each store the inventory level before the actual demand occurs will take one of two values (corresponding to the two aforementioned solutions of \eqref{eq:R-relaxed}). These inventory levels will be interpreted as store base-stock levels in subsequent analysis.

\begin{lemma}
    \label{lemma:relaxed base-stock fills each store}
    Suppose Assumption~\ref{assumption:sold-unit-worth} holds, and let $\underline{y}$ and $\overline{y}$ be optimal solutions for problem \eqref{eq:R-relaxed} with total inventory $Z_t = \hat{S} - \Dh_t$ and $\hat{S} - \Dl_t$, respectively. Then, $\gammah \uh^k \geq \overline{y}^k \geq \underline{y}^k \geq \gammal \ul^k$ for each $k \in [K]$.
    In particular, an optimal policy for the fully relaxed system does not move inventory from a store to other stores or to the warehouse. Additionally, there exist continuous functions $\underline{G}, \overline{G}: \mathcal{H} \to \R_+$ (where $\mathcal{H}$ was defined in Appendix \ref{appendix:proof overview}), such that $\underline{y}^k = \underline{G}(\mathcal{R}^k)$ and $\overline{y}^k = \overline{G}(\mathcal{R}^k)$ for each $k \in [K]$.
\end{lemma}

We prove Lemma~\ref{lemma:relaxed base-stock fills each store} in Appendix~\ref{appendix:relaxed base-stock fills each store}. The argument leading to the assertion that an optimal policy for the fully relaxed system does not move inventory from a store to other stores or to the warehouse is as follows: Recall that Assumption~\ref{assumption:base-level-above} ensures that if store $k$ starts with an inventory no larger than $\gammah \uh^k$, the remaining inventory after demand is realized is below $\gammal  \ul^k$. Hence, Lemma~\ref{lemma:relaxed base-stock fills each store} allows us to conclude that the base-stock levels {$\overline{y}^k, \underline{y}^k$} for each store are always above the store's inventory level after demand occurs under our assumptions.
Therefore, an optimal policy for the fully relaxed system does not move inventory from a store to other stores or to the warehouse.
Moreover, in the Lemma we show how to compute $\underline{y}^k$ and $\overline{y}^k$ from the optimality conditions of problem \eqref{eq:R-relaxed} with initial inventories $\hat{S} - \Dh_t$ and $\hat{S} - \Dl_t$, respectively, which permits the construction of the aforementioned continuous functions $\underline{G}$ and $\overline{G}$.

We derive a feasible policy $\tilde{\pi}$ to the original problem by following the structure in \eqref{eq:optimal-message-passing-policy},
and setting $y^w = \hat{S}$ for $\hat{S}$ as defined in Lemma~\ref{lemma:relaxed policy optimal structure}, and $\overline{y}^k$ and $\underline{y}^k$ for $k \in [K]$ as per the definitions in Lemma~\ref{lemma:relaxed base-stock fills each store}.
That is, the warehouse will follow an echelon-stock policy with level $\hat{S}$ and stores will follow a base-stock policy, with level $\overline{y}^k$ (or $\underline{y}^k$) whenever the demand in the previous period is estimated as being low (high). We derive feasible actions by using the Proportional Allocation feasibility enforcement layer $g_1(I^w, b^{1}, \ldots, b^{K}) = [b^k]^+ \cdot \min\big\{1, \tfrac{I^w}{\sum_{j \in [K]} [b^j]^+} \big\}$.
We can therefore represent tentative allocations and actions under $\tilde{\pi}$ by
\begin{align*}
    a^{\tilde{\pi}w}_t &= \hat{S} - Z_t \\
    b^{\tilde{\pi} k}_t &= \max(0, \left[ \underline{y}^k + \tilde{W}_t (Z_t) (\overline{y}^k - \underline{y}^k) \right] - I^k_t) & \quad k \in [K] \\
    a^{\tilde{\pi} k}_t &=  \left[g_1(I_t^w, b^{\tilde{\pi} 1}_t, \ldots, b^{\tilde{\pi} K}_t)\right]_k   & k \in [K].
\end{align*}

We provide a brief explanation regarding the role of $\tilde{W}_t$ and its definition $\tilde{W}_t(Z) \equiv   \mathbbm{1}\left\{ 
y^w - Z_t < (\Dh + \Dl)/2 \right\}$. Recall our assumption that only the observed demand is available to us, without knowledge of the underlying high or low state. Therefore, an estimate is required based on the current system-wide inventory levels. We compare the realized total demand in the previous period $y^w-Z_t$ with the midpoint between $\Dh$ and $\Dl$ to estimate whether the demand in the previous period was high or low $B_{t-1}$. As the number of stores $K$ increases, the demand fluctuations caused by the sum of uniform random variables $\hat{U}_t$ scale by approximately $\sqrt{K}$, while the difference between the midpoint and both $\Dh$ and $\Dl$ scales linearly with $K$. Consequently, as $K$ grows larger, the probability of incorrect estimation $\tilde{W}_t \neq B_{t-1}$ decreases exponentially, as demonstrated in Appendix~\ref{appendix:bound-for-deviations}.

Recall that $\tilde{U}_{t}^K = \sum_{k \in [K]} (U_{t}^k - \mu^k)$ represents the centered sum of the $K$ uniforms at time $t$, \ie the sum minus its mean. Let $\alpha = |\hat{R}(\hat{S} - \Dl) - \hat{R}(\hat{S} - \Dh)| +  \max\{\overline{p}, \overline{h}\} (|\Dh - \Dl|)$.

\begin{proposition}
    \label{proposition: upper bound on feasible policy}
    The total expected cost in the original system under the policy $\tilde{\pi}$ is bounded above as
    \begin{align}
        \label{eq:upper-bound-feasible-policy}
        J_1^{\tilde{\pi}}(I_1) \leq \hat{J}_1(Z_1) + T \left[ \max\{\overline{p}, \overline{h}\} \gammah  \E \left [ \left|\tilde{U}^K_{1} \right| \right] + \alpha \mathbb{P}\left(\left|\tilde{U}^K_{1} \right| \geq \hat{\mu} \frac{\gammah  - \gammal }{2 \gammah }\right) \right]\, .
    \end{align}
(Recall that $\hat{J}_1(Z_1)$ is the total expected cost in the fully relaxed system, with the same starting inventory.)
\end{proposition}

We prove Proposition \ref{proposition: upper bound on feasible policy} in Appendix \ref{proof: upper bound on feasible policy}. The proposition bounds the additional expected cost per period in the real system relative to that under the fully relaxed system. The $T  \max\{\overline{p}, \overline{h}\} \gammah  \E [ |\tilde{U}^K_{1} | ]$ term bounds the impact of the aggregate demand deviating from $\Dl$ or $\Dh$, and the $T\alpha \mathbb{P}(|\tilde{U}^K_{1} | \geq \hat{\mu} \frac{\gammah  - \gammal }{2 \gammah })$ term bounds the impact of occasionally getting $W_t$ wrong $W_t \neq B_{t-1}$. In short,
the additional expected cost per period in the real system is at most
proportional to deviation of the sum of uniforms from its mean. We show the latter quantity to be ``small'', specifically we show that it is $O(\sqrt{K})$ by the central limit theorem, in Appendix~\ref{appendix:bound-for-deviations}. Plugging this bound into Proposition~\ref{proposition: upper bound on feasible policy}, together with the converse bound in Proposition~\ref{relaxation is lower bound} and \ref{proposition: upper bound on feasible policy} leads us to a proof of Theorem~\ref{theorem:message-passing-base-stock} in Appendix~\ref{appendix:concluding-theorem-appendix}.

\subsection{Bellman Equation for Optimal Cost-to-go \label{appendix:Bellman Equation for Optimal Cost-to-go}}

\begin{lemma}
   The Bellman equation for the original system can be written as
    \begin{align}
           J_t(I_t)= \inf_{a_t \in \mathcal{A}(I_t)} \left\{\E_{\xi_t} \left[J_{t+1}(f\left(I_{t}, a_t, \xi_t\right))\right] +  v\left(\sum_{k \in \mathcal{N}_+} I^k_t, I_t^1 + a_t^1, \ldots, I_t^K + a_t^K\right) \right\} \, .
    \end{align}
\end{lemma}

\begin{proof}{}
    
Recall that $J_t(I_t)$ represents the optimal cost-to-go for the original system from period $t \in [T]$ when starting at state $I_t$. The Bellman equation is
\begin{align*}
    J_t(I_t) &\overset{(i)}{=}  \inf_{a_t \in \mathcal{A}(I_t)} \left\{\E_{\xi_t} \left[J_{t+1}(f\left(I_{t}, a_t, \xi_t\right)) + \sum_{k \in [K]} \left( p^k (\xi^k_t - (I_t^k + a_t^k))^+ + h^k(I_t^k + a_t^k - \xi^k_t)^+\right) \right. \right. \\
    & \left. \left. \quad + h^w(I_t^w - \sum_{k \in [K]}a_t^k) \right] \right\} \\
    &\overset{(ii)}{=} \inf_{a_t \in \mathcal{A}(I_t)} \left\{\E_{\xi_t} \left[J_{t+1}(f\left(I_{t}, a_t, \xi_t\right)) + h^w I_t^w \right. \right. \\
    & \quad \left. \left. + \sum_{k \in [K]} \left( p^k (\xi^k_t - (I_t^k + a_t^k))^+ + h^k(I_t^k + a_t^k - \xi^k_t)^+ - h^w a^k_t\right) \right]\right\} \\
    &\overset{(iii)}{=} \inf_{a_t \in \mathcal{A}(I_t)} \left\{\E_{\xi_t} \left[J_{t+1}(f\left(I_{t}, a_t, \xi_t\right)) + h^w \sum_{k \in \mathcal{N}_+} I_t^k \right. \right. \\
    & \quad \left. \left.+ \sum_{k \in [K]} \left( p^k (\xi^k_t - (I_t^k + a_t^k))^+ + h^k(I_t^k + a_t^k - \xi^k_t)^+ - h^w (I^k_t + a^k_t)\right) \right]\right\} \\
    &\overset{(iv)}{=} \inf_{a_t \in \mathcal{A}(I_t)} \left\{\E_{\xi_t} \left[J_{t+1}(f\left(I_{t}, a_t, \xi_t\right))\right] +  v\left(\sum_{k \in \mathcal{N}_+} I^k_t, I_t^1 + a_t^1, \ldots, I_t^K + a_t^K\right) \right\}, 
\end{align*}
where $(ii)$ is obtained by reorganizing $h^w(I_t^w - \sum_{k \in [K]}a_t^k)$, $(iii)$ by adding and subtracting $h^w\sum_{k \in [K]}I^k_t$, and $(iv)$ from the definition of $v(\cdot)$ (see \eqref{eq: immediate cost}).

\end{proof}

\subsection{Cost-to-go of relaxed systems depends solely on $Z_t$ 
\label{appendix:cost-to-go on z} }

\begin{lemma}
    For any $I_t \in \mathbb{R}^{K+1}$, the cost-to-go for the partially relaxed system (and the fully relaxed system) depends on $I_t$ solely through the total inventory on hand $Z_t = \sum_{k \in \mathcal{N}_+} I_t^k$. 
\end{lemma}
\begin{proof}{}
We prove the result by backward induction in $t$. Let $\breve{J}_t(I_t)$ be the optimal cost-to-go for the partially relaxed setting when starting from period $t$ at state $I_t$. Clearly, $\breve{J}_T(I_T) = \hat{R}_T(Z_T)$ depends only on $Z_T$. Suppose $\breve{J}_{t+1}(I_{t+1})$ depends only on $Z_{t+1} = Z_t+a^w_t - \Xi_t$. As deduced in \eqref{eq:bellman partially relaxed}, $\breve{J}_t(I_t)$ satisfies the following Bellman Equation
\begin{align}
    \breve{J}_t(I_t) &= \hat{R}(Z_t) + \min_{a^w_t \in \mathbb{R}^+} \E_{\Xi_t}\left[ \breve{J}_{t+1}(Z_t + a^w_t - \Xi_t) \right]\, ,
\end{align}
for $\hat{R}(\cdot)$ defined in \eqref{eq:R-relaxed}. (We abuse notation in writing $\breve{J}_{t+1}(I_{t+1})= \breve{J}_{t+1}(Z_{t+1})$, since by induction hypothesis $\breve{J}_{t+1}(I_{t+1})$ depends only on $Z_{t+1}$.) Since the right-hand side depends only on $Z_t$ so does the left-hand side. Induction on $t= T-1, T-2, \dots, 1$ completes the proof. 
\end{proof}

\subsection{Convexity of $\hat{R}(Z)$, $\breve{J}_t$ and $\hat{J}_t$ 
\label{appendix:R-hat-convex-proof} }

\begin{proof}[Proof of Lemma~\ref{lemma:convexity}]
We first show the convexity of $\hat{R}(\cdot)$. Note that for every demand trace, $\sum_{k \in [K]} \left[ p^k(\xi^k_t - y^k)^+ + h^k(y^k - \xi_t^k)^+ - h^w y^k \right]$ is convex in $y^k$, as linear functions are convex, the positive part operator preserves convexity (as it is the maximum of linear functions), and the sum of convex functions is also convex. Further, as expectation preserves convexity, we have that $\sum_{k \in [K]} \E_{\xi^k_t}\left[\left( p^k(\xi^k_t - y^k)^+ + h^k(y^k - \xi_t^k)^+ - h^w y^k\right) \right]$ is also convex in $y^k$. 

Consider arbitrary $\underline{Z}\in \mathbb{R}^+$ and $\overline{Z}\in \mathbb{R}^+$ and let $\underline{y}$, $\overline{y}$ be the store inventory level vectors for which the minima of the problem 
\eqref{eq:R-relaxed} (which defines $\hat{R}(\cdot)$) with initial total inventory $\underline{Z}$ and $\overline{Z}$, respectively, are attained. For any $\lambda \in (0, 1)$, letting $y = \lambda \underline{y} + (1-\lambda)\overline{y}$, we have that
    \begin{align*}
        \sum_{k \in [K]} y^k &= \sum_{k \in [K]} (\lambda \underline{y}^k + (1-\lambda)\overline{y}^k) \\
        &= \lambda \sum_{k \in [K]} \underline{y}^k + (1-\lambda) \sum_{k \in [K]} \overline{y}^k \\
        & \leq \lambda \underline{Z} + (1 - \lambda)\overline{Z},
    \end{align*}
    where we used the feasibility of $\underline{y}$ (resp. $\overline{y}$)  for problem~\eqref{eq:R-relaxed} with $\underline{Z}$ (resp. $\overline{Z}$). Hence, $y$ is a feasible solution to the problem \eqref{eq:R-relaxed} with initial total inventory $Y=\lambda \underline{Z} + (1 - \lambda)\overline{Z}$. Further, note that
    \begin{align*}
        v(Y, y^1, \ldots, y^K) &= h^w (\lambda \underline{Z} + (1 - \lambda)\overline{Z}) + \sum_{k \in [K]} \E_{\xi^k_t}\left[\left( p^k(\xi^k_t - y^k)^+ + h^k(y^k - \xi_t^k)^+ - h^w y^k\right) \right] \\
        &\overset{(i)}{\leq} h^w (\lambda \underline{Z} + (1 - \lambda)\overline{Z}) + \lambda \sum_{k \in [K]} \E_{\xi^k_t}\left[\left( p^k(\xi^k_t - \underline{y}^k)^+ + h^k(\underline{y}^k - \xi_t^k)^+ - h^w \underline{y}^k\right) \right] \\
        & \quad + (1 - \lambda) \sum_{k \in [K]} \E_{\xi^k_t}\left[\left( p^k(\xi^k_t - \overline{y}^k)^+ + h^k(\overline{y}^k - \xi_t^k)^+ - h^w \overline{y}^k\right) \right] \\
        &= \lambda \hat{R}(\underline{Z}) + (1 - \lambda)\hat{R}(\overline{Z})\, ,
    \end{align*}
    where $(i)$ follows from the convexity of $\sum_{k \in [K]} \E_{\xi^k_t}\left[\left( p^k(\xi^k_t - y^k)^+ + h^k(y^k - \xi_t^k)^+ - h^w y^k\right) \right]$ in $y^k$. Thus, clearly, $\hat{R}(\lambda \underline{Z} + (1 - \lambda)\overline{Z}) = \hat{R}(Y) \leq v (Y, y^1, \dots, y^k)\leq \lambda \hat{R}(\underline{Z}) + (1 - \lambda)\hat{R}(\overline{Z})$. Since this holds for arbitrary arbitrary $\underline{Z}\in \mathbb{R}^+$ and $\overline{Z}\in \mathbb{R}^+$, we conclude that $\hat{R}(\cdot)$ is convex.

    The convexity of $\breve{J}_t(\cdot)$ and $\hat{J}_t(\cdot)$ follows by backward induction in $t$. We provide the argument for $\breve{J}_t(\cdot)$, and the argument for $\hat{J}_t(\cdot)$ is similar. Note that $\breve{J}_T(\cdot)=\hat{R}(\cdot)$ is convex as shown above. Suppose $\breve{J}_{t+1}(\cdot)$ is convex. Recall the Bellman equation \eqref{eq:bellman partially relaxed} for the partially relaxed system. We have shown above that the first term on the right $\hat{R}(\cdot)$ is convex. We will now argue that the second term on the right of \eqref{eq:bellman partially relaxed} is also convex.
    Let $y^w_t \equiv a^w_t +Z_t$ and $\breve{\ell}(y^w_t)\equiv \E_{\Xi_t}\left[ \breve{J}_{t+1}(y^w_t - \Xi_t) \right]$. Since $\breve{J}_{t+1}(\cdot)$ is convex, so is $\breve{\ell}(\cdot)$. The second term on the right of \eqref{eq:bellman partially relaxed} is nothing but
    \begin{align*}
        \min_{y^w_t \geq Z_t} \breve{\ell}(y^w_t) \, .
    \end{align*}
    Let $y^* \equiv \arg \min_{y^w_t \geq \mathbb{R}} \breve{\ell}(y^w_t)$ be the minimizer of $\breve{\ell}(\cdot)$. Then, 
    \begin{align*}
        \min_{y^w_t \geq Z_t} \breve{\ell}(y^w_t) = \left \{
        \begin{array}{ll}
            \breve{\ell}(y^*) &  \textup{for } Z_t \leq y^*  \\
             \breve{\ell}(Z_t) & \textup{for } Z_t > y^* \\
        \end{array} 
        \right . ,
    \end{align*} 
    and its convexity follows from the convexity of $\breve{\ell}(y^w_t)$ for $y^w_t \geq y^*$. Since the right-hand side of \eqref{eq:bellman partially relaxed} is the sum of convex functions, we deduce that $\breve{J}_t(\cdot)$ is convex. Induction on $t = T-1, T-2, \dots, 1$ completes the proof.
\end{proof}

\subsection{Proof of Proposition \ref{relaxation is lower bound}}
\label{proof: relaxation is lower bound}

\begin{proof}[Proof of Proposition~\ref{relaxation is lower bound}.]

We will first show the inequality $\hat{J}_1(Z_1) \leq \breve{J}_1(Z_1)$.
We will show by induction {for $t = T, T-1, \dots, 1$} that for any total inventory level $Z_t$ at time $t$, the optimal cost-to-go in the fully relaxed setting is weakly smaller than that of the partially relaxed setting $\hat{J}_t(Z_t) \leq \breve{J}_t(Z_t)$. The desired result $\hat{J}_1(Z_1) \leq \breve{J}_1(Z_1)$ will follow for the cost-to-go from the initial period.

We clearly have $\hat{J}_{T}(Z_T) = \breve{J}_{T}(Z_T) = \hat{R}(Z_T)$ for every $Z_T$. This forms our induction base. As our induction hypothesis, let us assume that $\hat{J}_{t+1}(Z_{t+1}) \leq \breve{J}_{t+1}(Z_{t+1})$ for all $Z_{t+1} \in \mathbb{R}_+$. We now show that $\hat{J}_{t}(Z_{t}) \leq \breve{J}_{t}(Z_{t})$ for all $Z_{t} \in \mathbb{R}_+$. Let $\hat{\pi}$ and $\breve{\pi}$ be optimal policies for the partially relaxed and the fully relaxed systems, respectively.
Recall that $\hat{U}_t = \sum_{k \in [K]} U_t^k$ is the sum of the $K$ uniforms at time $t$ and that $\breve{J}_{t+1}(Z)$ is convex in $Z$ by Lemma~\ref{lemma:convexity}. Therefore, for any $Z_{t}$, we have
\begin{align*}
    \breve{J}_{t}(Z_t) &= \hat{R}(Z_t) + \E_{B_t}\left[ \E_{\hat{U}_t}\left[\breve{J}_{t+1}(Z_t + a^{\breve{\pi}w}_t - B_t \hat{U}_t )  \right] \right] \\
    & \overset{(i)}{\geq} \hat{R}(Z_t) + \E_{B_t}\left[\breve{J}_{t+1}(Z_t + a^{\breve{\pi}w}_t - B_t \hat{\mu} )  \right] \\
    & \overset{(ii)}{\geq} \hat{R}(Z_t) + \E_{B_t}\left[\hat{J}_{t+1}(Z_t + a^{\breve{\pi}w}_t - B_t \hat{\mu} )  \right] \\
    & \overset{(iii)}{\geq} \hat{R}(Z_t) + \E_{B_t}\left[\hat{J}_{t+1}(Z_t + a^{\hat{\pi}w}_t - B_t \hat{\mu} )  \right] \\
    &\overset{(iv)}{=} \hat{J}_{t}({Z}_t),
\end{align*}
where $(i)$ follows by applying Jensen's inequality to $\breve{J}(\cdot)$, $(ii)$ by hypothesis, and $(iii)$ by the fact that $\hat{\pi}$ minimizes $\hat{J}_{t+1}(Z_t + a^{\pi w} - B_t \hat{\mu} )$,
and $(iv)$ is just the Bellman equation for the fully relaxed system. Therefore, we can conclude that $\hat{J}_{t} \leq \breve{J}_{t}$ for every $t \in [T]$. In particular, it holds for $t=1$ and $\hat{Z}_t = Z_t$.

We clearly have the second inequality $\breve{J}_{1}(Z_1) \leq J_1(Z_1)$ given that the action space in the partially relaxed setting contains that of the original setting.
\end{proof}

\subsection{Proof of Lemma~\ref{lemma:relaxed policy optimal structure} 
\label{appendix:relaxed policy optimal structure}}

\begin{proof}{}

Equation \eqref{eq:bellman fully relaxed} takes the form of an inventory problem for a single location under convex costs. Following \cite{federgruen1984approximations}, the optimal policy is given by a base-stock policy (we refer to it as an echelon-stock policy given that it considers the system-wide level of inventory).

To show that the policy is stationary, let $\hat{S}_t$ for every $t \in [T]$ be the optimal echelon-stock level at period $t$. 
Note that for all $t \in [T]$, 
\begin{align}
    \hat{S} = \arg \min_{S \in \mathbb{R}^+} \E_{B_t}\left[\hat{R}(S - B_t \hat{\mu})\right]
    \label{eq:Shat-identity}
\end{align}
by definition of $\hat{S}$ and the fact that $B_t$ is i.i.d. 
Consequently, $\inf_{a^w_t \geq 0}  \left\{\E_{B_t}\left[\hat{R}(Z_{t} + a_t^w - B_{t} \hat{\mu})\right] \right\} = \left\{\E_{B_t}\left[\hat{R}(\max (\hat{S}, Z_t) - B_{t} \hat{\mu})\right] \right\}$, where we further used convexity of $\hat{R}(\cdot)$ (Lemma~\ref{lemma:convexity}).

We will show that $\hat{S}_{t} = \hat{S}$ for all $t \in [T-1]$ by backward induction in $t= T-1, T-2, \dots, 1$. Plugging in the echelon stock policy, the Bellman Equation at period $T-1$ takes the form
\begin{align*}
    \hat{J}_{T-1}(Z_{T-1}) &= \hat{R}(Z_{T-1}) + \E_{B_{T-1}}\left[\hat{R}(\max(\hat{S}_{T-1}, Z_{T-1}) - B_{T-1} \hat{\mu})\right] ,
\end{align*}
and hence $\hat{S}_{T-1} = \hat{S}$ using \eqref{eq:Shat-identity}.

Suppose $\hat{S}_{t} = \hat{S}$. We will show that $\hat{S}_{t-1} = \hat{S}$.
For period $t-1$, the Bellman Equation takes the form
\begin{align*}
    \hat{J}_{t-1}(Z_{t-1}) &= \hat{R}(Z_{t-1}) + \E_{B_{t-1}}\left[\hat{J}_{t}(\hat{S}_{t-1} - B_{t-1} \hat{\mu} )  \right]  \\
    &= \hat{R}(Z_{t-1}) + \E_{B_{t-1}}\left[\hat{R}(\hat{S}_{t-1} - B_{t-1} \hat{\mu})\right] \\
    & \quad + \underbrace{\E_{B_{t-1}}\left[ \inf_{a_t^w \geq \hat{S}_{t-1} - B_{t-1}} \left\{\E_{B_{t}}\left[\hat{J}_{t+1}\left((a_t^w + \hat{S}_{t-1} - B_{t-1}\hat{\mu}) - B_{t} \hat{\mu} \right)  \right]\right\} \right]}_{A(\hat{S}_{t-1})} \, .
\end{align*}

Note that $\hat{S}_{t-1}$ affects the second and third terms on the right. We will show that $\hat{S}_{t-1} = \hat{S}$ causes each of the second and third terms to be individually minimized. For the second term, this follows from 
\eqref{eq:Shat-identity}. Consider the third term.
Setting $\hat{S}_{t-1} = \hat{S}$ implies that $\hat{S}_{t-1} - B_{t-1}\hat{\mu} \leq \hat{S}$ for both possible values of $B_{t-1}$. Consequently, 
\begin{align*}
    A(\hat{S})   = \E_{B_{t}}\left[\hat{J}_{t+1}(\hat{S} - B_{t} \hat{\mu} )  \right]  = \min_{S \in \mathbb{R}} \E_{B_{t}}\left[\hat{J}_{t+1}({S} - B_{t} \hat{\mu} )  \right]
\end{align*}
by our induction hypothesis that $\hat{S}_t = \hat{S}$, \ie setting $\hat{S}_{t-1} = \hat{S}$ also causes the third term to be minimized. We deduce that $\hat{S}_{t-1} = \hat{S}$. Induction completes the proof that the optimal echelon stock level is time invariant and equal to $\hat{S}$ for $t= T-1, T-2, \dots, 1$. 
\end{proof}

\subsection{Proof of Lemma \ref{lemma:relaxed base-stock fills each store}}
\label{appendix:relaxed base-stock fills each store}

\begin{proof}{}
    
    In Appendix \ref{appendix:relaxed policy optimal structure} we showed that the optimal base-stock level $\hat{S}$ for the fully relaxed system must minimize $q \hat{R}(\hat{S} - \Dh)  +  (1-q) \hat{R}(\hat{S} - \Dl)$. Recall that $\underline{y}$ and $\overline{y}$ denote the optimal solutions for problem \eqref{eq:R-relaxed}  with total inventory $\hat{S} - \Dh_t$ and $\hat{S} - \Dl_t$, respectively.

    We will first prove that $\hat{S}$ is such that $\underline{y}^k \geq \gammal a^k$ for every $k$. Suppose the opposite, \ie $\hat{S}$ is such that $\underline{y}^k < \gammal a^k$ for some $k$. Let $\delta^k > 0$ be such that $\underline{y}^k + \delta^k < \gammal a^k$, and consider a base-stock level $\hat{S} + \delta^k$. In the problem \eqref{eq:R-relaxed}  with total inventory $(\hat{S} + \delta^k - \Dh_t)$, we can construct a feasible solution $\underline{y}(\delta^k)$ such that $\underline{y}(\delta^k)^k = \underline{y}^k + \delta^k$ and $\underline{y}(\delta^k)^j = \underline{y}^j$ for all $j \neq k$, reducing cost $D$ by $p^k \delta^k$ as the extra $\delta_k$ will be sold immediately w.p. 1 given that the realized demand will be at least $\gammal a^k$. We can keep $\overline{y}$ unperturbed, increasing cost by $h^w \delta^k$, as there are an additional $\delta^k$ units being held at the warehouse. Therefore, the expected cost is reduced by at least $q p^k \delta^k - (1-q)h^w \delta^k > 0$, so $\hat{S}$ could not have been optimal. We thus conclude that  $\underline{y}^k \geq \gammal a^k$ for every $k \in [K]$.

    Now, note that the KKT conditions of $\hat{R}(Z)$ can be written as:
    \begin{equation}
    \label{eq: KKT-condtions}
    \begin{aligned}        
        -p^k \mathbb{P}(\xi^k_t \geq y^k) + h^k \mathbb{P}(\xi^k_t \leq y^k) -h^w + \lambda &= 0  \quad \forall y \in [K] \\
        \lambda \cdot (\sum_{k \in [K]} y^k - Z) &= 0, \\
        \sum_{k \in [K]} y^k - Z &\leq 0, \\
        \lambda &\geq 0.
    \end{aligned}
    \end{equation}

    It can be easily verified that the optimal $y^k$ for a problem $\hat{R}(Z)$ are weakly increasing in $Z$ by the optimality conditions \eqref{eq: KKT-condtions} and the convexity of costs. Therefore, $\overline{y} \geq \underline{y}$.
    Clearly, as $h^w < h^k$, a solution such that $\overline{y}^k > \gammah  b^k$ cannot be optimal, as the units above $\gammah  b^k$ will not be sold immediately w.p. 1. We can hence conclude that $\gammah b^k \geq \overline{y}^k \geq \underline{y}^k \geq \gammal a^k$ for each $k \in [K]$.
    
    The proof for the fact that an optimal policy for the fully relaxed system does not move inventory from a store to other stores or to the warehouse was provided in the paragraph immediately after the lemma statement.
    
    Additionally, the first condition in \eqref{eq: KKT-condtions} implies that $y^k = F_{\xi^k_t}^{-1} \left(\frac{p^k + h^w - \lambda}{p^k + h^k}\right)$ for every $k \in [K]$. This allows us to write $\underline{y}^k  = \underline{G}(\mathcal{R}^k) = F_{\xi^k_t}^{-1} \left(\frac{p^k + h^w - \underline{\lambda}}{p^k + h^k}\right)$ and $\overline{y}^k  = \overline{G}(\mathcal{R}^k) = F_{\xi^k_t}^{-1} \left(\frac{p^k + h^w - \overline{\lambda}}{p^k + h^k}\right)$, with $\underline{\lambda}$ and $\overline{\lambda}$ the optimal values of dual variable $\lambda$ for problem \eqref{eq:R-relaxed} with initial total inventories $\hat{S} - \Dh_t$ and $\hat{S} - \Dl_t$, respectively. Note that $\underline{G}$ and $\overline{G}$ are continuous functions of problem primitives $\mathcal{R}^k$, since the assumptions outlined in Example \ref{example: 1d context} imply that $p^k + h^k \geq \underline{p} + \underline{h} > 0$ for every $k \in [K]$.
\end{proof}

\subsection{Proof of Proposition \ref{proposition: upper bound on feasible policy}}
\label{proof: upper bound on feasible policy}
To prove this proposition, we define some specialized notation.  
Let $c^{\tilde{\pi}}(I_t) = v \left(Z_t, a^{\tilde{\pi}1}_t(I_t), \ldots, a^{\tilde{\pi}K}_t(I_t) \right)$ be the immediate cost under $\tilde{\pi}$. Recall that $Z_{t+1} = \hat{S} - \Xi_t$ under $\tilde{\pi}$ and that $\hat{Z}_{t+1} = \hat{S} - D_t$ following the optimal policy for the fully relaxed setting. We will use the following two lemmas, which we will prove at the end of this section.

Let $\alpha = |\hat{R}(\hat{S} - \Dl) - \hat{R}(\hat{S} - \Dh)| + \max\{\overline{p}, \overline{h}\} (|\Dh - \Dl|)$.

The next lemma bounds the immediate cost under $\tilde{\pi}$ by $\hat{R}(Z_t)$, which plays the role of the immediate cost in the partially relaxed system (See \eqref{eq:bellman partially relaxed}) plus two extra terms. The term $\max\{\overline{p}, \overline{h}\} (|\hat{Z}_t- Z_t|)$ penalizes deviations in overall system inventory from the level $\hat{Z}_t$ observed in the fully relaxed system and the term  $\mathbbm{1}(\tilde{W}_t(Z_t) \neq W_t) \alpha$ penalizes cases when the the ``estimate'' $\tilde{W}_t$ of whether demand was high or low in the previous period was incorrect. 
\begin{lemma}
\label{lemma: upper bound with indicators}
In any period $t$, with probability 1, 
    \begin{align}
    \label{upper bound with indicators}
    c^{\tilde{\pi}}(I_t) & \leq \hat{R}(\hat{Z_t}) + \max\{\overline{p}, \overline{h}\} (|\hat{Z}_t- Z_t|) + \mathbbm{1}(\tilde{W}_t \neq W_t) \alpha.
    \end{align}
\end{lemma}

The next lemma bounds key terms in \eqref{upper bound with indicators} in terms of the deviations of the sum of zero-mean uniforms random variables from the sum's mean. Recall that $\tilde{U}_{1}^K = \sum_{k \in [K]} (U_{1}^k - \mu^k)$ is the sum of $K$ zero-mean uniforms. 
\begin{lemma} 
\label{lemma: upper bound on E and P of uniform deviation}
For $t \in \{2, \ldots, T\}$,  
    \begin{align}
    \label{upper bound E uniform deviation}
        \E \left[\left|\hat{Z}_t - Z_t \right| \right] \leq \gammah  \E\left[ \left|\tilde{U}_{1}^K \right| \right].
    \end{align}
and
    \begin{align}
    \label{upper bound P uniform deviation}
    \mathbb{P}\left(\tilde{W}_t \neq W_t \right) \leq \mathbb{P}\left( \left| \tilde{U}_{1}^K \right| \geq \hat{\mu} \frac{\gammah  - \gammal }{2 \gammah }\right).
    \end{align}
\end{lemma}

We now prove Proposition \ref{proposition: upper bound on feasible policy}. 
\begin{proof}[Proof of Proposition \ref{proposition: upper bound on feasible policy}]
By taking expectation on \eqref{upper bound with indicators}, using \eqref{upper bound E uniform deviation} and \eqref{upper bound P uniform deviation}, we get that for $t=2,\ldots, T$
\begin{align*}
    \begin{split}
        \E[c^{\tilde{\pi}}(I_t)] &\leq \E\left[ \left(\hat{R}(\hat{Z}_t) + \max\{\overline{p}, \overline{h}\} (|\hat{Z}_t - Z_t|) \right) + \alpha \mathbbm{1}(\tilde{W}_t \neq W_t) \right] \\
        &\leq \E[\hat{R}(\hat{Z}_t)] + \max\{\overline{p}, \overline{h}\} \gammah  \E\left[ \left|\tilde{U}_{1}^K \right| \right] + \alpha \mathbb{P}\left(\left|\tilde{U}^K_{1} \right| \geq \hat{\mu} \frac{\gammah  - \gammal }{2 \gammah }\right).
    \end{split}
\end{align*}

Using this inequality, we conclude that
\begin{align*}
    J_1^{\tilde{\pi}}(I_t) & = c^{\tilde{\pi}}(I_1) + \E \left[ \sum_{t=2}^{T} c^{\tilde{\pi}}(I_t)\right]   \\
    &\leq c^{\tilde{\pi}}(I_1) +  \E\left[ \sum_{t=2}^{T} \hat{R}(\hat{Z}_t) \right]    + (T-1)\left[ \max\{\overline{p}, \overline{h}\} \gammah  \E_{\tilde{U}^K_{1}} \left [ \left|\tilde{U}^K_{1} \right| \right] + \alpha \mathbb{P}\left(\left|\tilde{U}^K_{1} \right| \geq \hat{\mu} \frac{\gammah  - \gammal }{2 \gammah }\right) \right] \\
    &= \hat{R}(Z_1) +  \E\left[ \sum_{t=2}^{T} \hat{R}(\hat{Z}_t) \right]    + (T-1)\left[ \max\{\overline{p}, \overline{h}\} \gammah  \E_{\tilde{U}^K_{1}} \left [ \left|\tilde{U}^K_{1} \right| \right] + \alpha \mathbb{P}\left(\left|\tilde{U}^K_{1} \right| \geq \hat{\mu} \frac{\gammah  - \gammal }{2 \gammah }\right) \right] \\
    &= \hat{J}_1(Z_1) + (T-1) \left[ \max\{\overline{p}, \overline{h}\} \gammah  \E_{\tilde{U}^K_{1}} \left [ \left|\tilde{U}^K_{1} \right| \right] + \alpha \mathbb{P}\left(\left|\tilde{U}^K_{1} \right| \geq \hat{\mu} \frac{\gammah  - \gammal }{2 \gammah }\right) \right]\\
    &\leq \hat{J}_1(Z_1) + T \left[ \max\{\overline{p}, \overline{h}\} \gammah  \E_{\tilde{U}^K_{1}} \left [ \left|\tilde{U}^K_{1} \right| \right] + \alpha \mathbb{P}\left(\left|\tilde{U}^K_{1} \right| \geq \hat{\mu} \frac{\gammah  - \gammal }{2 \gammah }\right) \right].
\end{align*}
The first and third equalities apply the definitions of the cost-to-go functions under the two systems. The second equality uses that $c^{\tilde{\pi}}(I_1) = v \left(Z_1, a^{\tilde{\pi}1}_1(I_1), \ldots, a^{\tilde{\pi}K}_1(I_1) \right) = \hat{R}(Z_1) = \hat{R}(\hat{Z}_1)$ because of how the policy $\tilde{\pi}$ is constructed and because the initial inventory levels ($Z_1$ and $\hat{Z}_1$) in the two systems are equal by construction.  
\end{proof}

We now return to prove the lemmas stated above. 

Lemma~\ref{lemma: upper bound with indicators} follows from two facts: (i) Whenever the policy $\tilde{\pi}$ correctly ``estimated'' whether the demand in the previous period was low (\ie $\tilde{W}_t = W_t$) 
we can bound the difference in the costs incurred in the fully relaxed and original systems by a quantity proportional to the difference in system-wide inventories among systems. Here we make use of Lemma~\ref{lemma:relaxed base-stock fills each store} which assures us that we can do as well in the original system as in a fully relaxed system with the same system-wide inventory. (ii) Whenever the estimation is wrong (\ie $\tilde{W}_t \neq W_t$), we incur an additional cost of at most $|\hat{R}(\hat{S} - \Dh) - \hat{R}(\hat{S} - \Dl)|$.

\begin{proof}[Proof of Lemma~\ref{lemma: upper bound with indicators}.]

Let $r^k(y^k) = \E_{\xi^k_t}\left[\left( p^k(\xi^k_t - y^k)^+ + h^k(y^k - \xi_t^k)^+ - h^w y^k\right) \right]$, so that $v(Z, y^1, \ldots, y^K) = h^w Z + \sum_{k \in [K]} r^k(y^k)$. It is clear that
\begin{align}
    \label{immediate cost dif upper bound}
    r^k (y^k - x) - r^k (y^k) \leq \overline{p}x
\end{align}
for all $k \in [K]$ as, in the worst case, $x$ fewer units can cause an additional underage cost of $\overline{p}x$.

Now, recall that, given the definition of $\tilde{W}_t$, we will have that $\tilde{\pi}$ will follow the base-stock levels $\overline{y}$ whenever $\hat{S} - \frac{\Dl + \Dh}{2} \leq Z_t$. Let us first address two separate cases in which $\tilde{\pi}$ follows $\overline{y}$. These differ on whether $Z_t \leq \hat{S} - \Dl$ or $Z_t > \hat{S} - \Dl$. 

Let us first consider the case $Z_t > \hat{S} - \Dl$. We have
\begin{align*}
    R^{\tilde{\pi}}(I_t) &= h^w Z_t + \sum_{k \in [K]} r^k(a^{\tilde{\pi}k}_t(I_t)) \\
    & \overset{(i)}{=} h^w Z_t + \sum_{k \in [K]} r^k(\overline{y}^k) \\
    & \overset{(ii)}{=} h^w (Z_t - (\hat{S} - \Dl)) + h^w(\hat{S} - \Dl) + \sum_{k \in [K]} r^k(\overline{y}^k) \\
    & \overset{(iii)}{=} h^w (Z_t - (\hat{S} - \Dl)) + \hat{R}(\hat{S} - \Dl), 
\end{align*}
where $(i)$ follows by the fact that, as implied by Lemma \ref{lemma:relaxed base-stock fills each store}, $a^{\tilde{\pi}k}_t(I_t) = \overline{y}^k$ whenever $Z_t > \hat{S} - \Dl$, $(ii)$ by adding and subtracting $h^w(\hat{S} - \Dl)$ and $(iii)$ by the definition of $\hat{R}(\hat{S} - \Dl)$

Similarly, if  $\hat{S} - \frac{\Dl + \Dh}{2} \leq Z_t \leq \hat{S} - \Dl$, we have
\begin{align*}
    R^{\tilde{\pi}}(I_t) &= h^w Z_t + \sum_{k \in [K]} r^k(a^{\tilde{\pi}k}_t(I_t)) \\
    &\overset{(i)}{\leq} h^w (\hat{S} - \Dl) + \sum_{k \in [K]} r^k(a^{\tilde{\pi}k}_t(I_t)) \\
    & \overset{(ii)}{=} h^w (\hat{S} - \Dl) + \sum_{k \in [K]} r^k(\overline{y}^k) + \sum_{k \in [K]} (r^k(a^{\tilde{\pi}k}_t(I_t)) - r^k(\overline{y}^k)) \\
    & \overset{(iii)}{\leq} h^w (\hat{S} - \Dl) + \sum_{k \in [K]} r^k(\overline{y}^k) +  \overline{p} (\hat{S} - \Dl - Z_t) \\
    & \overset{(iv)}{=}\hat{R}(\hat{S} - \Dl) +   \overline{p} (\hat{S} - \Dl - Z_t),
\end{align*}
where $(i)$ follows from $Z_t \leq \hat{S} - \Dl$, $(ii)$ by adding and subtracting $\sum_{k \in [K]} r^k(\overline{y}^k)$, $(iii)$ from \eqref{immediate cost dif upper bound} and the fact that, using Lemma \ref{lemma:relaxed base-stock fills each store}, $\sum_{k \in [K]} (\overline{y}^k - a^{\tilde{\pi}k}_t(I_t)) = (\hat{S} - \Dl - Z_t)$ whenever $Z_t \leq \hat{S} - \Dl$, and $(iv)$ from the definition of $\hat{R}(\hat{S} - \Dl)$.

Let $\overline{c} = \max\{\overline{p}, \overline{h}\}$. Given that $\overline{h} > h^w$, we obtain that whenever $\hat{S} - \frac{\Dl + \Dh}{2} \leq Z_t$,
\begin{align}
\label{eq:upper bound with indicators}
    \begin{split}
    R^{\tilde{\pi}}(I_t) &\leq \hat{R}(\hat{S} - \Dl) +   \overline{c} (|\hat{S} - \Dl - Z_t|) \\
    &= \mathbbm{1}(\tilde{W}_t = W_t)[\hat{R}(\hat{S} - \Dl) +   \overline{c} (|\hat{S} - \Dl - Z_t|)]  \\
    &\quad + \mathbbm{1}(\tilde{W}_t \neq W_t)[\hat{R}(\hat{S} - \Dl) +  \overline{c} (|\hat{S} - \Dl - Z_t|)]
    \end{split}
\end{align}
Note that under $\{\tilde{W}_t = W_t\}$, we have that $\hat{S} - \Dl = \hat{Z_t}$. Conversely, if $\{\tilde{W}_t \neq W_t\}$, we have that $\hat{S} - \Dh = \hat{Z_t}$. Under $\{\tilde{W}_t \neq W_t\}$, we have
\begin{align*}
    \hat{R}(\hat{S} - \Dl) +  \overline{c} (|\hat{S} - \Dl - Z_t|) &= 
    \hat{R}(\hat{S} - \Dl) - \hat{R}(\hat{S} - \Dh) +  \overline{c} (|\hat{S} - \Dl - Z_t| - |\hat{S} - \Dh - Z_t|) \\
    &\quad + [\hat{R}(\hat{Z}_t) +  \overline{c} (|\hat{Z}_t - Z_t|)] \\
    &\overset{(i)}{\leq} 
    |\hat{R}(\hat{S} - \Dl) - \hat{R}(\hat{S} - \Dh)| +  \overline{c} (|\Dh - \Dl|) \\
    &\quad + [\hat{R}(\hat{Z}_t) +  \overline{c} (|\hat{Z}_t - Z_t|)] \\
    & \leq \alpha + [\hat{R}(\hat{Z}_t) + \overline{c}(|\hat{Z}_t - Z_t|)],
\end{align*}
with $\alpha = |\hat{R}(\hat{S} - \Dl) - \hat{R}(\hat{S} - \Dh)| +  \overline{c} (|\Dh - \Dl|)$, where $(i)$ follows by the triangle inequality.

Applying our previous identities on $\eqref{eq:upper bound with indicators}$ and rearranging terms, we obtain
\begin{align}
\begin{split}
    R^{\tilde{\pi}}(I_t) 
    &\leq
    \left(\hat{R}(\hat{Z}_t) + \max\{\overline{p}, \overline{h}\} (|\hat{Z}_t - Z_t|) \right) + \mathbbm{1}(\tilde{W}_t \neq W_t)\alpha
\end{split}
\end{align}

The procedure for the cases in which $\tilde{\pi}$ follows $\underline{y}$ are analogous, so we will omit them. We therefore conclude the desired result.

\end{proof}

\begin{proof}[Proof of Lemma 
\ref{lemma: upper bound on E and P of uniform deviation}.]

Let $\tilde{U}_{t-1}^k = \sum_{k \in [K]} (U_{t-1}^k - \mu^k)$ be the sum of the zero-mean uniforms. Then, following $\tilde{\pi}$, we have $Z_t = \hat{S} - \Xi_{t-1}$.
Therefore,
\begin{align*}
    \begin{split}
    \left|\hat{Z}_t - Z_t \right|
    &\overset{(i)}{=} \left|\hat{S} - D_t - (\hat{S} - \Xi_{t-1})\right| \\
    &\overset{(ii)}{=} \left|\hat{S} - B_{t-1}\hat{\mu} - (\hat{S} - B_{t-1}\sum_{k \in [K]} U_{t-1}^k) \right| \\
    &\overset{(iii)}{=} \left|B_{t-1}\sum_{k \in [K]} (U_{t-1}^k - \mu^k) \right| \\
    &\overset{(iv)}{\leq} \gammah  \left|\tilde{U}_{t-1}^K \right|
    \end{split}
\end{align*}
where $(i)$ and $(ii)$ follow from the definition of $Z_t$ and $\hat{Z}_t$ under their respective policies, $(iii)$ by $\hat{\mu} = \sum_{k\in [K]} \mu^k$ and $(iv)$ by $B_t \leq \gammah $. Similarly, noting that the distance between $D_t$ and the "crossing point" $\frac{\Dh - \Dl}{2}$ is given by
\begin{align*}
    \left| D_t - \frac{\Dh - \Dl}{2}\right| = \hat{\mu} \frac{\gammah  - \gammal }{2}
\end{align*}
and recalling that the cumulative demand in the original system can be written as $\Xi_{t-1} = B_{t-1} \hat{\mu} +B_{t-1} \tilde{U}_{t-1}^K$, we have that 
\begin{align*}
    \left(\tilde{W}_t \neq W_t \right)  \quad \text{ implies} \quad \left( \left| B_{t-1} \tilde{U}_{t-1}^K \right| \geq \hat{\mu} \frac{\gammah  - \gammal }{2}\right).
\end{align*}

As $B_{t - 1} \geq \gammah $ and $\tilde{U}_{t-1}^K$ are i.i.d, we deduce that
\begin{align*}
    \mathbb{P}\left(\tilde{W}_t \neq W_t \right) \leq \mathbb{P}\left( \left| \tilde{U}_{1}^K \right| \geq \hat{\mu} \frac{\gammah  - \gammal }{2 \gammah }\right)\, .
\end{align*}

\end{proof}

\subsection{Bounding $\E \left [ \left|\tilde{U}^K_{1} \right| \right]$ and $\mathbb{P}\left(\left|\tilde{U}^K_{1} \right| \geq \hat{\mu} \frac{\gammah  - \gammal }{2 \gammah }\right)$
\label{appendix:bound-for-deviations}}

Recall the definition of the primitive $\kappa \equiv \max_{k \in [K]} \max(\uh^k, 1/\uh^k)$.
\begin{lemma}
We have
\begin{align}
    \E \left [ \left|\tilde{U}^K_{1} \right| \right] \leq \kappa \sqrt{K}
    \, .
\end{align}
\label{lemma:sum_of_uniforms}
\end{lemma}
\begin{proof}{}
Let $\tilde{U}^k_1 = U^k_1 - \mu^k$ be the zero-mean uniform corresponding to store $k$. Clearly, $\tilde{U}^k_1 \sim \textup{Uniform}(\ul^k - \mu^k, \uh^k - \mu^k)$, and has variance $\textup{Var}(\tilde{U}^k_1) = (\uh^k - \ul^k)^{2}/12$. Thus, $\textup{Var}(\tilde{U}^K_{1}) = \sum_{k \in K}(\uh^k - \ul^k)^{2}/12$. 
Let $\Delta \equiv \max_{k \in [K]}{(\uh^k - \ul^k)} \leq \kappa$. 

Now, notice that
\begin{align*}
    \E \left [ \left|\tilde{U}^K_{1} \right| \right] 
    &\leq \left( \E \left [ \left|\tilde{U}^K_{1} \right|^2 \right]\right)^{1/2} = (\textup{Var}(\tilde{U}^K_1))^{1/2} \\
    &\leq \sqrt{K} \Delta/\sqrt{12} \leq \kappa \sqrt{K/12}
    \, .
\end{align*}
We drop the $\sqrt{12}$ to reduce the notational burden, leading to the lemma.
\end{proof}

\begin{lemma}
We have
\begin{align}
    \mathbb{P}\left(\left|\tilde{U}^K_{1} \right| \geq \hat{\mu} \frac{\gammah  - \gammal }{2 \gammah }\right) \leq \frac{4 \gammah \kappa^2 }{ (\gammah  - \gammal )\sqrt{K}}\, .
\end{align}
\label{lemma:What-error-probability}
\end{lemma}
\begin{proof}
Using the definition of $\hat{\mu}$ as the total expected sum of uniforms and $\uh^k \geq 1/\kappa \, \forall k \in [K]$, we have the lower bound 
\begin{align}
    \hat{\mu} \frac{\gammah  - \gammal }{2 \gammah } \geq d_1 K \, , \qquad \textup{for }  d_1 = \frac{\gammah  - \gammal }{4 \gammah \kappa} \, .
    \label{eq:d_1}
\end{align} 
By Markov's inequality, we have 
\begin{align}
\mathbb{P}\left(\left|\tilde{U}^K_{1} \right| \geq \hat{\mu} \frac{\gammah  - \gammal }{2 \gammah }\right) \leq \E [ |\tilde{U}^K_{1} | ] \frac{2 \gammah }{\hat{\mu} (\gammah  - \gammal )} \leq \frac{4 \gammah \kappa^2 }{ (\gammah  - \gammal )\sqrt{K}} \, ,
\end{align}
using Lemma~\ref{lemma:sum_of_uniforms} and \eqref{eq:d_1} to get the second inequality.
\end{proof}

\subsection{Proof of Theorem \ref{theorem:message-passing-base-stock}
\label{appendix:concluding-theorem-appendix}}

We will use the following lemmas, which we will prove at the end of this subsection.

\begin{lemma}
\label{lemma:lower-bound-hat-J}
    There exists a constant $C_3 = C_3(q, \underline{p},  \underline{h}, \kappa)>0$ such that $\hat{J}_1(I_1) \geq C_3 T \sqrt{K} $ for every $I_1$ and $K$.
\end{lemma}

\begin{lemma}
There exists a constant $C_4 = C_4(\overline{p}, \overline{h}, \gammah, \gammal, \kappa)< \infty$ such that $\alpha \leq C_4 K $ for every $K$.
\label{lemma:bound_on_alpha}
\end{lemma}

\begin{proof}[Proof of Theorem~\ref{theorem:message-passing-base-stock}.]

Using the bounds in Lemmas~\ref{lemma:sum_of_uniforms} and \ref{lemma:What-error-probability} in Proposition~\ref{proposition: upper bound on feasible policy} (achievability) we obtain
\begin{align}
    J_1^{\tilde{\pi}}(I_1) &\leq \hat{J}_1(Z_1) + T \sqrt{K} \left[\max\{\overline{p}, \overline{h}\} \gammah \kappa + \alpha  \frac{4 \gammah \kappa^2 }{ (\gammah  - \gammal )} \right] \nonumber\\
    &\leq J_1(I_1) + C_5 T \sqrt{K}  \, \qquad \textup{for } C_5 = C_5(\gammah, \gammal, \overline{p}, \overline{h}, \kappa) \, , 
    \label{eq:J1 + T sqrt(K)}
\end{align}
where the second inequality uses $\hat{J}_1(Z_1) \leq J_1(I_1)$ from Proposition~\ref{relaxation is lower bound} (converse) and Lemma~\ref{lemma:bound_on_alpha}.

By Lemma~\ref{lemma:lower-bound-hat-J}, there exists $C_3=C_3(q, \underline{p},  \underline{h}, \kappa) > 0$ such that $J_1(I_1) \geq C_3 K T$. Dividing \eqref{eq:J1 + T sqrt(K)} by $J_1(I_1)$, we obtain
\begin{align*}
    \frac{J_1^{\tilde{\pi}}(I_1)}{J_1(I_1)} &\leq  1 + \frac{C_5 T \sqrt{K}}{J_1(I_1)} \\
    &\leq  1 + \frac{ C_5 }{C_3 \sqrt{K}} \, .
\end{align*}

Defining $C \equiv C_5/C_3 < \infty$ yields the theorem. 

\end{proof}

We now prove Lemmas~\ref{lemma:lower-bound-hat-J} and \ref{lemma:bound_on_alpha}.
\begin{proof}[Proof of Lemma~\ref{lemma:lower-bound-hat-J}.]
Let $w^k(y^k) = \E_{\xi^k_t}\left[\left( p^k(\xi^k_t - y^k)^+ + h^k(y^k - \xi_t^k)^+\right) \right]$ be the expected cost incurred at period $t$ at store $k$ given that the inventory level before demand is realized is $y^k$. Define $\overline{\nu}^k = \gammah (\ul^k+\uh^k)/2$ and $ \underline{\nu}^k = \gammal (\ul^k+\uh^k)/2$, the "mid-points" of the support of the demand distribution conditioned on the general demand being \emph{high} and \emph{low}, respectively. Let $\beta = (\gammah - \gammal)/(4\kappa) > 0$. For any choice of $y^k$, we have that the larger distance $\varepsilon(y^k)$ between $y^k$ and $\underline{\nu}^k, \overline{\nu}^k$, defined as
\begin{align*}
    \varepsilon(y^k) = \max\{\ |\underline{\nu}^k - y^k|, |\overline{\nu}^k - y^k|\},
\end{align*}
satisfies $\varepsilon(y^k) \geq (\overline{\nu}-\underline{\nu})/2= (\gammah - \gammal)(\ul^k+\uh^k)/4 \geq (\gammah - \gammal)(\uh^k)/4 \geq (\gammah - \gammal)/(4\kappa) = \beta$. Without loss of generality, let us assume $|\overline{\nu}^k - y^k| \geq \beta$ and that $y^k < \overline{\nu}^k$ (all other cases are analogous). Now, $\mathbb{P}(\xi^k_t \geq \overline{\nu}^k) \geq q/2 \geq q(1-q)/2$. Further, whenever  $\xi^k_t \geq \overline{\nu}^k$, a cost of at least $\min\{\underline{h}, \underline{p} \} \beta$ is incurred. Therefore, we must have that the expected cost per store for each period satisfies $w^k(y^k) \geq \frac{\min\{\underline{h}, \underline{p} \} q (1-q) \beta}{2}$ for every possible $y^k$. Defining $C_3 \equiv  \frac{\min\{\underline{h}, \underline{p} \} q (1-q)\beta}{2} = \frac{\min\{\underline{h}, \underline{p} \} q (\gammah - \gammal)}{8 \kappa}$, we immediately infer that $\hat{J}_1(I_1) \geq C_3 TK$ for all $I_1$.
\end{proof}

\begin{proof}[Proof of Lemma~\ref{lemma:bound_on_alpha}.]
Recall that $\alpha = |\hat{R}(\hat{S} - \Dl) - \hat{R}(\hat{S} - \Dh)| + \max\{\overline{p}, \overline{h}\} (|\Dh - \Dl|)$. First, given that $\uh^k < \kappa$ for every $k \in [K]$, we must have that 
\begin{align}
\label{eq:bound-on-Dh-Dl}
\begin{split}
    D^H - D^L &= (\gammah - \gammal)\hat{\mu}\leq (\gammah - \gammal) K \kappa.
\end{split}
\end{align}

Now, for the problem \eqref{eq:R-relaxed} with inventory $\hat{S} - D^H$, consider its minimizer $\underline{y}$. We can build a feasible solution $y$ to the problem with inventory $\hat{S} - D^L$ by setting $y^k = \underline{y}^k$ for every $k$. Therefore,
\begin{align}
\label{eq:bound-R-dif-1}
    \hat{R}(\hat{S} - D^L) - \hat{R}(\hat{S} - D^H) \leq h^w(D^H - D^L),
\end{align}
as following $y$ with initial inventory $\hat{S} - D^L$ leads to holding an additional $(D^H - D^L)$ units of inventory in the warehouse, as compared to that by following $\underline{y}$ with initial inventory $\hat{S} - D^H$.

On the other hand, for the problem \eqref{eq:R-relaxed} with inventory $\hat{S} - D^L$, consider its minimizer $\overline{y}$. Let us first consider the case that $\sum_{k \in [K]}\overline{y}^k > D^H - D^L$. We can build a feasible solution $\tilde{y}$ to the problem with inventory $\hat{S} - D^H$ by setting $\tilde{y}^k = \overline{y}^k \frac{\sum_{k \in [K]}\overline{y}^k - (D^H - D^L)}{\sum_{k \in [K]}\overline{y}^k}$. 
Note that $\sum_{k \in [K]}\overline{y}^k - \sum_{k \in [K]}\tilde{y}^k = D^H - D^L$. 
Note that the warehouse inventory is the same in the two cases, hence the warehouse holding costs are the same, and each store has lower inventory under $\tilde{y}$ and hence store holding costs are weakly lower under $\tilde{y}$. Recalling that the per-unit cost of "scarcity" is upper bounded by $\overline{p}$ (see~\eqref{immediate cost dif upper bound}), we can conclude that 
\begin{align}
\label{eq:bound-R-dif-2}
    \hat{R}(\hat{S} - D^H) - \hat{R}(\hat{S} - D^L) \leq \overline{p}(D^H - D^L)\, .
\end{align}
Now consider the case that $\sum_{k \in [K]}\overline{y}^k < D^H - D^L$. 
We build a feasible solution to the problem with inventory $\hat{S} - D^H$ by setting $\tilde{y}^k = 0$ for every $k$. First, clearly, $\sum_{k \in [K]}\overline{y}^k - \sum_{k \in [K]}\tilde{y}^k = \sum_{k \in [K]}\overline{y}^k < D^H - D^L$, so we can bound the excess underage costs incurred across stores by $\overline{p}(D^H - D^L)$, and stores incur no holding costs. Furthermore, the inventory at the warehouse when following $\overline{y}$ for the problem with inventory $\hat{S} - D^L$ is given by 
\begin{align*}
    \overline{y}^w = \hat{S} - D^L - \sum_{k \in [K]}\overline{y}^k &>
    \hat{S} - D^H = \tilde{y}^w \, .
\end{align*}
Hence, warehouse holding costs are smaller in the case of $\tilde{y}$. Therefore, the upper bound in \eqref{eq:bound-R-dif-2} still applies. 

Using \eqref{eq:bound-on-Dh-Dl}, \eqref{eq:bound-R-dif-1} and \eqref{eq:bound-R-dif-2}, and given that $\overline{h} > h^w$, we conclude that
\begin{align*}
    \alpha &= |\hat{R}(\hat{S} - \Dl) - \hat{R}(\hat{S} - \Dh)| +  \max\{\overline{p}, \overline{h}\} (|\Dh - \Dl|) \\
    &\leq \max\{\overline{p}, \overline{h}\} (D^H - D^L) + \max\{\overline{p}, \overline{h}\}(D^H - D^L) \\
    &= 2\max\{\overline{p}, \overline{h}\}(D^H - D^L) \\
    &\leq 2\max\{\overline{p}, \overline{h}\}(\gammah - \gammal) \kappa K  \\
    &\leq  C_4 K,
\end{align*}
where we set $C_4 = 2\max\{\overline{p}, \overline{h}\}(\gammah - \gammal) \kappa$.
\end{proof}

\end{APPENDIX}
\end{document}